\documentclass[twoside,11pt]{article}
\usepackage{journal}

\usepackage{natbib}
\usepackage{epsfig,ifthen}
\usepackage{latexsym}
\usepackage{amsfonts}
\usepackage{eepic}
\usepackage{amsopn}
\usepackage{amsmath,amssymb,rotating,graphicx,rotating,float}
\usepackage{accents}
\usepackage[mathscr]{eucal}
\usepackage[backref,pageanchor=true,plainpages=false,
pdfpagelabels,bookmarks,bookmarksnumbered,breaklinks,
pdfborder={0 0 0},
]{hyperref}

\usepackage{cite}
\usepackage{url}

\usepackage{ulem}
\renewenvironment{proof}[1][]{\par\noindent{\bf Proof #1\ }}{\hfill\BlackBox\\[2mm]}

\newcommand{\Yone}{{\rm Y1}}
\newcommand{\Ytwo}{{\rm Y2}}

\newcommand{\X}{\mathcal X}

\newcommand{\Y}{\mathcal Y}
\newcommand{\F}{\mathcal F}
\newcommand{\G}{\mathcal G}

\newcommand{\A}{\mathcal A}

\renewcommand{\L}{\mathcal L}

\newcommand{\bclower}{\underline{\phi}}
\newcommand{\bcupper}{\overline{\phi}}

\newcommand{\T}{\mathcal T}

\renewcommand{\P}{\mathbb P}

\newcommand{\nats}{\mathbb{N}}
\newcommand{\ints}{\mathbb{Z}}
\newcommand{\reals}{\mathbb{R}}

\newcommand{\E}{\mathbb E}

\newcommand{\eps}{\varepsilon}

\newcommand{\argmax}{\mathop{\rm argmax}}
\newcommand{\argmin}{\mathop{\rm argmin}}

\renewcommand{\limsup}{\mathop{\rm limsup}}
\renewcommand{\liminf}{\mathop{\rm liminf}}

\DeclareSymbolFont{bbold}{U}{bbold}{m}{n}
\DeclareSymbolFontAlphabet{\mathbbold}{bbold}
\newcommand{\ind}{\mathbbold{1}}

\newcommand{\Borel}{{\cal B}}

\newcommand{\ProcX}{\mathbb{X}}
\newcommand{\ProcY}{\mathbb{Y}}

\newcommand{\target}{f^{\star}}
\newcommand{\goodfun}{\bar{f}}
\newcommand{\loss}{\ell}
\newcommand{\lossdom}{\loss_{o}}
\newcommand{\domfunc}{\chi}
\newcommand{\maxloss}{\bar{\loss}}
\newcommand{\triconst}{c_{\loss}}

\newcommand{\SUIL}{{\rm SUIL}}
\newcommand{\WUIL}{{\rm WUIL}}
\newcommand{\SUAL}{{\rm SUAL}}
\newcommand{\WUAL}{{\rm WUAL}}
\newcommand{\SUOL}{{\rm SUOL}}
\newcommand{\WUOL}{{\rm WUOL}}
\newcommand{\ProcSet}{\mathcal{C}}
\newcommand{\KC}{\mathcal{C}_{1}}
\newcommand{\OKC}{\mathcal{C}_{2}}
\newcommand{\UKC}{\mathcal{C}_{3}}

\newcommand{\CRF}{{\rm CRF}}

\newcommand{\zo}{{\scriptscriptstyle{01}}}
\newcommand{\sq}{{\rm sq}}
\newcommand{\sqsmall}{{\scriptscriptstyle{sq}}}

\renewcommand{\i}{{\mathrm{i}}}
\renewcommand{\j}{{\mathrm{j}}}

\newcommand{\ignore}[1]{}
\newcommand{\private}[1]{}

\makeatletter
\newcommand{\vast}{\bBigg@{3}}
\newcommand{\Vast}{\bBigg@{4}}
\makeatother

\newsavebox{\savepar}

\newtheorem{condition}{Condition}
\newtheorem{problem}{Open Problem}

\jmlrheading{}{2020}{}{}{}{Steve Hanneke}

\ShortHeadings{Learning Whenever Learning is Possible}{Hanneke}
\firstpageno{1}

\begin{document}

\title{Learning Whenever Learning is Possible: \\Universal Learning under General Stochastic Processes}

\author{\name Steve Hanneke \email steve.hanneke@gmail.com\\
\addr Toyota Technological Institute at Chicago}

\editor{G\'{a}bor Lugosi}

\maketitle

\begin{abstract}
This work initiates a general study of learning and generalization without the i.i.d.~assumption,
starting from first principles.
While the traditional approach to statistical learning theory
typically relies on standard assumptions from probability theory (e.g., i.i.d.~or stationary ergodic),
in this work
we are interested in developing a theory of learning based only
on the most fundamental and 
necessary assumptions implicit in the requirements of the learning 
problem itself.  
We specifically study
universally consistent function learning, 
where the objective is to obtain low long-run average loss for any target function, when 
the data follow a given stochastic process.
We are then interested in the question of whether there exist learning rules guaranteed to be universally consistent 
given \emph{only} the assumption that universally consistent learning is \emph{possible} for the given data process.
The reasoning that motivates this criterion emanates from a kind of \emph{optimist's decision theory}, 
and so we refer to such learning rules as being \emph{optimistically universal}.
We study this question in three natural learning settings:
\emph{inductive}, \emph{self-adaptive}, and \emph{online}.
Remarkably, as our strongest positive result, we find that
optimistically universal learning rules do indeed exist in the self-adaptive learning setting.  
Establishing this fact requires us to develop 
new approaches to the design of learning algorithms. 
Along the way, we also identify concise characterizations of the family of processes 
under which universally consistent learning is possible in the inductive and self-adaptive settings.
We additionally pose a number of enticing open problems, particularly for the 
online learning setting.
\end{abstract}


\section{Introduction}
\label{sec:intro}

At least since 
the time of
the ancient Pyrrhonists,
it has been observed that learning in general is sometimes not possible.
Rather than turning to radical skepticism, 
modern learning theorists have 
preferred to introduce 
constraining assumptions, under which learning becomes possible, 
and have established positive guarantees for various learning strategies under these assumptions.
However, the assumptions we have focused on in the literature 
tend to be imported from the probability theory literature, rather than being rooted in a principled approach to the learning problem itself.
This is typified by the overwhelming 
reliance on the assumption that training samples are independent and identically distributed, or resembling this (e.g., stationary ergodic).
While such assumptions are known to be sufficient for learning due to their convenient convergence properties (i.e., laws of large numbers), 
it is clear that they are not \emph{necessary} for learning.
In the present work, we revisit the issue of the assumptions at the foundations of statistical learning theory, starting from first principles,
without relying on traditional probabilistic assumptions 
about the data, such as independence and stationarity 
(which will be recovered as special cases).

{\vskip 1mm}
We approach this via a kind of {\bf optimist's decision theory}, reasoning that if we are tasked with achieving a given objective $O$ in some scenario,
then already we have implicitly committed to the assumption that achieving objective $O$ is at least \emph{possible} in that scenario.
We may therefore \emph{rely} on this assumption in our strategy for achieving the objective.
We are then most interested in strategies guaranteed to achieve objective $O$ in \emph{all} scenarios 
where it is possible to do so: that is, strategies that rely \emph{only} on the assumption that objective $O$ is achievable.
Such strategies have the satisfying property that, if ever they fail to achieve the objective, we may rest assured that 
no other strategy could have succeeded, so that nothing was lost.

{\vskip 1mm}
Thus, in approaching the problem of learning (suitably formalized), we may restrict focus to those scenarios in which \emph{learning is possible}.
This assumption --- that learning is possible --- essentially represents a most ``natural'' assumption,
since it is \emph{necessary} for a theory of learning.
Concretely, in this work, we initiate this line of exploration by focusing on
(arguably) the most basic type of learning problem: \emph{universal consistency} in learning a function.
Following the optimist's reasoning above, 
we are interested in determining whether there exist
\emph{optimistically} universal learners, 
in the sense that they are guaranteed to be universally consistent 
given only the assumption that universally consistent learning is \emph{possible} under the given data process:
that is, they are universally consistent under all data processes that admit the existence of universally consistent learners.
We find that, in certain learning protocols,
such optimistically universal learners do indeed exist, and we provide a construction of such a learning rule.
Interestingly, it turns out that 
not all learning rules consistent under the i.i.d.\ assumption satisfy this 
type of universality, so that this criterion can serve as an informative 
desideratum in the design of learning methods.
Along the way, we are also interested in expressing concise necessary and sufficient conditions 
for universally consistent learning to be possible under a given data process.

{\vskip 1mm} 
We specifically consider three natural learning settings --- \emph{inductive}, \emph{self-adaptive}, and \emph{online} --- 
distinguished by the level of access to the data available to the learner.
In all three settings, we suppose there is an unknown \emph{target function} $\target$ and a sequence of data $(X_{1},Y_{1}),(X_{2},Y_{2}),\ldots$ with $Y_{t} = \target(X_{t})$,\footnote{Later we 
also discuss extensions to allow noisy responses $Y_{t}$.} 
of which the learner is permitted to observe the first $n$ samples $(X_{1},Y_{1}),\ldots,(X_{n},Y_{n})$: the \emph{training data}.  
Based on these observations, the learner is tasked with producing a predictor $f_{n}$.  
The performance of the learner is determined by how well $f_{n}(X_{t})$ approximates the (unobservable) $Y_{t}$ value 
for data $(X_{t},Y_{t})$ encountered in the \emph{future} (i.e., $t > n$).\footnote{Of course, in certain real learning scenarios, 
these future $Y_{t}$ values might never actually be observable, and therefore should be considered merely as hypothetical values 
for the purpose of theoretical analysis.}
To quantify this, we suppose there is a \emph{loss function} $\loss$, 
and we are interested in obtaining a small \emph{long-run average} value of $\loss(f_{n}(X_{t}),Y_{t})$.
A learning rule is said to be \emph{universally consistent} under the process $\{X_{t}\}$ if it achieves this (almost surely, as $n \to \infty$) for all target functions $\target$.\footnote{Technically, 
to be consistent with the terminology used in the literature on universal consistency, we should qualify this as 
``universally consistent for function learning,'' to indicate that $Y_{t}$ is a fixed function of $X_{t}$.
We omit this qualification and simply write ``universally consistent'' for brevity.  The more-general case 
of random variable pairs $(X_{t},Y_{t})$, where $Y_{t}$ may be noisy, will be discussed in Section~\ref{sec:noise}.}

The three different settings
are then 
formed 
as natural variants of this high-level description.
The first is the basic \emph{inductive} learning setting, in which $f_{n}$ is fixed 
after observing the initial $n$ samples, and we are interested in obtaining a 
small value of $\frac{1}{m} \sum_{t=n+1}^{n+m} \loss(f_{n}(X_{t}),Y_{t})$ for all large $m$.
This inductive setting is perhaps the most commonly-studied in the prior literature on statistical learning theory
\citep*[see e.g.,][]{devroye:96}.
The second setting is a more-advanced variant, which we call \emph{self-adaptive} learning, 
in which $f_{n}$ may be updated after each subsequent prediction $f_{n}(X_{t})$, 
based on the additional (unlabeled) \emph{test} observations $X_{n+1},\ldots,X_{t}$: that is, 
it continues to learn from its \emph{test data}.  In this case, denoting by $f_{n,t}$ the 
predictor after observing $(X_{1},Y_{1}),\ldots,(X_{n},Y_{n}),X_{n+1},\ldots,X_{t}$, 
we are interested in obtaining a small value of $\frac{1}{m} \sum_{t=n+1}^{n+m} \loss(f_{n,t-1}(X_{t}),Y_{t})$ for all large $m$.
In principle, self-adaptive learning should be possible in many common learning scenarios where the test data are observed sequentially, 
such as in pattern recognition based on a data stream from a camera or other sensors.
This setting is related to several
others studied in the literature,
including \emph{semi-supervised} learning \citep*{chapelle:10},
\emph{transductive} learning \citep*{vapnik:82,vapnik:98},
and (perhaps most-closely related) the problems of \emph{domain adaptation}
and \emph{covariate shift} \citep*{huang:07,cortes:08,ben-david:10,hanneke:19b}.
Finally, the strongest
setting considered in this work is the 
\emph{online} learning setting, in which, after each prediction $f_{n}(X_{t})$, the learner 
is permitted to \emph{observe} $Y_{t}$ and update its predictor $f_{n}$.  
We are then interested in obtaining a small value of 
$\frac{1}{m} \sum_{n=0}^{m-1} \loss(f_{n}(X_{n+1}),Y_{n+1})$ for all large $m$.
This is a particularly strong setting, since it requires that the supervisor providing the $Y_{t}$ responses remains
present in perpetuity.  Nevertheless, this is sometimes the case to a certain extent (e.g., in forecasting problems),  
and consequently the online setting has received considerable attention 
\citep*[e.g.,][]{littlestone:88,haussler:94,cesa-bianchi:06,ben-david:09,rakhlin:15}.

Our
most-complete result is for the self-adaptive setting, where we propose a new learning rule and prove that it is 
universally consistent under \emph{every} data process $\{X_{t}\}$ for which there exist universally consistent self-adaptive learning 
rules.  As mentioned above, we refer to this property as being \emph{optimistically universal}.  Interestingly, we also 
prove that there is \emph{no} optimistically universal \emph{inductive} learning rule, so that the additional 
ability to learn from the (unlabeled) test data is crucial.  For both inductive and self-adaptive learning, we 
also prove that the family of processes $\{X_{t}\}$ that admit the existence of universally consistent learning rules 
is completely characterized by a simple condition on the tail behavior of empirical frequencies.  In particular, 
this also means that these two families of processes are equal.  In contrast, 
we find that the family of processes admitting the existence of universally consistent \emph{online} learning rules 
forms a strict \emph{superset} of these other two families.  However, beyond this, 
the treatment of the online learning setting in this work remains incomplete, 
and leaves a number of enticing open problems regarding whether or not there exist optimistically universal online 
learning rules, and concisely characterizing the family of processes admitting the existence of universally consistent 
online learners.
In addition to results about learning rules, we also argue that there is no consistent 
\emph{hypothesis test} for whether a given process admits the existence of universally consistent learners
(in any of these settings),
indicating that the possibility of learning must indeed be considered an \emph{assumption}, 
rather than merely a verifiable hypothesis.
The above results are all established for general bounded losses.
We also discuss the case of unbounded losses, a much more demanding setting for universal learners.
In that setting, the theory becomes significantly simpler, and we are able to resolve the essential 
questions of interest for all three learning settings, with the exception of one particular question on the 
types of processes that admit the existence of universally consistent learning rules.

In addition to these general results for function learning, we also discuss extensions of the theory 
to allow \emph{noisy responses} $Y_{t}$, in Section~\ref{sec:noise}. Specifically, we consider 
the case of responses $Y_{t}$ that are \emph{conditionally independent} given $X_{t}$, 
with the further requirement that there is a time-invariant optimal function $\target$.  
We find that the results for inductive and self-adaptive learning indeed extend to these noisy scenarios,
for certain families of losses: for instance, regression with the squared loss.

\subsection{Formal Definitions}
\label{subsec:notation}

We begin our formal discussion with a few basic definitions.
Let $(\X,\Borel)$ be a measurable space, 
with $\Borel$ a Borel $\sigma$-algebra generated by a 
separable metrizable topological space $(\X,\T)$,
where $\X$ is called the \emph{instance space} and is assumed to be nonempty. 
Fix a space $\Y$, called the \emph{value space}, and a function $\loss : \Y^{2} \to [0,\infty)$, called the \emph{loss function}. 
We also define $\maxloss = \sup\limits_{y,y^{\prime} \in \Y} \loss(y,y^{\prime})$.
Unless otherwise indicated explicitly, we will suppose $\maxloss < \infty$ (i.e., $\loss$ is \emph{bounded});
the sole exception to this is Section~\ref{sec:unbounded-losses}, which is devoted to exploring the setting of unbounded $\loss$.
Furthermore, to focus on nontrivial scenarios, we will suppose 
$\X$ and $\Y$ are nonempty and 
$\maxloss > 0$
throughout.

For simplicity, we suppose that $\loss$ is a \emph{near-metric}:
that is, $\forall y_{1},y_{2}\in \Y$, $\loss$ satisfies $\loss(y_{1},y_{2}) = \loss(y_{2},y_{1})$, and $\loss(y_{1},y_{2})=0$ if and only if $y_{1}=y_{2}$, 
and also satisfies a relaxed triangle inequality, namely, there is a finite constant $\triconst \geq 1$ such that 
$\forall y_{1},y_{2},y_{3} \in \Y$, $\loss(y_{1},y_{2}) \leq \triconst ( \loss(y_{1},y_{3}) + \loss(y_{3},y_{2}) )$.
We further suppose that $(\Y,\loss)$ is \emph{separable}, in the usual sense that there exists a countable $\tilde{\Y} \subseteq \Y$ 
with $\sup\limits_{y \in \Y} \inf\limits_{\tilde{y} \in \tilde{\Y}} \loss(\tilde{y},y) = 0$.
For instance, these conditions are satisfied for discrete classification under the $0$-$1$ loss ($\Y$ countable, $\loss(a,b)=\ind[a \neq b]$), 
or bounded real-valued regression under the squared loss ($\Y = [-B,B]$, $\loss(a,b) = (a-b)^{2}$) or indeed any $L_{p}$ loss ($\loss(a,b)=|a-b|^{p}$, $p > 0$),
as well as many other losses.
Most of the theory developed here also easily extends
to any $\loss$ that is merely \emph{dominated} by a separable near-metric $\lossdom$, in the sense that
$\forall y,y^{\prime} \in \Y, \loss(y,y^{\prime}) \leq \domfunc(\lossdom(y,y^{\prime}))$ for a continuous nondecreasing function 
$\domfunc : [0,\infty) \to [0,\infty)$ with $\domfunc(0)=0$ and satisfying a non-triviality condition $\sup\limits_{y_{0},y_{1}} \inf\limits_{y} \max\{ \loss(y,y_{0}),\loss(y,y_{1}) \} > 0$.
This then admits discrete classification with 
asymmetric misclassification costs, 
and many other interesting cases.
We include a brief discussion of this generalization in Section~\ref{subsec:nonmetric-losses}.

Below, any reference to a \emph{measurable set} 
$A \subseteq \X$ should be taken to mean $A \in \Borel$, unless otherwise specified.  
Additionally, let $\mathcal{T}_{y}$ be the topology on $\Y$ generated by the open balls of $\ell$,  $\{\{ y \in \Y : \loss(y,y_{0}) < r \} : y_{0} \in \Y, r > 0 \}$,
and let $\Borel_{y} = \sigma(\mathcal{T}_{y})$ denote the Borel $\sigma$-algebra on $\Y$ generated by $\mathcal{T}_{y}$;
references to measurability of subsets $B \subseteq \Y$ below should be taken to indicate $B \in \Borel_{y}$.
We will be interested in the problem of learning from data described by a discrete-time
stochastic process $\ProcX = \{X_{t}\}_{t=1}^{\infty}$ on $\X$.  
We do not make any assumptions about the nature of this process.
For any $s \in \nats$ and $t \in \nats \cup \{\infty\}$, and any sequence $\{x_{i}\}_{i=1}^{\infty}$, 
define $x_{s:t} = \{x_{i}\}_{i=s}^{t}$, or $x_{s:t} = \{\}$ if $t < s$, where $\{\}$ or $\emptyset$ denotes the empty sequence
(overloading notation, as these may also denote the empty \emph{set}); 
for convenience, also define $x_{s:0} = \{\}$.
For any function $f$ and sequence $\mathbf{x} = \{x_{i}\}_{i=1}^{\infty}$ in the domain of $f$, 
we define $f(\mathbf{x}) = \{f(x_{i})\}_{i=1}^{\infty}$ and $f(x_{s:t}) = \{f(x_{i})\}_{i=s}^{t}$.
Also, for any set $A \subseteq \X$, we denote by $x_{s:t} \cap A$ or $A \cap x_{s:t}$ 
the subsequence of all entries of $x_{s:t}$ contained in $A$, 
and $|x_{s:t} \cap A|$ denotes the number of indices $i \in \nats \cap [s,t]$ with $x_{i} \in A$.

For any function $g : \X \to \reals$, and any sequence $\mathbf{x} = \{x_{t}\}_{t=1}^{\infty}$ in $\X$, define
\begin{equation*}
\hat{\mu}_{\mathbf{x}} (g) = \limsup_{n\to\infty} \frac{1}{n} \sum_{t=1}^{n} g(x_t).
\end{equation*}
In particular, we will often use this notation with $\mathbf{x}=\ProcX$, 
for a process $\ProcX = \{X_t\}_{t=1}^{\infty}$, in which case $\hat{\mu}_{\ProcX}(g)$ is a random variable.
For any set $A \subseteq \X$ we overload this 
notation, defining $\hat{\mu}_{\mathbf{x}}(A) = \hat{\mu}_{\mathbf{x}}(\ind_{A})$,
where $\ind_{A}$ is the binary indicator function for the set $A$.
We also use the notation $\ind[p]$, for any logical proposition $p$, 
to denote a value that is $1$ if $p$ holds (evaluates to ``True''), and $0$ otherwise.
We also make use of the standard notation for limits of sequences $\{A_{i}\}_{i=1}^{\infty}$ of sets \citep*[see e.g.,][]{ash:00}:
$\limsup\limits_{i \to \infty} A_{i} = \bigcap\limits_{k = 1}^{\infty} \bigcup\limits_{i = k}^{\infty} A_{i}$,
$\liminf\limits_{i\to\infty} A_{i} = \bigcup\limits_{k=1}^{\infty} \bigcap\limits_{i=k}^{\infty} A_{i}$,
and $\lim\limits_{i\to\infty} A_{i}$ exists and equals $\limsup\limits_{i\to\infty} A_{i}$
if and only if $\limsup\limits_{i\to\infty} A_{i} = \liminf\limits_{i\to\infty} A_{i}$. 
As one final remark on notation, we note that we will generally interpret claims 
regarding conditional expectations to mean that there exist \emph{versions} 
of the corresponding conditional expectations for which the claims hold, 
such as in $\E[Z|Y] \leq \E[Z|X] = \E[W|X]$.

As discussed above, we are interested in three learning settings, defined as follows.
An \emph{inductive} learning rule is any sequence of measurable functions $f_{n} : \X^n \times \Y^n \times \X \to \Y$, $n \in \nats \cup \{0\}$.
A \emph{self-adaptive} learning rule is any array of measurable functions $f_{n,m} : \X^{m} \times \Y^{n} \times \X \to \Y$, $n,m \in \nats \cup \{0\}$, $m \geq n$.
An \emph{online} learning rule is any sequence of measurable functions $f_{n} : \X^{n} \times \Y^{n} \times \X \to \Y$, $n \in \nats\cup\{0\}$.
In each case, these functions can potentially be stochastic 
(that is, we allow $f_{n}$ itself to be a random variable),
though independent from $\ProcX$.
For any measurable $\target : \X \to \Y$, any inductive learning rule $f_{n}$, any self-adaptive learning rule $g_{n,m}$, and any online learning rule $h_{n}$, we define 
\begin{align*}
\hat{\L}_{\ProcX}(f_{n}, \target;n) & = \limsup\limits_{t \to \infty} \frac{1}{t} \sum_{m=n+1}^{n+t} \loss\left( f_{n}(X_{1:n}, \target(X_{1:n}), X_{m}), \target(X_{m}) \right),\\
\hat{\L}_{\ProcX}(g_{n,\cdot},\target;n) & = \limsup\limits_{t \to \infty} \frac{1}{t+1} \sum_{m=n}^{n+t} \loss\!\left( g_{n,m}(X_{1:m},\target(X_{1:n}),X_{m+1}), \target(X_{m+1}) \right),\\
\hat{\L}_{\ProcX}(h_{\cdot},\target;n) & = \frac{1}{n} \sum_{t=0}^{n-1} \loss\!\left(h_{t}(X_{1:t},\target(X_{1:t}),X_{t+1}),\target(X_{t+1})\right).
\end{align*}

In each case, $\hat{\L}_{\ProcX}(\cdot,\target; n)$ measures a kind of limiting loss of the learning rule, relative to the source of the target values: $\target$.
In this context, we refer to $\target$ as the \emph{target function}.
Note that, in the cases of inductive and self-adaptive learning rules, we are interested in the average \emph{future} losses 
after some initial number $n$ of ``training'' observations,
for which target values are provided, and after which no further target values are observable.
Thus, a small value of the loss $\hat{\L}_{\ProcX}$ in these settings represents a kind of \emph{generalization} to future (possibly previously-unseen) data points.
In particular, in the special case of i.i.d.\! $\ProcX$ with marginal distribution $\P_{X}$, the strong law of large numbers implies that 
the loss $\hat{\L}_{\ProcX}(f_{n},\target;n)$ of an inductive learning rule $f_{n}$ 
is equal (almost surely) to the usual notion of the \emph{risk} of $f_{n}(X_{1:n},\target(X_{1:n}),\cdot)$ 
--- namely, $\int \loss(f_{n}(X_{1:n},\target(X_{1:n}),x),\target(x)) \P_{X}({\rm d}x)$ --- 
commonly studied in the statistical learning theory literature, 
so that $\hat{\L}_{\ProcX}(f_{n},\target;n)$ represents a generalization of the notion of \emph{risk}. 
Note that, in the case of general processes $\ProcX$, the average loss $\frac{1}{t} \sum\limits_{m=n+1}^{n+t} \loss\!\left( f_{n}(X_{1:n}, \target(X_{1:n}), X_{m}), \target(X_{m}) \right)$ 
might not have a well-defined limit as $t \to \infty$, particularly for \emph{non-stationary} processes $\ProcX$, 
and it is for this reason that we use the limit superior in the definition (and similarly for $\hat{\L}_{\ProcX}(g_{n,\cdot},\target;n)$).
We also note that, since the loss function is always finite, we could have included the losses on the $n$ training samples 
in the summation in the inductive $\hat{\L}_{\ProcX}(f_{n},\target;n)$ definition without affecting its value.
This observation implies the following simple equality.
\begin{equation}
\label{eqn:inductive-simple-equiv}
\hat{\L}_{\ProcX}(f_{n}, \target;n) = \hat{\mu}_{\ProcX}\!\left( \loss\!\left(f_n(X_{1:n}, \target(X_{1:n}), \cdot), \target(\cdot)\right)\right).
\end{equation}

The distinction between the inductive and self-adaptive settings is merely the fact that the self-adaptive learning rule is able to 
\emph{update} the function used for prediction after observing each ``test'' point $X_{t}$, $t > n$.  Note that the target values
are not available for these test points: only the ``unlabeled'' $X_{t}$ values.
In the special case of an i.i.d.\! process, the self-adaptive setting is closely related to the \emph{semi-supervised} learning setting
studied in the statistical learning theory literature \citep*{chapelle:10}.  For more-general processes, 
it has relations to problems of \emph{domain adaptation} and \emph{covariate shift} \citep*{huang:07,cortes:08,ben-david:10,hanneke:19b}, 
as the additional samples $X_{(n+1):m}$ provide important information about how representative the training samples $X_{1:n}$ are
for the purpose of estimating certain relevant long-run averages $\hat{\mu}_{\ProcX}(g)$ (see Section~\ref{sec:universal2-adaptive} for details).
In particular, for this purpose, it is important that these additional unlabeled samples are actual \emph{test} samples, 
rather than (for instance) taken from an independent copy of the process, 
since general (non-ergodic) processes may have very different long-run behaviors in different sample paths.

{\vskip 2mm}In the case of online learning, the prediction function is again allowed to update after every test point,
but in this case the target value for the test point \emph{is} accessible (after the prediction is made).
This online setting, with precisely this same $\hat{\L}_{\ProcX}(h_{\cdot},\target;n)$ objective function, 
has been studied in the learning theory literature, both in the case of i.i.d.\! processes 
and relaxations thereof \citep*[e.g.,][]{haussler:94,gyorfi:02} 
and in the very-general setting of $\ProcX$ an \emph{arbitrary} process \citep*[e.g.,][]{littlestone:88,cesa-bianchi:06,rakhlin:15}.

{\vskip 2mm}Our interest in the present work is the basic problem of \emph{universal consistency}, 
wherein the objective is to design a learning rule with the guarantee that the long-run average loss $\hat{\L}_{\ProcX}$ 
approaches \emph{zero} (almost surely) as the training sample size $n$ grows large, and that this fact holds true 
for \emph{any} target function $\target$.
Specifically, we have the following definitions.

\newpage
\begin{definition}
\label{def:suil}
We say an inductive learning rule $f_n$ is strongly universally consistent under $\ProcX$
if, for every measurable $\target : \X \to \Y$, 
$\lim\limits_{n\to\infty} \hat{\L}_{\ProcX}(f_n,\target;n) = 0\text{ (a.s.)}$.
\\We say a process $\ProcX$ admits strong universal inductive learning
if there exists an inductive learning rule $f_n$ that is 
strongly universally consistent under $\ProcX$.
\\We denote by $\SUIL$ the set of all processes $\ProcX$ that admit strong universal inductive learning.
\end{definition}
\begin{definition}
\label{def:sual}
We say a self-adaptive learning rule $f_{n,m}$ is strongly universally consistent under $\ProcX$
if, for every measurable $\target : \X \to \Y$, 
$\lim\limits_{n\to\infty} \hat{\L}_{\ProcX}(f_{n,\cdot},\target;n) = 0\text{ (a.s.)}$.
\\We say a process $\ProcX$ admits strong universal self-adaptive learning
if there exists a self-adaptive learning rule $f_{n,m}$ that is 
strongly universally consistent under $\ProcX$.
\\We denote by $\SUAL$ the set of all processes $\ProcX$ that admit strong universal self-adaptive learning.
\end{definition}
\begin{definition}
\label{def:suol}
We say an online learning rule $f_{n}$ is strongly universally consistent under $\ProcX$
if, for every measurable $\target : \X \to \Y$, 
$\lim\limits_{n\to\infty} \hat{\L}_{\ProcX}(f_{\cdot},\target;n) = 0\text{ (a.s.)}$.
\\We say a process $\ProcX$ admits strong universal online learning
if there exists an online learning rule $f_{n}$ that is 
strongly universally consistent under $\ProcX$.
\\We denote by $\SUOL$ the set of all processes $\ProcX$ that admit strong universal online learning.
\end{definition}

Technically, the above definitions of universal consistency are defined relative 
to the loss function $\loss$.  However, we will establish below that $\SUIL$ and $\SUAL$
are in fact \emph{invariant} to the choice of $(\Y,\loss)$, subject to the basic assumptions stated above 
(separable near-metric, $0 < \maxloss < \infty$).  We will also find that this is true of $\SUOL$, subject to the 
additional constraint that $(\Y,\loss)$ is \emph{totally bounded}.  Furthermore, for unbounded losses 
we find that all three families are invariant to $(\Y,\loss)$, subject to separability and $\maxloss > 0$.

As noted above, much of the prior literature on universal consistency without the i.i.d.\ assumption has focused on 
relaxations of the i.i.d.\ assumption to more-general families of processes, such as stationary mixing, stationary ergodic,
or certain limited forms of non-stationarity (see e.g., \citealp*{steinwart:09}, Chapter 27 of \citealp*{gyorfi:02}, and references therein).
Though the analysis of learning techniques becomes significantly more challenging under these 
relaxations, in many cases the essential features of the i.i.d.\ setting 
useful for proving consistency are preserved (particularly, laws of large numbers).
In contrast, our primary interest in the present work is to study the \emph{natural} assumption 
\emph{intrinsic to the universal consistency problem itself}:
the assumption that universal consistency is \emph{possible}.  By definition, this is a \emph{necessary} assumption for universal consistency. 
Thus, the important question is whether there is a learning rule for which it is also a \emph{sufficient} assumption for establishing universal consistency.
In other words, we are interested in the following abstract question:
\begin{center}
{\bf Do there exist learning rules that are strongly universally consistent under \emph{every} process $\ProcX$ that admits strong universal learning?}
\end{center}
Each of the three learning settings yields a concrete instantiation of this question.
For the reason discussed in the introductory remarks, we refer to any such learning rule as being 
\textbf{optimistically universal}.
Hence we have the following definition.

\begin{definition}
\label{def:optimistically-universal}
An (inductive/self-adaptive/online) learning rule is \emph{optimistically universal} if it is strongly universally consistent 
under every process $\ProcX$ that admits strong universal (inductive/self-adaptive/online) learning.
\end{definition}

\subsection{Summary of the Main Results}
\label{subsec:main}

Here we briefly summarize the main results of this work.  
Their proofs, along with several other results, will be developed throughout the rest of this article.

The main positive result in this work is the following theorem, which establishes that optimistically universal self-adaptive learning is indeed possible.
In fact, in proving this result, we develop a specific construction of one such self-adaptive learning rule.

\begin{theorem}
\label{thm:optimistic-self-adaptive}
There exists an optimistically universal self-adaptive learning rule.
\end{theorem}

Interestingly, it turns out that the additional capabilities of self-adaptive learning, compared to inductive learning, 
are actually \emph{necessary} for optimistically universal learning.  This is reflected in the following result.

\begin{theorem}
\label{thm:no-optimistic-inductive}
There \emph{does not exist} an optimistically universal inductive learning rule, if $(\X,\T)$ is an uncountable Polish space.
\end{theorem}

Taken together, these two results are interesting indeed, as they indicate there can be strong advantages to 
designing learning methods to be self-adaptive.  
This seems particularly interesting when we note that 
very few learning methods in common use are designed to exploit this capability: that is, to adjust their 
trained predictor based on the (unlabeled) test samples they encounter.  
As mentioned, self-adaptive learning 
should be possible in many common learning scenarios where the unlabeled test data are observed sequentially, 
such as in pattern recognition based on a data stream from a camera or other sensors.
In light of these results, it seems worthwhile 
to revisit the definitions of commonly-used learning methods with a view toward designing self-adaptive variants.
In the self-adaptive method we propose
in Section~\ref{sec:universal2-adaptive} below, 
the main utility of being self-adaptive is in a model selection component: 
for each hypothesis class $\F_{i}$ in a hierarchy of classes, 
we use $X_{1:n}$ and $X_{1:m}$ to produce two estimates of 
$\hat{\mu}_{\ProcX}(\loss(f(\cdot),f^{\prime}(\cdot)))$ for all $f,f^{\prime} \in \F_{i}$, 
and select the largest class $\F_{i}$ in the hierarchy 
for which these two estimates are uniformly close.
If $\F_{i}$ is sufficiently rich to approximate $\target$, 
this technique functions as an approximate test for 
whether a particular estimate of $\hat{\mu}_{\ProcX}(\loss(f(\cdot),\target(\cdot)))$ 
based on $(X_{1:n},\target(X_{1:n}))$ is close to the true value 
$\hat{\mu}_{\ProcX}(\loss(f(\cdot),\target(\cdot)))$, 
for all $f \in \F_{i}$, so that minimizing the estimate 
over $f \in \F_{i}$ produces a function $f$ 
with relatively small $\hat{\mu}_{\ProcX}(\loss(f(\cdot),\target(\cdot)))$; 
see Section~\ref{sec:universal2-adaptive} for the details.\\

As for the online learning setting, the present work makes only partial progress toward resolving the 
question of the existence of optimistically universal online learning rules (in Section~\ref{sec:online}).  In particular, the following 
question remains open at this time.

\begin{problem}
\label{prob:optimistic-online}
Does there exist an optimistically universal online learning rule?
\end{problem}

To be clear, as we discuss in Section~\ref{sec:online}, one can convert 
the optimistically universal self-adaptive learning rule from Theorem~\ref{thm:optimistic-self-adaptive}
into an online learning rule that is strongly universally consistent for any process $\ProcX$ 
that admits strong universal \emph{self-adaptive} learning.  However, as we prove below, 
the set of processes $\ProcX$ that admit strong universal \emph{online} learning is a strict 
superset of these, and so optimistically universal online learning represents a much stronger 
requirement for the learner.

In the process of studying the above, we also investigate the problem of concisely \emph{characterizing} 
the family of processes that admit strong universal learning, of each of the three types: that is, $\SUIL$, $\SUAL$, and $\SUOL$.
In particular, consider the following simple condition on the tail behavior of a given process $\ProcX$.

\begin{condition}
\label{con:kc}
For every monotone sequence $\{A_k\}_{k=1}^{\infty}$ of sets in $\Borel$ with $A_k \downarrow \emptyset$,
\begin{equation*}
\lim\limits_{k\to\infty} \E\!\left[ \hat{\mu}_{\ProcX}\!\left( A_k \right) \right] = 0.
\end{equation*}
\end{condition}

Denote by $\KC$ the set of all processes $\ProcX$ satisfying Condition~\ref{con:kc}.
In Section~\ref{sec:equiv} below, we discuss this condition in detail, and also provide several
equivalent forms of the condition.  One interesting instance of this is Theorem~\ref{thm:maharam}, 
which notes that Condition~\ref{con:kc} is equivalent to the condition that the set function 
$\E\!\left[ \hat{\mu}_{\ProcX}(\cdot) \right]$ is a \emph{continuous submeasure} (Definition~\ref{def:maharam} below).
For our present interest, the most important fact about Condition~\ref{con:kc} is that 
it precisely identifies which processes $\ProcX$ admit strong universal inductive or self-adaptive learning,
as the following theorem states.

\begin{theorem}
\label{thm:main}
The following statements are equivalent for any process $\ProcX$.
\begin{itemize}
\item $\ProcX$ satisfies Condition~\ref{con:kc}.
\item $\ProcX$ admits strong universal inductive learning.
\item $\ProcX$ admits strong universal self-adaptive learning.
\end{itemize}
Equivalently, $\SUIL = \SUAL = \KC$.
\end{theorem}

Certainly any i.i.d.\! process satisfies Condition~\ref{con:kc} (by the strong law of large numbers).
Indeed, we argue in Section~\ref{sec:lln} that any process satisfying the law of large numbers 
--- or more generally, having pointwise convergent relative frequencies --- satisfies Condition~\ref{con:kc},
and hence by Theorem~\ref{thm:main} admits strong universal learning (in both settings).
For instance, this implies that \emph{all stationary processes} admit strong universal inductive and self-adaptive learning.
However, as we also demonstrate in Section~\ref{sec:lln}, there are many other types of processes,
which do not have convergent relative frequencies, but which do satisfy Condition~\ref{con:kc}, and
hence admit universal learning, so that Condition~\ref{con:kc} represents a strictly more-general condition.

Other than the fact that Condition~\ref{con:kc} precisely characterizes the families of processes 
that admit strong universal inductive or self-adaptive learning,
another interesting fact established by Theorem~\ref{thm:main} is that these two families are 
actually \emph{equivalent}: that is, $\SUIL = \SUAL$.
Interestingly, as alluded to above, we find that this equivalence does \emph{not} extend to \emph{online} learning.
Specifically, in Section~\ref{sec:online} we find that $\SUAL \subseteq \SUOL$, with \emph{strict} inclusion iff $\X$ is infinite.

As for the problem of concisely characterizing the family of processes that admit strong universal \emph{online} learning, 
again the present work only makes partial progress.
Specifically, in Section~\ref{sec:online}, we formulate a concise \emph{necessary} condition for 
a process $\ProcX$ to admit strong universal online learning (Condition~\ref{con:okc} below), 
but we leave open the important question of whether this condition is also \emph{sufficient},
or more-broadly of identifying a concise condition on $\ProcX$ equivalent to the condition that 
$\ProcX$ admits strong universal online learning.

In addition to the questions of optimistically universal learning and concisely characterizing the family of processes admitting universal learning, 
another interesting question is whether it is possible to empirically \emph{test} whether a given process admits universal learning (of any of the three types).
However, in Section~\ref{sec:no-consistent-test} we find that 
in all three settings this is \emph{not} the case.  Specifically, in Theorem~\ref{thm:no-consistent-test-for-suil} 
we prove that (when $\X$ is infinite) there does not exist a consistent hypothesis test for whether a given $\ProcX$ admits strong universal (inductive/self-adaptive/online) learning.
Hence, the assumption that learning is possible truly is an \emph{assumption}, rather than a testable hypothesis.

While all of the above results are established for \emph{bounded} losses, Section~\ref{sec:unbounded-losses} is devoted to 
the study of these same issues in the case of \emph{unbounded} losses.
In that case, the theory becomes significantly simplified, as universal consistency is much more difficult to achieve, 
and hence the family of processes that admit universal learning is severely restricted.
We specifically find that, when the loss is unbounded, there exists an optimistically universal learning rule of \emph{all three} types.
We also identify a concise condition (Condition~\ref{con:ukc} below) that 
is necessary and sufficient for a process to admit strong universal learning in any/all of the three settings.

In Section~\ref{sec:noise}, we extend the theory to allow the $Y_t$ response values 
to be \emph{noisy}, subject to being conditionally independent. 
We discuss other extensions of the theory in Section~\ref{sec:extensions}, 
admitting more-general loss functions, as well as relaxation of the requirement of \emph{strong} consistency 
to mere \emph{weak} consistency.
Finally, we conclude the article in Section~\ref{sec:open-problems} by summarizing several interesting open questions 
that arise from the theory developed below.

\subsection{Related Work}
\label{subsec:related-work}

There are several important works in the literature 
related to universal consistency under non-i.i.d.\ processes.  
Questions about consistency under general stationary ergodic 
processes were posed by \citet*{cover:75} for the forecasting problem   
(i.e., predicting $Y_{t+1}$ based on $Y_{1:t}$) and related settings.
In particular, Cover's question of whether there is an estimator 
$\hat{m}(Y_{1:t})$ with $| \hat{m}(Y_{1:t}) - \E[Y_{t+1}|Y_{1:t}] | \to 0$ (a.s.)
for all stationary ergodic $\ProcY = \{Y_{t}\}_{t=1}^{\infty}$ on $\{0,1\}$ was answered negatively by 
\citet*{bailey:76} and \citet*{ryabko:88}.  A related negative result was also established by 
\citet*{nobel:99} for regression, showing there is no estimator  
$\hat{f}_{t}(X_{1:t},Y_{1:t},\cdot)$ with $\E | \hat{f}(X_{1:t},Y_{1:t},X) - \E[Y|X] | \to 0$
for all stationary ergodic processes $(\ProcX,\ProcY) = \{(X_{t},Y_{t})\}_{t=1}^{\infty}$ on $[0,1]^2$, 
and where $(X,Y)$ is an independent random variable of the same marginal distribution.  In contrast, there is a substantial literature on estimators that are
consistent (in various senses)
under \emph{mixing} conditions,
which are stronger than ergodicity
\citep*[e.g.,][]{steinwart:09,lozano:06,roussas:88,collomb:84,irle:97}.

On the other hand, a series of works by  
\citet*{ornstein:78},
\citet*{algoet:92,algoet:94,algoet:99}, \citet*{morvai:96}, 
\citet*{gyorfi:99}, \citet*{gyorfi:02a}, and \citet*{nobel:03} showed
(among other results) that 
there \emph{do} exist universally consistent forecasting rules under general (with bounded moment)
stationary ergodic processes $\ProcY = \{Y_{t}\}_{t \in \ints}$ on $\reals$, if we are merely interested in 
the long-run \emph{average} loss:
that is, 
$\frac{1}{n} \sum_{t=0}^{n-1} | \hat{m}(Y_{1:t}) - \E[Y_{t+1}|Y_{-\infty:t}] | \to 0$ (a.s.).
This is analogous to the \emph{online} setting studied in the present work.
This result was extended to classification and bounded regression settings by 
\citet*{morvai:96}, \citet*{gyorfi:99}, and \citet*{gyorfi:02a},
yielding an online learning rule $\hat{f}_{t}$ for which 
$\frac{1}{n} \sum_{t=0}^{n-1} ( \hat{f}_{t}(X_{1:t},Y_{1:t},X_{t+1}) - \E[ Y_{t+1} | X_{-\infty:(t+1)}, Y_{-\infty:t} ] )^{2} \to 0$ (a.s.) 
for all stationary ergodic processes $(\ProcX,\ProcY) = \{(X_{t},Y_{t})\}_{t \in \ints}$ on $\reals^{d} \times \reals$ with $|Y_{t}|$ bounded.

In contrast, as we discuss below (Section~\ref{sec:examples}), 
an immediate implication of Theorems~\ref{thm:optimistic-self-adaptive}, \ref{thm:main}, and \ref{thm:suil-subset-suol}  
is that the ergodicity assumption is superfluous for the existence of such estimators (i.e., stationarity alone suffices), 
if we restrict to cases where 
$Y_{t} = \target(X_{t})$ for arbitrary unknown functions $\target$, 
or more generally, cases where
$Y_t$ is conditionally independent of $\{(X_s,Y_s)\}_{s \neq t}$ given $X_t$.
Indeed, universal consistency is even possible in the much weaker self-adaptive setting for such stationary processes.
One interpretation of this is that, while the stationary ergodic assumption enables a learner to 
estimate and optimize its \emph{expected} risk, stationarity alone already suffices 
if we are only interested in estimating and optimizing its average loss on 
the actual future samples in the individual sample path, 
so that information about the expected risk is unnecessary.
We also remark that the results established here also hold for many non-stationary processes as well.

Other works have considered learning under various \emph{non-stationary} processes.
A mild form of non-stationarity was discussed by \citet*{irle:97},
who constructs consistent regression estimators 
under mixing processes that have vanishing average
total variation distance of the marginals to a fixed distribution
the risk is defined under.
\citet*{steinwart:09} generalize this to the family of all processes
for which a \emph{law of large numbers} holds, which includes all
\emph{asymptotically mean stationary} ergodic processes \citep*{gray:80,gray:09}.
However, the learning rule of \citet*{steinwart:09}
has a dependence on the distribution of
$(\ProcX,\ProcY)$, which is in fact necessary
(due to the negative result of \citealp*{nobel:99};
see also the proof of Theorem~\ref{thm:no-optimistic-inductive} below).
However, \citet*{steinwart:09} show that this dependence can be removed 
if we further restrict to processes $(\ProcX,\ProcY)$ satisfying a mixing condition 
(constrained weak $\alpha$-bi-mixing rate), in which case an $(\ProcX,\ProcY)$-independent 
choice of the parameter sequence yields weak consistency.
This relaxes the requirements of an earlier result of \citet*{lozano:06} 
establishing strong consistency (for a different learning rule) 
for stationary processes satisfying a stronger type of mixing condition (constrained $\beta$-mixing rate).
\citet*{morvai:99} also studied consistency under general processes satisfying a law of large numbers, 
in a bounded regression setting on $\reals$.  The results there even hold for deterministic processes, 
as long as the frequencies converge to a probability measure in the limit.  
They specifically show that a particular learning rule is consistent as long as the 
regression function of the limit distribution satisfies a known constraint on its 
total variation in each bounded interval.

Other works have considered learning with families of non-stationary processes
not even satisfying a law of large numbers.
\citet*{kulkarni:02}
established a very general result, showing that for the regression setting (generalized to Hilbert spaces $\Y$), 
if we only require the learner to be consistent for \emph{continuous} target functions $\target$, 
then there is an online learning rule that is strongly consistent under every process $\ProcX = \{X_{t}\}_{t=1}^{\infty}$ 
such that the set $\{X_t : t \in \nats\}$ is almost surely totally bounded.  For instance, if $\X$ is totally bounded, 
then this is the case for all processes $\ProcX$ on $\X$.  They in fact establish a more general
result that also allows $Y_t$ to be corrupted by conditionally independent noise
(subject to $\target(X_t)=\E[Y_t|X_t]$), a topic we discuss in Section~\ref{sec:noise} below.
The results in the present work reveal that, if we seek truly \emph{universal} learners, 
consistent for \emph{all} possible target functions $\target$, including discontinuous functions, 
then even in totally bounded spaces $\X$, there exist processes $\ProcX$ where no such 
universal learners exist.  Thus, the best we can aim for is a learning rule that is universally 
consistent under every process $\ProcX$ that \emph{admits} universal learning: 
i.e., an \emph{optimistically} universal learner.

\citet*{ryabko:06} introduced another type of non-stationary process for the classification setting with finite $\Y$: namely, 
processes $(\ProcX,\ProcY)$ where the process $\ProcY = \{Y_t\}_{t=1}^{\infty} = \{\target(X_t)\}_{t=1}^{\infty}$ 
is arbitrary (subject to each $y \in \Y$ occurring with non-vanishing liminf frequency), 
and the $X_t$ sequence is ``conditionally i.i.d.'', 
meaning that the $X_t$ variables are conditionally independent given their respective $Y_t = \target(X_t)$ values, 
with time-invariant conditional distribution.
This family of processes captures many interesting scenarios beyond the i.i.d.\ assumption, 
including many non-stationary processes.
Under this condition, \citet*{ryabko:06} shows that certain learning rules known 
to be universally consistent under the i.i.d.\ assumption remain (weakly) 
consistent under this more-general family of processes (in the online setting). 
In particular, this is true of the nearest neighbor rule.
He also shows strong universal consistency for learning rules based on empirical risk minimization 
for a sequence of hypothesis classes becoming rich as $n \to \infty$.
The consistency result of \citet*{ryabko:06} is stated in a stronger form: 
$\P( h_n(X_{1:n},Y_{1:n},X_{n+1}) \neq Y_{n+1} | X_{1:n},Y_{1:n} ) \to 0$ (a.s.).  However, under the 
conditions considered by \citet*{ryabko:06}, this would actually be satisfied by any inductive learning rule $h_n$ 
satisfying $\hat{\L}_{\ProcX}(h_n,\target;n) \to 0$ (a.s.),
and indeed
the learning rules considered by \citet*{ryabko:06} are of this type.

The nature of the conditional i.i.d.\ assumption of \citet*{ryabko:06} is somewhat different 
from the conditions studied in the present work, in that it is a condition on the joint process $(\ProcX,\ProcY)$
rather than $\ProcX$ alone.
Nevertheless, it is straightforward to verify that for finite $\Y$, 
for any $(\ProcX,\ProcY)$ satisfying the conditional i.i.d.\ condition, 
$\ProcX$ satisfies Condition~\ref{con:kc}, and hence by Theorems~\ref{thm:main} and \ref{thm:suil-subset-suol} 
it admits strong universal learning (in any of the three settings studied here); 
this is true even without 
restrictions on the frequencies of each $y \in \Y$.
Note that this also implies consistency even for target functions $\target$ for which 
$(\ProcX,\target(\ProcX))$ is not conditionally i.i.d.

\citet*{ryabko:06} also asks a question of how to extend beyond 
the setting of deterministic responses $Y_t = \target(X_t)$ 
to allow more-general distributions of $Y_t$ given $X_t$.
The results in Section~\ref{sec:noisy-classification} on the topic of handling 
noisy labels are relevant to this question, in particular studying the case 
where the noise is conditionally independent: that is, the optimal prediction $\target(X_t)$ 
remains a $t$-invariant function of $X_t$, but the observed response $Y_t$ may be stochastic.
This condition can be combined with Ryabko's conditional i.i.d.\ assumption by 
taking the $X_t$ sequence to be conditionally i.i.d.\ given an (arbitrary) $\target(X_t)$ sequence, 
and then taking the $Y_t$ responses to be conditionally independent given the respective $X_t$ values 
(subject to the requirement that the conditional distribution of $Y_t$ given $X_t$ 
is a $t$-invariant function of $X_t$,
and $y = \target(X_t)$ minimizes $\E[\loss(y,Y_t)|X_t]$).
The results in Section~\ref{sec:noisy-classification} then imply that there is a learning 
rule that is strongly consistent for every process $(\ProcX,\ProcY)$ of this type 
(in the self-adaptive or online setting).\footnote{One can show this result also holds in 
the inductive setting for this special case, though we will not discuss this extension, for the sake of brevity.}
Moreover, the results in Section~\ref{sec:noise} 
also imply consistency under a much broader family of processes.

In the broader subject of learning theories beyond i.i.d.\ processes, 
the topic of finding a function with near-minimal risk within a fixed
restricted hypothesis class $\F$ has received considerably more attention.
There, the goal is consistency relative to $\F$: that is, finding a function $\hat{f} \in \F$ with risk converging
to at most the best risk achievable by functions in $\F$.
Most of this work has studied learning under stationary processes
satisfying various \emph{mixing} conditions.
Of particular relevance in this context is the set
$\F_{\loss} = \{ (x,y) \mapsto \loss(f(x),y) : f \in \F \}$.
When $\F_{\loss}$ is a VC class, or generally has bounded covering numbers,
methods based on variants of empirical risk minimization have been shown to be
consistent relative to $\F$  
\citep*[see e.g.,][]{yu:94,karandikar:02,karandikar:04,vidyasagar:05,zou:09}.
Moving beyond mixing assumptions, 
when $\F_{\loss}$ is a VC class, \citet*{adams:10a,adams:10b,adams:12} showed
that empirical risk minimization is consistent relative to $\F$ under
every stationary ergodic process.
Later, \citet*{van-handel:13} extended this to all classes such that 
$\F_{\loss}$ is a universal Glivenko-Cantelli class.
\citet*{kuznetsov:14} and \citet*{hanneke:19a} have also considered 
extensions of some of these results to some restricted families of
\emph{non-stationary} mixing processes,
constrained by the rate of change of the single-index marginal distribution.

A significant change in the present work compared to the above is that 
much of the prior work on statistical learning without the i.i.d.\ assumption 
essentially studies the same learning methods developed for i.i.d.\ processes, 
such as local averaging estimators or empirical risk minimization.
In contrast, we argue below that such methods will fail in certain 
non-stationary scenarios, in which other methods would be consistent. 
As such, the techniques we develop in this 
work necessarily differ significantly from those designed with i.i.d.\ processes in 
mind.

\section{Equivalent Expressions of Condition~\ref{con:kc}}
\label{sec:equiv}

Before getting into the analysis of learning, we first discuss basic properties of the $\hat{\mu}_{\mathbf{x}}$ functional.
In particular, we find that there are several equivalent ways to state Condition~\ref{con:kc}, which 
will be useful in various parts of the proofs below, and which may themselves be of independent interest in some cases.  

\subsection{Basic Lemmas}

We begin by proving 
some basic properties of the $\hat{\mu}_{\mathbf{x}}$ functional 
that will be indispensable 
in the main proofs below.

\begin{lemma}
\label{lem:expectation}
For any sequence $\mathbf{x} = \{x_t\}_{t=1}^{\infty}$ in $\X$, and any functions $f : \X \rightarrow \reals$, $g : \X \rightarrow \reals$,
if $\hat{\mu}_{\mathbf{x}}(f)$ and $\hat{\mu}_{\mathbf{x}}(g)$ are not both infinite and of opposite signs, the following properties hold.
\begin{align*}
&1. \text{ (monotonicity)}& &\text{if } f \leq g, \text{ then } \hat{\mu}_{\mathbf{x}}(f) \leq \hat{\mu}_{\mathbf{x}}(g),&\\
&2. \text{ (homogeneity)}& &\forall c \in (0,\infty), \hat{\mu}_{\mathbf{x}}(c f) = c \hat{\mu}_{\mathbf{x}}(f),&\\
&3. \text{ (subadditivity)}& &\hat{\mu}_{\mathbf{x}}(f+g) \leq \hat{\mu}_{\mathbf{x}}(f) + \hat{\mu}_{\mathbf{x}}(g).& \phantom{aaaaaaaaaaaaaaaaaaaaaaaaaaaaaa}
\end{align*}
\end{lemma}
\begin{proof}
Properties $1$ and $2$ follow directly from the definition of $\hat{\mu}_{\mathbf{x}}$, and monotonicity and homogeneity (for positive constants) of $\limsup$.
Property $3$ is established by noting 
\begin{align*}
\limsup_{n\rightarrow\infty} \frac{1}{n} \sum_{t=1}^{n} \left( f(x_t)+g(x_t) \right) &\leq \lim_{k\rightarrow\infty} \left(\sup_{n\geq k} \frac{1}{n}\sum_{t=1}^{n}f(x_t)\right) + \left(\sup_{n\geq k} \frac{1}{n} \sum_{t=1}^{n} g(x_t)\right) 
\\ &= \left(\limsup_{n\rightarrow\infty} \frac{1}{n}\sum_{t=1}^{n}f(x_t)\right) + \left(\limsup_{n\rightarrow\infty}\frac{1}{n}\sum_{t=1}^{n}g(x_t)\right).
\end{align*}
\end{proof}

These properties immediately imply related properties for the \emph{set function} $\hat{\mu}_{\mathbf{x}}$.

\begin{lemma}
\label{lem:mu}
For any sequence $\mathbf{x} = \{x_t\}_{t=1}^{\infty}$ in $\X$, and any sets $A,B \subseteq \X$,
\begin{align*}
&1. \text{ (nonnegativity)}&& 0 \leq \hat{\mu}_{\mathbf{x}}(A),&\\
&2. \text{ (monotonicity)}&& \hat{\mu}_{\mathbf{x}}(A\cap B) \leq \hat{\mu}_{\mathbf{x}}(A),&\\
&3. \text{ (subadditivity)}&&\hat{\mu}_{\mathbf{x}}(A\cup B) \leq \hat{\mu}_{\mathbf{x}}(A) + \hat{\mu}_{\mathbf{x}}(B).&\phantom{aaaaaaaaaaaaaaaaaaaaaaaaaaaaaa}
\end{align*}
\end{lemma}
\begin{proof}
These follow directly from the properties listed in Lemma~\ref{lem:expectation}, since $0 \leq \ind_{A}$, $\ind_{A \cap B} \leq \ind_{A}$, and $\ind_{A\cup B} \leq \ind_{A} + \ind_{B}$.
\end{proof}

\subsection{An Equivalent Expression in Terms of Continuous Submeasures}
\label{subsec:submeasure}

Next, we note a connection to a much-studied definition from the measure theory literature:
namely, the notion of a \emph{continuous submeasure}.  This notion appears in the measure theory literature,
most commonly under the name \emph{Maharam submeasure} \citep*[see e.g.,][]{maharam:47,talagrand:08,bogachev:07},
but is also referred to as a \emph{subadditive Dobrakov submeasure} \citep*[see e.g.,][]{dobrakov:74,dobrakov:84},
and related notions arise in discussions of 
\emph{Choquet capacities} \citep*[see e.g.,][]{choquet:54,obrien:94}.

\begin{definition}
\label{def:maharam}
A \emph{submeasure} on $\Borel$ is a function $\nu : \Borel \to [0,\infty]$ satisfying the following properties.
\begin{itemize}
\item[$1.$] $\nu(\emptyset) = 0$.
\item[$2.$] $\forall A,B \in \Borel$, $A \subseteq B \Rightarrow \nu(A) \leq \nu(B)$.
\item[$3.$] $\forall A,B \in \Borel$, $\nu(A \cup B) \leq \nu(A) + \nu(B)$.
\end{itemize}
A submeasure is called \emph{continuous} if it additionally satisfies the condition
\begin{itemize}
\item[$4.$] For every monotone sequence $\{A_{k}\}_{k=1}^{\infty}$ in $\Borel$ with $A_{k} \downarrow \emptyset$, $\lim\limits_{k\to\infty} \nu(A_{k}) = 0$.
\end{itemize}
\end{definition}

Note that we have defined ``submeasure'' to only require \emph{finite} subadditivity.  However, it immediately 
follows that any \emph{continuous} submeasure would also be \emph{countably} subadditive 
\citep*[see][Chapter 39, Lemma 392H]{fremlin:02v3}.
The relevance of this definition to our present discussion is via the set function $\E[\hat{\mu}_{\ProcX}(\cdot)]$,
which we can easily show is always a submeasure, as follows.

\begin{lemma}
\label{lem:Ehatmu-submeasure}
For any process $\ProcX$, $\E[ \hat{\mu}_{\ProcX}(\cdot) ]$ is a submeasure.
\end{lemma}
\begin{proof}
Since $\hat{\mu}_{\ProcX}(\emptyset) = 0$ follows directly from the definition of $\hat{\mu}_{\ProcX}$,
we have $\E[\hat{\mu}_{\ProcX}(\emptyset)] = \E[0] = 0$ as well (property $1$ of Definition~\ref{def:maharam}).
Furthermore, monotonicity of $\hat{\mu}_{\ProcX}$ (Lemma~\ref{lem:expectation})
and monotonicity of the expectation imply monotonicity of $\E[\hat{\mu}_{\ProcX}(\cdot)]$ (property $2$ of Definition~\ref{def:maharam}).
Likewise, finite subadditivity of $\hat{\mu}_{\ProcX}$ (Lemma~\ref{lem:mu}) implies
that for $A,B \in \Borel$, $\hat{\mu}_{\ProcX}(A \cup B) \leq \hat{\mu}_{\ProcX}(A) + \hat{\mu}_{\ProcX}(B)$,
so that monotonicity and linearity of the expectation imply 
$\E[\hat{\mu}_{\ProcX}(A \cup B)] \leq \E[\hat{\mu}_{\ProcX}(A) + \hat{\mu}_{\ProcX}(B)] = \E[\hat{\mu}_{\ProcX}(A)] + \E[\hat{\mu}_{\ProcX}(B)]$
(property 3 of Definition~\ref{def:maharam}).
\end{proof}

Together with the definition of Condition~\ref{con:kc}, this immediately implies the following theorem, 
which states that Condition~\ref{con:kc} is \emph{equivalent} to $\E[\hat{\mu}_{\ProcX}(\cdot)$] being a 
continuous submeasure.

\begin{theorem}
\label{thm:maharam}
A process $\ProcX$ satisfies Condition~\ref{con:kc} if and only if 
$\E[\hat{\mu}_{\ProcX}(\cdot)]$ is a continuous submeasure.
\end{theorem}

\subsection{Other Equivalent Expressions of Condition~\ref{con:kc}}

We next state several other results expressing equivalent formulations of Condition~\ref{con:kc}, 
and other related properties.  These equivalent forms will be useful in proofs below.

\begin{lemma}
\label{lem:limsup-equiv}
The following conditions are all equivalent to Condition~\ref{con:kc}.
\begin{itemize}
\item[$\bullet$] For every monotone sequence $\{A_{k}\}_{k=1}^{\infty}$ of sets in $\Borel$ with $A_{k} \downarrow \emptyset$, 
\begin{equation*}
\lim\limits_{k \rightarrow \infty} \hat{\mu}_{\ProcX}( A_k ) = 0 \text{ (a.s.)}.
\end{equation*}
\item[$\bullet$] For every sequence $\{A_{k}\}_{k=1}^{\infty}$ of sets in $\Borel$,
\begin{equation*}
\lim\limits_{i \rightarrow \infty} \hat{\mu}_{\ProcX}\!\left(\bigcup\limits_{k \geq i} A_k \right) = \hat{\mu}_{\ProcX}\!\left(\limsup\limits_{k \to \infty} A_k\right) \text{ (a.s.)}.
\end{equation*}
\item[$\bullet$] For every disjoint sequence $\{A_{k}\}_{k=1}^{\infty}$ of sets in $\Borel$,
\begin{equation*}
\lim\limits_{i \rightarrow \infty} \hat{\mu}_{\ProcX}\!\left(\bigcup\limits_{k \geq i} A_k \right) = 0 \text{ (a.s.)}.
\end{equation*}
\end{itemize}
\end{lemma}
\begin{proof}
First, suppose $\ProcX$ satisfies Condition~\ref{con:kc},
and let $\{A_{k}\}_{k=1}^{\infty}$ be any monotone sequence in $\Borel$ with $A_{k} \downarrow \emptyset$.
By monotonicity and nonnegativity of the set function $\hat{\mu}_{\ProcX}$ (Lemma~\ref{lem:mu}),
$\lim\limits_{k\to\infty} \hat{\mu}_{\ProcX}(A_{k})$ always exists and is nonnegative.
Therefore, since the set function $\hat{\mu}_{\ProcX}$ is bounded in $[0,1]$, the dominated convergence theorem implies 
$\E\!\left[ \lim\limits_{k\to\infty} \hat{\mu}_{\ProcX}(A_{k}) \right] = \lim\limits_{k\to\infty} \E\left[ \hat{\mu}_{\ProcX}(A_{k}) \right] = 0$,
where the last equality is due to Condition~\ref{con:kc}.
Combined with the fact that $\lim\limits_{k\to\infty} \hat{\mu}_{\ProcX}(A_{k}) \geq 0$,
it follows that $\lim\limits_{k \to \infty} \hat{\mu}_{\ProcX}(A_{k}) = 0$ (a.s.)
\citep*[e.g.,][Theorem 1.6.6]{ash:00}.
Thus, Condition~\ref{con:kc} implies the first condition in the lemma.

Next, let $\ProcX$ be any process satisfying the first condition in the lemma, 
and let $\{A_{k}\}_{k=1}^{\infty}$ be any sequence in $\Borel$.
For each $k \in \nats$, let $B_{k} = A_{k} \setminus \bigcup\limits_{j > k} A_{j}$.
Note that $\{B_{k}\}_{k=1}^{\infty}$ is a sequence of disjoint measurable sets. 
In particular, this implies $\bigcup\limits_{k \geq i} B_{k} \downarrow \emptyset$,
so that (since $\ProcX$ satisfies the first condition) 
$\lim\limits_{i\to\infty} \hat{\mu}_{\ProcX}\!\left( \bigcup\limits_{k \geq i} B_{k} \right) = 0 \text{ (a.s.)}$.
Furthermore, for any $i \in \nats$,
we have $\bigcup\limits_{k \geq i} A_{k} = \left( \limsup\limits_{j\to\infty} A_{j} \right) \cup \bigcup\limits_{k \geq i} B_{k}$.
Therefore,
by finite subadditivity of $\hat{\mu}_{\ProcX}$ (Lemma~\ref{lem:mu}), 
\begin{align*}
\lim_{i\to\infty} \hat{\mu}_{\ProcX}\!\left( \bigcup_{k \geq i} A_{k} \right) 
& = \lim_{i\to\infty} \hat{\mu}_{\ProcX}\!\left( \left( \limsup_{j\to\infty} A_{j} \right) \cup \bigcup_{k \geq i} B_{k} \right)
\\ & \leq \hat{\mu}_{\ProcX}\!\left( \limsup_{j\to\infty} A_{j} \right) + \lim_{i\to\infty} \hat{\mu}_{\ProcX}\!\left( \bigcup_{k \geq i} B_{k} \right)
= \hat{\mu}_{\ProcX}\!\left( \limsup_{j\to\infty} A_{j} \right) \text{ (a.s.)}.
\end{align*}
Furthermore, since $\limsup\limits_{j\to\infty} A_{j} \subseteq \bigcup\limits_{k \geq i} A_{k}$ for every $i \in \nats$, 
monotonicity of $\hat{\mu}_{\ProcX}$ (Lemma~\ref{lem:expectation}) implies $\hat{\mu}_{\ProcX}\!\left( \bigcup\limits_{k \geq i} A_{k} \right) \geq \hat{\mu}_{\ProcX}\!\left( \limsup\limits_{j\to\infty} A_{j} \right)$,
which implies $\lim\limits_{i \to \infty} \hat{\mu}_{\ProcX}\!\left( \bigcup\limits_{k \geq i} A_{k} \right) \geq \hat{\mu}_{\ProcX}\!\left( \limsup\limits_{j\to\infty} A_{j} \right)$.
Together, we have that the first condition implies the second condition in this lemma.
Furthermore, the second condition in this lemma trivially implies the third condition, 
since any \emph{disjoint} sequence $\{A_{k}\}_{k=1}^{\infty}$ in $\Borel$
has $\limsup\limits_{k\to\infty} A_{k} = \emptyset$, and $\hat{\mu}_{\ProcX}(\emptyset) = 0$ is immediate from the definition of $\hat{\mu}_{\ProcX}$.

Finally, suppose the third condition in this lemma holds, and
let $\{A_{k}\}_{k=1}^{\infty}$ be a monotone sequence in $\Borel$ with $A_{k} \downarrow \emptyset$.
For each $k \in \nats$, let $B_{k} = A_{k} \setminus \bigcup\limits_{j > k} A_{j}$.
Note that $\{B_{k}\}_{k=1}^{\infty}$ is a sequence of disjoint sets in $\Borel$, 
and that monotonicity of $\{A_{k}\}_{k=1}^{\infty}$ implies
$\forall k \in \nats$, $A_{k} = \left(\limsup\limits_{j\to\infty} A_{j}\right) \cup \bigcup\limits_{i \geq k} B_{i}$;
furthermore,
$A_{k} \downarrow \emptyset$ implies $\limsup\limits_{j\to\infty} A_{j} = \emptyset$, so that $A_{k} = \bigcup\limits_{i \geq k} B_{i}$.
Therefore, the third condition in the lemma implies 
\begin{equation*}
\lim\limits_{k\to\infty} \hat{\mu}_{\ProcX}(A_{k})
= \lim\limits_{k\to\infty} \hat{\mu}_{\ProcX}\!\left( \bigcup_{i \geq k} B_{i} \right)
= 0 \text{ (a.s.)}.
\end{equation*}
Since the set function $\hat{\mu}_{\ProcX}$ is bounded in $[0,1]$, 
combining this with the dominated convergence theorem implies 
$\lim\limits_{k\to\infty} \E\!\left[ \hat{\mu}_{\ProcX}(A_{k}) \right]
= \E\!\left[ \lim\limits_{k\to\infty} \hat{\mu}_{\ProcX}(A_{k}) \right]
= 0$.
Since this applies to any such sequence $\{A_{k}\}_{k=1}^{\infty}$, 
we have that Condition~\ref{con:kc} holds.
This completes the proof of the lemma.
\end{proof}

In combination with Lemma~\ref{lem:limsup-equiv}, the following lemma allows us to extend Condition~\ref{con:kc} to other useful equivalent forms.
In particular, the form expressed in \eqref{eqn:raw-condition} will be a key component (via Corollary~\ref{cor:missing-mass}) 
in the proof below (in Lemma~\ref{lem:sual-subset-kc}) that Condition~\ref{con:kc} 
is a \emph{necessary} condition for a process $\ProcX$ to admit strong universal self-adaptive learning.

\begin{lemma}
\label{lem:equiv}
For any sequence $\mathbf{x} = \{x_t\}_{t=1}^{\infty}$ of elements of $\X$,
and any sequence $\{A_i\}_{i=1}^{\infty}$ of disjoint subsets of $\X$,
the following conditions are all equivalent.
\begin{align}
& \lim_{k\rightarrow\infty} \hat{\mu}_{\mathbf{x}}\!\left(\bigcup_{i \geq k} A_i\right) = 0. \label{eqn:tail-condition}\\
& \lim_{n\rightarrow\infty} \hat{\mu}_{\mathbf{x}}\!\left(\bigcup\{A_i : x_{1:n} \cap A_i = \emptyset\}\right) = 0. \label{eqn:raw-condition} \\
& \lim_{m\rightarrow\infty} \lim_{n\rightarrow\infty} \hat{\mu}_{\mathbf{x}}\!\left(\bigcup\{A_i : |x_{1:n} \cap A_i | < m\}\right) = 0. \label{eqn:growing-sets-condition}
\end{align}
\end{lemma}
\begin{proof}
Fix $\mathbf{x}$ and $\{A_{i}\}_{i=1}^{\infty}$ as described.
For each $x \in \bigcup\limits_{i=1}^{\infty} A_{i}$, let $i(x)$ denote the index $i \in \nats$ with $x \in A_{i}$;
for each $x \in \X \setminus \bigcup\limits_{i=1}^{\infty} A_{i}$, let $i(x) = 0$.
First, suppose \eqref{eqn:raw-condition} is satisfied.
For any $k \in \nats$, let 
\begin{equation*}
n_{k} = \max\left\{ n \in \nats \cup \{0,\infty\} : x_{1:n} \cap \bigcup\limits_{i \geq k} A_{i} = \emptyset \right\}.
\end{equation*}
By definition of $n_{k}$, we have $\bigcup\limits_{i \geq k} A_{i} \subseteq \bigcup \{ A_{i} : x_{1:n_{k}} \cap A_{i} = \emptyset \}$,
so that monotonicity of $\hat{\mu}_{\mathbf{x}}$ (Lemma~\ref{lem:mu}) implies
\begin{equation}
\label{eqn:raw-to-tail-initial}
\lim_{k \to \infty} \hat{\mu}_{\mathbf{x}}\!\left( \bigcup_{i \geq k} A_{i} \right)
\leq \lim_{k \to \infty} \hat{\mu}_{\mathbf{x}}\!\left( \bigcup \{ A_{i} : x_{1:n_{k}} \cap A_{i} = \emptyset \} \right).
\end{equation}
Next note that monotonicity of $\bigcup\limits_{i \geq k} A_{i}$ implies $n_{k}$ is nondecreasing in $k$.
In particular, this implies that if any $k \in \nats$ has $n_{k} = \infty$, then
\begin{equation*}
\lim_{k \to \infty} \hat{\mu}_{\mathbf{x}}\!\left( \bigcup \{ A_{i} : x_{1:n_{k}} \cap A_{i} = \emptyset \} \right)
= \hat{\mu}_{\mathbf{x}}\!\left( \bigcup \{ A_{i} : \mathbf{x} \cap A_{i} = \emptyset \} \right)
= 0
\end{equation*}
by definition of $\hat{\mu}_{\mathbf{x}}$, which establishes \eqref{eqn:tail-condition} when combined with \eqref{eqn:raw-to-tail-initial}.
Otherwise, suppose $n_{k} < \infty$ for all $k \in \nats$.
In this case, we will argue that $n_{k} \to \infty$.
Note that $\forall k \in \nats$, by maximality of $n_{k}$, we have $x_{n_{k}+1} \in \bigcup\limits_{i \geq k} A_{i}$, 
so that $i(x_{n_{k}+1}) \geq k$.
Together with the definition of $n_{k}$ this also implies 
that for $k^{\prime} = i(x_{n_{k}+1})+1$ we have 
$x_{1:(n_{k}+1)} \cap \bigcup\limits_{i \geq k^{\prime}} A_{i} = \emptyset$, 
and therefore $n_{k^{\prime}} \geq n_{k}+1$.
Together with monotonicity of $n_{k}$ in $k$, this implies $n_{k} \to \infty$.
Combined with \eqref{eqn:raw-condition} and monotonicity of 
$\hat{\mu}_{\mathbf{x}}\!\left( \bigcup \{ A_{i} : x_{1:n} \cap A_{i} = \emptyset \} \right)$ in $n$, this implies that
\begin{equation*}
\lim_{k \to \infty} \hat{\mu}_{\mathbf{x}}\!\left( \bigcup \{ A_{i} : x_{1:n_{k}} \cap A_{i} = \emptyset \} \right)
= \lim_{n \to \infty} \hat{\mu}_{\mathbf{x}}\!\left( \bigcup \{ A_{i} : x_{1:n} \cap A_{i} = \emptyset \} \right)
= 0,
\end{equation*}
which establishes \eqref{eqn:tail-condition} when combined with \eqref{eqn:raw-to-tail-initial}
and nonnegativity of $\hat{\mu}_{\mathbf{x}}$ (Lemma~\ref{lem:mu}).

Next, suppose \eqref{eqn:tail-condition} is satisfied, and fix any $m \in \nats$.
By inductively applying the finite subadditivity property of $\hat{\mu}_{\mathbf{x}}$ (Lemma~\ref{lem:mu}), for any $n,k \in \nats$,
\begin{equation}
\label{eqn:tail-to-growing-sets-initial}
\hat{\mu}_{\mathbf{x}}\!\left( \bigcup \left\{A_{i} : |x_{1:n} \cap A_{i} | < m \right\} \right)
\leq \hat{\mu}_{\mathbf{x}}\!\left( \bigcup \left\{ A_{i} : | x_{1:n} \cap A_{i} | < m, i \geq k \right\} \right) + \sum_{\substack{i \in \{1,\ldots,k-1\} :\\ |x_{1:n} \cap A_{i}| < m}} \hat{\mu}_{\mathbf{x}}(A_{i}).
\end{equation}
Note that, for any $i \in \nats$ with $\hat{\mu}_{\mathbf{x}}(A_{i}) > 0$,
there must be an infinite subsequence of $\mathbf{x}$ contained in $A_{i}$;
in particular, this implies $\exists n_{i}^{\prime} \in \nats$ with $|x_{1:n_{i}^{\prime}} \cap A_{i}| = m$.
Also define $n_{i}^{\prime} = 0$ for every $i \in \nats$ with $\hat{\mu}_{\mathbf{x}}(A_{i}) = 0$.
Therefore, for every $n \in \nats$, defining 
\begin{equation*}
k_{n} = \min\left( \left\{ i \in \nats : n_{i}^{\prime} > n \right\}\cup\{\infty\} \right),
\end{equation*}
we have that every $i < k_{n}$ has either $|x_{1:n} \cap A_{i}| \geq m$ or $\hat{\mu}_{\mathbf{x}}(A_{i}) = 0$.
Thus, it follows that 
\begin{equation}
\label{eqn:tail-to-growing-sets-first-term}
\sum_{\substack{i \in \{1,\ldots,k_{n}-1\} :\\ |x_{1:n} \cap A_{i}| < m}} \hat{\mu}_{\mathbf{x}}(A_{i}) = 0.
\end{equation}

Next we argue that $k_{n} \to \infty$.  To see this, note that by definition $k_{n}$ is nondecreasing,
and if $k_{n} < \infty$, then
any $n^{\prime} \geq n_{k_{n}}^{\prime}$ has $n^{\prime} > n$ 
(since $n_{k_{n}}^{\prime} > n$ by the definition of $k_{n}$), 
and hence $n_{i}^{\prime} \leq n^{\prime}$ for every $i \leq k_{n}$ 
(since minimality of $k_{n}$ implies $n_{i}^{\prime} \leq n < n^{\prime}$ for every $i < k_{n}$, and by assumption $n_{k_{n}}^{\prime} \leq n^{\prime}$),
which implies $k_{n^{\prime}} \geq k_{n}+1$.
Therefore, we have $k_{n} \to \infty$.
Thus, combined with \eqref{eqn:tail-to-growing-sets-initial} and \eqref{eqn:tail-to-growing-sets-first-term}, and monotonicity of $\hat{\mu}_{\mathbf{x}}$ (Lemma~\ref{lem:mu}),
we have
\begin{align*}
& \lim_{n\to\infty} \hat{\mu}_{\mathbf{x}}\!\left( \bigcup \left\{A_{i} : |x_{1:n} \cap A_{i} | < m \right\} \right)
\\ & \leq \lim_{n\to\infty}  \hat{\mu}_{\mathbf{x}}\!\left( \bigcup \left\{ A_{i} : | x_{1:n} \cap A_{i} | < m, i \geq k_{n} \right\} \right)
+ \sum_{\substack{i \in \{1,\ldots,k_{n}-1\} :\\ |x_{1:n} \cap A_{i}| < m}} \hat{\mu}_{\mathbf{x}}(A_{i})
\\ & = \lim_{n\to\infty} \hat{\mu}_{\mathbf{x}}\!\left( \bigcup \left\{ A_{i} : | x_{1:n} \cap A_{i} | < m, i \geq k_{n} \right\} \right)
\leq \lim_{k\to\infty} \hat{\mu}_{\mathbf{x}}\!\left( \bigcup_{i \geq k} A_{i} \right).
\end{align*}
If \eqref{eqn:tail-condition} is satisfied, this last expression is $0$.
Thus, 
\begin{equation*}
\lim_{n\to\infty} \hat{\mu}_{\mathbf{x}}\!\left( \bigcup \left\{A_{i} : |x_{1:n} \cap A_{i} | < m \right\} \right)
= 0
\end{equation*}
for all $m \in \nats$.  Taking the limit of both sides as $m \to \infty$ establishes \eqref{eqn:growing-sets-condition}.

Finally, note that for any $n \in \nats$,
\begin{equation*}
\hat{\mu}_{\mathbf{x}}\!\left( \bigcup \{A_{i} : x_{1:n} \cap A_{i} = \emptyset \} \right)
= \hat{\mu}_{\mathbf{x}}\!\left( \bigcup \{A_{i} : |x_{1:n} \cap A_{i}| < 1\} \right),
\end{equation*}
and monotonicity of $\hat{\mu}_{\mathbf{x}}$ (Lemma~\ref{lem:mu}) implies that for any $m \in \nats$,
\begin{equation*}
\hat{\mu}_{\mathbf{x}}\!\left( \bigcup \{A_{i} : |x_{1:n} \cap A_{i}| < 1\} \right)
\leq \hat{\mu}_{\mathbf{x}}\!\left( \bigcup \{A_{i} : |x_{1:n} \cap A_{i}| < m\} \right).
\end{equation*}
Taking limits of both sides, we have
\begin{equation*}
\lim_{n\to\infty} \hat{\mu}_{\mathbf{x}}\!\left( \bigcup \{A_{i} : |x_{1:n} \cap A_{i}| < 1\} \right)
\leq \lim_{m\to\infty} \lim_{n\to\infty} \hat{\mu}_{\mathbf{x}}\!\left( \bigcup \{A_{i} : |x_{1:n} \cap A_{i}| < m\} \right).
\end{equation*}
Thus, if \eqref{eqn:growing-sets-condition} is satisfied, then \eqref{eqn:raw-condition} must also hold.
\end{proof}

This further implies the following corollary relating Condition~\ref{con:kc} 
to a requirement of having vanishing \emph{missing mass} in any countable discretization of $\X$.

\begin{corollary}
\label{cor:missing-mass}
A process $\ProcX$ satisfies Condition~\ref{con:kc} if and only if 
every disjoint sequence $\{A_i\}_{i=1}^{\infty}$ in $\Borel$ with $\bigcup\limits_{i=1}^{\infty} A_i = \X$ (i.e., every countable measurable partition)
satisfies
\begin{equation}
\label{eqn:missing-mass-condition}
\lim_{n\rightarrow\infty} \hat{\mu}_{\ProcX}\!\left(\bigcup\{A_i : X_{1:n} \cap A_i = \emptyset\}\right) = 0 \text{ (a.s.)}.
\end{equation}
\end{corollary}
\begin{proof}
If $\ProcX$ satisfies Condition~\ref{con:kc}, 
then for any disjoint sequence $\{A_i\}_{i=1}^{\infty}$ in $\Borel$, 
Lemma~\ref{lem:limsup-equiv} implies that, 
on an event of probability one, 
$\lim\limits_{k \to \infty} \hat{\mu}_{\ProcX}\!\left(\bigcup\limits_{i \geq k} A_i \right) = 0$.
The equivalence of \eqref{eqn:tail-condition} and \eqref{eqn:raw-condition} 
in Lemma~\ref{lem:equiv} then implies that, on this same event, 
$\lim\limits_{n\rightarrow\infty} \hat{\mu}_{\ProcX}\!\left(\bigcup\{A_i : X_{1:n} \cap A_i = \emptyset\}\right) = 0$,
so that \eqref{eqn:missing-mass-condition} holds.

On the other hand, if $\ProcX$ does not satisfy Condition~\ref{con:kc}, 
then Lemma~\ref{lem:limsup-equiv} implies that there exists a disjoint sequence 
$\{A_{i}^{\prime}\}_{i=1}^{\infty}$ in $\Borel$ such that, on an event of nonzero probability, 
$\lim\limits_{k \to \infty} \hat{\mu}_{\ProcX}\!\left(\bigcup\limits_{i \geq k} A_{i}^{\prime} \right) > 0$.
Since this claim only involves the tail of the $A_{i}^{\prime}$ sequence, 
if we define $A_{1} = \X \setminus \bigcup\limits_{i=1}^{\infty} A_{i}^{\prime}$ 
and $A_{i+1} = A_{i}^{\prime}$ for $i \in \nats$ (so that $\{A_i\}_{i=1}^{\infty}$ is a countable measurable partition of $\X$), 
then on this same event 
we have $\lim\limits_{k \to \infty} \hat{\mu}_{\ProcX}\!\left(\bigcup\limits_{i \geq k} A_{i} \right) > 0$.
The equivalence of \eqref{eqn:tail-condition} and \eqref{eqn:raw-condition} 
in Lemma~\ref{lem:equiv} then implies that, on this same event, 
$\lim\limits_{n\rightarrow\infty} \hat{\mu}_{\ProcX}\!\left(\bigcup\{A_{i} : X_{1:n} \cap A_{i} = \emptyset\}\right) > 0$, 
so that \eqref{eqn:missing-mass-condition} does not hold.
\end{proof}

One interesting property of processes $\ProcX$ satisfying Condition~\ref{con:kc}
is that $\hat{\mu}_{\ProcX}$ is \emph{countably subadditive} (almost surely),
as implied by the following two lemmas.
Note that this is not necessarily true
of processes $\ProcX$ failing to satisfy Condition~\ref{con:kc} 
(e.g., the process $X_{i} = i$ on $\nats$ does not have countably subadditive $\hat{\mu}_{\ProcX}$).  
However, we note that this kind of countable subadditivity is not actually \emph{equivalent} to Condition~\ref{con:kc}, 
as not every process satisfying this countable subadditivity condition also satisfies Condition~\ref{con:kc}
(e.g., any deterministic process $\ProcX$ on $\nats$ with $\forall i \in \nats, \hat{\mu}_{\ProcX}(\{i\})=1$  
has countably subadditive $\hat{\mu}_{\ProcX}$ and yet $\ProcX \notin \KC$).

\begin{lemma}
\label{lem:subadditive}
For any sequence $\mathbf{x} = \{x_t\}_{t=1}^{\infty}$ of elements of $\X$,
and any sequence $\{A_i\}_{i=1}^{\infty}$ of disjoint subsets of $\X$,
if \eqref{eqn:tail-condition} is satisfied, then
\begin{equation*}
\hat{\mu}_{\mathbf{x}}\!\left(\bigcup_{i = 1}^{\infty} A_i\right) \leq \sum_{i = 1}^{\infty} \hat{\mu}_{\mathbf{x}}(A_i).
\end{equation*}
\end{lemma}
\begin{proof}
By finite subadditivity of $\hat{\mu}_{\mathbf{x}}$ (Lemma~\ref{lem:mu} and induction), we have that for any $k \in \nats$, 
\begin{equation}
\label{eqn:tail-to-subadd-initial}
\hat{\mu}_{\mathbf{x}}\!\left( \bigcup_{i = 1}^{\infty} A_{i} \right) 
\leq \hat{\mu}_{\mathbf{x}}\!\left( \bigcup_{i \geq k} A_{i} \right) + \sum_{i=1}^{k-1} \hat{\mu}_{\mathbf{x}}( A_{i} ).
\end{equation}
If \eqref{eqn:tail-condition} is satisfied, then 
$\lim\limits_{k \to \infty} \hat{\mu}_{\mathbf{x}}\!\left( \bigcup\limits_{i \geq k} A_{i} \right) = 0$,
so that taking the limit as $k \to \infty$ in \eqref{eqn:tail-to-subadd-initial} yields the claimed inequality, completing the proof.
\end{proof}

\begin{lemma}
\label{lem:kc-subadditive}
If $\ProcX$ satisfies Condition~\ref{con:kc}, then for any sequence $\{A_{i}\}_{i=1}^{\infty}$ in $\Borel$, 
\begin{equation*}
\hat{\mu}_{\ProcX}\!\left( \bigcup_{i=1}^{\infty} A_{i} \right) \leq \sum_{i=1}^{\infty} \hat{\mu}_{\ProcX}(A_{i})~~~\text{ (a.s.)}.
\end{equation*}
\end{lemma}
\begin{proof}
Let $B_{1} = A_{1}$, and for each $i \in \nats \setminus \{1\}$, let $B_{i} = A_{i} \setminus \bigcup\limits_{j=1}^{i-1} A_{j}$.
Then $\{B_{i}\}_{i=1}^{\infty}$ is a disjoint sequence in $\Borel$.  If $\ProcX$ satisfies Condition~\ref{con:kc}, then
Lemma~\ref{lem:limsup-equiv} implies $\lim\limits_{k\to\infty} \hat{\mu}_{\ProcX}\!\left( \bigcup\limits_{j \geq k} B_{j} \right) = 0$ (a.s.).
Combined with Lemma~\ref{lem:subadditive}, this implies that $\hat{\mu}_{\ProcX}\!\left( \bigcup\limits_{i=1}^{\infty} B_{i} \right) \leq \sum\limits_{i=1}^{\infty} \hat{\mu}_{\ProcX}(B_{i})$ (a.s.).
Noting that $\bigcup\limits_{i=1}^{\infty} B_{i} = \bigcup\limits_{i=1}^{\infty} A_{i}$, we have $\hat{\mu}_{\ProcX}\!\left( \bigcup\limits_{i=1}^{\infty} A_{i} \right) \leq \sum\limits_{i=1}^{\infty} \hat{\mu}_{\ProcX}(B_{i})$ (a.s.).
Finally, since $B_{i} \subseteq A_{i}$ for every $i \in \nats$, monotonicity of $\hat{\mu}_{\ProcX}$ (Lemma~\ref{lem:mu}) implies $\hat{\mu}_{\ProcX}(B_{i}) \leq \hat{\mu}_{\ProcX}(A_{i})$,
so that $\hat{\mu}_{\ProcX}\!\left( \bigcup\limits_{i=1}^{\infty} A_{i} \right) \leq \sum\limits_{i=1}^{\infty} \hat{\mu}_{\ProcX}(A_{i})$ (a.s.).
\end{proof}

\section{Relation to the Condition of Convergent Relative Frequencies}
\label{sec:examples}

Before proceeding with the general analysis, we first discuss the relation between Condition~\ref{con:kc} 
and the commonly-studied condition of \emph{convergent relative frequencies}.
In particular, we show that Condition~\ref{con:kc} is a \emph{strictly more-general} condition.
This is interesting in the context of learning, as the vast majority of the prior literature on statistical 
learning theory without the i.i.d.\ assumption studies learning rules designed for and analyzed under 
assumptions that imply convergent relative frequencies.  These results therefore indicate that we should 
not expect such learning rules to be optimistically universal, and hence that we will need to seek more general 
strategies in designing optimistically universal learning rules. 
In particular, in Section~\ref{subsec:inconsistent-nn} we provide an example of 
a process satisfying Condition~\ref{con:kc} under which 
the nearest neighbor predictor fails to be universally consistent.

Formally, define $\CRF$ as the set of processes $\ProcX$ such that, $\forall A \in \Borel$, 
\begin{equation}
\label{eqn:crf}
\lim_{m \to \infty} \frac{1}{m} \sum_{t=1}^{m} \ind_{A}(X_{t}) \text{ exists (a.s.)}.
\end{equation}
These processes are said to have \emph{convergent relative frequencies}.
Equivalently, this is the family of processes with \emph{ergodic properties} with respect to the 
class of measurements $\{ \ind_{A \times \X^{\infty}} : A \in \Borel \}$ \citep*{gray:09}.
Certainly $\CRF$ contains every i.i.d.\ process, by the \emph{strong law of large numbers}.
More generally, it is known that 
any \emph{stationary} process $\ProcX$ is contained in $\CRF$ (by Birkhoff's ergodic theorem),
and in fact, it suffices for the process to be \emph{asymptotically mean stationary}
\citep*[][Theorem 8.1]{gray:09}: 
that is, $\forall A \in \Borel^{\infty}$, $\lim\limits_{n \to \infty} \frac{1}{n} \sum\limits_{t=1}^{n} \P( X_{t:\infty} \in A )$ exists.

\subsection{Processes with Convergent Relative Frequencies Satisfy Condition~\ref{con:kc}}
\label{sec:lln}

The following theorem establishes that every $\ProcX$ with convergent relative frequencies satisfies Condition~\ref{con:kc},
and that the inclusion is \emph{strict} in all nontrivial cases.

\begin{theorem}
\label{thm:crf-implies-kc}
$\CRF \subseteq \KC$, and the inclusion is \emph{strict} iff $|\X| \geq 2$.
\end{theorem}
\begin{proof}
Fix any $\ProcX \in \CRF$.
For each $A \in \Borel$, define $\pi_{m}(A) = \frac{1}{m} \sum_{t=1}^{m} \P( X_{t} \in A )$.
One can easily verify that $\pi_{m}$ is a probability measure.
The definition of $\CRF$ implies that, $\forall A \in \Borel$, there exists an event $E_{A}$ of probability one,
on which $\lim\limits_{m \to \infty} \frac{1}{m} \sum_{t=1}^{m} \ind_{A}(X_{t})$ exists;
in particular, this implies
$\hat{\mu}_{\ProcX}(A) = \lim\limits_{m \to \infty} \frac{1}{m} \sum_{t=1}^{m} \ind_{A}(X_{t}) \ind_{E_{A}}$ almost surely.
Together with the dominated convergence theorem and linearity of expectations, this implies
\begin{align*}
\E\!\left[ \hat{\mu}_{\ProcX}( A ) \right]
& = \E\!\left[ \lim_{m \to \infty} \frac{1}{m} \sum_{t=1}^{m} \ind_{A}(X_{t}) \ind_{E_{A}} \right]
= \lim_{m \to \infty} \E\!\left[ \frac{1}{m} \sum_{t=1}^{m} \ind_{A}(X_{t}) \ind_{E_{A}} \right]
\\ & = \lim_{m \to \infty} \frac{1}{m} \sum_{t=1}^{m} \P\!\left( X_{t} \in A \right)
= \lim_{m \to \infty} \pi_{m}(A).
\end{align*}
In particular, this establishes that the limit in the rightmost expression exists.
The Vitali-Hahn-Saks theorem then implies that $\lim\limits_{m \to \infty} \pi_{m}(\cdot)$ is also a probability measure \citep*[see][Lemma 7.4]{gray:09}.
Thus, we have established that $A \mapsto \E\!\left[ \hat{\mu}_{\ProcX}(A) \right]$ is a probability measure,
and hence is a continuous submeasure \citep*[see e.g.,][Theorem A.19]{schervish:95}.
That $\CRF \subseteq \KC$ now follows from Theorem~\ref{thm:maharam}.

For the claim about strict inclusion, first note that if $|\X|=1$ then there is effectively 
only one possible process (infinitely repeating the sole element of $\X$), and it is trivially in $\CRF$, 
so that $\CRF = \KC$.  On the other hand, suppose $|\X| \geq 2$, let $x_{0},x_{1}$ be distinct elements of $\X$,
and define a deterministic process $\ProcX$ such that, for every $i \in \nats$ and every $t \in \{3^{i-1},\ldots,3^{i}-1\}$, $X_{t} = x_{i-2\lfloor i/2 \rfloor}$: 
that is, $X_{t} = x_{0}$ if $i$ is even and $X_{t} = x_{1}$ if $i$ is odd. 
Since any monotone sequence $\{A_{k}\}_{k=1}^{\infty}$ in $\Borel$ with $A_{k} \downarrow \emptyset$ 
necessarily has some $k_{0} \in \nats$ with $\{x_{0},x_{1}\} \cap A_{k} = \emptyset$ for all $k \geq k_{0}$, 
we have $\E[ \hat{\mu}_{\ProcX}(A_{k}) ] = 0$ for all $k \geq k_{0}$, so that $\ProcX \in \KC$.
However, for any odd $i$, $\frac{1}{3^{i}-1} \sum\limits_{t=1}^{3^{i}-1} \ind_{\{x_{1}\}}(X_{t}) \geq \frac{2}{3}$,
so that $\limsup\limits_{m \to \infty} \frac{1}{m} \sum\limits_{t=1}^{m} \ind_{\{x_{1}\}}(X_{t}) \geq \frac{2}{3}$, 
while for any even $i$, $\frac{1}{3^{i}-1} \sum\limits_{t=1}^{3^{i}-1} \ind_{\{x_{1}\}}(X_{t}) \leq \frac{1}{3}$,
so that $\liminf\limits_{m \to \infty} \frac{1}{m} \sum\limits_{t=1}^{m} \ind_{\{x_{1}\}}(X_{t}) \leq \frac{1}{3}$.
Therefore $\frac{1}{m} \sum\limits_{t=1}^{m} \ind_{\{x_{1}\}}(X_{t})$ does not have a limit as $m \to \infty$,
and hence $\ProcX \notin \CRF$.
\end{proof}

\subsection{Inconsistency of the Nearest Neighbor Rule}
\label{subsec:inconsistent-nn}

The separation $\KC \setminus \CRF \neq \emptyset$ established above indicates that, 
in approaching the design of consistent inductive or self-adaptive learning rules 
under processes in $\KC$, we should not rely on the property of having convergent relative frequencies,
as it is not generally guaranteed to hold.  Since most learning rules in the prior literature rely heavily on 
this property for their performance guarantees, we should not generally expect them to be consistent under processes in $\KC$. 
To give a concrete example illustrating this, consider $\X \subseteq \reals^d$ (with the standard topology),
and let $f_{n}$ be the well-known \emph{nearest neighbor} learning rule:
an inductive learning rule defined by the property that $f_{n}(x_{1:n},y_{1:n},x) = y_{i_{n}}$, 
where $i_{n} = \argmin\limits_{i \in \{1,\ldots,n\}} \|x - x_{i}\|$ (with an appropriate policy for breaking ties).
For classification and regression in $\reals^d$  
this learning rule is known to be strongly universally consistent 
(in the sense of Definition~\ref{def:suil}) under every i.i.d.\! process \citep*[e.g.,][]{devroye:96}.

We exhibit a process $\ProcX \in \KC$ for $\X = [0,1]$, under which the nearest neighbor inductive learning rule is \emph{not} universally consistent
for binary classification.\footnote{Of course, 
Theorem~\ref{thm:no-optimistic-inductive} indicates that \emph{any} inductive learning rule has processes in $\KC$ for which it is not universally consistent.
However, the construction here
is more direct, and illustrates a common failing of many 
learning rules designed for i.i.d.\! 
data, so 
it is worth presenting this specialized
argument as well.}
This also provides a second proof that $\KC \setminus \CRF \neq \emptyset$ for this space, 
as this process will not have convergent relative frequencies.
Specifically, let $\{W_{i}\}_{i=1}^{\infty}$ be independent ${\rm Uniform}(0,1/2)$ random variables.
Let $n_{1} = 1$, and for each $k \in \nats$ with $k \geq 2$, inductively define $n_{k} = n_{k-1} + k \cdot n_{k-1}^{2}$.
Now for each $k \in \nats$, let $a_{k} = k - 2 \lfloor k/2 \rfloor$ (i.e., $a_{k} = 1$ if $k$ is odd, and otherwise $a_{k} = 0$),
and let $b_{k} = 1-a_{k}$.  Define $X_{1} = 0$, and for each $k \in \nats$ with $k \geq 2$, and each $i \in \{1,\ldots,n_{k-1}^{2}\}$,
define $X_{n_{k-1} + (i-1)k+1} = \frac{b_{k}}{2}+\frac{i-1}{2 n_{k-1}^{2}}$,
and for each $j \in \{2,\ldots,k\}$, define $X_{n_{k-1} + (i-1)k + j} = \frac{a_{k}}{2} + W_{n_{k-1}+(i-1)k+j}$.

The intention in constructing this process is that there are segments of the sequence in which the fraction of the data in $[0,1/2)$ is 
relatively small compared to $[1/2,1]$, and other segments of the sequence in which the fraction of the data in $[1/2,1]$ is relatively
small compared to $[0,1/2)$.  Furthermore, at certain time points (namely, the $n_{k}$ times), the vast 
majority of the points on the sparse side are determined \emph{a priori}, in contrast to the points on the 
dense side, which are uniform random.  This is designed to frustrate most learning rules designed under
the $\CRF$ assumption, many of which would base their predictions in the sparse side on these deterministic
points, rather than the relatively very-sparse random points in the same region left over from the previous
epoch (i.e., when that region was relatively dense, and the majority of points in that region were uniform random).
It is easy to verify that, because of this switching of which side is dense and which side sparse, which occurs 
infinitely many times, this process $\ProcX$ does \emph{not} have convergent relative frequencies.

We first argue that $\ProcX$ satisfies Condition~\ref{con:kc}.
Let $I = \{1\} \cup \{ n_{k-1} + (i-1)k + 1 : k \in \nats \setminus \{1\}, i \in \{1,\ldots,n_{k-1}^{2}\} \}$.
Note that, for any $k \in \nats \setminus \{1,2\}$ and any $m \in \{n_{k-1}+1,\ldots,n_{k}\}$,
\begin{align*}
|\{ t \in I : t \leq m \}| 
& \leq n_{k-2} + \frac{n_{k-1} - n_{k-2}}{k-1} + \left\lceil \frac{m-n_{k-1}}{k} \right\rceil
\leq 1 + n_{k-2} + \frac{n_{k-1}}{k-1} + \frac{m - n_{k-1}}{k-1}
\\ & = 1 + n_{k-2} + \frac{m}{k-1}
\leq 1 + \sqrt{\frac{n_{k-1}}{k-1}} + \frac{m}{k-1}
\leq 1 + \sqrt{\frac{m}{k-1}} + \frac{m}{k-1}.
\end{align*}
Thus, letting $k_{m} = \min\{ k \in \nats : m \leq n_{k} \}$ for each $m \in \nats$, 
and noting that $k_{m} \to \infty$ (since each $n_{k}$ is finite), 
we have that
\begin{equation*}
\lim_{m \to \infty} \frac{|\{t \in I : t \leq m \}|}{m} 
\leq \lim_{m \to \infty} \frac{1}{m} + \sqrt{\frac{1}{m(k_{m}-1)}} + \frac{1}{k_{m}-1}
= 0.
\end{equation*}
We therefore have that, for any set $A \in \Borel$, 
\begin{equation*}
\hat{\mu}_{\ProcX}(A)
\leq \limsup_{m \to \infty} \frac{1}{m} \sum_{t=1}^{m} \ind_{\nats \setminus I}(t) \ind_{A}(X_{t}) + \!\lim_{m \to \infty} \!\frac{|\{t \!\in\! I \!:\! t \!\leq\! m\}|}{m}
= \limsup_{m \to \infty} \frac{1}{m} \sum_{t=1}^{m} \ind_{\nats \setminus I}(t) \ind_{A}(X_{t}).
\end{equation*}
Furthermore,
noting that  
every $t \in \nats \setminus I$ has $X_{t} \in \left\{ W_{t},\frac{1}{2} + W_{t} \right\}$,
we have 
\begin{equation*}
\limsup_{m \to \infty} \frac{1}{m} \sum_{t=1}^{m} \ind_{\nats \setminus I}(t) \ind_{A}(X_{t})
\leq \limsup_{m \to \infty} \frac{1}{m} \!\sum_{t=1}^{m} \!\left(\ind_{A}(W_{t}) \!+\! \ind_{A}\!\left(\frac{1}{2}\!+\!W_{t}\right) \right),
\end{equation*}
and the strong law of large numbers implies that, with probability one, the expression on the right hand side equals 
$2\lambda(A \cap (0,1/2)) + 2 \lambda( A \cap (1/2,1) ) = 2 \lambda(A)$, where $\lambda$ is the Lebesgue measure.
In particular, this implies $\E\!\left[ \hat{\mu}_{\ProcX}(A) \right] \leq 2\lambda(A)$ for every $A \in \Borel$.
Therefore,
for any monotone sequence $\{A_{k}\}_{k=1}^{\infty}$ in $\Borel$ with $A_{k} \downarrow \emptyset$,
$\lim\limits_{k \to \infty} \E\!\left[ \hat{\mu}_{\ProcX}(A_{k}) \right] \leq \lim\limits_{k \to \infty} 2 \lambda(A_{k}) = 0$ 
since $2 \lambda(\cdot)$ is a finite measure (because $\X$ is bounded) and therefore is continuous 
\citep*[see e.g.,][Theorem A.19]{schervish:95}.
Thus, $\ProcX$ satisfies Condition~\ref{con:kc}.

\ignore{We can also easily verify that $\ProcX \notin \CRF$.
A clear example of a set $A$ for which \eqref{eqn:crf} is violated
is $A = (0,1/2)$.  Here we may simply note that, for any $k \in \nats$ with $a_{k} = 1$ (i.e., any \emph{odd} $k$),
we have $|X_{1:n_{k}} \cap (0,1/2)| \leq n_{k-1} + n_{k-1}^{2} \leq 2 n_{k-1}^{2} \leq 2 \frac{n_{k}}{k}$,
while for any $k \in \nats$ with $a_{k} = 0$ (i.e., any \emph{even} $k$), we have
$|X_{1:n_{k}} \cap (0,1/2)| \geq (k-1) n_{k-1}^{2} = \frac{k-1}{k} (n_{k} - n_{k-1})$.
Thus, since $n_{k} \to \infty$, we have
\begin{equation*}
\liminf_{m \to \infty} \frac{1}{m} \sum_{t=1}^{m} \ind_{(0,1/2)}(X_{t}) \leq \liminf_{k \to \infty} \frac{|X_{1:n_{k}} \cap (0,1/2)|}{n_{k}}\leq \liminf_{k \to \infty} 2 \frac{1}{k} = 0,
\end{equation*}
while
\begin{align*}
& \limsup_{m \to \infty} \frac{1}{m} \sum_{t=1}^{m} \ind_{(0,1/2)}(X_{t}) 
\geq \limsup_{k \to \infty} \frac{|X_{1:n_{k}} \cap (0,1/2)|}{n_{k}} 
\\ & \geq \limsup_{k \to \infty} \frac{k-1}{k} \left( 1 - \frac{n_{k-1}}{n_{k}} \right) \geq \limsup_{k \to \infty} \frac{k-1}{k} \left( 1 - \frac{1}{k n_{k-1}} \right) = 1.
\end{align*}
Since these are different, the limit in \eqref{eqn:crf} does not exist for $A = (0,1/2)$ under this process $\ProcX$,
which implies $\ProcX \notin \CRF$.}

Now to see that the nearest neighbor rule is not universally consistent under this process $\ProcX$, 
let $y_{0},y_{1} \in \Y$ be such that $\loss(y_{0},y_{1}) > 0$.
Define
\begin{equation*}
V = \left\{ \frac{b_{k}}{2} + \frac{i-1}{2 n_{k-1}^{2}} : k \in \nats \setminus \{1\}, i \in \left\{1,\ldots,n_{k-1}^{2}\right\} \right\},
\end{equation*}
and take $\target(x) = y_{1}$ for $x \in [0,1] \setminus V$, and $\target(x) = y_{0}$ for $x \in V$,
and note that this is a measurable function since $V$ is countable. 
Note that we have defined $\target$ so that every $t \in I$ has $\target(X_{t}) = y_{0}$, 
and with probability one every $t \in \nats \setminus I$ has $\target(X_{t}) = y_{1}$.
Then note that, for any $k \in \nats \setminus \{1,2\}$ with $a_{k} = 1$,
the points $\{X_{i} : 1 \leq i \leq n_{k}, \target(X_{i}) = y_{0}\}$ form a $\frac{1}{2n_{k-1}^{2}}$ cover of $[0,1/2)$.
Furthermore, the set $\{X_{i} : 1 \leq i \leq n_{k}, \target(X_{i}) = y_{1} \} \cap (0,1/2)$ contains at most $n_{k-1}$ points.
Together, these facts imply that for the nearest neighbor inductive learning rule $f_n$, 
letting $N_{k} = \{ x \in [0,1] : f_{n_{k}}(X_{1:n_{k}},\target(X_{1:n_{k}}),x) = y_{0} \}$, it holds that   
$\lambda( N_{k} \cap (0,1/2) ) \geq \frac{1}{2} - \frac{n_{k-1}}{2 n_{k-1}^{2}} = \frac{1}{2} \left( 1 - \frac{1}{n_{k-1}} \right)$.
In particular, this implies that a ${\rm Uniform}(0,1/2)$ random variable 
(independent from $f_{n_{k}}$ and $X_{1:n_{k}}$) has probability at least $1 - \frac{1}{n_{k-1}}$ of being in $N_{k}$.
However, for every $k^{\prime} \in \nats \setminus \{1\}$ with $2k^{\prime} > k$, we have $a_{2k^{\prime}} = 0$, so that
the set $\{X_{i} : n_{2k^{\prime}-1} < i \leq n_{2k^{\prime}} \} \cap (0,1/2)$ consists of 
$(2k^{\prime}-1) n_{2k^{\prime}-1}^{2} = \frac{2k^{\prime}-1}{2k^{\prime}} (n_{2k^{\prime}} - n_{2k^{\prime}-1})$
independent ${\rm Uniform}(0,1/2)$ samples (also independent from $f_{n_{k}}$ and $X_{1:n_{k}}$).
Since $V$ is countable, with probability one every one of these samples has $\target(X_{i}) = y_{1}$.
Furthermore, a Chernoff bound (under the conditional distribution given $f_{n_{k}}$ and $X_{1:n_{k}}$) and the law of total probability imply that
\begin{equation*}
\left|N_{k} \cap \{X_{i} : n_{2k^{\prime}-1} < i \leq n_{2k^{\prime}} \} \cap (0,1/2) \right|
\geq \left(1 \!-\! \frac{1}{2k^{\prime} \!-\! 1}\right) \!\left( 1 \!-\! \frac{1}{n_{k-1}} \right) \! \frac{2k^{\prime} \!-\! 1}{2k^{\prime}} ( n_{2k^{\prime}} - n_{2k^{\prime}-1} )
\end{equation*}
with probability at least 
$1 - \exp\!\left\{ - \frac{1}{2 (2k^{\prime}-1)^{2}} \!\left( 1 - \frac{1}{n_{k-1}} \right)\!  (2k^{\prime}-1) n_{2k^{\prime}-1}^{2} \right\} > 1 - e^{-(2k^{\prime}-1)/4}$
(since $n_{2k^{\prime}-1} \geq 2k^{\prime}-1$ and $n_{k-1} > 2$).
Noting that $\sum\limits_{k^{\prime}=2}^{\infty} e^{ - (2k^{\prime}-1)  / 4 } < \infty$, 
the Borel-Cantelli lemma implies that with probability one this event occurs for all sufficiently large $k^{\prime}$.
Thus, by the union bound, we have that with probability one, 
\begin{align*}
& \hat{\mu}_{\ProcX}\!\left( \loss\!\left( f_{n_{k}}(X_{1:n_{k}},\target(X_{1:n_{k}}),\cdot), \target(\cdot) \right) \right)
\\ & \geq \limsup_{k^{\prime} \to \infty} \frac{1}{n_{2k^{\prime}}} \sum_{t=1}^{n_{2k^{\prime}}} \loss\!\left( f_{n_{k}}\left( X_{1:n_{k}}, \target(X_{1:n_{k}}), X_{t} \right), \target(X_{t}) \right)
\\ & \geq \limsup_{k^{\prime} \to \infty} \frac{\left|N_{k} \cap \{X_{i} : n_{2k^{\prime}-1} < i \leq n_{2k^{\prime}} \} \cap (0,1/2) \right|}{n_{2k^{\prime}}} \loss(y_{0},y_{1})
\\ & \geq \loss(y_{0},y_{1}) \limsup_{k^{\prime} \to \infty} \left( 1 \!-\! \frac{1}{2k^{\prime} \!-\! 1} \right) \! \left( 1 \!-\! \frac{1}{n_{k-1}} \right) \frac{2k^{\prime} \!-\! 1}{2k^{\prime}} \left(1 \!-\! \frac{n_{2k^{\prime}-1}}{n_{2k^{\prime}}}\right)
= \loss(y_{0},y_{1})  \left( 1 - \frac{1}{n_{k-1}} \right).
\end{align*}
By the union bound, with probability one, this holds for every odd value of $k \in \nats \setminus \{1,2\}$.
Thus, with probability one, 
\begin{align*}
\limsup_{n \to \infty} \hat{\L}_{\ProcX}(f_{n},\target;n)
& \geq \limsup_{k \to \infty} \hat{\mu}_{\ProcX}\left( \loss\left( f_{n_{2k+1}}(X_{1:n_{2k+1}},\target(X_{1:n_{2k+1}}),\cdot), \target(\cdot) \right) \right)
\\ & \geq \limsup_{k \to \infty} \loss(y_{0},y_{1}) \left( 1 - \frac{1}{n_{2k}} \right)
= \loss(y_{0},y_{1}).
\end{align*}
In particular,
this implies $f_{n}$ is not strongly universally consistent under $\ProcX$.
Similar arguments can be constructed for most learning methods in common use (e.g., kernel rules, the $k$-nearest neighbors rule, 
support vector machines with radial basis kernel).

It is clear from this example that obtaining consistency under general $\ProcX$ satisfying Condition~\ref{con:kc}
will require a new approach to the design of learning rules.  We develop such an approach in the sections below.
The essential innovation is to base the predictions not only on performance on points that seem typical relative to the present data set $X_{1:n}$, 
but also on the \emph{prefixes} $X_{1:n^{\prime}}$ of the data set (for a well-chosen range of values $n^{\prime} \leq n$).

\section{Condition~\ref{con:kc} is Necessary and Sufficient for Universal Inductive and Self-Adaptive Learning}
\label{sec:main}

This section presents the proof of Theorem~\ref{thm:main} from Section~\ref{subsec:main}, establishing equivalence
of the set of processes admitting strong universal inductive learning, the set of processes
admitting strong universal self-adaptive learning, and the set of processes satisfying Condition~\ref{con:kc}.
For convenience, we restate that result here (in simplified form) as follows.\\

\noindent {\bf Theorem~\ref{thm:main} (restated)}~ $\SUIL = \SUAL = \KC$.\\

The proof is by way of three lemmas: Lemma~\ref{lem:sual-subset-kc}, representing necessity of Condition~\ref{con:kc} for strong universal self-adaptive learning, 
Lemma~\ref{lem:kc-subset-suil}, representing sufficiency of Condition~\ref{con:kc} for strong universal inductive learning,
and Lemma~\ref{lem:suil-subset-sual}, which indicates that any process admitting strong universal inductive learning necessarily admits strong
universal self-adaptive learning.
We begin with the last (and simplest) of these.

\begin{lemma}
\label{lem:suil-subset-sual}
$\SUIL \subseteq \SUAL$.
\end{lemma}
\begin{proof}
Let $\ProcX \in \SUIL$, and let $f_{n}$ be an inductive learning rule such that, 
for every measurable $\target : \X \to \Y$, 
$\lim\limits_{n\to\infty} \hat{\L}_{\ProcX}(f_{n},\target;n) = 0$ (a.s.).
Then define a self-adaptive learning rule $g_{n,m}$ as follows.
For every $n,m \in \nats$, and every $\{x_{i}\}_{i=1}^{m} \in \X^{m}$, $\{y_{i}\}_{i=1}^{n} \in \Y^{n}$, and $z \in \X$, 
if $n \leq m$, define $g_{n,m}(x_{1:m},y_{1:n},z) = f_{n}(x_{1:n},y_{1:n},z)$.
With this definition, we have that for every measurable $\target : \X \to \Y$, for every $n \in \nats$, 
\begin{align*}
& \hat{\L}_{\ProcX}(g_{n,\cdot},\target;n) 
= \limsup\limits_{t\to\infty} \frac{1}{t+1} \sum_{m=n}^{n+t} \loss(g_{n,m}(X_{1:m},\target(X_{1:n}),X_{m+1}),\target(X_{m+1}))
\\ & = \limsup\limits_{t\to\infty} \frac{1}{t+1} \sum_{m=n}^{n+t} \loss(f_{n}(X_{1:n},\target(X_{1:n}),X_{m+1}),\target(X_{m+1}))
= \hat{\L}_{\ProcX}(f_{n},\target;n),
\end{align*}
so that 
$\lim\limits_{n\to\infty} \hat{\L}_{\ProcX}(g_{n,\cdot},\target;n)
= \lim\limits_{n\to\infty} \hat{\L}_{\ProcX}(f_{n},\target;n)
= 0 \text{ (a.s.)}$.
\end{proof}

Next, we prove necessity of Condition~\ref{con:kc} for strong universal self-adaptive learning.

\begin{lemma}
\label{lem:sual-subset-kc}
$\SUAL \subseteq \KC$.
\end{lemma}
\begin{proof}
We prove this result in the contrapositive.
Suppose $\ProcX \notin \KC$.  
By Corollary~\ref{cor:missing-mass}, there exists a disjoint sequence $\{A_i\}_{i=1}^{\infty}$ in $\Borel$ with $\bigcup\limits_{i=1}^{\infty} A_i = \X$ 
such that, with probability greater than $0$, $\lim\limits_{n\to\infty} \hat{\mu}_{\ProcX}\!\left( \bigcup\{ A_{i} : X_{1:n} \cap A_{i} = \emptyset \} \right) > 0$.
For any $n \in \nats$, let $\bar{\A}(X_{1:n}) = \bigcup\{A_i : X_{1:n} \cap A_i = \emptyset\}$.
Now take any two distinct values $y_0, y_1 \in \Y$, and construct a set of target functions $\{ \target_{\kappa} : \kappa \in [0,1) \}$ as follows. 
For any $\kappa \in [0,1)$ and $i \in \nats$, let $\kappa_i = \lfloor 2^{i} \kappa \rfloor - 2 \lfloor 2^{i-1} \kappa \rfloor$: the $i^{th}$ bit of the binary representation of $\kappa$.  
For each $i \in \nats$ and each $x \in A_i$, define $\target_{\kappa}(x) = y_{\kappa_i}$.
Note that $(x,\kappa) \mapsto \target_{\kappa}(x)$ is measurable in the product $\sigma$-algebra 
(under $\Borel$ for the $x$ argument, and the usual Borel $\sigma$-algebra on $[0,1)$ for the $\kappa$ argument), 
since the inverse image of $\{y_1\}$ is 
$\bigcup\limits_{i=1}^{\infty} \!\left( A_i \times \{ \kappa : \kappa_i = 1 \} \right)$ (a countable union of measurable rectangle sets) 
and the inverse image of $\{y_0\}$ is the complement of this set.

For any $t \in \nats$, let $i_t$ denote the value of $i \in \nats$ for which $X_t \in A_i$.
Now fix any self-adaptive learning rule $g_{n,m}$, and for brevity define a function $f_{n,m}^{\kappa} : \X \to \Y$ as $f_{n,m}^{\kappa}(\cdot) = g_{n,m}(X_{1:m},\target_{\kappa}(X_{1:n}),\cdot)$
(a composition of measurable functions, and therefore measurable).
Then we have 
\begin{align*}
&\sup_{\kappa \in [0,1)} \E\!\left[ \limsup_{n\to\infty} \hat{\L}_{\ProcX}(g_{n,\cdot},\target_{\kappa};n) \right]
\geq \int_0^1 \E\!\left[ \limsup_{n\rightarrow\infty} \hat{\L}_{\ProcX}(g_{n,\cdot},\target_{\kappa};n) \right] {\rm d}\kappa
\\ & \geq \int_0^1 \E\!\left[ \limsup_{n\rightarrow\infty} \limsup_{t\to\infty} \frac{1}{t+1} \sum_{m=n}^{n+t} \loss\!\left( f_{n,m}^{\kappa}(X_{m+1}), \target_{\kappa}(X_{m+1}) \right) \ind_{\bar{\A}(X_{1:n})}(X_{m+1}) \right] {\rm d}\kappa.
\end{align*}
By Fubini's theorem, this last expression is equal
\begin{equation*}
\E\!\left[ \int_{0}^{1} \limsup_{n\rightarrow\infty} \limsup_{t\to\infty} \frac{1}{t+1} \sum_{m=n}^{n+t} \loss\!\left( f_{n,m}^{\kappa}(X_{m+1}), \target_{\kappa}(X_{m+1}) \right) \ind_{\bar{\A}(X_{1:n})}(X_{m+1}) {\rm d}\kappa \right].
\end{equation*}
Since $\loss$ is bounded, Fatou's lemma implies this is at least as large as
\begin{equation*}
\E\!\left[ \limsup_{n\rightarrow\infty} \limsup_{t\to\infty} \int_{0}^{1} \frac{1}{t+1} \sum_{m=n}^{n+t} \loss\!\left( f_{n,m}^{\kappa}(X_{m+1}), \target_{\kappa}(X_{m+1}) \right) \ind_{\bar{\A}(X_{1:n})}(X_{m+1}) {\rm d}\kappa \right],
\end{equation*}
and linearity of integration implies this equals
\begin{equation}
\label{eqn:sual-subset-kc-pre-integration}
\E\!\left[ \limsup_{n\rightarrow\infty} \limsup_{t\to\infty} \frac{1}{t+1} \sum_{m=n}^{n+t} \ind_{\bar{\A}(X_{1:n})}(X_{m+1}) \int_{0}^{1} \loss\!\left( f_{n,m}^{\kappa}(X_{m+1}), \target_{\kappa}(X_{m+1}) \right) {\rm d}\kappa \right].
\end{equation}
Note that, for any $m$, the value of $f_{n,m}^{\kappa}(X_{m+1})$ is a function of $\ProcX$ and the values $\kappa_{i_{1}},\ldots,\kappa_{i_{n}}$.
Therefore, for any $m$ with $X_{m+1} \in \bar{\A}(X_{1:n})$, the value of $f_{n,m}^{\kappa}(X_{m+1})$ is functionally independent of $\kappa_{i_{m+1}}$.
Thus, letting $K \sim {\rm Uniform}( [0,1) )$ be independent of $\ProcX$ and $g_{n,m}$, 
for any such $m$ we have
\begin{align*} 
& \int_{0}^{1} \loss\!\left( f_{n,m}^{\kappa}(X_{m+1}), \target_{\kappa}(X_{m+1}) \right) {\rm d}\kappa
= \E\!\left[ \loss\!\left( f_{n,m}^{K}(X_{m+1}), \target_{K}(X_{m+1}) \right) \Big| \ProcX, g_{n,m} \right]
\\ & = \E\!\left[ \E\!\left[ \loss\!\left( g_{n,m}(X_{1:m},\{y_{K_{i_{j}}}\}_{j=1}^{n}, X_{m+1}), y_{K_{m+1}}\right) \Big| \ProcX, g_{n,m}, K_{i_{1}},\ldots,K_{i_{n}} \right] \Big| \ProcX, g_{n,m} \right]
\\ & = \E\!\left[\sum_{b\in\{0,1\}} \frac{1}{2}\loss\!\left(g_{n,m}(X_{1:m},\{y_{K_{i_{j}}}\}_{j=1}^{n}, X_{m+1}), y_{b} \right) \Big| \ProcX, g_{n,m} \right].
\end{align*}
By the relaxed triangle inequality, this last line is no smaller than $\E\!\left[\frac{1}{2\triconst}\loss(y_{0},y_{1}) \Big| \ProcX, g_{n,m} \right] = \frac{1}{2\triconst} \loss(y_{0},y_{1})$, 
so that \eqref{eqn:sual-subset-kc-pre-integration} is at least as large as
\begin{align*}
& \E\!\left[ \limsup_{n\rightarrow\infty} \limsup_{t\to\infty} \frac{1}{t+1} \sum_{m=n}^{n+t} \ind_{\bar{\A}(X_{1:n})}(X_{m+1}) \frac{1}{2\triconst} \loss(y_{0},y_{1}) \right]
\\ & = \frac{1}{2\triconst} \loss(y_{0},y_{1}) \E\!\left[ \lim_{n\rightarrow\infty} \hat{\mu}_{\ProcX}\!\left( \bigcup \left\{ A_{i} : X_{1:n} \cap A_{i} = \emptyset \right\} \right) \right].
\end{align*}
Since any nonnegative random variable with mean $0$ necessarily equals $0$ almost surely \citep*[e.g.,][Theorem 1.6.6]{ash:00},
and since $\lim\limits_{n\to\infty} \!\hat{\mu}_{\ProcX}\!\left( \bigcup\{ A_{i} : X_{1:n} \cap A_{i} = \emptyset \} \right) > 0$ with probability strictly greater than $0$, 
and the left hand side of this inequality is nonnegative, 
we have that $\E\!\left[ \lim\limits_{n\rightarrow\infty} \hat{\mu}_{\ProcX}\!\left( \bigcup \!\left\{ A_{i} : X_{1:n} \cap A_{i} = \emptyset \right\} \right) \right] > 0$.
Furthermore, since $\loss$ is a near-metric, we also have $\loss(y_{0},y_{1}) > 0$.
Altogether we have that
\begin{equation*}
\sup_{\kappa \in [0,1)} \E\!\left[ \limsup_{n\to\infty} \hat{\L}_{\ProcX}(g_{n,\cdot},\target_{\kappa};n) \right]
 \geq \frac{1}{2\triconst} \loss(y_{0},y_{1}) \E\!\left[ \limsup_{n\rightarrow\infty} \hat{\mu}_{\ProcX}\!\left( \bigcup \left\{ A_{i} : X_{1:n} \cap A_{i} = \emptyset \right\} \right) \right]
> 0.
\end{equation*}
In particular, this implies $\exists \kappa \in [0,1)$ such that 
$\E\!\left[ \limsup\limits_{n\to\infty} \hat{\L}_{\ProcX}(g_{n,\cdot},\target_{\kappa};n) \right] > 0$.
Since any random variable equal $0$ (a.s.) necessarily has expected value $0$, 
and since we have $\limsup\limits_{n\to\infty} \hat{\L}_{\ProcX}(g_{n,\cdot},\target_{\kappa};n) \geq 0$,
it must be that 
$\limsup\limits_{n\to\infty} \hat{\L}_{\ProcX}(g_{n,\cdot},\target_{\kappa};n) > 0$
with probability strictly greater than $0$, 
so that $g_{n,m}$ is not strongly universally consistent.
Since $g_{n,m}$ was an arbitrary self-adaptive learning rule, 
we conclude that there does not exist a self-adaptive learning rule that is
strongly universally consistent under $\ProcX$: that is, $\ProcX \notin \SUAL$.
Since this argument holds for any $\ProcX \notin \KC$, the lemma follows.
\end{proof}

Finally, to complete the proof of Theorem~\ref{thm:main}, we prove that Condition~\ref{con:kc} is sufficient
for $\ProcX$ to admit strong universal inductive learning.  We prove this via a more general strategy: namely,
a kind of constrained maximum empirical risk minimization.
Though the lemmas below are in fact somewhat stronger than needed to prove Theorem~\ref{thm:main},
some of them are useful later for establishing Theorem~\ref{thm:optimistic-self-adaptive}, and some 
should also be of independent interest.
We propose to study an inductive learning rule $\hat{f}_{n}$ such that,
for any $n \in \nats$, $x_{1:n} \in \X^{n}$, and $y_{1:n} \in \Y^{n}$, 
the function $\hat{f}_{n}(x_{1:n},y_{1:n},\cdot)$ is defined as
\begin{equation}
\label{eqn:suil-rule}
\argmin\limits_{f \in \F_{n}} \max\limits_{\hat{m}_{n} \leq m \leq n} \frac{1}{m} \sum_{t=1}^{m} \loss(f(x_{t}),y_{t}),
\end{equation}
where $\F_{n}$ is a well-chosen finite class of functions $\X \to \Y$, 
and $\hat{m}_{n}$ is a well-chosen integer.
Suppose ties in the $\argmin$ are broken based on a fixed preference ordering of $\F_{n}$.
In particular, this makes $\hat{f}_{n}$ a measurable function,
and hence \eqref{eqn:suil-rule} defines a valid inductive learning rule.
For our purposes, $\hat{f}_{0}(\{\},\{\},\cdot)$ can be defined as an arbitrary measurable function $\X \to \Y$.

The sequence of classes $\F_{n}$ and values $\hat{m}_{n}$, and the guarantees they provide, originate in the following several lemmas.
The general strategy is to define $\F_{n}$ so that, for large $n$, $\F_{n}$ is rich enough to contain a function 
$f$ with small $\hat{\mu}_{\ProcX}(\loss(f(\cdot),\target(\cdot)))$, 
while at the same time not too rich, so that (for an appropriate choice of $\hat{m}_{n}$) 
$\max\limits_{\hat{m}_{n} \leq m \leq n} \frac{1}{m} \sum\limits_{t=1}^{m} \loss(f(X_{t}),\target(X_{t}))$ 
is a reasonable estimate of $\hat{\mu}_{\ProcX}(\loss(f(\cdot),\target(\cdot)))$ for all $f \in \F_{n}$.
The fact that these two properties can exist simultaneously, and for all possible $\target$, 
is enabled by $\ProcX$ satisfying Condition~\ref{con:kc}.
We now proceed with the details.

\begin{lemma}
\label{lem:bracketing-convergence}
For any finite set $\G$ of bounded measurable functions $\X \to \reals$,
for any process $\ProcX$, there exists a (nonrandom) nondecreasing sequence 
$\{s_{n}\}_{n=1}^{\infty}$ in $\nats$ with $s_{n} \leq n$ and $s_{n} \to \infty$
such that
\begin{equation*}
\lim_{n\to\infty} \E\!\left[ \sup_{n^{\prime} \geq n} \max_{g \in \G} \left| \hat{\mu}_{\ProcX}( g ) ~- \!\max_{s_{n} \leq m \leq n^{\prime}} \frac{1}{m} \sum_{t=1}^{m} g(X_{t}) \right| \right] = 0.
\end{equation*}
\end{lemma}
\begin{proof}
The proof of this lemma proceeds in three steps.
First, we note that the result would follow almost immediately from the definition of $\limsup$, 
if the sequence $\ProcX$ were \emph{deterministic} and we were merely interested in the case of a \emph{single} function $g$.
Second, we extend this observation to any finite \emph{set} $\G$ of functions by taking the $s_{n}$ 
sequence as the minimum of the corresponding $s_{n}$ values for each individual $g$.
The final component is to extend the result to hold for non-deterministic processes $\ProcX$, 
by replacing the $s_{n}$ sequence corresponding to each sample path of $\ProcX$ with an appropriate 
confidence bound on its value.  This last step requires us to introduce some notation in the 
first two steps to explicitly define these $s_{n}$ sequences for the sample paths, 
so that together they are a measurable function of $\ProcX$, and hence confidence bounds are well-define.
We now turn to the details of the proof.

Fix any sequence $\mathbf{x} = \{x_{t}\}_{t=1}^{\infty}$ in $\X$
and any bounded function $g : \X \to \reals$.
By definition, 
\begin{equation*}
\hat{\mu}_{\mathbf{x}}(g) = \lim_{s \to \infty} \lim_{n \to \infty} \max_{s \leq m \leq n} \frac{1}{m} \sum_{t=1}^{m} g(x_{t}).
\end{equation*}
In particular, for each $s \in \nats$, 
since $\max\limits_{s \leq m \leq n} \frac{1}{m} \sum\limits_{t=1}^{m} g(x_{t})$ is nondecreasing in $n \geq s$, and $g$ is bounded,
$\lim\limits_{n \to \infty} \max\limits_{s \leq m \leq n} \frac{1}{m} \sum\limits_{t=1}^{m} g(x_{t})$ exists and is finite.
This implies that, for each $s \in \nats$, $\exists n_{s}^{g}(\mathbf{x}) \in \nats$ s.t. 
$n_{s}^{g}(\mathbf{x}) \geq s$ and every $n^{\prime} \geq n_{s}^{g}(\mathbf{x})$ has
\begin{equation}
\label{eqn:bracketing-convergence-1}
\max_{s \leq m \leq n^{\prime}} \frac{1}{m} \sum_{t=1}^{m} g(x_{t}) 
\leq \sup_{s \leq m < \infty} \frac{1}{m} \sum_{t=1}^{m} g(x_{t}) 
\leq 2^{-s} + \max_{s \leq m \leq n^{\prime}} \frac{1}{m} \sum_{t=1}^{m} g(x_{t}).
\end{equation}
In particular, let us define $n_{s}^{g}(\mathbf{x})$ to be the minimal value in $\nats$ with this property.
We first argue that $n_{s}^{g}(\mathbf{x})$ is nondecreasing in $s$.
To see this, first note that the left inequality in \eqref{eqn:bracketing-convergence-1} is trivially satisfied for every $s,n^{\prime} \in \nats$ with $n^{\prime} \geq s$.
Moreover, for any $n^{\prime},s \in \nats$ with $s \geq 2$ and $n^{\prime} \geq n_{s}^{g}(\mathbf{x})$, 
either $\sup\limits_{s-1 \leq m < \infty} \frac{1}{m} \sum\limits_{t=1}^{m} g(x_{t}) = \frac{1}{s-1} \sum\limits_{t=1}^{s-1} g(x_{t})$, 
in which case it is clearly less than $2^{-(s-1)} + \max\limits_{s-1 \leq m \leq n^{\prime}} \frac{1}{m} \sum\limits_{t=1}^{m} g(x_{t})$,
or else $\sup\limits_{s-1 \leq m < \infty} \frac{1}{m} \sum\limits_{t=1}^{m} g(x_{t}) = \sup\limits_{s \leq m < \infty} \frac{1}{m} \sum\limits_{t=1}^{m} g(x_{t})$,
in which case (since $n^{\prime} \geq n_{s}^{g}(\mathbf{x})$) it is at most 
$2^{-s} + \max\limits_{s \leq m \leq n^{\prime}} \frac{1}{m} \sum\limits_{t=1}^{m} g(x_{t}) \leq 2^{-(s-1)} +  \max\limits_{s-1 \leq m \leq n^{\prime}} \frac{1}{m} \sum\limits_{t=1}^{m} g(x_{t})$.
Furthermore, we have $n_{s}^{g}(\mathbf{x}) \geq s \geq s-1$.
Altogether, we have $n_{s-1}^{g}(\mathbf{x}) \leq n_{s}^{g}(\mathbf{x})$, so that $n_{s}^{g}(\mathbf{x})$ is indeed nondecreasing in $s$.

For each $n \in \nats$ with $n \geq n_{1}^{g}(\mathbf{x})$, let $s_{n}^{g}(\mathbf{x}) = \max\{ s \in \{1,\ldots,n\} : n \geq n_{s}^{g}(\mathbf{x}) \}$;
for completeness, let $s_{n}^{g}(\mathbf{x}) = 0$ for $n < n_{1}^{g}(\mathbf{x})$.
Then, for any finite set $\G$ of bounded functions $\X \to \reals$, 
define $s_{n}^{\G}(\mathbf{x}) = \min\limits_{g \in \G} s_{n}^{g}(\mathbf{x}) = \max\!\left(\left\{ s \in \{1,\ldots,n\} : n \geq \max\limits_{g \in \G} n_{s}^{g}(\mathbf{x}) \right\} \cup \{0\}\right)$.
Since $n_{s}^{g}(\mathbf{x})$ is nondecreasing in $s$, 
we have for any $n,n^{\prime} \in \nats$ with $n^{\prime} \geq n$, for $1 \leq s \leq s_{n}^{\G}(\mathbf{x})$, 
every $g \in \G$ has $n^{\prime} \geq n_{s}^{g}(\mathbf{x})$,
so that \eqref{eqn:bracketing-convergence-1} is satisfied for every $g \in \G$,
which implies
\begin{equation*}
\max_{g \in \G} \left| \sup_{s \leq m < \infty} \frac{1}{m} \sum_{t=1}^{m} g(x_{t}) - \max_{s \leq m \leq n^{\prime}} \frac{1}{m} \sum_{t=1}^{m} g(x_{t}) \right| \leq 2^{-s}.
\end{equation*}
Therefore, for any sequence $s_{n} \to \infty$ with $s_{n} \leq n$ 
such that $\exists n_{0} \in \nats$ with $1 \leq s_{n} \leq s_{n}^{\G}(\mathbf{x})$ for all $n \geq n_{0}$, 
we have
\begin{equation*}
\lim_{n \to \infty} \sup_{n^{\prime} \geq n} \max_{g \in \G} \left| \sup_{s_{n} \leq m < \infty} \frac{1}{m} \sum_{t=1}^{m} g(x_{t}) ~- \max_{s_{n} \leq m \leq n^{\prime}} \frac{1}{m} \sum_{t=1}^{m} g(x_{t}) \right|
\leq \lim_{n\to\infty} 2^{-s_{n}} = 0.
\end{equation*}
Furthermore, for any such sequence $s_{n}$, for every $g \in \G$, by definition 
\begin{equation*}
\lim_{n \to \infty} \sup_{s_{n} \leq m < \infty} \frac{1}{m} \sum_{t=1}^{m} g(x_{t}) = \hat{\mu}_{\mathbf{x}}(g),
\end{equation*}
and since $\G$ has finite cardinality, this implies
\begin{equation*}
\lim_{n\to\infty} \max_{g \in \G} \left| \hat{\mu}_{\mathbf{x}}(g) ~- \!\sup_{s_{n} \leq m < \infty} \frac{1}{m} \sum_{t=1}^{m} g(x_{t}) \right|
= \max_{g \in \G} \lim_{n\to\infty} \left| \hat{\mu}_{\mathbf{x}}(g) ~- \!\sup_{s_{n} \leq m < \infty} \frac{1}{m} \sum_{t=1}^{m} g(x_{t}) \right|
= 0.
\end{equation*}
Altogether, the triangle inequality implies
\begin{align}
& \lim_{n \to \infty} \sup_{n^{\prime} \geq n} \max_{g \in \G} \left| \hat{\mu}_{\mathbf{x}}(g) ~- \!\max_{s_{n} \leq m \leq n^{\prime}} \frac{1}{m} \sum_{t=1}^{m} g(x_{t}) \right| 
\notag
\\ & \leq 
\lim_{n \to \infty} \sup_{n^{\prime} \geq n} \max_{g \in \G} \left| \sup_{s_{n} \leq m < \infty} \frac{1}{m} \sum_{t=1}^{m} g(x_{t}) ~- \max_{s_{n} \leq m \leq n^{\prime}} \frac{1}{m} \sum_{t=1}^{m} g(x_{t}) \right| 
\notag
\\ & {\hskip 32mm}+ \lim_{n\to\infty} \max_{g \in \G} \left| \hat{\mu}_{\mathbf{x}}(g) ~- \sup_{s_{n} \leq m < \infty} \frac{1}{m} \sum_{t=1}^{m} g(x_{t}) \right|
= 0. \label{eqn:bracketing-convergence-2}
\end{align}

Next, suppose the bounded functions in the set $\G$ are measurable.
Note that this implies that, for any $g \in \G$, the set of sequences $\mathbf{x}$ satisfying \eqref{eqn:bracketing-convergence-1}
for a given $s,n \in \nats$ is a measurable subset of $\X^{\infty}$,
so that for each $s,n^{\prime} \in \nats$ the set of sequences $\mathbf{x}$ with $n_{s}^{g}(\mathbf{x}) = n^{\prime}$
is also a measurable set, 
so that $n_{s}^{g}$ is a measurable function.  Since the value of $s_{n}^{g}$ is obtained
from the values $n_{s}^{g}$ via operations that preserve measurability, we also have that $s_{n}^{g}$ is a measurable
function.  Since the minimum of a finite set of measurable functions is also measurable, we also have that
$s_{n}^{\G}$ is a measurable function.

Now fix any process $\ProcX$.
At this point, it may be tempting to use $s_{n}^{\G}(\ProcX)$ to complete the proof.
However, recall that the lemma requires a \emph{nonrandom} sequence $s_{n}$, 
whereas $s_{n}^{\G}(\ProcX)$ is a random variable.
To address this, we will replace $s_{n}^{\G}(\ProcX)$ with an appropriate sequence of \emph{confidence bounds} on its value, as follows.
For any $n \in \nats$ and $\delta \in (0,1]$ 
define  
\begin{equation*}
s_{n}^{\G}(\delta) = \max\left\{ s \in \{0,1,\ldots,n\} : \P( s_{n}^{\G}(\ProcX) \geq s ) \geq 1-\delta \right\}.
\end{equation*}
Since $s_{n}^{\G}(\mathbf{x})$ is nondecreasing for each sequence $\mathbf{x}$, we must also have that $s_{n}^{\G}(\delta)$ is nondecreasing in $n$.
Furthermore, since each $s \in \nats$ and $g \in \G$ have $n_{s}^{g}(\mathbf{x}) < \infty$, 
and $\G$ is a finite set, 
we have $s_{n}^{\G}(\mathbf{x}) \to \infty$ for any sequence $\mathbf{x}$;
thus, by continuity of probability measures \citep*[e.g.,][Theorem A.19]{schervish:95}, 
$\forall s \in \nats$, $\lim\limits_{n \to \infty} \P( s_{n}^{\G}(\ProcX) < s ) = 0$.  
We therefore have $s_{n}^{\G}(\delta) \to \infty$ for any $\delta \in (0,1]$.
In particular, letting
\begin{equation*}
s_{n} = \max \left\{ s \in \nats \cup \{0\} : s_{n}^{\G}(2^{-s}) \geq s \right\}
\end{equation*}
for each $n \in \nats$,
we have that $s_{n}$ is nondecreasing, and $s_{n} \to \infty$.  
Furthermore, by definition, we have $\P( s_{n}^{\G}(\ProcX) \geq s_{n} ) \geq 1 - 2^{-s_{n}}$, 
and since any $\delta \in (0,1]$ has $s_{n}^{\G}(\delta) \leq n$, 
the definition of $s_{n}$ also implies $s_{n} \leq n$.
Let $n_{1} = 1$, and let $n_{2},n_{3},\ldots$ denote the increasing subsequence of all values $n \in \nats \setminus \{1\}$ for which $s_{n} > s_{n-1}$;
since $s_{n} \to \infty$ while each $n$ has $s_{n} \leq n < \infty$, there are indeed an infinite number of such $n_{k}$ values.
Note that, since $s_{n}$ is nondecreasing, and hence these $s_{n_{k}}$ are each distinct values in $\nats \cup \{0\}$, we have
\begin{equation*}
\sum_{k=1}^{\infty} \P( s_{n_{k}}^{\G}(\ProcX) < s_{n_{k}} )
\leq \sum_{k=1}^{\infty} 2^{-s_{n_{k}}}
\leq \sum_{i=0}^{\infty} 2^{-i} = 2 < \infty.
\end{equation*}
Therefore, the Borel-Cantelli Lemma implies that, with probability one, 
for all sufficiently large $k$, we have $s_{n_{k}}^{\G}(\ProcX) \geq s_{n_{k}}$.
Furthermore, since $s_{n}^{\G}(\ProcX)$ is nondecreasing in $n$, and $s_{n} = s_{n_{k}}$ for all $n \in \{n_{k},\ldots,n_{k+1}-1\}$ (due to $s_{n}$ nondecreasing),
if $s_{n_{k}}^{\G}(\ProcX) \geq s_{n_{k}}$ for a given $k \in \nats$, then $s_{n}^{\G}(\ProcX) \geq s_{n}$ for every $n \in \{n_{k},\ldots,n_{k+1}-1\}$.
Combining this again with the fact that $s_{n} \to \infty$, 
we may conclude that, with probability one, for all sufficiently large $n \in \nats$, 
we have $1 \leq s_{n} \leq s_{n}^{\G}(\ProcX)$. 
Thus, $s_{n}$ almost surely satisfies the requirements for \eqref{eqn:bracketing-convergence-2} to hold for $\mathbf{x} = \ProcX$,
which therefore implies
\begin{equation}
\label{eqn:bracketing-convergence-3}
\lim_{n \to \infty} \sup_{n^{\prime} \geq n} \max_{g \in \G} \left| \hat{\mu}_{\ProcX}(g) ~- \!\max_{s_{n} \leq m \leq n^{\prime}} \frac{1}{m} \sum_{t=1}^{m} g(X_{t}) \right| = 0 \text{ (a.s.)}.
\end{equation}
Finally, since the functions in $\G$ are bounded and $\G$ has finite cardinality,
\begin{equation*}
\left\{ \sup_{n^{\prime} \geq n} \max_{g \in \G} \left| \hat{\mu}_{\ProcX}(g) ~- \!\max_{s_{n} \leq m \leq n^{\prime}} \frac{1}{m} \sum_{t=1}^{m} g(X_{t}) \right| \right\}_{n=1}^{\infty}
\end{equation*}
is a uniformly bounded sequence of random variables,
so that combining \eqref{eqn:bracketing-convergence-3} with the dominated convergence theorem implies 
\begin{equation*}
\lim_{n \to \infty} \E\!\left[ \sup_{n^{\prime} \geq n} \max_{g \in \G} \left| \hat{\mu}_{\ProcX}(g) ~- \!\max_{s_{n} \leq m \leq n^{\prime}} \frac{1}{m} \sum_{t=1}^{m} g(X_{t}) \right| \right] = 0.
\end{equation*}
\end{proof}

\begin{lemma}
\label{lem:srm}
Suppose $\{\G_{i}\}_{i=1}^{\infty}$ is a sequence of nonempty finite sets of bounded measurable functions $\X \to \reals$,
with $\G_{1} \subseteq \G_{2} \subseteq \cdots$,
and $\{\gamma_{i}\}_{i=1}^{\infty}$ is a sequence in $(0,\infty)$ with
$\gamma_{1} \geq \max\limits_{g \in \G_{1}} \left( \sup\limits_{x \in \X} g(x) - \inf\limits_{x \in \X} g(x) \right)$.
Then for any process $\ProcX$, there exist (nonrandom) nondecreasing sequences $\{m_{i}\}_{i=1}^{\infty}$ and $\{i_{n}\}_{n=1}^{\infty}$ in $\nats$ with $m_{i} \to \infty$ and $i_{n} \to \infty$
such that $\forall n \in \nats$, $m_{i_{n}} \leq n$ and 
\begin{equation*}
\E\!\left[ \max_{g \in \G_{i_{n}}} \left| \hat{\mu}_{\ProcX}( g ) ~- \!\max_{m_{i_{n}} \leq m \leq n} \frac{1}{m} \sum_{t=1}^{m} g(X_{t}) \right|\right] \leq \gamma_{i_{n}}.
\end{equation*}
\end{lemma}
\begin{proof}
For each $i \in \nats$, let $\{m_{i,n}\}_{n=1}^{\infty}$ denote a nondecreasing sequence in $\nats$ with $\lim\limits_{n \to \infty} m_{i,n} = \infty$, $m_{i,n} \leq n$, and
\begin{equation}
\label{eqn:srm-1}
\lim_{n \to \infty} \E\!\left[ \sup_{n^{\prime} \geq n} \max_{g \in \G_{i}} \left| \hat{\mu}_{\ProcX}(g) ~- \!\max_{m_{i,n} \leq m \leq n^{\prime}} \frac{1}{m} \sum_{t=1}^{m} g(X_{t}) \right| \right] = 0.
\end{equation}
Such a sequence is guaranteed to exist by Lemma~\ref{lem:bracketing-convergence}.
From here, it would be straightforward to produce a sequence $i_{n}$ satisfying 
$\E\!\left[ \max\limits_{g \in \G_{i_{n}}} \left| \hat{\mu}_{\ProcX}( g ) ~- \!\max\limits_{m_{i_{n},n} \leq m \leq n} \frac{1}{m} \sum\limits_{t=1}^{m} g(X_{t}) \right|\right] \leq \gamma_{i_{n}}$, 
simply letting $i_{n}$ grow sufficiently slowly.  
However, comparing this to the claim in the lemma, we require slightly more than this: namely, replacing $m_{i_{n},n}$ with a single-index sequence $m_{i_{n}}$.
The existence of such a sequence $m_{i_{n}}$ is enabled by the additional supremum in the expression on the left of \eqref{eqn:srm-1}.
The basic idea is that this allows us to define a sequence $n_{i}$ such that 
$\E\!\left[ \max\limits_{g \in \G_{i}} \left| \hat{\mu}_{\ProcX}( g ) ~- \!\max\limits_{m_{i,n_{i}} \leq m \leq n} \frac{1}{m} \sum\limits_{t=1}^{m} g(X_{t}) \right|\right] \leq \gamma_{i}$ 
for any $n \geq n_{i}$.  We may then conclude by defining $m_{i} = m_{i,n_{i}}$, and $i_{n}$ maximal such that $n \geq n_{i_{n}}$.
The formal argument becomes somewhat more technical in order to verify such sequences are well-defined and to satisfy the monotonicity requirements of $m_{i}$ and $i_{n}$ from the lemma.
The details follow.

Formally, for each $n \in \nats$, define 
\begin{equation*}
j_{n} \!=\! \max\!\left\{ i \!\in \!\{1,\ldots,n\} : \forall i^{\prime} \!\leq\! i, \sup_{n^{\prime\prime} \geq n} \!\E\!\left[ \sup_{n^{\prime} \geq n^{\prime\prime}} \max_{g \in \G_{i^{\prime}}} \left| \hat{\mu}_{\ProcX}(g) ~-\!\!\!\! \max_{m_{i^{\prime},n^{\prime\prime}} \leq m \leq n^{\prime}} \frac{1}{m} \!\sum_{t=1}^{m} g(X_{t}) \right| \right] \!\!\leq\! \gamma_{i^{\prime}} \right\}\!.
\end{equation*}
First note that the set on the right hand side is nonempty, since every $n^{\prime\prime} \in \nats$ has 
\begin{equation*}
\E\!\left[ \sup_{n^{\prime} \geq n^{\prime\prime}} \max_{g \in \G_{1}} \left| \hat{\mu}_{\ProcX}(g) ~- \!\max_{m_{1,n^{\prime\prime}} \leq m \leq n^{\prime}} \frac{1}{m} \sum_{t=1}^{m} g(X_{t}) \right| \right] \leq \max_{g \in \G_{1}} \left( \sup_{x \in \X} g(x) - \inf_{x \in \X} g(x) \right) \leq \gamma_{1}.
\end{equation*}
Thus, $j_{n}$ is well-defined for every $n \in \nats$.
In particular, by this definition, we have $\forall n \in \nats$, $\forall i \in \{1,\ldots,j_{n}\}$, 
\begin{equation}
\label{eqn:srm-2}
\E\!\left[ \sup_{n^{\prime} \geq n} \max_{g \in \G_{i}} \left| \hat{\mu}_{\ProcX}(g) ~- \!\max_{m_{i,n} \leq m \leq n^{\prime}} \frac{1}{m} \sum_{t=1}^{m} g(X_{t}) \right| \right] \leq \gamma_{i}.
\end{equation}
Furthermore, since 
\begin{equation*}
\sup_{n^{\prime\prime} \geq n} \E\!\left[ \sup_{n^{\prime} \geq n^{\prime\prime}} \max_{g \in \G_{i}} \left| \hat{\mu}_{\ProcX}(g) ~- \!\max_{m_{i,n^{\prime\prime}} \leq m \leq n^{\prime}} \frac{1}{m} \sum_{t=1}^{m} g(X_{t}) \right| \right]
\end{equation*}
is nonincreasing in $n$ for every $i \in \nats$, we have that $j_{n}$ is nondecreasing.
Also note that, for any $i \in \nats$, since $\gamma_{i} > 0$, \eqref{eqn:srm-1} implies that $\exists n_{i}^{\prime} \in \nats$ such that, $\forall n \geq n_{i}^{\prime}$, 
\begin{equation*}
\sup_{n^{\prime\prime} \geq n} \E\!\left[ \sup_{n^{\prime} \geq n^{\prime\prime}} \max_{g \in \G_{i}} \left| \hat{\mu}_{\ProcX}(g) ~- \!\max_{m_{i,n^{\prime\prime}} \leq m \leq n^{\prime}} \frac{1}{m} \sum_{t=1}^{m} g(X_{t}) \right| \right] \leq \gamma_{i}.
\end{equation*}
Therefore $j_{n} \geq i$ for every $n \geq \max\!\left\{ i, \max\limits_{1 \leq i^{\prime} \leq i} n_{i^{\prime}}^{\prime} \right\}$.
Since this is true of every $i \in \nats$, we have that $j_{n} \to \infty$.

Next, let $n_{1} = 1$, and for each $i \in \nats \setminus \{1\}$, inductively define
\begin{equation*}
n_{i} = \min\left\{ n \in \nats : j_{n} \geq i, m_{i,n} > m_{i-1,n_{i-1}} \right\}.
\end{equation*}
Note that, given the value $n_{i-1} \in \nats$, 
the value $n_{i}$ is well-defined since $\lim\limits_{n \to \infty} j_{n} = \infty$ and $\lim\limits_{n\to\infty} m_{i,n} = \infty$.
Thus, by induction, $n_{i}$ is a well-defined value in $\nats$ for all $i \in \nats$.
For each $i \in \nats$, define $m_{i} = m_{i,n_{i}}$.
In particular, by definition of $n_{i}$, for all $i \in \nats$ we have 
$m_{i+1} = m_{i+1,n_{i+1}} > m_{i,n_{i}} = m_{i}$, so that $m_{i}$ is strictly increasing, with $m_{i} \to \infty$.
Finally, for each $n \in \nats$, define $i_{n} = \max\left\{ i \in \{1,\ldots,n\} : n_{i} \leq n \right\}$.
Since $n_{1} = 1$,
$i_{n}$ is a well-defined value in $\nats$ for all $n \in \nats$.
Also, any $i \in \{1,\ldots,n\}$ with $n_{i} \leq n$ also has $n_{i} \leq n+1$, so that $i_{n}$ is nondecreasing in $n$.
Furthermore, since $n_{i} < \infty$ for every $i \in \nats$, we have $i_{n} \to \infty$.
Also note that, $\forall n \in \nats$, we have $n \geq n_{i_{n}}$, which also implies $m_{i_{n}} \leq n_{i_{n}} \leq n$ (by the assumed property that $m_{i,n} \leq n$ for any $n$).
Thus, for every $n \in \nats$, 
\begin{equation*}
\E\!\left[ \max_{g \in \G_{i_{n}}} \!\left| \hat{\mu}_{\ProcX}(g) - \!\!\!\max_{m_{i_{n}} \leq m \leq n} \frac{1}{m} \!\sum_{t=1}^{m} g(X_{t}) \right| \right]
\!\leq\! \E\!\left[ \sup_{n^{\prime} \geq n_{i_{n}}} \max_{g \in \G_{i_{n}}} \!\left| \hat{\mu}_{\ProcX}(g) - \!\!\!\max_{m_{i_{n},n_{i_{n}}} \leq m \leq n^{\prime}} \frac{1}{m} \!\sum_{t=1}^{m} g(X_{t}) \right| \right]\!.
\end{equation*}
By definition of $n_{i_{n}}$, we have $j_{n_{i_{n}}} \geq i_{n}$ (this is immediate from the $n_{i}$ definition if $i_{n} \geq 2$, 
and is also trivially true for $i_{n} = 1$ since $j_{1} \geq 1$),
so that \eqref{eqn:srm-2} implies the rightmost expression above is at most $\gamma_{i_{n}}$,
which completes the proof.
\end{proof}

The following lemma represents the first 
use of Condition~\ref{con:kc} in the proof of 
sufficiency of Condition~\ref{con:kc} for strong universal inductive learning.
Indeed, in the special case of binary classification, 
this is actually the \emph{only} use of Condition~\ref{con:kc} 
needed for the proof.  For the case of general $(\Y,\loss)$, 
one additional use of Condition~\ref{con:kc} (in Lemma~\ref{lem:approximating-functions} below) 
will be needed, to extend this lemma from set approximation to function approximation.

\begin{lemma}
\label{lem:approximating-sets}
There exists a countable set $\T_{1} \subseteq \Borel$ such that, 
$\forall \ProcX \in \KC$, $\forall A \in \Borel$, 
\begin{equation*}
\inf\limits_{G \in \T_{1}} \E\!\left[ \hat{\mu}_{\ProcX}( G \bigtriangleup A ) \right] = 0.
\end{equation*}
\end{lemma}
\begin{proof} 
By assumption, $\Borel$ is generated by a separable metrizable topology $\T$,
and since every separable metrizable topological space is second countable \citep*[see][Proposition 2.1.9]{srivastava:98},
we have that there exists a \emph{countable} set $\T_{0} \subseteq \T$ such that,
$\forall A \in \T$, $\exists \A \subseteq \T_{0}$ s.t. $A = \bigcup \A$.
Now from this, there is an immediate proof of the lemma if we were to take $\T_{1}$ 
as the algebra generated by $\T_{0}$ (which is a countable set)
via the monotone class theorem \citep*[][Theorem 1.3.9]{ash:00},
using Condition~\ref{con:kc} to argue that the sets $A$ satisfying the claim
in the lemma form a monotone class.
However, here we will instead establish the lemma with a \emph{smaller} choice of the set $\T_{1}$, 
which thereby simplifies the problem of implementing the resulting learning rule.  
Specifically, we take 
$\T_{1} = \{ \bigcup \A : \A \subseteq \T_{0}, |\A| < \infty \}$:
all finite unions of sets in $\T_{0}$.
Note that, given an indexing of $\T_{0}$ by $\nats$, 
each $A \in \T_{1}$ can be indexed by a finite subset of $\nats$ (the indices of elements of the corresponding $\A$), 
of which there are countably many, so that $\T_{1}$ is countable.
Now fix any $\ProcX \in \KC$ and let 
\begin{equation*}
\Lambda = \left\{ A \in \Borel : \inf\limits_{G \in \T_{1}} \E\!\left[ \hat{\mu}_{\ProcX}( G \bigtriangleup A ) \right] = 0 \right\}.
\end{equation*}
We will prove 
that $\Lambda = \Borel$ 
by establishing that $\T \subseteq \Lambda$ and that $\Lambda$ is a $\sigma$-algebra.

First consider any $A \in \T$.
As mentioned above, $\exists \{B_{i}\}_{i=1}^{\infty}$ in $\T_{0}$ such that $A = \bigcup\limits_{i=1}^{\infty} B_{i}$.
But then letting $A_{k} = \bigcup\limits_{i=1}^{k} B_{i}$ for each $k \in \nats$, we have $A_{k} \bigtriangleup A = A \setminus A_{k} \downarrow \emptyset$,
and $A_{k} \in \T_{1}$ for each $k \in \nats$.
Therefore, 
$\inf\limits_{G \in \T_{1}} \E\left[ \hat{\mu}_{\ProcX}( G \bigtriangleup A ) \right]
\leq \lim\limits_{k \to \infty} \E\left[ \hat{\mu}_{\ProcX}( A_{k} \bigtriangleup A ) \right]$,
and the right hand side equals $0$ by Condition~\ref{con:kc}.
Together with nonnegativity of $\inf\limits_{G \in \T_{1}} \E\left[ \hat{\mu}_{\ProcX}( G \bigtriangleup A ) \right]$ (Lemma~\ref{lem:mu}),
this implies $A \in \Lambda$.  Since this holds for any $A \in \T$, we have $\T \subseteq \Lambda$.

Next, we argue that $\Lambda$ is a $\sigma$-algebra.
We begin by showing it is closed under complements.
Toward this end, consider any $A \in \Lambda$,
and for any $k \in \nats$ 
denote by $G_{k}$ an element of $\T_{1}$ 
with $\E\left[ \hat{\mu}_{\ProcX}( G_{k} \bigtriangleup A ) \right] < 1/k$ (guaranteed to exist by the definition of $\Lambda$).
Since $G_{k} \in \T_{1} \subseteq \T$, it follows that $\X \setminus G_{k}$ is a closed set.
Therefore, since $(\X,\T)$ is metrizable, $\exists \{B_{k i}\}_{i=1}^{\infty}$ in $\T$ such that 
$\X \setminus G_{k} = \bigcap\limits_{i=1}^{\infty} B_{k i}$ \citep*[][Proposition 3.7]{kechris:95}.
Letting $C_{k j} = \bigcap\limits_{i=1}^{j} B_{k i}$ for each $j \in \nats$, 
we have that $C_{k j} \bigtriangleup (\X \setminus G_{k}) = C_{k j} \setminus (\X \setminus G_{k}) \downarrow \emptyset$ as $j \to \infty$,
and $C_{k j} \in \T$ for each $j \in \nats$.
In particular, by
Condition~\ref{con:kc}, 
$\exists j_{k} \in \nats$ such that
$\E\!\left[ \hat{\mu}_{\ProcX}( C_{k j_{k}} \bigtriangleup (\X \setminus G_{k}) ) \right] < 1/k$.
Also, since $C_{k j_{k}} \in \T$, and we proved above that $\T \subseteq \Lambda$, 
$\exists D_{k} \in \T_{1}$ such that
$\E\left[ \hat{\mu}_{\ProcX}( D_{k} \bigtriangleup C_{k j_{k}} ) \right] < 1/k$.
Together with the facts that 
$D_{k} \bigtriangleup (\X \setminus A) 
\subseteq (D_{k} \bigtriangleup C_{k j_{k}}) \cup ( C_{k j_{k}} \bigtriangleup (\X \setminus G_{k}) ) \cup ( (\X \setminus G_{k}) \bigtriangleup (\X \setminus A) )$
and $(\X \setminus G_{k}) \bigtriangleup (\X \setminus A) = G_{k} \bigtriangleup A$,
by subadditivity of $\E[ \hat{\mu}_{\ProcX}(\cdot) ]$ (Lemma~\ref{lem:Ehatmu-submeasure}), 
we have that 
\begin{equation*}
\E\!\left[ \hat{\mu}_{\ProcX}( D_{k} \bigtriangleup (\X \!\setminus\! A) ) \right]
\leq 
\E\!\left[ \hat{\mu}_{\ProcX}(D_{k} \bigtriangleup C_{k j_{k}}) \right]
+ \E\!\left[ \hat{\mu}_{\ProcX}( C_{k j_{k}} \bigtriangleup (\X \!\setminus\! G_{k}) ) \right]
+ \E\!\left[ \hat{\mu}_{\ProcX}( G_{k} \bigtriangleup A ) \right]
< 3/k.
\end{equation*}
Since $D_{k} \in \T_{1}$, and this argument holds for any $k \in \nats$, we have
\begin{equation*}
\inf_{G \in \T_{1}} \E\!\left[ \hat{\mu}_{\ProcX}( G \bigtriangleup (\X \setminus A) ) \right]
\leq \inf_{k \in \nats} 3/k
= 0.
\end{equation*}
Together with nonnegativity of the left hand side (Lemma~\ref{lem:mu}),
this implies $\X \setminus A \in \Lambda$.
Thus, $\Lambda$ is closed under complements.

Next, we argue that $\Lambda$ is closed under countable unions.
Let $\{A_{i}\}_{i=1}^{\infty}$ be a  
sequence in $\Lambda$, let $A = \bigcup\limits_{i=1}^{\infty} A_{i}$,
and fix any $\eps > 0$.
Letting $B_{k} = \bigcup\limits_{i=1}^{k} A_{i}$ for each $k \in \nats$,
we have $B_{k} \bigtriangleup A = A \setminus B_{k} \downarrow \emptyset$.
Therefore, Condition~\ref{con:kc} 
implies $\exists k_{\eps} \in \nats$ such that 
$\E\!\left[ \hat{\mu}_{\ProcX}( B_{k_{\eps}} \bigtriangleup A ) \right] < \eps$.
Next, for each $i \in \nats$, let $G_{i}$ be an element of $\T_{1}$ with 
$\E\!\left[ \hat{\mu}_{\ProcX}( G_{i} \bigtriangleup A_{i} ) \right] < \eps / k_{\eps}$
(guaranteed to exist, since $A_{i} \in \Lambda$).
Let $C_{k_{\eps}} = \bigcup\limits_{i=1}^{k_{\eps}} G_{i}$.
Noting that it follows immediately from its definition 
that $\T_{1}$ is closed under finite unions, we have that $C_{k_{\eps}} \in \T_{1}$.
Then noting that 
\begin{equation*}
C_{k_{\eps}} \bigtriangleup A 
\subseteq ( B_{k_{\eps}} \bigtriangleup A ) \cup ( C_{k_{\eps}} \bigtriangleup B_{k_{\eps}} )
\subseteq ( B_{k_{\eps}} \bigtriangleup A ) \cup \bigcup_{i=1}^{k_{\eps}} (G_{i} \bigtriangleup A_{i}),
\end{equation*}
altogether we have that
\begin{equation*}
\inf_{G \in \T_{1}} \!\E\!\left[ \hat{\mu}_{\ProcX}( G \!\bigtriangleup\! A ) \right]
\!\leq\! \E\!\left[ \hat{\mu}_{\ProcX}( C_{k_{\eps}} \!\bigtriangleup\! A ) \right]
\!\leq\! \E\!\left[ \hat{\mu}_{\ProcX}( B_{k_{\eps}} \!\bigtriangleup\! A ) \right]
\!+ \!\sum_{i=1}^{k_{\eps}} \E\!\left[ \hat{\mu}_{\ProcX}( G_{i} \!\bigtriangleup\! A_{i} ) \right]
\!<\! \eps + \!\sum_{i=1}^{k_{\eps}} \frac{\eps}{k_{\eps}} = 2\eps,
\end{equation*}
where the second inequality is due to subadditivity of $\E[\hat{\mu}_{\ProcX}(\cdot)]$ (Lemma~\ref{lem:Ehatmu-submeasure}).
Since this argument holds for any $\eps > 0$, taking the limit as $\eps \to 0$ reveals that
$\inf\limits_{G \in \T_{1}} \E\!\left[ \hat{\mu}_{\ProcX}( G \bigtriangleup A ) \right] \leq 0$.
Together with nonnegativity of the left hand side (Lemma~\ref{lem:mu}),
this implies $A \in \Lambda$.  Thus, $\Lambda$ is closed under countable 
unions.

Finally, recalling that $\T$ is a topology, by definition we have $\X \in \T$,
and since $\T \subseteq \Lambda$, this implies $\X \in \Lambda$.
Altogether, we have established that $\Lambda$ is a $\sigma$-algebra.
Therefore, since $\Borel$ is the $\sigma$-algebra generated by $\T$, and $\T \subseteq \Lambda$,
it immediately follows that 
$\Borel \subseteq \Lambda$
(which also implies $\Lambda = \Borel$).
Since this argument holds for any choice of $\ProcX \in \KC$, the lemma immediately follows.
\end{proof}

For example, in the special case of $\X = \reals^{d}$ ($d \in \nats$) with the Euclidean topology,
the above proof implies it suffices to take the set $\T_{1}$ as the finite unions of rational-centered rational-radius open balls.
Now, continuing with the general case, 
the next lemma extends Lemma~\ref{lem:approximating-sets} from set approximation to 
function approximation, again using Condition~\ref{con:kc}.

\begin{lemma}
\label{lem:approximating-functions}
There exists a countable set $\tilde{\F}$ of measurable functions $\X \to \Y$ 
such that, for every $\ProcX \in \KC$, for every measurable $f : \X \to \Y$, 
\begin{equation*}
\inf\limits_{\tilde{f} \in \tilde{\F}} \E\!\left[ \hat{\mu}_{\ProcX}( \loss(\tilde{f}(\cdot),f(\cdot)) ) \right] = 0.
\end{equation*}
\end{lemma}
\begin{proof}
The proof will establish this claim for the set $\tilde{\F}$ of finite-depth \emph{decision list} functions, 
where the decision region of each node is specified by an element from the countable set $\T_{1}$ (from Lemma~\ref{lem:approximating-sets}) 
and the values are taken from a countable dense set $\tilde{\Y} \subseteq \Y$.

We will first prove that there exists a countable set $\tilde{\F}$ of measurable functions $\X \to \Y$ such that, 
for every $\ProcX \in \KC$, $\forall \eps > 0$, for every measurable $f : \X \to \Y$, 
$\exists \tilde{f}_{\eps} \in \tilde{\F}$ s.t. $\E\!\left[ \hat{\mu}_{\ProcX}( \loss(\tilde{f}_{\eps}(\cdot), f(\cdot)) ) \right] < 3\triconst\eps$.
The lemma will follow immediately from this (for this same set $\tilde{\F}$) by taking $\eps \to 0$.
Let $\T_{1}$ be as in Lemma~\ref{lem:approximating-sets}, 
and let $\tilde{\Y} \subseteq \Y$ be a countable set with $\sup\limits_{y \in \Y} \inf\limits_{\tilde{y} \in \tilde{\Y}} \loss(\tilde{y},y) = 0$;
this must exist, by the assumption that $(\Y,\loss)$ is separable.  
Fix some arbitrary value $y_{0} \in \Y$, and let $A_{0} = \X$.
For any $k \in \nats$, values $y_{1},\ldots,y_{k} \in \Y$, and sets $A_{1},\ldots,A_{k} \in \Borel$, 
for any $x \in \X$, define $\tilde{f}(x ; \{y_{i}\}_{i=1}^{k},\{A_{i}\}_{i=1}^{k}) = y_{\max\{j \in \{0,\ldots,k\} : x \in A_{j}\}}$;
one can easily verify that $\tilde{f}(\cdot ; \{y_{i}\}_{i=1}^{k}, \{A_{i}\}_{i=1}^{k})$ is a measurable function (indeed, it is a \emph{simple} function).
Define 
\begin{equation*}
\tilde{\F} = \left\{ \tilde{f}(\cdot ;\{y_{i}\}_{i=1}^{k}, \{A_{i}\}_{i=1}^{k}) : k \in \nats, \forall i \leq k, y_{i} \in \tilde{\Y}, A_{i} \in \T_{1} \right\},
\end{equation*}
and note that, given an indexing of $\tilde{\Y}$ and $\T_{1}$ by $\nats$, we can index $\tilde{\F}$ by finite tuples of integers 
(the indices of the corresponding $y_{i}$ and $A_{i}$ values), of which there are countably many, so that $\tilde{\F}$ is countable.

Enumerate the elements of $\tilde{\Y}$ as $\tilde{y}_{1},\tilde{y}_{2},\ldots$ (for simplicity of notation, we suppose this sequence is infinite; otherwise, we can simply repeat the elements to get an infinite sequence).
For each $\eps > 0$, let $B_{\eps,1} = \{ y \in \Y : \loss(\tilde{y}_{1},y) \leq \eps \}$,
and for each integer $i \geq 2$, inductively define $B_{\eps,i} = \{ y \in \Y : \loss(\tilde{y}_{i},y) \leq \eps \} \setminus \bigcup\limits_{j = 1}^{i-1} B_{\eps,j}$.
Note that the sets $B_{\eps,i}$ are measurable and disjoint over $i \in \nats$,
and that $\bigcup\limits_{i=1}^{\infty} B_{\eps,i} = \Y$.

Now fix any $\ProcX \in \KC$, any measurable $f : \X \to \Y$, and any $\eps > 0$.
For each $i \in \nats$, define $C_{\eps,i} = f^{-1}(B_{\eps,i})$, which is an element of $\Borel$ by measurability of $f$ and $B_{\eps,i}$.
Note that $\bigcup\limits_{i=1}^{\infty} C_{\eps,i} = f^{-1}\left( \bigcup\limits_{i=1}^{\infty} B_{\eps,i} \right) = f^{-1}(\Y) = \X$,
and furthermore that (since the $B_{\eps,i}$ sets are disjoint) the sets $C_{\eps,i}$ are disjoint over $i \in \nats$.
It follows that $\lim\limits_{k \to \infty} \bigcup\limits_{i=k}^{\infty} C_{\eps,i} = \emptyset$,
with $\bigcup\limits_{i=k}^{\infty} C_{\eps,i}$ nonincreasing in $k$, so that 
Condition~\ref{con:kc} 
entails $\lim\limits_{k \to \infty} \E\!\left[ \hat{\mu}_{\ProcX}\left( \bigcup\limits_{i=k}^{\infty} C_{\eps,i} \right) \right] = 0$.
In particular, this implies $\exists k_{\eps} \in \nats$ such that $\E\!\left[ \hat{\mu}_{\ProcX}\left( \bigcup\limits_{i=k_{\eps}+1}^{\infty} C_{\eps,i} \right) \right] < \triconst \eps / \maxloss$.

For each $i \in \{1,\ldots,k_{\eps}\}$, let $A_{\eps,i} \in \T_{1}$ be a set with $\E\!\left[ \hat{\mu}_{\ProcX}(A_{\eps,i} \bigtriangleup C_{\eps,i}) \right] < \eps / (k_{\eps} \maxloss)$,
which exists by the defining property of $\T_{1}$ from Lemma~\ref{lem:approximating-sets}.  Finally, let 
\begin{equation*}
\tilde{f}_{\eps}(\cdot) = \tilde{f}\left(\cdot ; \{\tilde{y}_{i}\}_{i=1}^{k_{\eps}},\{A_{\eps,i}\}_{i=1}^{k_{\eps}}\right),
\end{equation*}
and note that $\tilde{f}_{\eps} \in \tilde{\F}$.  Furthermore, for any $x \in \X = \bigcup\limits_{i=1}^{\infty} C_{\eps,i}$, 
\begin{align}
\loss(f(x),\tilde{f}_{\eps}(x)) 
& \leq \maxloss \ind_{\bigcup_{i=k_{\eps}+1}^{\infty} C_{\eps,i}}(x) + \sum_{i = 1}^{k_{\eps}} \loss(f(x), \tilde{f}_{\eps}(x)) \ind_{C_{\eps,i}}(x) \notag
\\ & \leq \maxloss \ind_{\bigcup_{i=k_{\eps}+1}^{\infty} C_{\eps,i}}(x) + \sum_{i=1}^{k_{\eps}} \triconst \left( \loss(f(x),\tilde{y}_{i}) \ind_{C_{\eps,i}}(x) + \loss(\tilde{y}_{i},\tilde{f}_{\eps}(x)) \ind_{C_{\eps,i}}(x) \right) \notag
\\ & \leq \maxloss \ind_{\bigcup_{i=k_{\eps}+1}^{\infty} C_{\eps,i}}(x) + \triconst \eps + \triconst \sum_{i=1}^{k_{\eps}} \loss(\tilde{y}_{i},\tilde{f}_{\eps}(x)) \ind_{C_{\eps,i}}(x). \label{eqn:approx-sets-bound-1}
\end{align}
Focusing now on the rightmost summation, let $[k_{\eps}] = \{1,\ldots,k_{\eps}\}$.
If $x \notin \bigcup\limits_{i \in [k_{\eps}]} C_{\eps,i}$ then this term is trivially zero due to the $\ind_{C_{\eps,i}}(x)$ factors.
Otherwise, let $j \in [k_{\eps}]$ be such that $x \in C_{\eps,j}$; this $j$ is unique by disjointness of the $C_{\eps,i}$ sets,
and for this same reason we have 
\begin{equation}
\label{eqn:approx-sets-bound-1-3a}
\sum\limits_{i=1}^{k_{\eps}} \loss(\tilde{y}_{i},\tilde{f}_{\eps}(x)) \ind_{C_{\eps,i}}(x) = \loss(\tilde{y}_{j},\tilde{f}_{\eps}(x)) \ind_{C_{\eps,j}}(x).
\end{equation}
Now note that if $x \in A_{\eps,j} \setminus \bigcup\limits_{i \in [k_{\eps}] \setminus \{j\}} A_{\eps,i}$, then $\loss(\tilde{y}_{j},\tilde{f}_{\eps}(x)) = 0$.
Thus, if $\loss(\tilde{y}_{j},\tilde{f}_{\eps}(x)) \ind_{C_{\eps,j}}(x) \neq 0$, 
then either $x \in C_{\eps,j} \setminus A_{\eps,j}$, 
or else $\exists i \in [k_{\eps}] \setminus \{j\}$ with $x \in C_{\eps,j} \cap A_{\eps,i} \subseteq A_{\eps,i} \setminus C_{\eps,i}$
(where this last inclusion follows from $C_{\eps,j} \cap C_{\eps,i} = \emptyset$).
Either way, we see that if 
$\loss(\tilde{y}_{j},\tilde{f}_{\eps}(x)) \ind_{C_{\eps,j}}(x) \neq 0$ then $\exists i \in [k_{\eps}]$ with $x \in C_{\eps,i} \bigtriangleup A_{\eps,i}$, 
so that 
\begin{equation}
\label{eqn:approx-sets-bound-1-3b}
\loss(\tilde{y}_{j},\tilde{f}_{\eps}(x)) \ind_{C_{\eps,j}}(x) 
\leq \loss(\tilde{y}_{j},\tilde{f}_{\eps}(x)) \sum\limits_{i =1}^{k_{\eps}} \ind_{C_{\eps,i} \bigtriangleup A_{\eps,i}}(x)
\leq \maxloss \sum\limits_{i =1}^{k_{\eps}} \ind_{C_{\eps,i} \bigtriangleup A_{\eps,i}}(x).
\end{equation}
Combining \eqref{eqn:approx-sets-bound-1-3a} and \eqref{eqn:approx-sets-bound-1-3b} 
and plugging back into \eqref{eqn:approx-sets-bound-1} yields
\begin{equation*}
\loss(f(x),\tilde{f}_{\eps}(x)) \leq \maxloss \ind_{\bigcup_{i=k_{\eps}+1}^{\infty} C_{\eps,i}}(x) + \triconst \eps + \triconst \maxloss \sum_{i=1}^{k_{\eps}} \ind_{C_{\eps,i} \bigtriangleup A_{\eps,i}}(x).
\end{equation*}
Therefore, by linearity of the expectation, together with monotonicity, homogeneity, and finite subadditivity of $\hat{\mu}_{\ProcX}$ (Lemma~\ref{lem:expectation}),
\begin{equation*}
\E\!\left[ \hat{\mu}_{\ProcX}\!\left( \loss\!\left(f(\cdot),\tilde{f}_{\eps}(\cdot)\right) \right) \right]
\leq \triconst \eps + \maxloss \E\!\left[ \hat{\mu}_{\ProcX}\!\left( \bigcup_{i=k_{\eps}+1}^{\infty} C_{\eps,i} \right) \right] + \triconst \maxloss \sum_{i=1}^{k_{\eps}} \E\!\left[ \hat{\mu}_{\ProcX}\!\left( C_{\eps,i} \bigtriangleup A_{\eps,i} \right) \right]
< 3 \triconst \eps.
\end{equation*}

The lemma now follows directly from this (together with non-negativity and symmetry of $\loss$), 
since each $\tilde{f}_{\eps} \in \tilde{\F}$, so that 
$\inf\limits_{\tilde{f} \in \tilde{\F}} \E\!\left[ \hat{\mu}_{\ProcX}( \loss(\tilde{f}(\cdot),f(\cdot)) ) \right] 
\leq \lim\limits_{\eps \to 0} \E\!\left[ \hat{\mu}_{\ProcX}( \loss(\tilde{f}_{\eps}(\cdot),f(\cdot)) ) \right] 
\leq \lim\limits_{\eps \to 0} 3 \triconst \eps = 0$.
\end{proof}

\noindent {\bf Remark:}~ Before proceeding, we remark that since $\T_{1}$ in the proof of Lemma~\ref{lem:approximating-sets} is defined as 
the set of finite unions of elements of $\T_{0}$ (where $\T_{0}$ is any countable base for the topology $\T$),
we can in fact represent any $f \in \tilde{\F}$ as a function 
$\tilde{f}(\cdot; \{y_{i}\}_{i=1}^{k}, \{A_{i}\}_{i=1}^{k})$, $k \in \nats$, $\{y_{i}\}_{i=1}^{k} \in \tilde{\Y}^{k}$, with $\{A_{i}\}_{i=1}^{k} \in \T_{0}^{k}$ 
(for $\tilde{f}(\cdot;\cdot,\cdot)$ and $\tilde{\Y}$ as defined in the above proof: 
that is, in the definition of $\tilde{\F}$ in the proof of Lemma~\ref{lem:approximating-functions}, we can replace $\T_{1}$ with $\T_{0}$ and the set $\tilde{\F}$ remains unchanged.
For instance, in the special case of $\X = \reals^{d}$ ($d \in \nats$) and $\Y=[0,1]$ with $\loss$ the squared loss ($\loss(a,b)=(a-b)^2$), 
we can take $\tilde{\F}$ as the set of rational-valued finite-depth decision lists, with 
the region of each decision node being a rational-centered rational-radius open ball.

{\vskip 2mm}
We will use Lemma~\ref{lem:approximating-functions} via the following immediate implication.

\begin{lemma}
\label{lem:approximating-sequence}
There exists a sequence $\{\F_{i}\}_{i=1}^{\infty}$ of nonempty finite sets of measurable functions $\X \to \Y$
with $\F_{1} \subseteq \F_{2} \subseteq \cdots$ such that, for every $\ProcX \in \KC$, for every measurable $f : \X \to \Y$, 
\begin{equation*}
\lim\limits_{i \to \infty} \min\limits_{f_{i} \in \F_{i}} \E\!\left[ \hat{\mu}_{\ProcX}( \loss(f_{i}(\cdot),f(\cdot)) ) \right] = 0.
\end{equation*}
\end{lemma}
\begin{proof}
Enumerate the elements of the countable set $\tilde{\F}$ from Lemma~\ref{lem:approximating-functions} as $\tilde{f}_1,\tilde{f}_2,\ldots$, 
and define $\F_{i} = \left\{\tilde{f}_{1},\ldots,\tilde{f}_{i}\right\}$.  
With this definition, by Lemma~\ref{lem:approximating-functions}, any $\ProcX \in \KC$ and measurable $f : \X \to \Y$ satisfy 
$\lim\limits_{i \to \infty} \min\limits_{f_{i} \in \F_{i}} \E\!\left[ \hat{\mu}_{\ProcX}( \loss(f_{i}(\cdot),f(\cdot)) ) \right] = \inf\limits_{\tilde{f} \in \tilde{\F}} \E\!\left[ \hat{\mu}_{\ProcX}( \loss(\tilde{f}(\cdot),f(\cdot)) ) \right] = 0$.
\end{proof}

Additionally, we have the following property for the $f$-approximating sequences of sets $\F_{i}$ implied by Lemma~\ref{lem:approximating-sequence}.

\begin{lemma}
\label{lem:exponential-approximating-sequence}
Fix any process $\ProcX$ on $\X$, any measurable function $f : \X \to \Y$, any nondecreasing sequence $\{u_{i}\}_{i=1}^{\infty}$ in $\nats$ with $u_{i} \to \infty$,
and any sequence $\{\F_{i}\}_{i=1}^{\infty}$ of
sets of measurable functions $\X \to \Y$ with $\F_{1} \subseteq \F_{2} \subseteq \cdots$ 
such that $\lim\limits_{i \to \infty} \inf\limits_{g \in \F_{i}} \E\!\left[ \hat{\mu}_{\ProcX}(\loss(g(\cdot),f(\cdot))) \right] = 0$.
There exists a (nonrandom) sequence $\{f_{i}\}_{i=1}^{\infty}$, with $f_{i} \in \F_{i}$ for each $i \in \nats$, 
and a (nonrandom) sequence $\{\alpha_{i}\}_{i=1}^{\infty}$ in $(0,\infty)$ with $\alpha_{i} \to 0$,
such that, on an event $K$ of probability one, $\exists \iota_{0} \in \nats$ such that $\forall i \geq \iota_{0}$,
\begin{equation*}
\sup_{u_{i} \leq m < \infty} \frac{1}{m} \sum_{t=1}^{m} \loss(f_{i}(X_{t}),f(X_{t})) \leq \alpha_{i}.
\end{equation*}
\end{lemma}
\begin{proof}
Let $\{g_{i}\}_{i=1}^{\infty}$ be a sequence with $g_{i} \in \F_{i}$ for each $i \in \nats$, s.t. 
$\lim\limits_{i \to \infty} \E\!\left[ \hat{\mu}_{\ProcX}( \loss(g_{i}(\cdot),f(\cdot)) ) \right] = 0$.
Then $\forall k \in \nats$, $\exists j_{k} \in \nats$ such that
$\E\!\left[ \hat{\mu}_{\ProcX}( \loss( g_{j_{k}}(\cdot), f(\cdot) ) ) \right] < 4^{-k} \maxloss$.
Let us fix any sequence $\{j_{k}\}_{k=1}^{\infty}$ in $\nats$ such that $j_{k}$ has this property for every $k$.  
For completeness, also define $j_{0} = 1$.
Furthermore, since $u_{i} \to \infty$, the dominated convergence theorem implies that $\forall j \in \nats$, 
\begin{align*}
&& \lim_{i \to \infty} \E\!\left[ \sup_{u_{i} \leq m < \infty} \frac{1}{m} \sum_{t=1}^{m} \loss(g_{j}(X_{t}),f(X_{t})) \right]
= \E\!\left[ \lim_{i \to \infty} \sup_{u_{i} \leq m < \infty} \frac{1}{m} \sum_{t=1}^{m} \loss(g_{j}(X_{t}),f(X_{t})) \right]
\\ && = \E\!\left[ \limsup_{m \to \infty} \frac{1}{m} \sum_{t=1}^{m} \loss(g_{j}(X_{t}),f(X_{t})) \right]
= \E\!\left[ \hat{\mu}_{\ProcX}( \loss(g_{j}(\cdot),f(\cdot)) ) \right].
\end{align*}
In particular, this implies that $\forall k \in \nats$, $\exists i_{k} \in \nats$ such that 
\begin{equation}
\label{eqn:ik-jk-expected}
\E\!\left[ \sup_{u_{i_{k}} \leq m < \infty} \frac{1}{m} \sum_{t=1}^{m} \loss(g_{j_{k}}(X_{t}),f(X_{t})) \right]
\leq \E\!\left[ \hat{\mu}_{\ProcX}( \loss(g_{j_{k}}(\cdot),f(\cdot)) ) \right] + 4^{-k} \maxloss
< 2 \cdot 4^{-k} \maxloss.
\end{equation}
Also note that, since the leftmost expression in \eqref{eqn:ik-jk-expected} is nonincreasing in $i_{k}$, 
we may choose $i_{k} > i_{k-1}$ if $k \geq 2$ (or $i_{k} > 1$ for $k = 1$).
Thus, letting $i_{0} = 1$, there exists a strictly increasing sequence 
$\{i_{k}\}_{k=0}^{\infty}$ in $\nats$ such that $i_{k}$ has the property \eqref{eqn:ik-jk-expected} for every $k \in \nats$.
We may then note that, by Markov's inequality,
\begin{align*}
& \sum_{k=0}^{\infty} \P\left( \sup_{u_{i_{k}} \leq m < \infty} \frac{1}{m} \sum_{t=1}^{m} \loss(g_{j_{k}}(X_{t}),f(X_{t})) > 2^{(1/2)-k} \sqrt{\maxloss} \right)
\\ & \leq \sum_{k=0}^{\infty} \frac{1}{2^{(1/2)-k}\sqrt{\maxloss}} \E\!\left[ \sup_{u_{i_{k}} \leq m < \infty} \frac{1}{m} \sum_{t=1}^{m} \loss(g_{j_{k}}(X_{t}),f(X_{t})) \right]
\\ & \leq \sum_{k=0}^{\infty} \frac{1}{2^{(1/2)-k}\sqrt{\maxloss}} 2 \cdot 4^{-k} \maxloss
= \sum_{k=0}^{\infty} 2^{(1/2)-k} \sqrt{\maxloss}
= 2^{3/2} \sqrt{\maxloss} < \infty.
\end{align*}
Therefore, by the Borel-Cantelli Lemma, there exists an event $K$ of probability one, on which
$\exists \kappa_{0} \in \nats$ such that, $\forall k \geq \kappa_{0}$, 
\begin{equation}
\label{eqn:k-over-kappa0}
\sup_{u_{i_{k}} \leq m < \infty} \frac{1}{m} \sum_{t=1}^{m} \loss(g_{j_{k}}(X_{t}),f(X_{t})) \leq 2^{(1/2)-k} \sqrt{\maxloss}.
\end{equation}

Now, $\forall i \in \nats$, define 
\begin{equation*}
k_{i} = \max\left\{ k \in \nats \cup \{0\} : \max\{i_{k},j_{k}\} \leq i \right\},
\end{equation*}
and let $\alpha_{i} = 2^{(1/2)-k_{i}} \sqrt{\maxloss}$.
To see that the value $k_{i}$ is well-defined for every $i \in \nats$, note that $\max\{i_{0},j_{0}\} = 1 \leq i$,
so that the set on the right hand side is nonempty, and furthermore, since $\{i_{k}\}_{k=0}^{\infty}$ is strictly increasing,
every $k \geq i$ has $\max\{i_{k},j_{k}\} > i$, so that the set is finite, and hence has a maximum element.
Also, since $i_{k}$ and $j_{k}$ are finite for every $k$, we have that $\lim\limits_{i \to \infty} k_{i} = \infty$.
In particular, this implies that, on the event $K$, $\exists \iota_{0} \in \nats$ such that $\forall i \geq \iota_{0}$,
$k_{i} \geq \kappa_{0}$, so that \eqref{eqn:k-over-kappa0} implies
\begin{equation}
\label{eqn:i-over-iota0}
\sup_{u_{i_{k_{i}}} \leq m < \infty} \frac{1}{m} \sum_{t=1}^{m} \loss(g_{j_{k_{i}}}(X_{t}),f(X_{t})) \leq \alpha_{i}.
\end{equation}
Now define $f_{i} = g_{j_{k_{i}}}$ for every $i \in \nats$.
Note that, since $j_{k_{i}} \leq i$ (by definition of $k_{i}$)
and $\F_{1} \subseteq \F_{2} \subseteq \cdots$,
we have $\F_{j_{k_{i}}} \subseteq \F_{i}$.
In particular, since $f_{i} = g_{j_{k_{i}}} \in \F_{j_{k_{i}}}$ (by definition), 
this implies $f_{i} \in \F_{i}$ for every $i \in \nats$.
Also note that, since $i_{k_{i}} \leq i$ (by definition of $k_{i}$), 
and $\{u_{t}\}_{t=1}^{\infty}$ is a nondecreasing sequence, 
$u_{i_{k_{i}}} \leq u_{i}$ for every $i \in \nats$.
Together with \eqref{eqn:i-over-iota0}, these facts imply that, on the event $K$, $\forall i \geq \iota_{0}$,
\begin{equation*}
\sup_{u_{i} \leq m < \infty} \frac{1}{m} \sum_{t=1}^{m} \loss(f_{i}(X_{t}),f(X_{t}))
\leq \sup_{u_{i_{k_{i}}} \leq m < \infty} \frac{1}{m} \sum_{t=1}^{m} \loss(f_{i}(X_{t}),f(X_{t})) \leq \alpha_{i}.
\end{equation*}
\end{proof}

With these results in hand, we are finally ready for the proof of sufficiency of Condition~\ref{con:kc}
for strong universal inductive learning.

\begin{lemma}
\label{lem:kc-subset-suil}
$\KC \subseteq \SUIL$.
\end{lemma}
\begin{proof}
Suppose $\ProcX \in \KC$.
Lemma~\ref{lem:approximating-sequence} implies that there exists a sequence $\{\G_{i}\}_{i=1}^{\infty}$ 
of finite sets of measurable functions with $\G_{1} \subseteq \G_{2} \subseteq \cdots$
such that, for every measurable function $\target : \X \to \Y$, 
$\lim\limits_{i \to \infty} \min\limits_{g_{i} \in \G_{i}} \E\!\left[ \hat{\mu}_{\ProcX}(\loss(g_{i}(\cdot),\target(\cdot))) \right] = 0$.  
Furthermore, applying Lemma~\ref{lem:srm} to the sequence of sets $\{ \loss(f(\cdot),g(\cdot)) : f,g \in \G_{i} \}$, with 
$\gamma_{i} = 4^{1-i} \maxloss$, 
we find that there exist (nonrandom) nondecreasing sequences $\{m_{i}\}_{i=1}^{\infty}$ and $\{i_{n}\}_{n=1}^{\infty}$ in $\nats$ with $m_{i} \to \infty$ and $i_{n} \to \infty$
such that $\forall n \in \nats$, $m_{i_{n}} \leq n$ and 
\begin{equation}
\label{eqn:kc-subset-suil-1}
\E\!\left[ \max_{f,g \in \G_{i_{n}}} \left| \hat{\mu}_{\ProcX}( \loss(f(\cdot),g(\cdot)) ) ~- \!\max_{m_{i_{n}} \leq m \leq n} \frac{1}{m} \sum_{t=1}^{m} \loss(f(X_{t}),g(X_{t})) \right|\right] \leq \gamma_{i_{n}}.
\end{equation}

Let $I = \{ i_{n} : n \in \nats \}$, and for each $i \in I$, define $n_{i} = \min\{ n \in \nats : i_{n} = i \}$.
Markov's inequality and \eqref{eqn:kc-subset-suil-1} imply
\begin{align*}
& \sum_{i \in I} \P\left( \max_{f,g \in \G_{i}} \left| \hat{\mu}_{\ProcX}( \loss(f(\cdot),g(\cdot)) ) ~- \!\max_{m_{i} \leq m \leq n_{i}} \frac{1}{m} \sum_{t=1}^{m} \loss(f(X_{t}),g(X_{t})) \right| > \sqrt{\gamma_{i}} \right)
\\ & \leq \sum_{i \in I} \frac{1}{\sqrt{\gamma_{i}}} \E\!\left[ \max_{f,g \in \G_{i}} \left| \hat{\mu}_{\ProcX}( \loss(f(\cdot),g(\cdot)) ) ~- \!\max_{m_{i} \leq m \leq n_{i}} \frac{1}{m} \sum_{t=1}^{m} \loss(f(X_{t}),g(X_{t})) \right| \right]
\\ & \leq \sum_{i \in I} \sqrt{\gamma_{i}}
\leq \sum_{i=1}^{\infty} 2^{1-i} \sqrt{\maxloss}
= 2 \sqrt{\maxloss} < \infty.
\end{align*}
Therefore, the Borel-Cantelli Lemma implies that there exists an event $K^{\prime}$ of probability one,
on which $\exists \iota_{1} \in \nats$ such that $\forall i \in I$ with $i \geq \iota_{1}$, 
\begin{equation}
\label{eqn:inductive-i-sqrt-gamma-bound}
\max_{f,g \in \G_{i}} \left( \hat{\mu}_{\ProcX}( \loss(f(\cdot),g(\cdot)) ) ~- \!\max_{m_{i} \leq m \leq n_{i}} \frac{1}{m} \sum_{t=1}^{m} \loss(f(X_{t}),g(X_{t})) \right) \leq \sqrt{\gamma_{i}}.
\end{equation}
Additionally, note that $\forall n \in \nats$, $n \geq n_{i_{n}}$, so that $\forall f,g \in \G_{i_{n}}$, 
\begin{equation}
\label{eqn:inductive-relax-n}
\max_{m_{i_{n}} \leq m \leq n} \frac{1}{m} \sum_{t=1}^{m} \loss(f(X_{t}),g(X_{t}))
\geq \max_{m_{i_{n}} \leq m \leq n_{i_{n}}} \frac{1}{m} \sum_{t=1}^{m} \loss(f(X_{t}),g(X_{t})).
\end{equation}
Furthermore, since $i_{n} \to \infty$, on the event $K^{\prime}$, 
$\exists \nu_{1} \in \nats$ such that $\forall n \geq \nu_{1}$,
we have $i_{n} \geq \iota_{1}$, so that \eqref{eqn:inductive-i-sqrt-gamma-bound} and \eqref{eqn:inductive-relax-n} imply
\begin{equation}
\max_{f,g \in \G_{i_{n}}} \left( \hat{\mu}_{\ProcX}( \loss(f(\cdot),g(\cdot)) ) ~- \!\max_{m_{i_{n}} \leq m \leq n} \frac{1}{m} \sum_{t=1}^{m} \loss(f(X_{t}),g(X_{t})) \right)
\leq \sqrt{\gamma_{i_{n}}}. \label{eqn:inductive-n-sqrt-gamma-bound}
\end{equation}

Now consider using the inductive learning rule $\hat{f}_{n}$ defined in \eqref{eqn:suil-rule},
with $\F_{n} = \G_{i_{n}}$ and $\hat{m}_{n} = m_{i_{n}}$ for each $n \in \nats$.
Fix any measurable function $\target : \X \to \Y$.
By the defining properties of the $\G_{i}$ sequence, 
and the fact that $m_{i}$ is nondecreasing with $\lim\limits_{i \to \infty} m_{i} = \infty$,
Lemma~\ref{lem:exponential-approximating-sequence} implies that there exists a (nonrandom) sequence
$\{\target_{i}\}_{i=1}^{\infty}$ with $\target_{i} \in \G_{i}$ for each $i \in \nats$, 
a (nonrandom) sequence $\{\alpha_{i}\}_{i=1}^{\infty}$ in $(0,\infty)$ with $\alpha_{i} \to 0$,
and an event $K$ of probability one, on which 
$\exists \iota_{0} \in \nats$ such that $\forall i \geq \iota_{0}$,
\begin{equation}
\label{eqn:inductive-targetn-alpha-bound}
\sup_{m_{i} \leq m < \infty} \frac{1}{m} \sum_{t=1}^{m} \loss( \target_{i}(X_{t}), \target(X_{t}) ) \leq \alpha_{i}.
\end{equation}
On this event, let $\nu_{0} \in \nats$ be a value such that $\forall n \in \nats$ with $n \geq \nu_{0}$, we have 
$i_{n} \geq \iota_{0}$; such a $\nu_{0}$ exists since $\lim\limits_{n \to \infty} i_{n} = \infty$.

For brevity, define $\hat{g}_{n}(\cdot) = \hat{f}_{n}(X_{1:n},\target(X_{1:n}),\cdot)$ for every $n \in \nats$.  
Note that, by the definition of $\hat{f}_{n}$ from \eqref{eqn:suil-rule} and the fact that $\target_{n} \in \F_{n} = \G_{i_{n}}$ and $\hat{m}_{n} = m_{i_{n}}$, $\forall n \in \nats$, 
we have 
\begin{equation*}
\max_{m_{i_{n}} \leq m \leq n} \frac{1}{m} \sum_{t=1}^{m} \loss(\hat{g}_{n}(X_{t}),\target(X_{t}))
\leq \max_{m_{i_{n}} \leq m \leq n} \frac{1}{m} \sum_{t=1}^{m} \loss(\target_{i_{n}}(X_{t}),\target(X_{t})).
\end{equation*}
Thus, on the event $K$, $\forall n \in \nats$ with $n \geq \nu_{0}$, 
\eqref{eqn:inductive-targetn-alpha-bound} implies
\begin{equation}
\label{eqn:inductive-hatg-alpha-bound}
\max_{m_{i_{n}} \leq m \leq n} \frac{1}{m} \sum_{t=1}^{m} \loss(\hat{g}_{n}(X_{t}),\target(X_{t}))
\leq \alpha_{i_{n}}.
\end{equation}

Now suppose the event $K \cap K^{\prime}$ occurs and fix any $n \in \nats$ with $n \geq \max\{\nu_{0},\nu_{1}\}$.
The relaxed triangle inequality and subadditivity of $\hat{\mu}_{\ProcX}$ (Lemma~\ref{lem:expectation}) imply
\begin{equation}
\label{eqn:inductive-direct-tribound}
\hat{\mu}_{\ProcX}(\loss(\hat{g}_{n}(\cdot),\target(\cdot)))
\leq \triconst \hat{\mu}_{\ProcX}(\loss(\hat{g}_{n}(\cdot),\target_{i_{n}}(\cdot))) + \triconst \hat{\mu}_{\ProcX}(\loss(\target_{i_{n}}(\cdot),\target(\cdot))).
\end{equation}
Furthermore, since $\hat{g}_{n}$ and $\target_{i_{n}}$ are both elements of $\G_{i_{n}}$, and since the event $K^{\prime}$ holds and $n \geq \nu_{1}$, 
the inequality \eqref{eqn:inductive-n-sqrt-gamma-bound}
implies 
\begin{equation}
\label{eqn:inductive-direct-tribound-1}
\hat{\mu}_{\ProcX}(\loss(\hat{g}_{n}(\cdot),\target_{i_{n}}(\cdot)))
\leq \max_{m_{i_{n}} \leq m \leq n} \frac{1}{m} \sum_{t=1}^{m} \loss(\hat{g}_{n}(X_{t}),\target_{i_{n}}(X_{t})) + \sqrt{\gamma_{i_{n}}}.
\end{equation}
Then the relaxed triangle inequality and symmetry of $\loss$, together with subadditivity of the $\max$, imply 
\begin{align}
& \max_{m_{i_{n}} \leq m \leq n} \frac{1}{m} \sum_{t=1}^{m} \loss(\hat{g}_{n}(X_{t}),\target_{i_{n}}(X_{t})) \notag
\\ & \leq \triconst \max_{m_{i_{n}} \leq m \leq n} \frac{1}{m} \sum_{t=1}^{m} \loss(\hat{g}_{n}(X_{t}),\target(X_{t})) + \triconst \max_{m_{i_{n}} \leq m \leq n} \frac{1}{m} \sum_{t=1}^{m} \loss(\target_{i_{n}}(X_{t}),\target(X_{t})).
\label{eqn:inductive-direct-tribound-2}
\end{align}
Also, since $m_{i_{n}}$ is finite, we generally have 
\begin{equation*}
\hat{\mu}_{\ProcX}(\loss(\target_{i_{n}}(\cdot),\target(\cdot))) 
\leq \sup_{m_{i_{n}} \leq m < \infty} \frac{1}{m} \sum_{t=1}^{m} \loss( \target_{i_{n}}(X_{t}), \target(X_{t}) ).
\end{equation*}
Combining this with \eqref{eqn:inductive-direct-tribound-1} and \eqref{eqn:inductive-direct-tribound-2} and plugging into \eqref{eqn:inductive-direct-tribound} yields
\begin{align*}
& \hat{\mu}_{\ProcX}(\loss(\hat{g}_{n}(\cdot),\target(\cdot)))
\\ & \leq \triconst^{2} \!\max_{m_{i_{n}} \leq m \leq n} \frac{1}{m} \!\sum_{t=1}^{m}\! \loss(\hat{g}_{n}(X_{t}),\target\!(X_{t})) + \triconst (\triconst\!+\!1)\!\! \sup_{m_{i_{n}} \leq m < \infty} \frac{1}{m} \!\sum_{t=1}^{m}\! \loss(\target_{i_{n}}\!(X_{t}),\target\!(X_{t}))
+ \triconst \sqrt{\gamma_{i_{n}}}.
\end{align*}
Since the event $K$ holds and $n \geq \nu_{0}$, 
the inequalities \eqref{eqn:inductive-hatg-alpha-bound} and \eqref{eqn:inductive-targetn-alpha-bound} provide upper bounds on the first two terms above, respectively, 
so that altogether we have 
\begin{equation*}
\hat{\mu}_{\ProcX}(\loss(\hat{g}_{n}(\cdot),\target(\cdot))) 
\leq \triconst^{2} \alpha_{i_{n}} + \triconst (\triconst+1) \alpha_{i_{n}} + \triconst \sqrt{\gamma_{i_{n}}}
= \triconst (2\triconst+1) \alpha_{i_{n}} + \triconst \sqrt{\gamma_{i_{n}}}.
\end{equation*}
In particular, recall that $i_{n} \to \infty$ and $\lim\limits_{i \to \infty} \alpha_{i} = \lim\limits_{i \to \infty} \gamma_{i} = 0$, 
so that the rightmost expression above converges to $0$ as $n \to \infty$.
Thus, on the event $K \cap K^{\prime}$, since $\max\{\nu_{0},\nu_{1}\} < \infty$, we have that 
\begin{equation*}
\limsup_{n \to \infty} \hat{\L}_{\ProcX}(\hat{f}_{n},\target; n)
= \limsup_{n \to \infty} \hat{\mu}_{\ProcX}(\loss(\hat{g}_{n}(\cdot),\target(\cdot)))
\leq \lim_{n \to \infty} \triconst (2\triconst+1) \alpha_{i_{n}} + \triconst \sqrt{\gamma_{i_{n}}}
= 0.
\end{equation*}

Since the event $K \cap K^{\prime}$ has probability one (by the union bound), 
and $\hat{\L}_{\ProcX}$ is nonnegative,
this establishes that $\hat{\L}_{\ProcX}(\hat{f}_{n},\target;n) \to 0$ (a.s.).  
Since this argument applies to \emph{any} measurable $\target : \X \to \Y$,
this establishes that $\hat{f}_{n}$ is strongly universally consistent under $\ProcX$, so that $\ProcX \in \SUIL$.
Since this argument applies to \emph{any} $\ProcX \in \KC$, this completes the proof that $\KC \subseteq \SUIL$.
\end{proof}

Combining Lemmas~\ref{lem:suil-subset-sual}, \ref{lem:sual-subset-kc}, and \ref{lem:kc-subset-suil} completes the proof of Theorem~\ref{thm:main}.\\

Interestingly, we may note that the \emph{only} reliance of the above proof of Lemma~\ref{lem:kc-subset-suil} on the assumption $\ProcX \in \KC$ 
is in the existence of the set $\tilde{\F}$ from Lemma~\ref{lem:approximating-functions} (used here via its implication in Lemma~\ref{lem:approximating-sequence}):
that is, we have in fact established that any $\ProcX$ for which there exists a countable set $\tilde{\F}$ with these properties
admits strong universal inductive learning, so that the existence of such a set implies $\ProcX \in \SUIL$.
Together with Theorem~\ref{thm:main} (implying $\KC = \SUIL$) and Lemma~\ref{lem:approximating-functions} 
(implying $\ProcX \in \KC$ suffices for such a set $\tilde{\F}$ to exist),
this establishes that $\KC$ is in fact \emph{equivalent} to the set of processes for which such a set exists 
(and hence so are $\SUIL$ and $\SUAL$, via Theorem~\ref{thm:main}).
Thus, we have yet another useful equivalent way of expressing Condition~\ref{con:kc}, 
stated formally in the following corollary.

\begin{corollary}
\label{cor:approximating-sequence-equivalence}
A process $\ProcX$ satisfies Condition~\ref{con:kc}
if and only if
there exists a countable set $\tilde{\G}$ of measurable functions $\X \to \Y$ such that,
for every measurable $f : \X \to \Y$, 
$\inf\limits_{\tilde{g} \in \tilde{\G}} \E\!\left[ \hat{\mu}_{\ProcX}(\loss(\tilde{g}(\cdot),f(\cdot))) \right] = 0$.
\end{corollary}

Indeed, we may further observe that, since Condition~\ref{con:kc} does not involve $\Y$ or $\loss$, 
applying the above equivalence to the special case of $\Y = \{0,1\}$ and $\loss(y,y^{\prime}) = \ind[y \neq y^{\prime}]$
admits another simple equivalent condition: 
namely,
a process $\ProcX$ satisfies Condition~\ref{con:kc} 
if and only if there exists a countable set $\T_{2} \subseteq \Borel$ with $\sup\limits_{A \in \Borel} \inf\limits_{G \in \T_{2}} \E\!\left[ \hat{\mu}_{\ProcX}(G \bigtriangleup A) \right] = 0$.
Recall that this was the guarantee for the set $\T_{1}$ from Lemma~\ref{lem:approximating-sets}.
However, Lemma~\ref{lem:approximating-sets} also guarantees the stronger property that this \emph{same} set $\T_{1}$ 
can serve as the above set $\T_{2}$ for \emph{every} $\ProcX$ satisfying Condition~\ref{con:kc}.
Similarly, the set $\tilde{\F}$ supplied by Lemma~\ref{lem:approximating-functions} is also defined independent of $\ProcX$, 
so that this same set $\tilde{\F}$ can serve as the set $\tilde{\G}$ in Corollary~\ref{cor:approximating-sequence-equivalence} 
for every $\ProcX$ satisfying Condition~\ref{con:kc}.
This universality of $\T_{1}$ and $\tilde{\F}$ will be crucial in the next section when discussing \emph{optimistically} universal learning.

\section{Optimistically Universal Learning}
\label{sec:universal2}

This section presents the proofs of two results on optimistically universal learning:
Theorems~\ref{thm:optimistic-self-adaptive} and \ref{thm:no-optimistic-inductive} stated in Section~\ref{subsec:main}.
For the first of these, we propose a new general self-adaptive learning rule, and prove that 
it is optimistically universal: that is, it is strongly universally consistent 
under \emph{every} process admitting strong universal self-adaptive learning.
For the second of these theorems, we prove that there is no optimistically universal 
inductive learning rule.  Together, these results imply that the additional capability of 
self-adaptive learning rules to adjust their predictor based on the unlabeled test data 
is crucial for optimistically universal learning.

\subsection{Existence of Optimistically Universal Self-Adaptive Learning Rules}
\label{sec:universal2-adaptive}

We now present the construction of an optimistically universal self-adaptive learning rule.
Fix a sequence $\{\F_{i}\}_{i=1}^{\infty}$ of nonempty finite sets of measurable functions $\X \to \Y$ with $\F_{1} \subseteq \F_{2} \subseteq \cdots$
such that
$\forall \ProcX \in \KC$, for every measurable $f \!:\! \X \!\to\! \Y$, 
$\lim\limits_{i \to \infty} \min\limits_{f_{i} \in \F_{i}} \E[ \hat{\mu}_{\ProcX}(\loss(f_{i}(\cdot),f(\cdot))) ] = 0$.
Recall that such a sequence $\{\F_{i}\}_{i=1}^{\infty}$ is guaranteed to exist by Lemma~\ref{lem:approximating-sequence}.
Let $\{u_{i}\}_{i=1}^{\infty}$ be an arbitrary nondecreasing sequence in $\nats$ with $u_{i} \to \infty$ and $u_{1} = 1$,
and let $\{\gamma_{i}\}_{i=1}^{\infty}$ be an arbitrary sequence in $(0,\infty)$ 
with $\gamma_{1} \geq \maxloss$ and $\gamma_{i} \to 0$.
Let $\{x_{i}\}_{i=1}^{\infty}$ be any sequence in $\X$ and let $\{y_{i}\}_{i=1}^{\infty}$ be any sequence in $\Y$.
For each $n,m \in \nats$ with $m \geq n$, let 
\begin{align}
\hat{i}_{n,m}(x_{1:m}) = &\max\Bigg\{ i \in \nats : u_{i} \leq n \text{ and } \label{eqn:sual-index}
\\ & \max_{f,g \in \F_{i}} \left( \max_{u_{i} \leq s \leq m} \frac{1}{s} \sum_{t=1}^{s} \loss(f(x_{t}),g(x_{t})) - \max_{u_{i} \leq s \leq n} \frac{1}{s} \sum_{t=1}^{s} \loss(f(x_{t}),g(x_{t})) \right) \leq \gamma_{i} \Bigg\}. \notag
\end{align}
This is a well-defined positive integer, since our constraints on $u_{1}$ and $\gamma_{1}$ guarantee that the set of $i$ values on the right hand side is nonempty,
while the fact that $u_{i} \to \infty$ implies this set of $i$ values is finite (and hence has a maximum element).
Finally, for every $n,m \in \nats$ with $m \geq n$, define the function $\hat{f}_{n,m}(x_{1:m},y_{1:n},\cdot)$ as
\begin{equation}
\label{eqn:sual-rule}
\argmin_{f \in \F_{\hat{i}_{n,m}(x_{1:m})}} \max_{u_{\hat{i}_{n,m}(x_{1:m})} \leq s \leq n} \frac{1}{s} \sum_{t=1}^{s} \loss(f(x_{t}),y_{t}).
\end{equation}
We break ties in the $\argmin$ based on a fixed preference ordering of $\F_{i}$.
Since the sets $\F_{i}$ are finite, one can easily verify that
this makes $\hat{f}_{n,m}$ a measurable function, and 
hence \eqref{eqn:sual-rule} defines a valid self-adaptive learning rule.
For completeness, for every $m \in \nats \cup \{0\}$, also define $\hat{f}_{0,m}(x_{1:m},\{\},\cdot)$ as an arbitrary element of $\F_{1}$ (chosen identically for every $m$ and $x_{1:m}$),
which is then also a measurable function.  

The essential difference between the self-adaptive learning rule \eqref{eqn:sual-rule} and the inductive learning rule \eqref{eqn:suil-rule} 
is that the self-adaptive rule uses the sequence of test samples $X_{1:m}$ for the \emph{model selection} component, 
selecting which class $\F_{i}$ to use in the optimization in \eqref{eqn:sual-rule},
whereas \eqref{eqn:suil-rule} uses a \emph{distribution-dependent} selection. 
Specifically, the self-adaptive rule replaces the distribution-dependent value $i_{n}$ 
from Lemma~\ref{lem:srm}, used in the proof of Lemma~\ref{lem:kc-subset-suil},
with a data-dependent value $\hat{i}_{n,m}(X_{1:m})$, thus removing all dependence on the 
distribution of $\ProcX$.  In the proof of Lemma~\ref{lem:kc-subset-suil}, the value $i_{n}$ 
is chosen to guarantee (via Lemma~\ref{lem:srm}) that the estimator 
$\max\limits_{m_{i_{n}} \leq m \leq n} \frac{1}{m} \sum\limits_{t=1}^{m} \loss(f(X_{t}),g(X_{t}))$
is close to $\hat{\mu}_{\ProcX}(\loss(f(\cdot),g(\cdot)))$ uniformly over all $f,g$ in the class 
$\G_{i_{n}}$ defined in the proof.  The value $\hat{i}_{n,m}(X_{1:m})$ in \eqref{eqn:sual-index} is designed to provide 
this guarantee \emph{directly}.
Specifically, in the analysis of $\hat{f}_{n,m}$ below, the value $\hat{i}_{n,m}(X_{1:m})$ 
ensures (essentially) that for large $m$, 
for all $f,g \in \F_{\hat{i}_{n,m}(X_{1:m})}$, 
$\max\limits_{u_{\hat{i}_{n,m}(X_{1:m})} \leq s \leq n} \frac{1}{s} \sum\limits_{t=1}^{s} \loss(f(X_{t}),g(X_{t}))$ is 
close to $\hat{\mu}_{\ProcX}(\loss(f(\cdot),g(\cdot)))$.
These can then be related to the losses relative to $\target$ 
via relaxed triangle inequalities and the approximation guarantees from Lemma~\ref{lem:exponential-approximating-sequence}, 
to conclude that the function $f \in \F_{\hat{i}_{n,m}(X_{1:m})}$ minimizing 
$\max\limits_{u_{\hat{i}_{n,m}(X_{1:m})} \leq s \leq n} \frac{1}{s} \sum\limits_{t=1}^{s} \loss(f(X_{t}),\target(X_{t}))$
achieves a relatively small value of $\hat{\mu}_{\ProcX}(\loss(f(\cdot),\target(\cdot)))$.
In both the inductive and self-adaptive cases, 
this approach is analogous to the traditional principles of model selection, whereby we 
constrain the function class so that empirical estimates of the risk are close enough to 
the corresponding population risks to guarantee that optimizing the estimate yields a 
function with relatively small population risk, but while also allowing the constraint 
to become less restrictive as $n$ grows, to admit increasingly good approximations of $\target$.

As discussed in the proofs of Lemmas~\ref{lem:approximating-sets}, \ref{lem:approximating-functions}, and \ref{lem:approximating-sequence}, 
and the remark following the proof of Lemma~\ref{lem:approximating-functions}, 
the sets $\F_{i}$ can be constructed based on an enumeration of finite-depth decision lists, 
with the region of each decision node being an element of a countable base for the topology $\T$, 
and with values from a countable dense subset of $\Y$.
For instance, in the special case of $\X = \reals^{d}$ ($d \in \nats$) with the Euclidean topology, 
and $\Y = [0,1]$ with the squared loss ($\loss(a,b)=(a-b)^{2}$), we can let $\tilde{f}_{1},\tilde{f}_{2},\ldots$ 
be an enumeration of the rational-valued finite-depth decision lists, with the region of 
each decision node being a rational-centered rational-radius open ball. 
Then we can let $\F_{i} = \{\tilde{f}_{1},\ldots,\tilde{f}_{i}\}$ for each $i \in \nats$.
In this case, in principle the learning rule $\hat{f}_{n,m}$ 
can be approximated by a digital computer (up to some finite precision for the points and predictions).

Continuing with the general case, 
we have the following theorem for this $\hat{f}_{n,m}$ rule.

\begin{theorem}
\label{thm:doubly-universal-adaptive}
The self-adaptive learning rule $\hat{f}_{n,m}$ is optimistically universal.
\end{theorem}
\begin{proof}
The proof proceeds along similar lines to that of Lemma~\ref{lem:kc-subset-suil},
except using the data-dependent values $\hat{i}_{n,m}(X_{1:m})$ in place of the distribution-dependent sequence $i_{n}$ from the proof of Lemma~\ref{lem:kc-subset-suil}.
Fix any $\ProcX \in \KC$ and any measurable $\target : \X \to \Y$.

Note that, for any given $i \in \nats$ and $f,g \in \F_{i}$, 
$\max\limits_{u_{i} \leq s \leq m} \frac{1}{s} \sum\limits_{t=1}^{s} \loss(f(X_{t}),g(X_{t}))$ is nondecreasing in $m$,
so that $\forall n \in \nats$, $\hat{i}_{n,m}(X_{1:m})$ is nonincreasing in $m$.  Since $\hat{i}_{n,m}(X_{1:m})$ is always positive, 
this implies $\hat{i}_{n,m}(X_{1:m})$ converges as $m \to \infty$;
in particular, since $\hat{i}_{n,m}(X_{1:m}) \in \nats$, this implies $\forall n \in \nats$, 
$\exists m_{n}^{*} \in \nats$ with $m_{n}^{*} \geq n$ such that $\forall m \geq m_{n}^{*}$, 
$\hat{i}_{n,m}(X_{1:m}) = \hat{i}_{n,m_{n}^{*}}(X_{1:m_{n}^{*}})$.
For brevity, let us define 
$\hat{i}_{n} = \hat{i}_{n,m_{n}^{*}}(X_{1:m_{n}^{*}})$.
By definition of $\hat{i}_{n,m}(X_{1:m})$, we have that every $m \geq m_{n}^{*}$ satisfies
\begin{equation*}
\max_{f,g \in \F_{\hat{i}_{n}}} \left( \max_{u_{\hat{i}_{n}} \leq s \leq m} \frac{1}{s} \sum_{t=1}^{s} \loss(f(X_{t}),g(X_{t})) - \max_{u_{\hat{i}_{n}} \leq s \leq n} \frac{1}{s} \sum_{t=1}^{s} \loss(f(X_{t}),g(X_{t})) \right) \leq \gamma_{\hat{i}_{n}}.
\end{equation*}
Taking the limiting case as $m \to \infty$, together with monotonicity of the $\max$ function, this implies
\begin{equation}
\label{eqn:adaptive-hati-gamma-bound}
\max_{f,g \in \F_{\hat{i}_{n}}} \left( \sup_{u_{\hat{i}_{n}} \leq s < \infty} \frac{1}{s} \sum_{t=1}^{s} \loss(f(X_{t}),g(X_{t})) - \max_{u_{\hat{i}_{n}} \leq s \leq n} \frac{1}{s} \sum_{t=1}^{s} \loss(f(X_{t}),g(X_{t})) \right)
\leq \gamma_{\hat{i}_{n}}.
\end{equation}
Furthermore, for each $i \in \nats$, since $\F_{i}$ is finite, continuity of the $\max$ function implies
\begin{align*}
& \limsup_{n \to \infty} \max_{f,g \in \F_{i}} \left( \max_{u_{i} \leq s \leq m_{n}^{*}} \frac{1}{s} \sum_{t=1}^{s} \loss(f(X_{t}),g(X_{t})) - \max_{u_{i} \leq s \leq n} \frac{1}{s} \sum_{t=1}^{s} \loss(f(X_{t}),g(X_{t})) \right)
\\ & \leq \limsup_{n \to \infty} \max_{f,g \in \F_{i}} \left( \max_{u_{i} \leq s < \infty} \frac{1}{s} \sum_{t=1}^{s} \loss(f(X_{t}),g(X_{t})) - \max_{u_{i} \leq s \leq n} \frac{1}{s} \sum_{t=1}^{s} \loss(f(X_{t}),g(X_{t})) \right)
\\ & = \max_{f,g \in \F_{i}} \left( \max_{u_{i} \leq s < \infty} \frac{1}{s} \sum_{t=1}^{s} \loss(f(X_{t}),g(X_{t})) - \lim_{n \to \infty} \max_{u_{i} \leq s \leq n} \frac{1}{s} \sum_{t=1}^{s} \loss(f(X_{t}),g(X_{t})) \right) 
= 0 < \gamma_{i}.
\end{align*}
Together with finiteness of every $u_{i}$, this implies 
\begin{equation}
\label{eqn:adaptive-hati-limit}
\lim_{n \to \infty} \hat{i}_{n} = \infty.
\end{equation}

Next note that, by our choices of the sequences $\{\F_{i}\}_{i=1}^{\infty}$ and $\{u_{i}\}_{i=1}^{\infty}$,
Lemma~\ref{lem:exponential-approximating-sequence} implies that
there exists a (nonrandom) sequence $\{\target_{i}\}_{i=1}^{\infty}$, with $\target_{i} \in \F_{i}$ for each $i \in \nats$,
a (nonrandom) sequence $\{\alpha_{i}\}_{i=1}^{\infty}$ in $(0,\infty)$ with $\alpha_{i} \to 0$,
and an event $K$ of probability one, on which $\exists \iota_{0} \in \nats$ such that 
$\forall i \geq \iota_{0}$,
\begin{equation*}
\sup_{u_{i} \leq s < \infty} \frac{1}{s} \sum_{t=1}^{s} \loss(\target_{i}(X_{t}),\target(X_{t})) \leq \alpha_{i}.
\end{equation*}
In particular, since $\lim\limits_{n \to \infty} \hat{i}_{n} = \infty$ by \eqref{eqn:adaptive-hati-limit}, this implies that, on the event $K$,
$\exists \nu_{0} \in \nats$ such that $\forall n \geq \nu_{0}$, we have $\hat{i}_{n} \geq \iota_{0}$, so that
the above implies
\begin{equation}
\label{eqn:adaptive-targeti-bound}
\sup_{u_{\hat{i}_{n}} \leq s < \infty} \frac{1}{s} \sum_{t=1}^{s} \loss(\target_{\hat{i}_{n}}(X_{t}),\target(X_{t})) \leq \alpha_{\hat{i}_{n}}.
\end{equation}

For brevity, for every $n,m \in \nats$ with $m \geq n$, define
$\hat{g}_{n,m}(\cdot) = \hat{f}_{n,m}(X_{1:m},\target(X_{1:n}),\cdot)$.
Since every $m \geq m_{n}^{*}$ has $\hat{i}_{n,m}(X_{1:m}) = \hat{i}_{n,m_{n}^{*}}(X_{1:m_{n}^{*}})$, the definition of $\hat{f}_{n,m}$
implies that any $m \geq m_{n}^{*}$ also has $\hat{g}_{n,m} = \hat{g}_{n,m_{n}^{*}}$ 
(recalling that ties are broken in the $\argmin$ based on a fixed ordering).
Define $\hat{g}_{n} = \hat{g}_{n,m_{n}^{*}}$.
Combining the definition of $\hat{f}_{n,m_{n}^{*}}$ with \eqref{eqn:adaptive-targeti-bound}
we have that, on the event $K$, $\forall n \in \nats$ with $n \geq \nu_{0}$,
\begin{align}
\max_{u_{\hat{i}_{n}} \leq s \leq n} \frac{1}{s} \sum_{t=1}^{s} \loss(\hat{g}_{n}(X_{t}),\target(X_{t}))
& \leq \max_{u_{\hat{i}_{n}} \leq s \leq n} \frac{1}{s} \sum_{t=1}^{s} \loss(\target_{\hat{i}_{n}}(X_{t}),\target(X_{t})) \notag
\\ & \leq \sup_{u_{\hat{i}_{n}} \leq s < \infty} \frac{1}{s} \sum_{t=1}^{s} \loss(\target_{\hat{i}_{n}}(X_{t}),\target(X_{t}))
\leq \alpha_{\hat{i}_{n}}. \label{eqn:adaptive-hatf-empirical-bound}
\end{align}

Now suppose the event $K$ occurs and fix any $n \in \nats$ with $n \geq \nu_{0}$.
The relaxed triangle inequality and subadditivity of the supremum imply 
\begin{align}
& \sup_{u_{\hat{i}_{n}} \leq s < \infty} \frac{1}{s} \sum_{t=1}^{s} \loss(\hat{g}_{n}(X_{t}),\target(X_{t})) \notag
\\ & \leq \triconst \sup_{u_{\hat{i}_{n}} \leq s < \infty} \frac{1}{s} \sum_{t=1}^{s} \loss(\hat{g}_{n}(X_{t}),\target_{\hat{i}_{n}}(X_{t}))
+ \triconst \sup_{u_{\hat{i}_{n}} \leq s < \infty} \frac{1}{s} \sum_{t=1}^{s} \loss(\target_{\hat{i}_{n}}(X_{t}),\target(X_{t})). \label{eqn:self-adaptive-direct-tribound}
\end{align}
Since $\hat{g}_{n}$ and $\target_{\hat{i}_{n}}$ are both elements of $\F_{\hat{i}_{n}}$, 
\eqref{eqn:adaptive-hati-gamma-bound} implies 
\begin{equation}
\label{eqn:self-adaptive-direct-tribound-1}
\sup_{u_{\hat{i}_{n}} \leq s < \infty} \frac{1}{s} \sum_{t=1}^{s} \loss(\hat{g}_{n}(X_{t}),\target_{\hat{i}_{n}}(X_{t})) 
\leq \max_{u_{\hat{i}_{n}} \leq s \leq n} \frac{1}{s} \sum_{t=1}^{s} \loss(\hat{g}_{n}(X_{t}),\target_{\hat{i}_{n}}(X_{t})) + \gamma_{\hat{i}_{n}}.
\end{equation}
The relaxed triangle inequality and symmetry of $\loss$, together with subadditivity of the $\max$, then imply 
\begin{align*}
& \max_{u_{\hat{i}_{n}} \leq s \leq n} \frac{1}{s} \sum_{t=1}^{s} \loss(\hat{g}_{n}(X_{t}),\target_{\hat{i}_{n}}(X_{t}))
\\ & \leq \triconst \max_{u_{\hat{i}_{n}} \leq s \leq n} \frac{1}{s} \sum_{t=1}^{s} \loss(\hat{g}_{n}(X_{t}),\target(X_{t}))
+ \triconst \max_{u_{\hat{i}_{n}} \leq s \leq n} \frac{1}{s} \sum_{t=1}^{s} \loss(\target_{\hat{i}_{n}}(X_{t}),\target(X_{t})).
\end{align*}
Combining this with \eqref{eqn:self-adaptive-direct-tribound-1} and plugging into \eqref{eqn:self-adaptive-direct-tribound} yields
\begin{align*}
& \sup_{u_{\hat{i}_{n}} \leq s < \infty} \frac{1}{s} \sum_{t=1}^{s} \loss(\hat{g}_{n}(X_{t}),\target(X_{t}))
\\ & \leq 
\triconst^{2} \max_{u_{\hat{i}_{n}} \leq s \leq n} \frac{1}{s} \sum_{t=1}^{s} \loss(\hat{g}_{n}(X_{t}),\target(X_{t}))
+ \triconst (\triconst+1) \sup_{u_{\hat{i}_{n}} \leq s < \infty} \frac{1}{s} \sum_{t=1}^{s} \loss(\target_{\hat{i}_{n}}(X_{t}),\target(X_{t}))
+ \triconst \gamma_{\hat{i}_{n}}.
\end{align*}
Since the event $K$ holds and $n \geq \nu_{0}$, 
the inequalities \eqref{eqn:adaptive-hatf-empirical-bound} and \eqref{eqn:adaptive-targeti-bound} 
provide upper bounds on the first two terms above, respectively, so that altogether we have 
\begin{equation}
\label{eqn:adaptive-hatf-infty-empirical-bound}
\sup_{u_{\hat{i}_{n}} \leq s < \infty} \frac{1}{s} \sum_{t=1}^{s} \loss(\hat{g}_{n}(X_{t}),\target(X_{t})) 
\leq
\triconst (2\triconst+1) \alpha_{\hat{i}_{n}} + \triconst \gamma_{\hat{i}_{n}}.
\end{equation}

Now note that, for every $n \in \nats$, since $\hat{g}_{n,m} = \hat{g}_{n}$ for every $m \geq m_{n}^{*}$, we have 
\begin{align*}
\hat{\L}_{\ProcX}\!\left( \hat{f}_{n,\cdot}, \target ; n \right)
& = \limsup_{s \to \infty} \frac{1}{s+1} \sum_{m=n}^{n+s} \loss\!\left( \hat{g}_{n,m}(X_{m+1}), \target(X_{m+1}) \right)
\\ & \leq \limsup_{s \to \infty} \frac{1}{s+1} (m_{n}^{*}-1) \maxloss + \frac{1}{s+1} \sum_{m=m_{n}^{*}}^{n+s} \loss\!\left( \hat{g}_{n}(X_{m+1}),\target(X_{m+1}) \right)
\\ & \leq \limsup_{s \to \infty} \frac{n+s+1}{s+1} \frac{1}{n+s+1} \sum_{t=1}^{n+s+1} \loss\!\left( \hat{g}_{n}(X_{t}), \target(X_{t}) \right)
\\ & = \limsup_{s \to \infty} \frac{1}{s} \sum_{t=1}^{s} \loss\!\left( \hat{g}_{n}(X_{t}), \target(X_{t}) \right)
\leq \sup_{u_{\hat{i}_{n}} \leq s < \infty} \frac{1}{s} \sum_{t=1}^{s} \loss\!\left( \hat{g}_{n}(X_{t}), \target(X_{t}) \right).
\end{align*}
Combined with \eqref{eqn:adaptive-hatf-infty-empirical-bound}, this implies that,
on the event $K$, every $n \in \nats$ with $n \geq \nu_{0}$ satisfies
\begin{equation*}
\hat{\L}_{\ProcX}\!\left( \hat{f}_{n,\cdot}, \target ; n \right) 
\leq \triconst (2\triconst+1) \alpha_{\hat{i}_{n}} + \triconst \gamma_{\hat{i}_{n}}.
\end{equation*}
Recalling that (by their definitions) $\lim\limits_{i \to \infty} \alpha_{i} = 0$ and $\lim\limits_{i \to \infty} \gamma_{i} = 0$, 
and that 
$\lim\limits_{n \to \infty} \hat{i}_{n} = \infty$ by \eqref{eqn:adaptive-hati-limit},
we have that on the event $K$,
\begin{equation*}
\limsup_{n \to \infty} \hat{\L}_{\ProcX}\!\left( \hat{f}_{n,\cdot}, \target ; n \right) 
\leq \lim_{n \to \infty} \triconst (2\triconst+1) \alpha_{\hat{i}_{n}} + \triconst \gamma_{\hat{i}_{n}} = 0.
\end{equation*}

Since the event $K$ has probability one, and $\hat{\L}_{\ProcX}$ is nonnegative, 
this establishes that $\hat{\L}_{\ProcX}\!\left( \hat{f}_{n,\cdot}, \target ; n \right) \to 0$ (a.s.).
Since this argument applies to \emph{any} measurable $\target : \X \to \Y$, 
this establishes that $\hat{f}_{n,m}$ is strongly universally consistent under $\ProcX$.
Furthermore, since this argument applies to \emph{any} $\ProcX \in \KC$, 
and Theorem~\ref{thm:main} implies $\SUAL = \KC$, this completes the proof 
that $\hat{f}_{n,m}$ is strongly universally consistent under every $\ProcX \in \SUAL$: 
that is, $\hat{f}_{n,m}$ is optimistically universal.
\end{proof}

An immediate consequence of Theorem~\ref{thm:doubly-universal-adaptive} is that 
there \emph{exist} optimistically universal self-adaptive learning rules, so that this 
also completes the proof of Theorem~\ref{thm:optimistic-self-adaptive} stated in Section~\ref{subsec:main}.

\subsection{Nonexistence of Optimistically Universal Inductive Learning Rules}
\label{sec:no-U2-inductive}

Given the positive result above on optimistically universal self-adaptive learning, 
it is natural to wonder whether the same is true of \emph{inductive} learning.
However, it turns out this is \emph{not} the case.
In fact, we find below that there do not even exist inductive learning rules that 
are strongly universally consistent under every $\ProcX$ with \emph{convergent relative frequencies}, 
which form a proper subset of $\SUIL$ (recall the discussion in Section~\ref{sec:examples}).  
We begin with the following result (restated from Section~\ref{subsec:main}).
For technical reasons, throughout Section~\ref{sec:no-U2-inductive} we assume that $(\X,\T)$ is a Polish space;
for instance, $\reals^{p}$ satisfies this for any $p \in \nats$, under the usual Euclidean topology.\\

\noindent {\bf Theorem~\ref{thm:no-optimistic-inductive} (restated)}~ 
\textit{There \emph{does not exist} an optimistically universal inductive learning rule,
if $\X$ is uncountable.}\\

Before presenting the proof, we first have a technical lemma regarding a basic fact about nonatomic probability measures.

\begin{lemma}
\label{lem:non-atomic-half-covering}
For any nonatomic probability measure $\pi_{0}$ on $\X$, 
there exists a sequence \break
$\{R_{k}\}_{k=1}^{\infty}$ in $\Borel$ such that, 
$\forall k \in \nats$, $\pi_{0}(R_{k}) = 1/2$, and
$\forall A \in \Borel$, $\lim\limits_{k \to \infty} \pi_{0}(A \cap R_{k}) = (1/2) \pi_{0}(A)$.
\end{lemma}
\begin{proof}
Denote by $\lambda$ the Lebesgue measure on $\reals$.
First, note that since $(\X,\T)$ is a Polish space, $(\X,\Borel)$ is a \emph{standard Borel space} \citep*[in the sense of][]{srivastava:98}. 
In particular, since $\pi_{0}$ is nonatomic,
this implies that there exists a Borel isomorphism $\psi : \X \to [0,1]$ such that,
for every Borel subset $B$ of $[0,1]$, $\pi_{0}(\psi^{-1}(B)) = \lambda(B)$ \citep*[see e.g.,][Theorem 3.4.23]{srivastava:98}.

For each $k \in \nats$ and each $i \in \ints$,
define $C_{k,i} = \left[ (i-1) 2^{-k}, i 2^{-k} \right)$,
let $B_{k} = \bigcup\limits_{i \in \ints} C_{k,2i}$,
and define $R_{k} = \psi^{-1}(B_{k} \cap [0,1])$.
Note that each $B_{k} \cap [0,1]$ is a Borel subset of $[0,1]$, so that measurability of $\psi$ implies $R_{k} \in \Borel$;
furthermore, $\pi_{0}(R_{k}) = \pi_{0}(\psi^{-1}(B_{k} \cap [0,1])) = \lambda(B_{k} \cap [0,1]) = 1/2$, as required.

Now fix any set $A \in \Borel$, and let $B \subseteq [0,1]$ be the Borel subset of $[0,1]$ with $A = \psi^{-1}(B)$ (which exists by the bimeasurability property of $\psi$).
Since $\lambda$ is a \emph{regular} measure \citep*[e.g.,][Proposition 1.4.1]{cohn:80},
for any $\eps > 0$, there exists an \emph{open} set $U_{\eps}$ with $B \subseteq U_{\eps} \subseteq \reals$ such that $\lambda(U_{\eps} \setminus B) < \eps$.
As any open subset of $\reals$ is a union of countably many pairwise-disjoint open intervals \citep*[e.g.,][Section 6, Theorem 6]{kolmogorov:75},
we let $(a_{1},b_{1}),(a_{2},b_{2}),\ldots$ be a sequence of disjoint open intervals ($a_{i} \in [-\infty,\infty)$, $b_{i} \in (-\infty,\infty]$) 
with $U_{\eps} = \bigcup\limits_{i=1}^{\infty} (a_{i},b_{i})$;
for notational simplicity, we suppose this sequence is infinite, which can always be achieved by adding an infinite number of empty 
intervals $(a_{i},b_{i})$ with $a_{i} = b_{i} \in \reals$.
Since $U_{\eps} \setminus \bigcup\limits_{i=1}^{j} (a_{i},b_{i}) \downarrow \emptyset$ as $j \to \infty$, 
and since $\lambda(U_{\eps}) = \lambda(U_{\eps} \setminus B) + \lambda(B) < \eps + 1 < \infty$, 
continuity of finite measures implies 
$\lim\limits_{j \to \infty}\lambda\!\left(U_{\eps} \setminus \bigcup\limits_{i=1}^{j} (a_{i},b_{i})\right) = 0$ \citep*[e.g.,][Theorem A.19]{schervish:95}.
In particular, for any $\delta > 0$, $\exists j_{\delta} \in \nats$ such that $\lambda\!\left( U_{\eps} \setminus \bigcup\limits_{i=1}^{j_{\delta}} (a_{i},b_{i}) \right) < \delta/2$.
Let $k_{\delta} = \left\lceil \log_{2}\left( \frac{4j_{\delta}}{\delta} \right) \right\rceil$.
Since $\lambda(U_{\eps}) < \infty$, we know that every $i$ has $a_{i} > -\infty$ and $b_{i} < \infty$.
Also, letting $\bar{a}_{i} = \min\{ t 2^{-k_{\delta}} : a_{i} < t 2^{-k_{\delta}}, t \in \ints \}$
and $\bar{b}_{i} = \max\{ t 2^{-k_{\delta}} : b_{i} > t 2^{-k_{\delta}}, t \in \ints \}$,
we have that 
\begin{equation*}
\lambda\left( (a_{i},b_{i}) \setminus \bigcup\left\{ C_{k_{\delta},t} : C_{k_{\delta},t} \subseteq (a_{i},b_{i}), t \in \ints \right\} \right)
\leq |\bar{a}_{i} - a_{i}| + |b_{i} - \bar{b}_{i}|
\leq 2 \cdot 2^{-k_{\delta}}
\leq \frac{\delta}{2 j_{\delta}}.
\end{equation*}
Thus,
\begin{align}
& \lambda\!\left( U_{\eps} \setminus \bigcup\left\{ C_{k_{\delta},t} : C_{k_{\delta},t} \subseteq U_{\eps}, t \in \ints \right\} \right) \notag
\\ & \leq \lambda\!\left( U_{\eps} \setminus \bigcup_{i=1}^{j_{\delta}} (a_{i},b_{i}) \right) + \lambda\!\left( \bigcup_{i=1}^{j_{\delta}} (a_{i},b_{i}) \setminus \bigcup\left\{ C_{k_{\delta},t} : C_{k_{\delta},t} \subseteq U_{\eps}, t \in \ints \right\} \right) \notag
\\ & < \delta/2 + \sum_{i=1}^{j_{\delta}} \lambda\!\left( (a_{i},b_{i}) \setminus \bigcup\left\{ C_{k_{\delta},t} : C_{k_{\delta},t} \subseteq U_{\eps}, t \in \ints \right\} \right) \notag
\\ & \leq \delta/2 + \sum_{i=1}^{j_{\delta}} \lambda\!\left( (a_{i},b_{i}) \setminus \bigcup\left\{ C_{k_{\delta},t} : C_{k_{\delta},t} \subseteq (a_{i},b_{i}), t \in \ints \right\} \right)
\leq \delta/2 + \sum_{i=1}^{j_{\delta}} \frac{\delta}{2j_{\delta}}
= \delta. \label{eqn:kdelta-Ueps-approx}
\end{align}

Now note that, for every $k > k_{\delta}$ and $i \in \ints$, 
each $j \in \ints$ has either $C_{k,j} \subseteq C_{k_{\delta},i}$ or $C_{k,j} \cap C_{k_{\delta},i} = \emptyset$,
and moreover each $j$ has $C_{k,2j} \subseteq C_{k_{\delta},i}$ if and only if $C_{k,2j-1} \subseteq C_{k_{\delta},i}$ 
(the smallest $j$ with $C_{k,j} \subseteq C_{k_{\delta},i}$ has $(j-1)2^{-k} = (i-1)2^{-k_{\eps}}$, which 
implies $j$ is an \emph{odd} number because $k > k_{\eps}$; similarly, the largest $j$ with $C_{k,j} \subseteq C_{k_{\delta},i}$ 
has $j 2^{-k} = i 2^{-k_{\eps}}$ and is therefore \emph{even}), 
so that 
\begin{equation*}
\lambda( B_{k} \cap C_{k_{\delta},i} )
= \lambda\!\left( \bigcup \{ C_{k,2j} : C_{k,2j} \subseteq C_{k_{\delta},i}, j \in \ints \} \right) 
= (1/2)\lambda(C_{k_{\delta},i}),
\end{equation*}
and hence (by disjointness of the $C_{k_{\delta},i}$ sets) 
\begin{align*}
& \lambda\!\left( B_{k} \cap \bigcup\left\{ C_{k_{\delta},i} : C_{k_{\delta},i} \subseteq U_{\eps}, i \in \ints \right\} \right)
= \sum_{i \in \ints : C_{k_{\delta},i} \subseteq U_{\eps}} \lambda( B_{k} \cap C_{k_{\delta},i} )
\\ & = \sum_{i \in \ints : C_{k_{\delta},i} \subseteq U_{\eps}} (1/2) \lambda( C_{k_{\delta},i} )
= (1/2) \lambda\!\left( \bigcup\left\{ C_{k_{\delta},i} : C_{k_{\delta},i} \subseteq U_{\eps}, i \in \ints \right\} \right).
\end{align*}
Therefore $\forall k > k_{\delta}$, 
\begin{align}
& \lambda( U_{\eps} \cap B_{k} ) \notag 
\\ & = \lambda\!\left( B_{k} \cap \bigcup\left\{ C_{k_{\delta},i} : C_{k_{\delta},i} \subseteq U_{\eps}, i \in \ints \right\} \right)
+ \lambda\!\left( B_{k} \cap  U_{\eps} \setminus \bigcup\left\{ C_{k_{\delta},i} : C_{k_{\delta},i} \subseteq U_{\eps}, i \in \ints \right\} \right) \notag
\\ & = (1/2) \lambda\!\left( \bigcup\left\{ C_{k_{\delta},i} : C_{k_{\delta},i} \subseteq U_{\eps}, i \in \ints \right\} \right)
+ \lambda\!\left( B_{k} \cap  U_{\eps} \setminus \bigcup\left\{ C_{k_{\delta},i} : C_{k_{\delta},i} \subseteq U_{\eps}, i \in \ints \right\} \right). \label{eqn:Rk-Ueps-decomposition}
\end{align}
The first term in \eqref{eqn:Rk-Ueps-decomposition} equals 
$(1/2) \left( \lambda(U_{\eps}) - \lambda\!\left( U_{\eps} \setminus \bigcup \left\{ C_{k_{\delta},i} : C_{k_{\delta},i} \subseteq U_{\eps}, i \in \ints \right\} \right) \right)$,
which by \eqref{eqn:kdelta-Ueps-approx} is greater than $(1/2)\lambda(U_{\eps}) - \delta/2$.
Furthermore, the second term in \eqref{eqn:Rk-Ueps-decomposition}
is no smaller than $0$, and no greater than 
$\lambda\!\left( U_{\eps} \setminus \bigcup\left\{ C_{k_{\delta},i} : C_{k_{\delta},i} \subseteq U_{\eps}, i \in \ints \right\} \right)$.
Thus, 
\begin{align*}
& (1/2)\lambda(U_{\eps}) - \delta/2 
< \lambda( U_{\eps} \cap B_{k} ) 
\\ & \leq (1/2)\!\left( \lambda(U_{\eps}) \!-\! \lambda\!\left( U_{\eps} \!\setminus \bigcup\!\left\{ C_{k_{\delta},i} : C_{k_{\delta},i} \!\subseteq\! U_{\eps}, i \!\in\! \ints \right\} \right) \right)
+ \lambda\!\left( U_{\eps} \!\setminus \bigcup\!\left\{ C_{k_{\delta},i} : C_{k_{\delta},i} \!\subseteq\! U_{\eps}, i \!\in\! \ints \right\} \right) 
\\ & = (1/2)\left( \lambda(U_{\eps}) + \lambda\!\left( U_{\eps} \setminus \bigcup\left\{ C_{k_{\delta},i} : C_{k_{\delta},i} \subseteq U_{\eps}, i \in \ints \right\} \right) \right)
< (1/2) \lambda(U_{\eps}) + \delta/2, 
\end{align*}
where this last inequality is by \eqref{eqn:kdelta-Ueps-approx}.

Since this holds for every $k > k_{\delta}$, and $k_{\delta}$ is finite for every $\delta \in (0,1)$, we have $\forall \delta \in (0,1)$, 
\begin{equation*}
(1/2)\lambda(U_{\eps}) - \delta/2  \leq \liminf_{k \to \infty} \lambda( U_{\eps} \cap B_{k} )  \leq \limsup_{k \to \infty} \lambda( U_{\eps} \cap B_{k} ) \leq (1/2) \lambda(U_{\eps}) + \delta/2,
\end{equation*}
and taking the limit as $\delta \to 0$ implies
\begin{equation*}
\lim_{k \to \infty} \lambda( U_{\eps} \cap B_{k} ) = (1/2) \lambda(U_{\eps}).
\end{equation*}
This further implies that
\begin{equation*}
\limsup_{k \to \infty} \lambda( B \cap B_{k} ) 
\leq \lim_{k \to \infty} \lambda( U_{\eps} \cap B_{k} )
= (1/2) \lambda(U_{\eps})
< (1/2) \lambda(B) + \eps/2,
\end{equation*}
and
\begin{align*}
\liminf_{k \to \infty} \lambda( B \cap B_{k} )
& \geq \lim_{k \to \infty} \lambda( U_{\eps} \cap B_{k} ) - \lambda( U_{\eps} \setminus B )
= (1/2) \lambda(U_{\eps}) - \lambda( U_{\eps} \setminus B )
\\ & = (1/2) \lambda(B) - (1/2) \lambda( U_{\eps} \setminus B )
> (1/2) \lambda(B) - \eps/2.
\end{align*}
Since these inequalities hold for every $\eps > 0$, 
taking the limit as $\eps \to 0$ reveals that 
\begin{equation*}
\lim\limits_{k \to \infty} \lambda(B \cap B_{k}) = (1/2) \lambda(B).
\end{equation*}
Furthermore, since $\psi^{-1}(B) \cap \psi^{-1}(B_{k} \cap [0,1]) = \psi^{-1}( B \cap B_{k} \cap [0,1]) = \psi^{-1}( B \cap B_{k} )$ for every $k \in \nats$, this implies that 
\begin{align*}
\lim_{k \to \infty} \pi_{0}(A \cap R_{k}) 
& = \lim_{k \to \infty} \pi_{0}( \psi^{-1}(B) \cap \psi^{-1}(B_{k} \cap [0,1]) ) 
= \lim_{k \to \infty} \pi_{0}( \psi^{-1}(B \cap B_{k}) )
\\ & = \lim_{k \to \infty} \lambda(B \cap B_{k})
= (1/2) \lambda(B)
= (1/2) \pi_{0}( \psi^{-1}(B) )
= (1/2) \pi_{0}(A).
\end{align*}
Since this argument holds $\forall A \in \Borel$, this completes the proof.
\end{proof}

We are now ready for the proof of 
Theorem~\ref{thm:no-optimistic-inductive}.
The proof is partly inspired by that of a related (but somewhat different) result of \citet*{nobel:99}, based on a technique of \citet*{adams:98}.
Specifically, \citet*{nobel:99} proves that there is no learning rule converging to the stationary regression function
for all \emph{joint} processes $(\ProcX,\ProcY)$ that are stationary and ergodic.
In contrast, we are interested in learning under a fixed target function $\target$,
and as such the construction of \citet*{nobel:99} needs to be modified for our purposes. 
However, the proof below does preserve the essential elements of the cutting and stacking argument of \citet*{adams:98},
though generalized to suit our abstract setting.  While the processes $\ProcX$ we construct do not have the
property of stationarity from the original proof of \citet*{nobel:99}, they  
\emph{do} have convergent relative frequencies ($\CRF$)
and are ergodic (indeed, they are \emph{product} processes). 
Thus, this establishes the stronger fact that (when $\X$ is uncountable) there is no inductive learning rule 
that is strongly universally consistent for every ergodic $\ProcX \in \CRF$ (as stated in Corollary~\ref{cor:no-lln-consistent-inductive} below);
this suffices to establish Theorem~\ref{thm:no-optimistic-inductive} 
since
Theorems~\ref{thm:crf-implies-kc} and \ref{thm:main} imply $\CRF \subseteq \SUIL$.\\

\begin{proof}[of Theorem~\ref{thm:no-optimistic-inductive}]
Fix any inductive learning rule $f_{n}$.
We begin by constructing the process $\ProcX$.
Since $\X$ is uncountable, and $(\X,\T)$ is a Polish space,
there exists a nonatomic probability measure $\pi_{0}$ on $\X$ (with respect to $\Borel$) \citep*[see][Chapter 2, Theorem 8.1]{parthasarathy:67}. 
Furthermore, fixing any such nonatomic $\pi_{0}$, Lemma~\ref{lem:non-atomic-half-covering} implies there exists a sequence 
$\{R_{k}\}_{k=1}^{\infty}$ in $\Borel$ such that, $\forall k \in \nats$, $\pi_{0}(R_{k}) = 1/2$, and $\forall A \in \Borel$, 
$\lim\limits_{k \to \infty} \pi_{0}(A \cap R_{k}) = (1/2) \pi_{0}(A)$.
Also define $R_{0} = \emptyset$.
Define random variables $U_{k,j}$ (for all $k,j \in \nats$), $V_{k,j}$ (for all $k,j \in \nats$),  and $W_{j}$ (for all $j \in \nats$), 
all mutually independent (and independent from $\{f_{n}\}_{n\in\nats}$), with distributions specified as follows.
For each $k,j \in \nats$, $U_{k,j}$ has distribution $\pi_{0}(\cdot | \X \setminus R_{k})$, while $V_{k,j}$ has distribution $\pi_{0}(\cdot | R_{k})$.
For each $j \in \nats$, $W_{j}$ has distribution $\pi_{0}$.
Let $\mathbf{U} = \{U_{k,j}\}_{k,j \in \nats}$, $\mathbf{V} = \{V_{k,j}\}_{k,j \in \nats}$, $\mathbf{W} = \{W_{j}\}_{j \in \nats}$.

Fix any $y_{0},y_{1} \in \Y$ with $\loss(y_{0},y_{1}) > 0$, 
and define $\Delta_{\zo} = \loss(y_{0},y_{1})/(2\triconst)$.
Importantly, the near-metric properties of $\loss$ imply that 
any $y \in \Y$ with $\loss(y,y_{1}) < \Delta_{\zo}$ necessarily has 
$\loss(y,y_{0}) 
> \loss(y,y_{0}) + \loss(y,y_{1}) - \Delta_{\zo}
\geq \loss(y_{0},y_{1})/\triconst - \Delta_{\zo} 
= \Delta_{\zo}$,
where the second inequality is due to the relaxed triangle inequality and symmetry.
Thus, it is not possible to simultaneously achieve $\loss(y,y_{1}) < \Delta_{\zo}$ and $\loss(y,y_{0}) \leq \Delta_{\zo}$.

For any array $\mathbf{v} = \{v_{k,j}\}_{k,j \in \nats}$, 
and any $K \in \nats$, define $\mathbf{v}_{< K} = \{v_{k,j}\}_{k,j \in \nats, k < K}$,
and define $\mathbf{v}_{K} = \{v_{K,j}\}_{j \in \nats}$.
Then, for any arrays $\mathbf{u} = \{u_{k,j}\}_{k,j \in \nats}$ and $\mathbf{v} = \{v_{k,j}\}_{k,j \in \nats}$ in $\X$,
any sequence $\mathbf{w} = \{w_{j}\}_{j \in \nats}$ in $\X$,
and any $K \in \nats$, define
\begin{equation*}
\target_{K}(x; \mathbf{u}_{< K}, \mathbf{v}_{< K},\mathbf{w}) = 
\begin{cases}
y_{0}, & \text{ if } x \in (\mathbf{v}_{<K} \cup R_{K}) \setminus (\mathbf{w} \cup \mathbf{u}_{<K}) \\
y_{1}, & \text{ otherwise}
\end{cases}
\end{equation*}
and
\begin{equation*}
\target_{0}(x; \mathbf{v}) = 
\begin{cases}
y_{0}, & \text{ if } x \in \mathbf{v} \\
y_{1}, & \text{ otherwise}
\end{cases},
\end{equation*}
where, for notational simplicity, in these definitions we treat
$\mathbf{v}_{<K}$, $\mathbf{w}$, $\mathbf{u}_{<K}$, $\mathbf{v}$ 
as the \emph{sets} of the distinct values in the respective arrays.
Note that the above functions are measurable, since each $R_{K}$ is measurable, 
and $\mathbf{v}$, $\mathbf{v}_{<K}$, $\mathbf{w}$, $\mathbf{u}_{<K}$ are all countable 
and hence measurable (recalling that singleton sets $\{x\}$ are closed, hence measurable).

Now, for any $\mathbf{u}$, $\mathbf{v}$, $\mathbf{w}$ as above, inductively define
values  
$X_{i}^{(k)}(\mathbf{u}_{< k},\!\mathbf{u}_{k},\!\mathbf{v}_{< k},\!\mathbf{v}_{k},\!\mathbf{w})$
as follows.
Let $n_{0} = 0$.  For this inductive definition, suppose that for some $k \in \nats$
the value $n_{k-1} \in \nats$ and the values $\{X_{i}^{(k-1)}(\mathbf{u}_{< k-1},\mathbf{u}_{k-1},\mathbf{v}_{< k-1},\mathbf{v}_{k-1},\mathbf{w}) : i \in \nats, i \leq n_{k-1}\}$ are already defined 
(taking this to be trivially satisfied in the case $k=1$, wherein this is an empty sequence).
For each $i \in \nats$ with $i \leq n_{k-1}$, define $\tilde{X}_{i}^{(k)}(\mathbf{u}_{< k},\mathbf{u}_{k},\mathbf{v}_{< k},\mathbf{v}_{k},\mathbf{w})$ and
$X_{i}^{(k)}(\mathbf{u}_{< k},\mathbf{u}_{k},\mathbf{v}_{< k},\mathbf{v}_{k},\mathbf{w})$ both equal to 
$X_{i}^{(k-1)}(\mathbf{u}_{<k-1},\mathbf{u}_{k},\mathbf{v}_{<k-1},\mathbf{v}_{k-1},\mathbf{w})$.
Then, for each $i \in \nats$, define $\tilde{X}_{n_{k-1} + k (i-1) + 1}^{(k)}(\mathbf{u}_{<k},\mathbf{u}_{k},\mathbf{v}_{<k},\mathbf{v}_{k},\mathbf{w}) = v_{k,n_{k-1} + k (i-1) + 1}$, 
and for each $j \in \nats$ with $2 \leq j \leq k$, define $\tilde{X}_{n_{k-1} + k(i-1) + j}^{(k)}(\mathbf{u}_{<k},\mathbf{u}_{k},\mathbf{v}_{<k},\mathbf{v}_{k},\mathbf{w}) = u_{k,n_{k-1} + k(i-1) + j}$.
To simplify notation, for each $i \in \nats$, abbreviate $\hat{X}_{i}^{(k)} = \tilde{X}_{i}^{(k)}(\mathbf{U}_{<k},\mathbf{U}_{k},\mathbf{V}_{<k},\mathbf{V}_{k},\mathbf{W})$.
If $\exists n \in \nats$ with $n > n_{k-1}$ such that 
\begin{equation}
\label{eqn:no-suil-u2-prob}
\P\!\left( \pi_{0}\!\left( \left\{ x : \loss\!\left(f_{n}\!\left(\hat{X}_{1:n}^{(k)}, \target_{k}\!\left(\hat{X}_{1:n}^{(k)};\mathbf{U}_{\!<k},\mathbf{V}_{\!<k},\mathbf{W}\right), x \right),y_{0}\right) \geq \Delta_{\zo} \right\} \right) \geq 3/4 \right) < 2^{-k},
\end{equation}
then fix the minimum such $n$, 
and $\forall i \in \{n_{k-1}+1,\ldots,n\}$ define $X_{i}^{(k)}(\mathbf{u}_{<k},\mathbf{u}_{k},\mathbf{v}_{<k},\mathbf{v}_{k},\mathbf{w}) = \tilde{X}_{i}^{(k)}(\mathbf{u}_{<k},\mathbf{u}_{k},\mathbf{v}_{<k},\mathbf{v}_{k},\mathbf{w})$.
Furthermore, for each $i \in \nats$ with $n+1 \leq i \leq n^{2}$,
define $X_{i}^{(k)}(\mathbf{u}_{<k},\mathbf{u}_{k},\mathbf{v}_{<k},\mathbf{v}_{k},\mathbf{w}) \!=\! w_{i}$.
Finally, define $n_{k} \!=\! n^{2}$.
Otherwise, if no such $n$ satisfies \eqref{eqn:no-suil-u2-prob}, then $\forall i \in \nats$ with $i > n_{k-1}$, define 
$X_{i}^{(k)}(\mathbf{u}_{<k},\mathbf{u}_{k},\mathbf{v}_{<k},\mathbf{v}_{k},\mathbf{w}) = \tilde{X}_{i}^{(k)}(\mathbf{u}_{<k},\mathbf{u}_{k},\mathbf{v}_{<k},\mathbf{v}_{k},\mathbf{w})$,
in which case the inductive definition is complete (upon reaching the smallest value of $k$ for which no such $n$ exists).
Note that, since we do not condition on any variables in \eqref{eqn:no-suil-u2-prob}, the values $n_{k}$ are \emph{not} random.

Now we consider two cases.  First, suppose there is a maximum value $k^{*}$ of $k \in \nats$ for which $n_{k-1}$ is defined.
In this case, $\nexists n \in \nats$ with $n > n_{k^{*}-1}$ satisfying \eqref{eqn:no-suil-u2-prob} with $k = k^{*}$,
and furthermore 
$X_{i}^{(k^{*})}(\mathbf{u}_{<k^{*}},\mathbf{u}_{k^{*}},\mathbf{v}_{<k^{*}},\mathbf{v}_{k^{*}},\mathbf{w}) = \tilde{X}_{i}^{(k^{*})}(\mathbf{u}_{<k^{*}},\mathbf{u}_{k^{*}},\mathbf{v}_{<k^{*}},\mathbf{v}_{k^{*}},\mathbf{w})$ for every $i \in \nats$, 
and every $\mathbf{u}$, $\mathbf{v}$, and $\mathbf{w}$.
Next note that, by the law of total probability and basic limit theorems for probabilities 
(e.g., based on Fatou's lemma),
defining $\mathbf{Q}_{k^{*}} = (\mathbf{U}_{<k^{*}},\mathbf{V}_{<k^{*}},\mathbf{W})$,
\begin{align*}
& \E\!\left[ \P\!\left( \limsup_{n \to \infty} \!\left\{ \pi_{0}\!\left( \left\{ x \!:\! \loss\!\left(f_{n}\!\left(\hat{X}_{1:n}^{(k^{*})}\!\!, \target_{k^{*}}\!\!\left(\hat{X}_{1:n}^{(k^{*})};\mathbf{Q}_{k^{*}} \!\right)\!, x \right)\!,y_{0}\right) \!\!\geq\! \Delta_{\zo} \right\} \right) \!\!\geq\! 3/4 \right\} \Big| \mathbf{Q}_{k^{*}} \right) \right]
\\ & = \P\!\left( \limsup_{n \to \infty} \left\{ \pi_{0}\!\left( \left\{ x : \loss\!\left(f_{n}\!\left(\hat{X}_{1:n}^{(k^{*})}, \target_{k^{*}}\!\left(\hat{X}_{1:n}^{(k^{*})};\mathbf{Q}_{k^{*}}\right), x \right),y_{0}\right) \geq \Delta_{\zo} \right\} \right) \geq 3/4 \right\} \right)
\\ & \geq \limsup_{n \to \infty} \P\!\left( \pi_{0}\!\left( \left\{ x : \loss\!\left(f_{n}\!\left(\hat{X}_{1:n}^{(k^{*})}, \target_{k^{*}}\!\left(\hat{X}_{1:n}^{(k^{*})};\mathbf{Q}_{k^{*}} \right), x \right),y_{0}\right) \geq \Delta_{\zo} \right\} \right) \geq 3/4 \right).
\end{align*}
The negation of \eqref{eqn:no-suil-u2-prob} implies this last expression is at least $2^{-k^{*}}$ (noting that the negation of \eqref{eqn:no-suil-u2-prob} holds for \emph{every} $n > n_{k^{*}-1}$ in the present case).
In particular, since the $U_{k,j}$,$V_{k^{\prime},j^{\prime}}$, and $W_{j^{\prime\prime}}$ variables are all independent,
this implies $\exists \mathbf{u},\mathbf{v},\mathbf{w}$ such that,
taking $X_{i} = X_{i}^{(k^{*})}(\mathbf{u}_{<k^{*}},\mathbf{U}_{k^{*}},\mathbf{v}_{<k^{*}},\mathbf{V}_{k^{*}},\mathbf{w})$ for every $i \in \nats$,
and $\target(\cdot) = \target_{k^{*}}(\cdot;\mathbf{u}_{<k^{*}},\mathbf{v}_{<k^{*}},\mathbf{w})$,
we have 
\begin{equation*}
\P\!\left( \limsup_{n \to \infty} \left\{ \pi_{0}\!\left( \left\{ x : \loss\!\left(f_{n}\!\left(X_{1:n}, \target(X_{1:n}), x \right),y_{0}\right) \geq \Delta_{\zo} \right\} \right) \geq 3/4 \right\} \right) \geq 2^{-k^{*}}.
\end{equation*}
Define the event
\begin{equation*}
E^{\prime} = \left\{ \limsup_{n\to\infty} \pi_{0}\!\left( \left\{ x \in R_{k^{*}} : \loss\!\left(f_{n}\!\left(X_{1:n}, \target(X_{1:n}), x\right),y_{0}\right) \geq \Delta_{\zo} \right\} \right) \geq 1/4 \right\}.
\end{equation*}
Since $\pi_{0}(R_{k^{*}}) = 1/2$, 
we have that
\begin{align*}
& \limsup_{n \to \infty} \left\{ \pi_{0}\!\left( \left\{ x : \loss\!\left(f_{n}\!\left(X_{1:n}, \target(X_{1:n}), x\right),y_{0}\right) \geq \Delta_{\zo} \right\} \right) \geq 3/4 \right\}
\\ & \subseteq \limsup_{n \to \infty} \left\{ \pi_{0}\!\left( \left\{ x \in R_{k^{*}} : \loss\!\left(f_{n}\!\left(X_{1:n}, \target(X_{1:n}), x\right),y_{0}\right) \geq \Delta_{\zo} \right\} \right) \geq 1/4 \right\} \subseteq E^{\prime},
\end{align*}
so that $E^{\prime}$ has probability at least $2^{-k^{*}}$.
Also let $E$ denote the event that $\forall k,j \in \nats$, $V_{k,j} \notin \{w_{j^{\prime}} : j^{\prime} \in \nats\} \cup \{u_{k^{\prime},j^{\prime}} : k^{\prime},j^{\prime} \in \nats\}$;
note that, since $\pi_{0}$ is nonatomic, and hence so is each $\pi_{0}(\cdot|R_{k})$ (since $\pi_{0}(R_{k}) > 0$), $E$ has probability one.

Define $t_{i} = n_{k^{*}-1} + k^{*}(i-1) + 1$ for each $i \in \nats$, and let $I_{k^{*}} = \{ t_{i} : i \in \nats \}$.
Note that, since every $V_{k^{*},j}$ is in $R_{k^{*}}$ and every $t \in I_{k^{*}}$ has $X_{t} = V_{k^{*},t}$ (by definition),
on the event $E$, every $t \in I_{k^{*}}$ has $\target(X_{t}) = y_{0}$ (by definition of $\target$).
Therefore, on the event $E$, every $n \in \nats$ with $n > n_{k^{*}-1}$ has
\begin{align*}
\hat{\L}_{\ProcX}(f_{n},\target;n)
 & \geq \limsup_{m \to \infty} \frac{1}{m} \sum_{t=n+1}^{n+m} \ind_{I_{k^{*}}}(t) \loss\!\left(f_{n}\!\left(X_{1:n},\target(X_{1:n}),X_{t}\right),y_{0}\right).
\end{align*}
Since $k^{*} \sum\limits_{t=n+1}^{n+m} \ind_{I_{k^{*}}}(t) > m - 2 k^{*}$, letting $i_{n} = \max\{ i \in \nats : t_{i} \leq n \}$,
the right hand side above is at least as large as
\begin{align*}
& \limsup_{s \to \infty} \frac{1}{k^{*} s + 2 k^{*}} \sum_{j=1}^{s} \loss\!\left(f_{n}\!\left(X_{1:n},\target(X_{1:n}),X_{t_{i_{n} + j}}\right),y_{0}\right)
\\ & = \limsup_{s \to \infty} \frac{1}{k^{*} s} \sum_{j=1}^{s} \loss\!\left(f_{n}\!\left(X_{1:n},\target(X_{1:n}),X_{t_{i_{n}+j}}\right),y_{0}\right)
\\ & \geq \limsup_{s \to \infty} \frac{1}{k^{*} s} \sum_{j=1}^{s} \ind\!\left[\loss\!\left(f_{n}\!\left(X_{1:n},\target(X_{1:n}),X_{t_{i_{n}+j}}\right),y_{0}\right) \geq \Delta_{\zo}\right] \Delta_{\zo}.
\end{align*}
Furthermore, the subsequence $\{ X_{t_{i_{n}+j}} \}_{j=1}^{\infty}$ is a sequence of independent random variables with distribution $\pi_{0}(\cdot|R_{k^{*}})$ (namely, a subsequence of $\mathbf{V}_{k^{*}}$),
also independent from the rest of the sequence $\{ X_{t} : t \notin \{ t_{i_{n}+j} : j \in \nats \} \}$ and $f_{n}$.
This implies that 
\begin{equation*}
\left\{ \ind\!\left[\loss\!\left(f_{n}\!\left(X_{1:n},\target(X_{1:n}),X_{t_{i_{n}+j}}\right),y_{0}\right) \geq \Delta_{\zo} \right] \right\}_{j=1}^{\infty}
\end{equation*}
is a sequence of 
conditionally i.i.d.\ ${\rm Bernoulli}$ random variables (given $X_{1:n}$ and $f_{n}$).
Thus, $\forall n \in \nats$ with $n > n_{k^{*}-1}$, by the strong law of large numbers (applied under the conditional distribution given $X_{1:n}$ and $f_{n}$)
and the law of total probability, there is an event $E_{n}^{\prime\prime}$ of probability one such that, on $E \cap E_{n}^{\prime\prime}$, 
\begin{align*}
& \limsup_{s \to \infty} \frac{1}{k^{*} s} \sum_{j=1}^{s} \ind\!\left[\loss\!\left(f_{n}\!\left(X_{1:n},\target(X_{1:n}),X_{t_{i_{n}+j}}\right),y_{0}\right) \geq \Delta_{\zo}\right] \Delta_{\zo}
\\ & = \frac{\Delta_{\zo}}{k^{*}} \pi_{0}\!\left( \left\{ x : \loss\!\left( f_{n}\!\left(X_{1:n},\target(X_{1:n}),x \right),y_{0} \right) \geq \Delta_{\zo} \right\} \Big| R_{k^{*}} \right)
\\ & = \frac{2\Delta_{\zo}}{k^{*}} \pi_{0}\!\left( \left\{ x \in R_{k^{*}} : \loss\!\left( f_{n}\!\left(X_{1:n},\target(X_{1:n}),x \right),y_{0} \right) \geq \Delta_{\zo} \right\} \right).
\end{align*}
Altogether, we have that on the event $E \cap E^{\prime} \cap \bigcap\limits_{n > n_{k^{*}-1}} E_{n}^{\prime\prime}$,
\begin{align*}
& \limsup_{n \to \infty} \hat{\L}_{\ProcX}(f_{n},\target;n) 
\\ & \geq \frac{2\Delta_{\zo}}{k^{*}}  \limsup_{n \to \infty} \pi_{0}\!\left( \left\{ x \in R_{k^{*}} : \loss\!\left( f_{n}\!\left(X_{1:n},\target(X_{1:n}),x \right),y_{0} \right) \geq \Delta_{\zo} \right\} \right)
\geq \frac{\Delta_{\zo}}{2 k^{*}}. 
\end{align*}
Since $\frac{\Delta_{\zo}}{2 k^{*}} > 0$, and since $E \cap E^{\prime} \cap \bigcap\limits_{n > n_{k^{*}-1}} E_{n}^{\prime\prime}$ 
has probability at least $2^{-k^{*}} > 0$ (by the union bound), 
this implies that $f_{n}$ is not strongly universally consistent under the process $\ProcX$ defined here. 

To complete this first case, we argue that $\ProcX \in \SUIL$; in fact, we will show the stronger claim that $\ProcX \in \CRF$.
Note that for every $t > n_{k^{*}-1}$, if $t - n_{k^{*}-1} - 1$ is an integer multiple of $k^{*}$, then $t \in I_{k^{*}}$, 
in which case $X_{t} = V_{k^{*},t}$, and otherwise $X_{t} = U_{k^{*},t}$.
Thus, since all $X_{t}$ are independent, all $V_{k^{*},t}$ are identically distributed, and all $U_{k^{*},t}$ are identically distributed, 
we have that for any $n > n_{k^{*}-1}$, 
$\{X_{t}\}_{t=n}^{\infty}$ and $\{X_{t}\}_{t=n+k^{*}}^{\infty}$ have identical distributions.
Thus, the process $\{X_{t}\}_{t = n_{k^{*}-1}+1}^{\infty}$ is $k^{*}$-stationary \citep*[see][Section 5.10]{gray:09}, 
and hence also asymptotically mean stationary (recall the definition from Section~\ref{sec:examples}).  
Since $\{X_{t}\}_{t = n_{k^{*}-1}+1}^{\infty}$ differs from $\ProcX$ 
only by removing an initial finite segment, this immediately implies $\ProcX$ is also asymptotically mean stationary.
Thus, since (as discussed in Section~\ref{sec:examples} above) 
Theorem 8.1 of \citet*{gray:09} implies that every asymptotically mean stationary process has convergent relative frequencies,
and Theorem~\ref{thm:crf-implies-kc} of Section~\ref{sec:examples} establishes that $\CRF \subseteq \KC$, 
we have that $\ProcX \in \KC$, and since Theorem~\ref{thm:main} establishes that $\SUIL = \KC$, 
this implies $\ProcX \in \SUIL$.  Therefore, in this first case, we conclude that the inductive learning rule $f_{n}$ 
is not optimistically universal.

Next, let us examine the second case, wherein $n_{k}$ is defined for every $k \in \nats \cup \{0\}$, 
so that $\{n_{k}\}_{k=0}^{\infty}$ is an infinite increasing sequence of nonnegative integers.
In this case, for every $k \in \nats$, \eqref{eqn:no-suil-u2-prob} and the definition of $n_{k}$ imply that, 
defining $\mathbf{Q}_{k} = (\mathbf{U}_{<k},\mathbf{V}_{<k},\mathbf{W})$,
\begin{equation*}
\P\!\left( \pi_{0}\!\left( \left\{ x : \loss\!\left(f_{\sqrt{n_{k}}}\!\left(\hat{X}_{1:\sqrt{n_{k}}}^{(k)}, \target_{k}\!\left(\hat{X}_{1:\sqrt{n_{k}}}^{(k)};\mathbf{Q}_{k}\right), x \right),y_{0}\right) \geq \Delta_{\zo} \right\} \right) \geq 3/4 \right) < 2^{-k}.
\end{equation*}
By the monotone convergence theorem and linearity of expectations, combined with the law of total probability, this implies 
\begin{align*}
& \E\!\left[ \sum_{k=1}^{\infty} 
\P\!\left( \pi_{0}\!\left( \left\{ x \!: \loss\!\left(f_{\!\sqrt{n_{k}}}\!\left(\hat{X}_{1:\sqrt{n_{k}}}^{(k)}, \target_{k}\!\left(\hat{X}_{1:\sqrt{n_{k}}}^{(k)};\mathbf{Q}_{k}\right), x \right),y_{0}\right) \geq \Delta_{\zo} \right\} \right) \geq 3/4 \Big| \mathbf{V} \right) \right]
\\ & = \sum_{k=1}^{\infty} 
\P\!\left( \pi_{0}\!\left( \left\{ x \!: \loss\!\left(f_{\!\sqrt{n_{k}}}\!\left(\hat{X}_{1:\sqrt{n_{k}}}^{(k)}, \target_{k}\!\left(\hat{X}_{1:\sqrt{n_{k}}}^{(k)};\mathbf{Q}_{k}\right)\!, x \right)\!,y_{0}\right) \geq \Delta_{\zo} \right\} \right) \geq 3/4 \right)
< 1.
\end{align*}
In particular, this implies that with probability one,
\begin{equation*}
 \sum_{k=1}^{\infty} 
\P\!\left( \pi_{0}\!\left( \left\{ x \!:\! \loss\!\left(f_{\!\sqrt{n_{k}}}\!\left(\hat{X}_{1:\sqrt{n_{k}}}^{(k)}, \target_{k}\!\left(\hat{X}_{1:\sqrt{n_{k}}}^{(k)};\mathbf{Q}_{k}\right)\!, x \right)\!,y_{0}\right) \!\geq\! \Delta_{\zo} \right\} \right) \geq 3/4 \Big| \mathbf{V} \right) < \infty.
\end{equation*} 
Since $\mathbf{U}$, $\mathbf{W}$, and $\{f_{n}\}_{n\in\nats}$ are independent from $\mathbf{V}$, 
and since every $k,j \in \nats$ has $V_{k,j}$ with distribution $\pi_{0}(\cdot | R_{k})$ and hence $V_{k,j} \in R_{k}$,
this implies $\exists \mathbf{v}$
with $v_{k,j} \in R_{k}$ for every $k,j \in \nats$, 
such that, defining $X_{i} = X_{i}^{(k)}(\mathbf{U}_{<k},\mathbf{U}_{k},\mathbf{v}_{<k},\mathbf{v}_{k},\mathbf{W})$ for every $k \in \nats$ and $i \in \{n_{k-1}+1,\ldots,n_{k}\}$, 
\begin{equation*}
 \sum_{k=1}^{\infty} 
\P\!\left( \pi_{0}\!\left( \left\{ x \!:\! \loss\!\left(f_{\!\sqrt{n_{k}}}\!\left(X_{1:\sqrt{n_{k}}}, \target_{k}\!\left(X_{1:\sqrt{n_{k}}};\mathbf{U}_{\!<k},\!\mathbf{v}_{\!<k},\!\!\mathbf{W}\right)\!, x \right)\!,y_{0}\right) \!\geq\! \Delta_{\zo} \right\} \right) \!\geq\! 3/4 \right) \!<\! \infty.
\end{equation*}
The Borel-Cantelli Lemma then implies that there exists an event $H^{\prime}$ of probability one, on which $\exists k_{0} \in \nats$ such that, $\forall k \in \nats$ with $k > k_{0}$, 
\begin{equation*}
\pi_{0}\!\left( \left\{ x : \loss\!\left(f_{\sqrt{n_{k}}}\!\left(X_{1:\sqrt{n_{k}}}, \target_{k}\!\left(X_{1:\sqrt{n_{k}}};\mathbf{U}_{<k},\mathbf{v}_{<k},\mathbf{W}\right), x \right),y_{0}\right) \geq \Delta_{\zo} \right\} \right) < 3/4.
\end{equation*}

Next, let $H$ denote the event that $\{ W_{j} : j \in \nats \} \cap \{v_{k,j} : k,j \in \nats\} = \emptyset$ and $\{ U_{k,j} : k,j \in \nats \} \cap \{v_{k,j} : k,j \in \nats\} = \emptyset$.
Note that, since $\pi_{0}$ is nonatomic, and so is $\pi_{0}(\cdot | \X \setminus R_{k})$ for every $k \in \nats$, $H$ has probability one.
Furthermore, for every $k \in \nats$, by definition of $\target_{k}$, 
$\forall j \in \nats$, $\target_{k}(W_{j}; \mathbf{U}_{<k}, \mathbf{v}_{<k}, \mathbf{W}) = y_{1}$,
and $\forall k^{\prime},j \in \nats$ with $k^{\prime} < k$, $\target_{k}(U_{k^{\prime},j}; \mathbf{U}_{<k}, \mathbf{v}_{<k}, \mathbf{W}) = y_{1}$.
Also, for every $j \in \nats$, the distribution of $U_{k,j}$ is $\pi_{0}(\cdot | \X \setminus R_{k})$, and therefore we have $U_{k,j} \notin R_{k}$;
together with the definition of $\target_{k}$, this implies $\target_{k}(U_{k,j};\mathbf{U}_{<k},\mathbf{v}_{<k},\mathbf{W}) = y_{1}$ on the event $H$.
The definition of $\target_{k}$ further implies that, on $H$, for every $k^{\prime},k,j \in \nats$ with $k^{\prime} < k$, 
$\target_{k}(v_{k^{\prime},j};\mathbf{U}_{<k},\mathbf{v}_{<k},\mathbf{W}) = y_{0}$.  Also, since $v_{k,j} \in R_{k}$ for every $k,j \in \nats$,
on the event $H$, every $k,j \in \nats$ has $\target_{k}(v_{k,j};\mathbf{U}_{<k},\mathbf{v}_{<k},\mathbf{W}) = y_{0}$.
Furthermore, by definition of $\target_{0}$, every $k,j \in \nats$ has $\target_{0}(v_{k,j};\mathbf{v}) = y_{0}$, 
and on the event $H$, every $j \in \nats$ has $\target_{0}(W_{j};\mathbf{v}) = y_{1}$, and $\forall k,j \in \nats$, $\target_{0}(U_{k,j};\mathbf{v}) = y_{1}$.
Altogether we have that, on the event $H$, every $k^{\prime},k,j \in \nats$ with $k^{\prime} \leq k$ has $\target_{k}(v_{k^{\prime},j};\mathbf{U}_{<k},\mathbf{v}_{<k},\mathbf{W}) = \target_{0}(v_{k^{\prime},j};\mathbf{v})$,
$\target_{k}(U_{k^{\prime},j};\mathbf{U}_{<k},\mathbf{v}_{<k},\mathbf{W}) = \target_{0}(U_{k^{\prime},j};\mathbf{v})$, and $\target_{k}(W_{j};\mathbf{U}_{<k},\mathbf{v}_{<k},\mathbf{W}) = \target_{0}(W_{j};\mathbf{v})$.
In particular, note that for any $k \in \nats$, every $i \in \left\{1,\ldots,\sqrt{n_{k}}\right\}$ has $X_{i} \in \{ v_{k^{\prime},i} : k^{\prime} \leq k \} \cup \{ U_{k^{\prime},i} : k^{\prime} \leq k \} \cup \{ W_{i} \}$,
so that, on the event $H$, $\target_{k}(X_{i};\mathbf{U}_{<k},\mathbf{v}_{<k},\mathbf{W}) = \target_{0}(X_{i};\mathbf{v})$.
Thus, taking $\target(\cdot) = \target_{0}(\cdot;\mathbf{v})$, 
on the event $H \cap H^{\prime}$, $\forall k \in \nats$ with $k > k_{0}$, 
\begin{equation*}
\pi_{0}\!\left( \left\{ x : \loss\!\left(f_{\sqrt{n_{k}}}\!\left(X_{1:\sqrt{n_{k}}}, \target(X_{1:\sqrt{n_{k}}}), x \right),y_{0}\right) \geq \Delta_{\zo} \right\} \right) < 3/4.
\end{equation*}

As mentioned above, the near-metric properties of $\loss$ imply that 
any $y \in \Y$ with $\loss(y,y_{1}) < \Delta_{\zo}$ necessarily has 
$\loss(y,y_{0}) > \Delta_{\zo}$.
Therefore, on $H \cap H^{\prime}$, $\forall k \in \nats$ with $k > k_{0}$, 
\begin{align}
& \pi_{0}\!\left( \left\{ x : \loss\!\left(f_{\sqrt{n_{k}}}\!\left(X_{1:\sqrt{n_{k}}}, \target(X_{1:\sqrt{n_{k}}}), x \right),y_{1}\right) \geq \Delta_{\zo} \right\} \right) \notag
\\ & = 1 - \pi_{0}\!\left( \left\{ x : \loss\!\left(f_{\sqrt{n_{k}}}\!\left(X_{1:\sqrt{n_{k}}}, \target(X_{1:\sqrt{n_{k}}}), x \right),y_{1}\right) < \Delta_{\zo} \right\} \right) \notag
\\ & \geq 1 - \pi_{0}\!\left( \left\{ x : \loss\!\left(f_{\sqrt{n_{k}}}\!\left(X_{1:\sqrt{n_{k}}}, \target(X_{1:\sqrt{n_{k}}}), x \right),y_{0}\right) > \Delta_{\zo} \right\} \right) > 1/4. \label{eqn:no-suil-u2-pi0-lb}
\end{align}

Now fix any $k,k^{\prime} \in \nats$ with $k^{\prime} \geq k$ and $k^{\prime} > 1$ (which implies $n_{k^{\prime}} > \sqrt{n_{k^{\prime}}}$),
and note that every $t \in \{\sqrt{n_{k^{\prime}}}+1,\ldots,n_{k^{\prime}}\}$ has $X_{t} = W_{t}$;
on $H$, this implies $\target(X_{t}) = y_{1}$.  Thus, on the event $H$, 
\begin{align*}
& \frac{1}{n_{k^{\prime}} - \sqrt{n_{k^{\prime}}}} \sum_{t=\sqrt{n_{k^{\prime}}}+1}^{n_{k^{\prime}}} \loss\!\left( f_{\sqrt{n_{k}}}\!\left( X_{1:\sqrt{n_{k}}}, \target\!\left( X_{1:\sqrt{n_{k}}} \right), X_{t} \right), \target(X_{t}) \right)
\\ & = \frac{1}{n_{k^{\prime}} - \sqrt{n_{k^{\prime}}}} \sum_{t=\sqrt{n_{k^{\prime}}}+1}^{n_{k^{\prime}}} \loss\!\left( f_{\sqrt{n_{k}}}\!\left( X_{1:\sqrt{n_{k}}}, \target\!\left( X_{1:\sqrt{n_{k}}} \right), X_{t} \right), y_{1} \right)
\\ & \geq \frac{1}{n_{k^{\prime}} - \sqrt{n_{k^{\prime}}}} \sum_{t=\sqrt{n_{k^{\prime}}}+1}^{n_{k^{\prime}}} \ind\!\left[ \loss\!\left( f_{\sqrt{n_{k}}}\!\left( X_{1:\sqrt{n_{k}}}, \target\!\left( X_{1:\sqrt{n_{k}}} \right), X_{t} \right), y_{1} \right) \geq \Delta_{\zo} \right] \Delta_{\zo}.
\end{align*}
Furthermore, the fact that $\{X_{t}\}_{t=\sqrt{n_{k^{\prime}}}+1}^{n_{k^{\prime}}} = \{W_{t}\}_{t=\sqrt{n_{k^{\prime}}}+1}^{n_{k^{\prime}}}$ also implies that 
$\{ X_{t} \}_{t = \sqrt{n_{k^{\prime}}}+1}^{n_{k^{\prime}}}$ are independent $\pi_{0}$-distributed random variables,
also independent from $X_{1:\sqrt{n_{k}}}$ (since $k \leq k^{\prime}$) and $f_{\sqrt{n_{k}}}$.  Therefore, Hoeffding's inequality (applied under the conditional distribution given $X_{1:\sqrt{n_{k}}}$ and $f_{\sqrt{n_{k}}}$) 
and the law of total probability imply that, on an event $H_{k,k^{\prime}}^{\prime\prime}$ of probability at least $1 - \frac{1}{(k^{\prime})^{3}}$, 
\begin{align*}
& \frac{1}{n_{k^{\prime}} - \sqrt{n_{k^{\prime}}}} \sum_{t=\sqrt{n_{k^{\prime}}}+1}^{n_{k^{\prime}}} \ind\!\left[ \loss\!\left( f_{\sqrt{n_{k}}}\!\left( X_{1:\sqrt{n_{k}}}, \target\!\left( X_{1:\sqrt{n_{k}}} \right), X_{t} \right), y_{1} \right) \geq \Delta_{\zo} \right]
\\ & \geq \pi_{0}\!\left( \left\{ x : \loss\!\left( f_{\sqrt{n_{k}}}\!\left( X_{1:\sqrt{n_{k}}}, \target\!\left( X_{1:\sqrt{n_{k}}} \right), x \right), y_{1} \right) \geq \Delta_{\zo} \right\} \right) - \sqrt{\frac{ (3/2) \ln(k^{\prime}) }{n_{k^{\prime}} - \sqrt{n_{k^{\prime}}}}}.
\end{align*}
Combining with \eqref{eqn:no-suil-u2-pi0-lb} we have that, on the event $H \cap H^{\prime} \cap \bigcap\limits_{k^{\prime} \in \nats \setminus \{1\}} \bigcap\limits_{k \leq k^{\prime}} H_{k,k^{\prime}}^{\prime\prime}$, 
every $k,k^{\prime} \in \nats$ with 
$k^{\prime} \geq k > k_{0}$ satisfy
\begin{multline*}
\frac{1}{n_{k^{\prime}} - \sqrt{n_{k^{\prime}}}} \sum_{t=\sqrt{n_{k^{\prime}}}+1}^{n_{k^{\prime}}} \loss\!\left( f_{\sqrt{n_{k}}}\!\left( X_{1:\sqrt{n_{k}}}, \target\!\left( X_{1:\sqrt{n_{k}}} \right), X_{t} \right), \target(X_{t}) \right)
\\ > \Delta_{\zo} \left( \frac{1}{4} - \sqrt{\frac{ (3/2) \ln(k^{\prime}) }{n_{k^{\prime}} - \sqrt{n_{k^{\prime}}}}} \right).
\end{multline*}

Since $n_{k}$ is strictly increasing in $k$, we have that on $H \cap H^{\prime} \cap \bigcap\limits_{k^{\prime} \in \nats \setminus \{1\}} \bigcap\limits_{k \leq k^{\prime}} H_{k,k^{\prime}}^{\prime\prime}$, 
\begin{align}
& \limsup_{n \to \infty} \hat{\L}_{\ProcX}(f_{n},\target;n)
\geq \limsup_{k \to \infty} \hat{\L}_{\ProcX}\!\left(f_{\sqrt{n_{k}}},\target;\sqrt{n_{k}}\right) \notag
\\ & = \limsup_{k \to \infty} \limsup_{m \to \infty} \frac{1}{m} \sum_{t=1}^{m} \loss\!\left( f_{\sqrt{n_{k}}}\!\left( X_{1:\sqrt{n_{k}}}, \target\!\left( X_{1:\sqrt{n_{k}}} \right), X_{t} \right), \target(X_{t}) \right) \notag
\\ & \geq \limsup_{k \to \infty} \limsup_{k^{\prime} \to \infty} \frac{1}{n_{k^{\prime}}} \sum_{t=\sqrt{n_{k^{\prime}}}+1}^{n_{k^{\prime}}} \loss\!\left( f_{\sqrt{n_{k}}}\!\left( X_{1:\sqrt{n_{k}}}, \target\!\left( X_{1:\sqrt{n_{k}}} \right), X_{t} \right), \target(X_{t}) \right) \notag
\\ & \geq \limsup_{k^{\prime} \to \infty} \frac{n_{k^{\prime}} - \sqrt{n_{k^{\prime}}}}{n_{k^{\prime}}} \Delta_{\zo} \left( \frac{1}{4} - \sqrt{\frac{ (3/2) \ln(k^{\prime}) }{n_{k^{\prime}} - \sqrt{n_{k^{\prime}}}}} \right).
\label{eqn:no-suil-u2-untidy-limits}
\end{align}
Since $n_{k^{\prime}}$ is strictly increasing in $k^{\prime}$, 
we have that for any $k^{\prime} \geq 4$, $0 \leq \frac{(3/2)\ln(k^{\prime})}{n_{k^{\prime}} - \sqrt{n_{k^{\prime}}}} \leq \frac{3\ln(n_{k^{\prime}})}{n_{k^{\prime}}}$, which converges to $0$ as $k^{\prime} \to \infty$.
Furthermore, $\frac{n_{k^{\prime}}-\sqrt{n_{k^{\prime}}}}{n_{k^{\prime}}} = 1 - \frac{1}{\sqrt{n_{k^{\prime}}}}$, which converges to $1$ as $k^{\prime} \to \infty$.
Therefore, the expression in \eqref{eqn:no-suil-u2-untidy-limits} 
equals $\Delta_{\zo}/4$.
By the union bound, the event $H \cap H^{\prime} \cap \bigcap\limits_{k^{\prime} \in \nats \setminus \{1\}} \bigcap\limits_{k \leq k^{\prime}} H_{k,k^{\prime}}^{\prime\prime}$ has probability at least 
\begin{equation*}
1 - \sum\limits_{k^{\prime} \in \nats \setminus \{1\}} \sum\limits_{k \leq k^{\prime}}\frac{1}{(k^{\prime})^{3}} 
= 1 - \sum\limits_{k^{\prime} \in \nats \setminus \{1\}} \frac{1}{(k^{\prime})^{2}} = 1 - \left( \frac{\pi^{2}}{6} - 1 \right) > 0,
\end{equation*}
so that there is a nonzero probability that $\limsup\limits_{n \to \infty} \hat{\L}_{\ProcX}(f_{n},\target;n) \geq \Delta_{\zo} / 4 > 0$.
Thus,
the inductive learning rule $f_{n}$ 
is not strongly universally consistent under $\ProcX$.

It remains to show that the process $\ProcX$ defined above for this second case is an element of $\SUIL$; again, we will in fact establish the stronger fact that $\ProcX \in \CRF$.
For this, for each $k \in \nats$, let 
$J_{k} = \{ n_{k-1} + (i-1) k + 1 : i \in \nats, n_{k-1} + (i-1) k + 1 \leq \sqrt{n_{k}} \}$.
For any $n \in \nats$, define $k_{n} = \max\{ k \in \nats : n_{k-1} < n \}$;
this is well-defined, since $n_{0} = 0$ (so that this set of $k$ values is nonempty), 
and $n_{k}$ is strictly increasing (so that this set of $k$ values is finite, and hence has a maximum value).
Note that, since $n_{k}$ is finite for every $k$, it follows that $k_{n} \to \infty$.
Fix any $A \in \Borel$.
By the construction of the process above, we have that, $\forall n \in \nats$, 
\begin{equation}
\label{eqn:no-u2-inductive-lln-raw-sum}
\frac{1}{n} \sum_{t=1}^{n} \ind_{A}(X_{t}) 
= \frac{1}{n} \sum_{k = 1}^{k_{n}} \left( \left( \sum_{t = n_{k-1} + 1}^{\min\{ \sqrt{n_{k}}, n \}}   \!\!\!\!\left( \ind_{J_{k}}\!(t) \ind_{A}(v_{k,t}) + \ind_{\nats \setminus J_{k}}\!(t) \ind_{A}(U_{k,t}) \right) \right) \!+\!\!\! \sum_{t = \sqrt{n_{k}} + 1}^{\min\{ n_{k}, n \}} \!\!\!\ind_{A}(W_{t}) \right)\!.
\end{equation}
By Kolmogorov's strong law of large numbers \citep*[][Theorem 6.2.2]{ash:00}, with probability one we have 
\begin{align}
\lim_{n \to \infty} \frac{1}{n} \sum_{k = 1}^{k_{n}} \left( \left( \sum_{t = n_{k-1} + 1}^{\min\{ \sqrt{n_{k}}, n \}}   \!\!\!\!\left( \ind_{J_{k}}\!(t) \left(\ind_{A}(v_{k,t}) \!-\! \ind_{A}(v_{k,t})\right) + \ind_{\nats \setminus J_{k}}\!(t) \left(\ind_{A}(U_{k,t}) \!-\! \pi_{0}(A|\X \setminus R_{k})\right) \right) \right) \right. \notag
\\ \left. + \sum_{t = \sqrt{n_{k}} + 1}^{\min\{ n_{k}, n \}} \left(\ind_{A}(W_{t}) - \pi_{0}(A) \right) \right)
= 0.
\label{eqn:no-u2-inductive-klln}
\end{align}
We therefore focus on establishing convergence of 
\begin{equation}
\label{eqn:no-u2-inductive-means}
\frac{1}{n} \sum_{k = 1}^{k_{n}} \left( \left( \sum_{t = n_{k-1} + 1}^{\min\{ \sqrt{n_{k}}, n \}}   \left( \ind_{J_{k}}(t) \ind_{A}(v_{k,t}) + \ind_{\nats \setminus J_{k}}(t) \pi_{0}(A| \X \setminus R_{k}) \right) \right) + \sum_{t = \sqrt{n_{k}} + 1}^{\min\{ n_{k}, n \}} \pi_{0}(A) \right).
\end{equation}

Note that, for any $k, n \in \nats$ with $n > n_{k-1}$, 
\begin{equation*}
\left| J_{k} \cap \left\{n_{k-1}+1,\ldots,\min\!\left\{\sqrt{n_{k}},n\right\}\right\}\right| 
= \left\lceil \frac{\min\!\left\{\sqrt{n_{k}},n\right\} - n_{k-1}}{k} \right\rceil
\leq \frac{n}{k} + 1,
\end{equation*}
and that $\max (J_{k-1}) \leq \sqrt{n_{k-1}}$ for any $k > 1$.
Thus, 
\begin{align}
0 \leq \frac{1}{n} \sum_{k=1}^{k_{n}} \sum_{t=n_{k-1}+1}^{\min\{\sqrt{n_{k}},n\}} \ind_{J_{k}}(t) \ind_{A}(v_{k,t})
& \leq \frac{1}{n} \sum_{k=1}^{k_{n}} \sum_{t=n_{k-1}+1}^{\min\{\sqrt{n_{k}},n\}} \ind_{J_{k}}(t)
\notag \\ & 
\leq \frac{\sqrt{n_{k_{n}-1}}}{n} + \frac{1}{n} \left( \frac{n}{k_{n}} + 1 \right)
= \frac{\sqrt{n_{k_{n}-1}}}{n} + \frac{1}{k_{n}} + \frac{1}{n}. \label{eqn:no-u2-inductive-sparse-v}
\end{align}
By definition of $k_{n}$, this rightmost expression at most 
$\frac{1}{\sqrt{n}} + \frac{1}{k_{n}} + \frac{1}{n}$, 
which has limit $0$ as $n \to \infty$ since $k_{n} \to \infty$.
Thus, 
\begin{equation}
\label{eqn:no-u2-inductive-sparse-v-measure}
\lim_{n \to \infty} \frac{1}{n} \sum_{k=1}^{k_{n}} \sum_{t=n_{k-1}+1}^{\min\{\sqrt{n_{k}},n\}} \ind_{J_{k}}(t) \ind_{A}(v_{k,t}) = 0.
\end{equation}

By the definition of the $R_{k}$ sequence, for any $\eps \in (0,1)$, $\exists k_{\eps} \in \nats$ such that, 
$\forall k \geq k_{\eps}$, $|\pi_{0}(A \cap R_{k}) - (1/2)\pi_{0}(A)| < \eps/2$.
For any $k \geq k_{\eps}$, we have 
$\pi_{0}(A | \X \setminus R_{k}) 
= 2 \pi_{0}(A \cap (\X \setminus R_{k})) 
= 2 ( \pi_{0}(A) - \pi_{0}(A \cap R_{k}) ) 
\in ( 2 ( \pi_{0}(A) - (1/2)\pi_{0}(A) - \eps/2), 2 (\pi_{0}(A) - (1/2)\pi_{0}(A) + \eps/2)) 
= (\pi_{0}(A)-\eps,\pi_{0}(A)+\eps)$.
Thus, for any $n \in \nats$ with $k_{n} \geq k_{\eps}$, we have that
\begin{align*}
& \frac{1}{n} \sum_{k=1}^{k_{n}} \left( \sum_{t = n_{k-1} + 1}^{\min\{ \sqrt{n_{k}}, n \}} \ind_{\nats \setminus J_{k}}(t) \pi_{0}(A| \X \setminus R_{k}) + \sum_{t = \sqrt{n_{k}} + 1}^{\min\{ n_{k}, n \}} \pi_{0}(A) \right)
\\ & \geq - \eps + \frac{1}{n} \sum_{k=k_{\eps}}^{k_{n}} \sum_{t = n_{k-1} + 1}^{\min\{ n_{k}, n \}} \ind_{\nats \setminus J_{k}}(t) \pi_{0}(A)
\geq - \eps + \frac{1}{n} \sum_{k=k_{\eps}}^{k_{n}} \sum_{t = n_{k-1} + 1}^{\min\{ n_{k}, n \}} \left( \pi_{0}(A) - \ind_{J_{k}}(t) \right)
\\ & \geq - \eps - \left( \frac{1}{n} \sum_{k=1}^{k_{n}} \sum_{t = n_{k-1} + 1}^{\min\{ n_{k}, n \}} \ind_{J_{k}}(t) \right) + \left( \frac{1}{n} \sum_{k=k_{\eps}}^{k_{n}} \sum_{t = n_{k-1} + 1}^{\min\{ n_{k}, n \}} \pi_{0}(A) \right)
\\ & = -\eps - \left( \frac{1}{n} \sum_{k=1}^{k_{n}} \sum_{t=n_{k-1}+1}^{\min\{ \sqrt{n_{k}}, n \}} \ind_{J_{k}}(t) \right) + \left(1 - \frac{n_{k_{\eps}-1}}{n}\right) \pi_{0}(A)
\end{align*}
and
\begin{align*}
& \frac{1}{n} \sum_{k=1}^{k_{n}} \left( \sum_{t = n_{k-1} + 1}^{\min\{ \sqrt{n_{k}}, n \}} \ind_{\nats \setminus J_{k}}(t) \pi_{0}(A| \X \setminus R_{k}) + \sum_{t = \sqrt{n_{k}}+1}^{\min\{ n_{k}, n \}} \pi_{0}(A) \right)
\\ & \leq \frac{n_{k_{\eps}-1}}{n} + \frac{1}{n} \sum_{k=k_{\eps}}^{k_{n}} \left( \sum_{t = n_{k-1} + 1}^{\min\{ \sqrt{n_{k}}, n \}} \pi_{0}(A| \X \setminus R_{k}) + \sum_{t = \sqrt{n_{k}} + 1}^{\min\{ n_{k}, n \}} \pi_{0}(A) \right)
\\ & \leq \frac{n_{k_{\eps}-1}}{n} + \frac{1}{n} \sum_{k=k_{\eps}}^{k_{n}} \sum_{t = n_{k-1} + 1}^{\min\{ n_{k}, n \}} \left( \pi_{0}(A) + \eps \right)
\leq \frac{n_{k_{\eps}-1}}{n} + \pi_{0}(A) + \eps.
\end{align*}
As mentioned above, the rightmost expression in \eqref{eqn:no-u2-inductive-sparse-v} has limit $0$,
which in particular also implies that 
$\lim\limits_{n \to \infty} \frac{1}{n} \sum\limits_{k=1}^{k_{n}} \sum\limits_{t=n_{k-1}+1}^{\min\{ \sqrt{n_{k}}, n \}} \ind_{J_{k}}(t) = 0$.
Furthermore, for any fixed $\eps \in (0,1)$, $\lim\limits_{n \to \infty} \frac{n_{k_{\eps}-1}}{n} = 0$.
Thus, since $k_{n} \to \infty$ implies $k_{n} \geq k_{\eps}$ for all sufficiently large $n$, 
we have
\begin{align*}
& \pi_{0}(A) - \eps 
\leq \liminf_{n \to \infty} \frac{1}{n} \sum_{k=1}^{k_{n}} \left( \sum_{t = n_{k-1} + 1}^{\min\{ \sqrt{n_{k}}, n \}} \ind_{\nats \setminus J_{k}}(t) \pi_{0}(A| \X \setminus R_{k}) + \sum_{t = \sqrt{n_{k}} + 1}^{\min\{ n_{k}, n \}} \pi_{0}(A) \right)
\\ & \leq \limsup_{n \to \infty} \frac{1}{n} \sum_{k=1}^{k_{n}} \left( \sum_{t = n_{k-1} + 1}^{\min\{ \sqrt{n_{k}}, n \}} \ind_{\nats \setminus J_{k}}(t) \pi_{0}(A| \X \setminus R_{k}) + \sum_{t = \sqrt{n_{k}} + 1}^{\min\{ n_{k}, n \}} \pi_{0}(A) \right)
\leq  \pi_{0}(A) + \eps.
\end{align*}
Taking the limit as $\eps \to 0$ reveals that
\begin{equation*}
\lim_{n \to \infty} \frac{1}{n} \sum_{k=1}^{k_{n}} \left( \sum_{t = n_{k-1} + 1}^{\min\{ \sqrt{n_{k}}, n \}} \ind_{\nats \setminus J_{k}}(t) \pi_{0}(A| \X \setminus R_{k}) + \sum_{t = \sqrt{n_{k}} + 1}^{\min\{ n_{k}, n \}} \pi_{0}(A) \right) = \pi_{0}(A),
\end{equation*}
which also establishes that the limit exists.
Combined with \eqref{eqn:no-u2-inductive-sparse-v-measure}, \eqref{eqn:no-u2-inductive-klln}, and \eqref{eqn:no-u2-inductive-lln-raw-sum}, 
we have
\begin{equation}
\label{eqn:no-inductive-universal2-limit-measure-2}
\frac{1}{n} \sum_{t=1}^{n} \ind_{A}(X_{t}) \to \pi_{0}(A) \text{ (a.s.)}.
\end{equation}
In particular, this implies that the limit of the left hand side \emph{exists} almost surely.
Since this holds for any choice of $A \in \Borel$,
we have that $\ProcX \in \CRF$.  Since (as argued above) it holds that $\CRF \subseteq \KC = \SUIL$, 
this further implies $\ProcX \in \SUIL$.
Thus, in this second case as well, we conclude that the inductive learning rule $f_{n}$ is not optimistically universal.
Since any inductive learning rule $f_{n}$ satisfies one of these two cases,
this completes the proof that no inductive learning rule is optimistically universal.
\end{proof}

Combining this result with a simple technique for learning in countable spaces,
we immediately have the following corollary.

\begin{corollary}
\label{cor:countable-doubly-universal-inductive}
There exists an optimistically universal inductive learning rule if and only if $\X$ is countable.
\end{corollary}
\begin{proof}
The ``only if'' part of the claim follows immediately from 
Theorem~\ref{thm:no-optimistic-inductive}.
For the ``if'' part, consider a simple inductive learning rule $\hat{f}_{n}$, defined as follows. 
For any $n \in \nats$, $x_{1:n} \in \X^{n}$, $y_{1:n} \in \Y^{n}$, and $x \in \X$,  
if $x \in \{x_{1},\ldots,x_{n}\}$, then letting $i(x; x_{1:n}) = \min\{ i \in \{1,\ldots,n\} : x_{i} = x \}$, 
we define $\hat{f}_{n}(x_{1:n},y_{1:n},x) = y_{i(x; x_{1:n})}$; 
define $\hat{f}_{n}(x_{1:n},y_{1:n},x) = y_{0}$ for some arbitrary fixed $y_{0} \in \Y$ 
if $x \notin \{x_{1},\ldots,x_{n}\}$.  
In other words, this method simply \emph{memorizes}
the observed data points $(x_{i},y_{i})$, $i \in \{1,\ldots,n\}$, and if the test point $x$ is among the observed
$x_{i}$ points, it simply reports the corresponding memorized $y_{i}$ value.

Suppose $\X$ is countable, and enumerate its elements $\X = \{z_{1},z_{2},\ldots\}$ (or in the case of finite $|\X|$, $\X = \{z_{1},z_{2},\ldots,z_{|\X|}\}$).
For each $k \in \nats$ with $k \leq |\X|$, let $A_{k} = \{z_{k}\}$; if $|\X| < \infty$, let $A_{k} = \emptyset$ for all $k \in \nats$ with $k > |\X|$.
Fix any $\ProcX \in \KC$.  By Corollary~\ref{cor:missing-mass}, we have 
\begin{equation*}
\lim_{n \to \infty} \hat{\mu}_{\ProcX}\!\left( \bigcup \{ A_{i} : X_{1:n} \cap A_{i} = \emptyset \} \right) = 0 \text{ (a.s.)}.
\end{equation*}
From the definition of $\hat{f}_{n}$, for each $n \in \nats$, any $\target : \X \to \Y$, and each $z_{i} \in \X$,
if $\hat{f}_{n}(X_{1:n},\target(X_{1:n}),z_{i}) \neq \target(z_{i})$, then necessarily $z_{i} \notin \{X_{1},\ldots,X_{n}\}$.
Therefore, 
\begin{equation*}
\bigcup \{A_{i} : X_{1:n} \cap A_{i} = \emptyset\} 
= \X \setminus \{X_{1},\ldots,X_{n}\} 
\supseteq \{ z_{i} : \hat{f}_{n}(X_{1:n},\target(X_{1:n}),z_{i}) \neq \target(z_{i}) \}.
\end{equation*}
Combining this with Lemma~\ref{lem:expectation} (for homogeneity and monotonicity of $\hat{\mu}_{\ProcX}$),
we have that for any
$\target : \X \to \Y$, 
\begin{align*}
\lim_{n \to \infty} \hat{\L}_{\ProcX}(\hat{f}_{n},\target; n)
& \leq \lim_{n \to \infty} \hat{\mu}_{\ProcX}\!\left( \ind_{\{x : \hat{f}_{n}(X_{1:n},\target(X_{1:n}),x) \neq \target(x)\}}(\cdot) \maxloss \right)
\\ & = \maxloss \lim_{n \to \infty} \hat{\mu}_{\ProcX}\!\left( \{x : \hat{f}_{n}(X_{1:n},\target(X_{1:n}),x) \neq \target(x)\} \right)
\\ & \leq \maxloss \lim_{n \to \infty} \hat{\mu}_{\ProcX}\!\left( \bigcup \{ A_{i} : X_{1:n} \cap A_{i} = \emptyset \} \right) = 0 \text{ (a.s.)}.
\end{align*}
Thus, since $\hat{\L}_{\ProcX}$ is nonnegative, $\hat{f}_{n}$ is strongly universally consistent under every $\ProcX \in \KC$.
Recalling that (by Theorem~\ref{thm:main}) $\SUIL = \KC$, this completes the proof.
\end{proof}

It is worth noting here that the proof of 
Theorem~\ref{thm:no-optimistic-inductive}
can be made somewhat simpler if we only wish to directly establish the theorem statement.
Specifically, the variables $V_{k,j}$ there can be replaced by i.i.d.\ $\pi_{0}$ samples,
while the $U_{k,j}$ variables can all be set equal to some fixed point $x_{0} \in \X$;
in this case, the sets $R_{k}$ are not needed (replaced by $\X\setminus\{x_{0}\}$), and several of the definitions can be simplified 
(e.g., the $\target_{k}$ functions can all be replaced by a fixed function $\target_{1}$, which simply outputs $y_{0}$ except on $w_{j}$ and $x_{0}$ points, where it outputs $y_{1}$).
The general approach to the proof of inconsistency remains essentially unchanged.  One can easily verify that the resulting
process satisfies Condition~\ref{con:kc}; however, it 
does not necessarily have convergent relative frequencies (specifically, in the second case discussed in the proof).
The details of this simpler proof are left as an exercise for the interested reader.
We have chosen the more-involved proof presented above so that the inductive learning rule is 
shown to not be universally consistent even under processes that are ergodic (since they are product processes) 
with convergent relative frequencies ($\CRF$),
as argued in the proof.  
Formally, we have established the following corollary.

\begin{corollary}
\label{cor:no-lln-consistent-inductive}
If $\X$ is uncountable, then 
there does not exist an inductive learning rule that is strongly universally consistent under every ergodic $\ProcX \in \CRF$.
\end{corollary}

\section{Online Learning}
\label{sec:online}

In this section, we discuss the \emph{online learning} setting, establishing a number of 
results related to the following question (restated from Section~\ref{subsec:main}) 
on the existence of optimistically universal learning rules.\\

\noindent {\bf Open Problem~\ref{prob:optimistic-online} (restated)}~ 
\textit{Does there exist an optimistically universal online learning rule?}\\

We approach this question and related issues  
in an analogous fashion to the above discussion of self-adaptive and inductive learning.
However, unlike the results on self-adaptive and inductive learning, the results presented here are only partial, 
and leave open a number of interesting core questions, including the above open problem.

After introducing some useful lemmas on online aggregation techniques in Section~\ref{subsec:weighted-majority}, 
we begin the discussion of universally consistent online learning in 
Section~\ref{subsec:okc}
with the subject of 
concisely characterizing the family of processes $\SUOL$.
We propose a concise condition (Condition~\ref{con:okc}) for a process $\ProcX$, 
and prove that it is generally a \emph{necessary} condition: i.e., it is satisfied by any $\ProcX$ that admits 
strong universal online learning.  We also argue that it is a \emph{sufficient} condition in the case that $\X$ is countable 
or that $\ProcX$ is deterministic, and at the same time positively resolve Open Problem~\ref{prob:optimistic-online} 
for countable $\X$.  However, for the \emph{general} case with uncountable $\X$, 
we leave open both the question in Open Problem~\ref{prob:optimistic-online} 
and the question of whether Condition~\ref{con:okc} is sufficient 
for $\ProcX$ to admit strong universal online learning (Open Problem~\ref{prob:suol-equals-okc}).
Following this, in Section~\ref{subsec:online-vs-inductive},
we address the relation between admission of strong universal online learning and admission of strong universal self-adaptive learning.
We specifically establish that the latter implies the former, but not vice versa (when $\X$ is infinite): 
that is, $\SUAL \subset \SUOL$ with \emph{strict} inclusion,
which establishes a separation of $\SUOL$ from $\SUAL$ and $\SUIL$.
We also construct an online learning rule that is universally consistent under every $\ProcX \in \SUAL$.
Although lacking a general concise (provable) characterization of $\SUOL$, 
we are at least able to show, in Section~\ref{subsec:suol-invariance-to-loss}, that the family 
$\SUOL$ is invariant to the choice of loss function $\loss$ 
(as was true of $\SUIL$ and $\SUAL$ above, from their equivalence to $\KC$ in Theorem~\ref{thm:main}),
under the additional restriction that $\loss$ is totally bounded.
We also argue that $\SUOL$ is invariant to the choice of $\loss$ among losses that are separable but \emph{not}
totally bounded, but we leave open the question of whether these two $\SUOL$ families are equal (Open Problem~\ref{prob:suol-multiclass}).

\subsection{Online Aggregation}
\label{subsec:weighted-majority}

Before getting into the new results of the present work on online learning, we first introduce some supporting lemmas based on 
a well-known aggregation technique from the literature on online learning with arbitrary sequences.
The first lemma is a regret guarantee for a weighted averaging prediction algorithm.
The technique and analysis are taken from classic works in the theory of online learning
\citep*{vovk:90,vovk:92,littlestone:94,cesa-bianchi:97,kivinen:99,singer:99,gyorfi:02a}.
For completeness, we include a brief proof: a version of this 
classic argument.

\begin{lemma}
\label{lem:weighted-majority} 
For each $n \in \nats$, let $\{z_{n,i}\}_{i=1}^{\infty}$ be a sequence of values in $[0,1]$,
and let $\{p_{i}\}_{i=1}^{\infty}$ be a sequence in $(0,1)$ with $\sum\limits_{i=1}^{\infty} p_{i} = 1$.
Fix a finite constant $b \in (0,1)$.
For each $n,i \in \nats$, define $L_{n,i} = \frac{1}{n} \sum\limits_{t=1}^{n} z_{t,i}$.
Then for each $i \in \nats$, define $w_{1,i} = v_{1,i} = p_{i}$,
and for each $n \in \nats \setminus \{1\}$, define 
$w_{n,i} = p_{i} b^{ (n-1) L_{(n-1),i} }$,
and $v_{n,i} = w_{n,i} / \sum\limits_{i=1}^{\infty} w_{n,i}$.
Finally, for each $n \in \nats$, define $\bar{z}_{n} = \sum\limits_{i=1}^{\infty} v_{n,i} z_{n,i}$.
Then for every $n \in \nats$, 
\begin{equation*}
\frac{1}{n} \sum_{t=1}^{n} \bar{z}_{t} \leq \inf_{i \in \nats} \left( \frac{\ln(1/b)}{1-b} L_{n,i} + \frac{1}{(1-b) n} \ln\!\left(\frac{1}{p_{i}}\right) \right).
\end{equation*}
\end{lemma}
\begin{proof}
Define $W_{n} = \sum\limits_{i=1}^{\infty} w_{n,i}$ for each $n \in \nats$.
Then note that $\forall n \in \nats$, 
$W_{n+1} = \sum\limits_{i=1}^{\infty} w_{n,i} b^{z_{n,i}} = W_{n} \sum\limits_{i=1}^{\infty} v_{n,i} b^{z_{n,i}}$.
Noting that $b^{z_{n,i}} \leq 1 - (1-b) z_{n,i}$,
we find that 
\begin{equation*}
\frac{W_{n+1}}{W_{n}} 
\leq \sum_{i=1}^{\infty} v_{n,i} (1 - (1-b) z_{n,i} )
= 1 - (1-b) \bar{z}_{n}.
\end{equation*}
Since $W_{1} = 1$, by induction we have $W_{n+1} \leq \prod\limits_{t=1}^{n} (1 - (1-b) \bar{z}_{t} )$.
Noting that $\ln(1 - (1-b) \bar{z}_{t}) \leq -(1-b) \bar{z}_{t}$,
we have that $\ln(W_{n+1}) \leq \sum\limits_{t=1}^{n} \ln( 1 - (1-b) \bar{z}_{t} ) \leq - (1-b) \sum\limits_{t=1}^{n} \bar{z}_{t}$.
Therefore, for any $n \in \nats$, 
\begin{align*}
\sum_{t=1}^{n} \bar{z}_{t}
& \leq \frac{1}{1-b} \ln\!\left(\frac{1}{W_{n+1}}\right)
= \frac{1}{1-b} \ln\!\left(\frac{1}{ \sum_{i=1}^{\infty} p_{i} b^{ n L_{n,i} } }\right)
\\ & \leq \frac{1}{1-b} \ln\!\left(\frac{1}{ \sup_{i\in\nats} p_{i} b^{ n L_{n,i} } }\right)
= \inf_{i \in \nats} \left( \frac{\ln(1/b)}{1-b} n L_{n,i} + \frac{1}{1-b} \ln\!\left(\frac{1}{p_{i}}\right) \right).
\end{align*}
Dividing the leftmost and rightmost expressions by $n$ completes the proof.
\end{proof}

For our purposes, we will need the following implication of this lemma.

\begin{lemma}
\label{lem:general-weighted-majority}
For any sequence $\left\{ \hat{h}^{(i)}_{n} \right\}_{i=1}^{\infty}$ of online learning rules, 
there exists an online learning rule $\hat{f}_{n}$ such that, 
for any process $\ProcX$ and any measurable function $\target : \X \to \Y$, 
if, with probability one, there exists a sequence $\{i_{n}\}_{n=1}^{\infty}$ in $\nats$ with $\ln(i_{n}) = o(n)$
s.t. 
$\lim\limits_{n \to \infty} \hat{\L}_{\ProcX}(\hat{h}^{(i_{n})}_{\cdot},\target;n) = 0$, 
then $\lim\limits_{n \to \infty} \hat{\L}_{\ProcX}(\hat{f}_{\cdot},\target;n) = 0 \text{ (a.s.)}$.
\end{lemma}
\begin{proof}
Fix any sequences $\mathbf{x} = \{x_{n}\}_{n=1}^{\infty}$ in $\X$ and $\mathbf{y} = \{y_{n}\}_{n=1}^{\infty}$ in $\Y$.
For each $n,i \in \nats$, define $\hat{z}_{n,i}(x_{1:n},y_{1:n}) = \loss( \hat{h}^{(i)}_{n-1}(x_{1:(n-1)},y_{1:(n-1)},x_{n}), y_{n} ) / \maxloss$
(which may be random, if $\hat{h}^{(i)}_{n-1}$ is a randomized learning rule).
For each $i \in \nats$, let $p_{i} = \frac{6}{\pi^{2} i^{2}}$, and note that $\sum\limits_{i=1}^{\infty} p_{i} = 1$.
Fix any $b \in (0,1)$, and for $n,i \in \nats$ define $v_{n,i}$ as in Lemma~\ref{lem:weighted-majority}, 
for these $p_{i}$ values, and for $z_{n,i} = \hat{z}_{n,i}(x_{1:n},y_{1:n}) \in [0,1]$ for each $n,i \in \nats$.
Finally, define $\bar{z}_{n}(x_{1:n},y_{1:n}) = \sum\limits_{i=1}^{\infty} v_{n,i} \hat{z}_{n,i}(x_{1:n},y_{1:n})$.
From this point, there are two possible routes toward defining the online learning rule $\hat{f}_{n}$,
depending on whether we involve randomization.  In the simplest definition, when predicting for 
$x_{n+1}$, we could simply sample an index $i$ (independently for each $n$) according to the distribution specified 
by $\{v_{(n+1),i}\}_{i=1}^{\infty}$, and take the $\hat{h}^{(i)}_{n}$ learning rules's prediction.  It is fairly 
straightforward to relate the expected performance of this method to the quantities $\bar{z}_{t}(x_{1:t},y_{1:t})$
and then apply Lemma~\ref{lem:weighted-majority} \citep*[see e.g.,][]{littlestone:94}, 
together with concentration inequalities to argue
that the bound from Lemma~\ref{lem:weighted-majority} almost surely becomes valid in the limit of $n \to \infty$.
However, instead of this approach, we will analyze a method that avoids randomization.\footnote{In general, 
randomization is known to be necessary for achieving optimal regret guarantees in online learning \citep*[see][Chapter 4]{cesa-bianchi:06}.
However, since the reference sequence $\hat{h}_{\cdot}^{(i_n)}$ itself has $\hat{\L}_{\ProcX}(\hat{h}_{\cdot}^{(i_n)},\target;n) \to 0$ (a.s.), 
we are not concerned with multiplicative constant factors for the purpose of achieving 
$\hat{\L}_{\ProcX}(\hat{f}_{\cdot},\target;n) \to 0$ (a.s.), and thus we can avoid randomization.}
Specifically, let $\{\eps_{n}\}_{n=0}^{\infty}$ be any sequence in $(0,\infty)$ with $\eps_{n} \to 0$, 
and for each $n \in \nats \cup \{0\}$, define
$\hat{f}_{n}(x_{1:n},y_{1:n},x_{n+1}) = \hat{y}_{n+1}$
for some value $\hat{y}_{n+1} \in \Y$ satisfying\footnote{Here 
we suppose the 
choice of $\hat{y}_{n+1}$ is such that the function $\hat{f}_{n}(\cdot,\cdot,\cdot)$ is measurable:
for instance, it would suffice to consider an enumeration of 
a countable dense subset of $\Y$ (which exists by the separability assumption on $\Y$) and then choose 
the first $y$ in this enumeration satisfying the $\eps_{n}$-excess criterion in the definition of $\hat{y}_{n+1}$.}
\begin{equation*}
\sum_{i=1}^{\infty} v_{(n+1),i} \loss\!\left( \hat{y}_{n+1}, \hat{h}^{(i)}_{n}(x_{1:n},y_{1:n},x_{n+1}) \right) 
\leq \eps_{n} + \inf\limits_{y \in \Y} \sum_{i=1}^{\infty} v_{(n+1),i} \loss\!\left( y, \hat{h}^{(i)}_{n}(x_{1:n},y_{1:n},x_{n+1}) \right).
\end{equation*}
We use this definition for any $n$ and any such sequences $\mathbf{x}$ and $\mathbf{y}$, so that this completes the definition of $\hat{f}_{n}$.
With this definition, for any $t \in \nats \cup \{0\}$ and sequences $\mathbf{x}$ and $\mathbf{y}$, by the relaxed triangle inequality 
and the fact that $\sum\limits_{i=1}^{\infty} v_{(t+1),i} = 1$, we have that 
\begin{align*}
& \loss\!\left( \hat{f}_{t}(x_{1:t},y_{1:t},x_{t+1}), y_{t+1} \right)
= \sum_{i=1}^{\infty} v_{(t+1),i} \loss\!\left( \hat{f}_{t}(x_{1:t},y_{1:t},x_{t+1}), y_{t+1} \right)
\\ & \leq \triconst \sum_{i=1}^{\infty} v_{(t+1),i} \loss\!\left( \hat{f}_{t}(x_{1:t},y_{1:t},x_{t+1}), \hat{h}^{(i)}_{t}\!(x_{1:t},y_{1:t},x_{t+1}) \right) 
\\ & \phantom{\leq } ~+ \triconst \sum_{i=1}^{\infty} v_{(t+1),i} \loss\!\left( \hat{h}^{(i)}_{t}\!(x_{1:t},y_{1:t},x_{t+1}), y_{t+1} \right).
\end{align*}
Then the definition of $\hat{f}_{t}$ guarantees the right hand side is at most
\begin{align*}
& \eps_{t} + \inf_{y \in \Y} \triconst \sum_{i=1}^{\infty} v_{(t+1),i} \loss\!\left(y, \hat{h}^{(i)}_{t}(x_{1:t},y_{1:t},x_{t+1}) \right) + \triconst \sum_{i=1}^{\infty} v_{(t+1),i} \loss\!\left( \hat{h}^{(i)}_{t}(x_{1:t},y_{1:t},x_{t+1}), y_{t+1} \right)
\\ & \leq \eps_{t} + 2 \triconst \sum_{i=1}^{\infty} v_{(t+1),i} \loss\!\left( \hat{h}^{(i)}_{t}(x_{1:t},y_{1:t},x_{t+1}), y_{t+1} \right)
= \eps_{t} + 2 \triconst \maxloss \bar{z}_{t+1}(x_{1:(t+1)},y_{1:(t+1)}),
\end{align*}
so that
\begin{equation*}
\frac{1}{n} \sum_{t=0}^{n-1} \loss\!\left( \hat{f}_{t}(x_{1:t},y_{1:t},x_{t+1}), y_{t+1} \right)
\leq \frac{1}{n} \sum_{t=0}^{n-1} \left(\eps_{t} + 2 \triconst \maxloss \bar{z}_{t+1}(x_{1:(t+1)},y_{1:(t+1)})\right).
\end{equation*}
Together with Lemma~\ref{lem:weighted-majority},
we have that
\begin{align}
\label{eqn:nonrandomized-weighted-majority-predictor-loss-bound}
& \frac{1}{n} \sum_{t=0}^{n-1} \loss\!\left( \hat{f}_{t}(x_{1:t},y_{1:t},x_{t+1}), y_{t+1} \right)
\\ & \leq \left( \frac{1}{n} \sum_{t=0}^{n-1} \eps_{t} \right) + 2 \triconst \inf_{i \in \nats} \!\left( \frac{\ln(1/b)}{1-b} \!\left( \frac{1}{n} \sum_{t=0}^{n-1} \loss\!\left( \hat{h}^{(i)}_{t}\!(x_{1:t},y_{1:t},x_{t+1}), y_{t+1} \right) \right) \!\!+\! \frac{\maxloss}{(1-b)n} \ln\!\left(\frac{1}{p_{i}}\right) \right)\!. \notag
\end{align}

Now fix $\ProcX$ and $\target$ such that,
with probability one, there exists a sequence $\{i_{n}\}_{n=1}^{\infty}$ in $\nats$ with $\ln(i_{n}) = o(n)$
such that $\lim\limits_{n \to \infty} \hat{\L}_{\ProcX}(\hat{h}^{(i_{n})}_{\cdot},\target;n) = 0$.
Then, on the event that this occurs, the inequality in \eqref{eqn:nonrandomized-weighted-majority-predictor-loss-bound} implies
\begin{align*}
& \hat{\L}_{\ProcX}(\hat{f}_{\cdot},\target;n)
\leq \left( \frac{1}{n} \sum_{t=0}^{n-1} \eps_{t} \right) + 2 \triconst \inf_{i \in \nats} \left( \frac{\ln(1/b)}{1-b} \hat{\L}_{\ProcX}(\hat{h}^{(i)}_{\cdot},\target;n) + \frac{\maxloss}{(1-b) n} \ln\!\left(\frac{1}{p_{i}}\right) \right)
\\ & \leq \left( \frac{1}{n} \sum_{t=0}^{n-1} \eps_{t} \right) + 2 \triconst \left( \frac{\ln(1/b)}{1-b} \hat{\L}_{\ProcX}(\hat{h}^{(i_{n})}_{\cdot},\target;n) + \frac{2\maxloss}{(1-b) n} \ln(i_{n}) + \frac{\maxloss}{(1-b)n} \ln\!\left(\frac{\pi^{2}}{6}\right) \right).
\end{align*}
Since $\eps_{t} \to 0$ implies $\lim\limits_{n \to \infty} \frac{1}{n} \sum\limits_{t=0}^{n-1} \eps_{t} = 0$, and since 
$\lim\limits_{n \to \infty} \hat{\L}_{\ProcX}(\hat{h}^{(i_{n})}_{\cdot},\target;n) = 0$ and $\lim\limits_{n \to \infty} \frac{1}{n}\ln(i_{n})$ $= 0$ in this context, 
and $\hat{\L}_{\ProcX}$ is nonnegative, 
it follows that $\lim\limits_{n \to \infty} \hat{\L}_{\ProcX}(\hat{f}_{\cdot},\target;n) = 0$ on this event. 
\end{proof}

The next lemma provides a technical fact useful in the proofs of the theorems below.

\begin{lemma}
\label{lem:converging-array-path}
Suppose $\{\beta_{i,n}\}_{i,n \in \nats}$ is an array of values in $[0,\infty)$ such that
$\lim\limits_{i \to \infty} \limsup\limits_{n \to \infty} \beta_{i,n} = 0$, and that $\{j_{n}\}_{n=1}^{\infty}$ is a sequence in $\nats$ with $j_{n} \to \infty$.
Then there exists a sequence $\{i_{n}\}_{n=1}^{\infty}$ in $\nats$ such that $i_{n} \leq j_{n}$ for every $n \in \nats$, and $\lim\limits_{n \to \infty} \beta_{i_{n},n} = 0$.
\end{lemma}
\begin{proof}
For each $i \in \nats$, let $n_{i} \in \nats$ be such that $\sup\limits_{n \geq n_{i}} \beta_{i,n} \leq \frac{1}{i} + \limsup\limits_{n \to \infty} \beta_{i,n}$;
such an $n_{i}$ is guaranteed to exist by the definition of the $\limsup$.
For each $n \in \nats$ with $n < n_{1}$, define $i_{n} = 1$,
and for each $n \in \nats$ with $n \geq n_{1}$, define $i_{n} = \max\{ i \in \{1,\ldots,j_{n}\} : n \geq n_{i} \}$.
By definition, we have $i_{n} \leq j_{n}$ for every $n \in \nats$.
Furthermore, by definition, we have $n \geq n_{i_{n}}$ for every $n \geq n_{1}$,
so that $\beta_{i_{n},n} \leq \frac{1}{i_{n}} + \limsup\limits_{n^{\prime} \to \infty} \beta_{i_{n},n^{\prime}}$.
Finally, since $n_{i}$ is finite for each $i \in \nats$, and $j_{n} \to \infty$, we have $i_{n} \to \infty$.
Altogether, we have 
\begin{equation*}
\limsup\limits_{n \to \infty} \beta_{i_{n},n} 
\leq \limsup\limits_{n \to \infty}  \left( \frac{1}{i_{n}} + \limsup\limits_{n^{\prime} \to \infty} \beta_{i_{n},n^{\prime}} \right)
\leq \limsup\limits_{i \to \infty} \left( \frac{1}{i} + \limsup\limits_{n \to \infty} \beta_{i,n} \right)
= 0.
\end{equation*}
Since $\liminf\limits_{n \to \infty} \beta_{i_{n},n} \geq 0$ by nonnegativity of the $\beta_{i,n}$ values,
the result follows.
\end{proof}

\subsection{Toward Concisely Characterizing $\SUOL$}
\label{subsec:okc}

We begin the discussion of universally consistent online learning with
the subject of concisely characterizing the family of processes $\SUOL$.
Specifically,
we consider the following candidate for such a characterization.
Though we succeed in establishing its \emph{necessity} for $\ProcX$ to admit strong universal online learning, 
determining whether it is also sufficient will be left as an open problem.

\begin{condition}
\label{con:okc}
For every sequence $\{A_{k}\}_{k=1}^{\infty}$ of disjoint elements of $\Borel$,
\begin{equation*}
\left |\left\{ k \in \nats : X_{1:T} \cap A_{k} \neq \emptyset \right\} \right| = o(T)\text{ (a.s.)}.
\end{equation*}
\end{condition}

Denote by $\OKC$ the set of all processes $\ProcX$ satisfying Condition~\ref{con:okc}.
With the aim of concisely characterizing the family of processes $\SUOL$, we consider now 
the specific question of
whether $\SUOL = \OKC$.  Formally, we make partial progress toward resolving the following question, 
which remains open at this writing.

\begin{problem}
\label{prob:suol-equals-okc}
Is $\SUOL = \OKC$?
\end{problem}

In this subsection, we show that in general, $\SUOL \subseteq \OKC$, 
and that equality holds when $\X$ is countable.  Equality also holds for the intersections
of these sets with the family of deterministic processes.  

We begin with the first of these claims.
First, as was true of $\KC$, we can also state Condition~\ref{con:okc} in an alternative equivalent form, 
which makes the necessity of Condition~\ref{con:okc} for learning more immediately clear.
In particular, we may note an interesting parallel to Corollary~\ref{cor:missing-mass}.

\begin{lemma}
\label{lem:okc-equiv}
A process $\ProcX$ satisfies Condition~\ref{con:okc} 
if and only if every disjoint sequence $\{A_{i}\}_{i=1}^{\infty}$ in $\Borel$ with $\bigcup\limits_{i=1}^{\infty} A_{i} = \X$ 
(i.e., every countable measurable partition) satisfies 
\begin{equation*}
\limsup_{T \to \infty} \frac{1}{T} \sum_{t=1}^{T} \ind\!\left[ X_{t} \in \bigcup \left\{ A_{i} : X_{1:(t-1)} \cap A_{i} = \emptyset \right\} \right] = 0 \text{ (a.s.)}.
\end{equation*}
\end{lemma}
\begin{proof}
First note that, for any sequence $\{A_{k}\}_{k=1}^{\infty}$ of disjoint sets in $\Borel$, 
defining $B_{1} = \X \setminus \bigcup\limits_{k=1}^{\infty} A_{k}$ and $B_{k} = A_{k-1}$ for $k \geq 2$, 
we have that $\{B_{k}\}_{k=1}^{\infty}$ is a disjoint sequence in $\Borel$ with $\bigcup\limits_{k=1}^{\infty} B_{k} = \X$ 
and 
$|\{ k : X_{1:T} \cap A_{k} \neq \emptyset \}| \leq |\{ k : X_{1:T} \cap B_{k} \neq \emptyset \}| \leq |\{ k : X_{1:T} \cap A_{k} \neq \emptyset \}| + 1$, 
so that $|\{ k : X_{1:T} \cap B_{k} \neq \emptyset \}| = o(T)\text{ (a.s.)}$ if and only if $|\{ k : X_{1:T} \cap A_{k} \neq \emptyset \}| = o(T)\text{ (a.s.)}$.
Thus, the set of processes $\ProcX$ satisfying Condition~\ref{con:okc} remains unchanged
if we restrict the disjoint sequences $\{A_{k}\}_{k=1}^{\infty}$ to those satisfying $\bigcup\limits_{k=1}^{\infty} A_{k} = \X$.

Now fix any process $\ProcX$ and any disjoint sequence $\{A_{i}\}_{i=1}^{\infty}$ 
in $\Borel$ with $\bigcup\limits_{i=1}^{\infty} A_{i} = \X$.
Then note that, for any $T \in \nats$, 
$\left| \left\{ k \in \nats : X_{1:T} \cap A_{k} \neq \emptyset \right\} \right|
= \ind\!\left[ X_{T} \!\in \bigcup \left\{ A_{i} : X_{1:(T-1)} \!\cap A_{i} = \emptyset \right\} \right] + \left| \left\{ k \in \nats : X_{1:(T-1)} \cap A_{k} \neq \emptyset \right\} \right|$.
By induction (taking $T=1$ as a trivially-satisfied base case), 
this implies that $\forall T \in \nats$, 
\begin{equation*}
\left| \left\{ k \in \nats : X_{1:T} \cap A_{k} \neq \emptyset \right\} \right| 
= \sum_{t=1}^{T} \ind\!\left[ X_{t} \in \bigcup \left\{ A_{i} : X_{1:(t-1)} \cap A_{i} = \emptyset \right\} \right].
\end{equation*}
In particular, this implies that 
$\sum\limits_{t=1}^{T} \ind\!\left[ X_{t} \in \bigcup \left\{ A_{i} : X_{1:(t-1)} \cap A_{i} = \emptyset \right\} \right] = o(T) \text{ (a.s.)}$
if and only if
$\left| \left\{ k \in \nats : X_{1:T} \cap A_{k} \neq \emptyset \right\} \right| = o(T) \text{ (a.s.)}$.
Since this equivalence holds for any choice of disjoint sequence $\{A_{i}\}_{i=1}^{\infty}$ in $\Borel$ with $\bigcup\limits_{i=1}^{\infty} A_{i} = \X$, 
the lemma follows.
\end{proof}

With this lemma in hand, we can now prove the following theorem,
which establishes that 
Condition~\ref{con:okc} is \emph{necessary} for a process to admit strong universal online learning.

\begin{theorem}
\label{thm:okc-nec}
$\SUOL \subseteq \OKC$.
\end{theorem}
\begin{proof}
This proof follows essentially the same outline as that of Lemma~\ref{lem:sual-subset-kc}.
We prove the result in the contrapositive.
Suppose $\ProcX \notin \OKC$.  
By Lemma~\ref{lem:okc-equiv}, there exists a disjoint sequence $\{A_{i}\}_{i=1}^{\infty}$ in $\Borel$ 
with $\bigcup\limits_{i=1}^{\infty} A_{i} = \X$ such that, 
with probability strictly greater than $0$, 
\begin{equation*}
\limsup_{T \to \infty} \frac{1}{T} \sum_{t=1}^{T} \ind\!\left[ X_{t} \in \bigcup \left\{ A_{i} : X_{1:(t-1)} \cap A_{i} = \emptyset \right\} \right] > 0.
\end{equation*}
Furthermore, since the left hand side is always nonnegative, this also implies \citep*[see e.g.,][Theorem 1.6.6]{ash:00}
\begin{equation}
\label{eqn:okc-nec-expectation-from-equiv}
\E\!\left[ \limsup_{T \to \infty} \frac{1}{T} \sum_{t=1}^{T} \ind\!\left[ X_{t} \in \bigcup \left\{ A_{i} : X_{1:(t-1)} \cap A_{i} = \emptyset \right\} \right] \right] > 0.
\end{equation}

Now take any two distinct values $y_{0}, y_{1} \in \Y$, 
and (as we did in the proof of Lemma~\ref{lem:sual-subset-kc}) 
for each $\kappa \in [0,1)$, $i \in \nats$, and $x \in A_{i}$, 
letting $\kappa_{i} = \lfloor 2^{i} \kappa \rfloor - 2 \lfloor 2^{i-1} \kappa \rfloor \in \{0,1\}$, 
define $\target_{\kappa}(x) = y_{\kappa_{i}}$.
Recall that we established in the proof of Lemma~\ref{lem:sual-subset-kc} that 
$(x,\kappa) \mapsto \target_{\kappa}(x)$ is measurable in the appropriate product $\sigma$-algebra.
Also for every $t \in \nats$ define $i_{t}$ as the unique $i \in \nats$ with $X_{t} \in A_{i}$, 
and for any $n \in \nats \cup \{0\}$, let $\bar{\A}(X_{1:n}) = \bigcup\{A_{i} : X_{1:n} \cap A_{i} = \emptyset\}$.

Now fix any online learning rule $g_{n}$, and for brevity define
$f_{n}^{\kappa}(\cdot) = g_{n}(X_{1:n},\target_{\kappa}(X_{1:n}),\cdot)$ for each $n \in \nats$.
Then
\begin{align*}
& \sup_{\kappa \in [0,1)} \E\!\left[ \limsup_{n\to\infty} \hat{\L}_{\ProcX}(g_{\cdot},\target_{\kappa};n) \right]
\geq \int_{0}^{1} \E\!\left[ \limsup_{n\rightarrow\infty} \hat{\L}_{\ProcX}(g_{\cdot},\target_{\kappa};n) \right] {\rm d}\kappa
\\ & \geq \int_{0}^{1} \E\!\left[ \limsup_{n\rightarrow\infty} \frac{1}{n} \sum_{t=0}^{n-1} \loss\left( f_{t}^{\kappa}(X_{t+1}), \target_{\kappa}(X_{t+1}) \right) \ind_{\bar{\A}(X_{1:t})}(X_{t+1}) \right] {\rm d}\kappa.
\end{align*}
By Fubini's theorem, this is equal
\begin{equation*}
\E\!\left[ \int_{0}^{1} \limsup_{n\rightarrow\infty} \frac{1}{n} \sum_{t=0}^{n-1} \loss\!\left( f_{t}^{\kappa}(X_{t+1}), \target_{\kappa}(X_{t+1}) \right) \ind_{\bar{\A}(X_{1:t})}(X_{t+1}) {\rm d}\kappa \right].
\end{equation*}
Since $\loss$ is bounded, Fatou's lemma implies this is at least as large as
\begin{equation*}
\E\!\left[ \limsup_{n\rightarrow\infty} \int_{0}^{1} \frac{1}{n} \sum_{t=0}^{n-1} \loss\!\left( f_{t}^{\kappa}(X_{t+1}), \target_{\kappa}(X_{t+1}) \right) \ind_{\bar{\A}(X_{1:t})}(X_{t+1}) {\rm d}\kappa \right],
\end{equation*}
and linearity of integration implies this equals
\begin{equation}
\label{eqn:suol-subset-okc-pre-integration}
\E\!\left[ \limsup_{n\rightarrow\infty} \frac{1}{n} \sum_{t=0}^{n-1} \ind_{\bar{\A}(X_{1:t})}(X_{t+1}) \int_{0}^{1} \loss\!\left( f_{t}^{\kappa}(X_{t+1}), \target_{\kappa}(X_{t+1}) \right) {\rm d}\kappa \right].
\end{equation}

For any $t \in \nats \cup \{0\}$, the value of $f_{t}^{\kappa}(X_{t+1})$ is a function of $\ProcX$ and $\kappa_{i_{1}},\ldots,\kappa_{i_{t}}$.
Therefore, for any $t \in \nats \cup \{0\}$ with $X_{t+1} \in \bar{\A}(X_{1:t})$, the value of $f_{t}^{\kappa}(X_{t+1})$ is functionally independent of $\kappa_{i_{t+1}}$.
Thus, for any $t \in \nats \cup \{0\}$, 
letting $K \sim {\rm Uniform}([0,1))$ be independent of $\ProcX$ and $g_{t}$, 
if $X_{t+1} \in \bar{\A}(X_{1:t})$, we have
\begin{align*}
& \int_{0}^{1} \loss\!\left( f_{t}^{\kappa}(X_{t+1}), \target_{\kappa}(X_{t+1}) \right) {\rm d}\kappa
= \E\!\left[ \loss\!\left( f_{t}^{K}(X_{t+1}), \target_{K}(X_{t+1}) \right) \Big| \ProcX, g_{t} \right]
\\ & = \E\!\left[ \E\!\left[ \loss\!\left( g_{t}(X_{1:t},\{y_{K_{i_{j}}}\}_{j=1}^{t}, X_{t+1}), y_{K_{t+1}}\right) \Big| \ProcX, g_{t}, K_{i_{1}},\ldots,K_{i_{t}} \right] \Big| \ProcX, g_{t} \right]
\\ & = \E\!\left[\sum_{b\in\{0,1\}} \frac{1}{2}\loss\!\left(g_{t}(X_{1:t},\{y_{K_{i_{j}}}\}_{j=1}^{t}, X_{t+1}), y_{b} \right) \Big| \ProcX, g_{t} \right].
\end{align*}
By the relaxed triangle inequality, this is no smaller than $\E\!\left[\frac{1}{2\triconst}\loss(y_{0},y_{1}) \Big| \ProcX, g_{t} \right] = \frac{1}{2\triconst} \loss(y_{0},y_{1})$, 
so that \eqref{eqn:suol-subset-okc-pre-integration} is at least as large as
\begin{align*}
& \E\!\left[ \limsup_{n\rightarrow\infty} \frac{1}{n} \sum_{t=0}^{n-1} \ind_{\bar{\A}(X_{1:t})}(X_{t+1}) \frac{1}{2\triconst} \loss(y_{0},y_{1}) \right]
\\ & = \frac{1}{2\triconst} \loss(y_{0},y_{1}) \E\!\left[ \limsup_{n\rightarrow\infty} \frac{1}{n} \sum_{t=1}^{n} \ind\!\left[ X_{t} \in \bigcup \left\{ A_{i} : X_{1:(t-1)} \cap A_{i} = \emptyset \right\} \right] \right] 
> 0,
\end{align*}
where this last inequality is immediate from \eqref{eqn:okc-nec-expectation-from-equiv} and the fact that (since $\loss$ is a near-metric) $\loss(y_{0},y_{1}) > 0$.
Altogether, we have that
\begin{equation*}
\sup_{\kappa \in [0,1)} \E\!\left[ \limsup_{n\to\infty} \hat{\L}_{\ProcX}(g_{\cdot},\target_{\kappa};n) \right]
> 0.
\end{equation*}
In particular, this implies $\exists \kappa \in [0,1)$ such that 
$\E\!\left[ \limsup\limits_{n\to\infty} \hat{\L}_{\ProcX}(g_{\cdot},\target_{\kappa};n) \right] > 0$.
Since any random variable equal $0$ (a.s.) necessarily has expected value $0$, 
this further implies that with probability strictly greater than $0$, 
$\limsup\limits_{n\to\infty} \hat{\L}_{\ProcX}(g_{\cdot},\target_{\kappa};n) > 0$.
Thus, $g_{n}$ is not strongly universally consistent.
Since $g_{n}$ was an arbitrary online learning rule, we conclude that there does not exist an online learning rule that is
strongly universally consistent under $\ProcX$: that is, $\ProcX \notin \SUOL$.
Since this argument holds for any $\ProcX \notin \OKC$, the theorem follows.
\end{proof}

Although this work falls short of establishing equivalence between $\SUOL$ and $\OKC$ in the general case 
(i.e., positively resolving Open Problem~\ref{prob:suol-equals-okc} in general),
we do show this equivalence in the special case of \emph{countable} $\X$,
and indeed also positively resolve Open Problem~\ref{prob:optimistic-online} for countable $\X$.
Note that, in this special case, Condition~\ref{con:okc} simplifies to the condition that 
the number of distinct points $x \in \X$ occurring in the sequence $X_{1:T}$ is $o(T)$ almost surely.
Specifically, we have the following result.

\begin{lemma}
\label{lem:countable-okc-equiv}
If $\X$ is countable, then a process $\ProcX$ satisfies Condition~\ref{con:okc} if and only if 
\begin{equation*}
|\{ x \in \X : X_{1:T} \cap \{x\} \neq \emptyset \}| = o(T) \text{ (a.s.)}.
\end{equation*}
\end{lemma}
\begin{proof}
Enumerate the elements of $\X$ as $z_{1},z_{2},\ldots$ (or $z_{1},\ldots,z_{|\X|}$ in the case 
of finite $|\X|$).
If Condition~\ref{con:okc} is satisfied, then choose $A_{k} = \{z_{k}\}$ for every $k \in \nats$ 
with $k \leq |\X|$, and if $|\X| < \infty$ then let $A_{k} = \emptyset$ for every $k > |\X|$. 
It immediately follows from Condition~\ref{con:okc} that $|\{ x \in \X : X_{1:T} \cap \{x\} \neq \emptyset \}| = o(T)$ (a.s.).
For the other direction, 
if Condition~\ref{con:okc} fails, there exists a disjoint sequence $\{A_{k}\}_{k=1}^{\infty}$ in $\Borel$ 
that has $|\{ k \in \nats : X_{1:T} \cap A_{k} \neq \emptyset \}| \neq o(T)$ with nonzero probability.
Noting that 
$|\{ k \in \nats : X_{1:T} \cap A_{k} \neq \emptyset \}| \leq |\{ x \in \X : X_{1:T} \cap \{x\} \}|$,
we have that, on this same event of nonzero probability, 
$|\{ x \in \X : X_{1:T} \cap \{x\} \}| \neq o(T)$ as well.
\end{proof}

We now state our result for strong universal online learning when $\X$ is countable.

\begin{theorem}
\label{thm:okc-countable}
If $\X$ is countable, then Condition~\ref{con:okc} is necessary and sufficient for a process $\ProcX$ 
to admit strong universal online learning: that is, $\SUOL = \OKC$.
Moreover, if $\X$ is countable, then there exists an optimistically universal online learning rule.
\end{theorem}
\begin{proof}
Suppose $\X$ is countable.
For the first claim, since we already know $\SUOL \subseteq \OKC$ from Theorem~\ref{thm:okc-nec},
it suffices to show $\OKC \subseteq \SUOL$, for this special case.  We will establish this fact, 
while simultaneously establishing the second claim, by showing that there is an online learning rule 
that is strongly universally consistent under every $\ProcX \in \OKC$ (which thereby also establishes
that every such process is in $\SUOL$).
Toward this end, fix any $y_{0} \in \Y$, and define an online learning rule $f_{n}$ such that, 
for each $n \in \nats \cup \{0\}$, 
$\forall x_{1:(n+1)} \in \X^{n+1}$, $\forall y_{1:n} \in \Y^{n}$, 
if $x_{n+1} = x_{i}$ for some $i \in \{1,\ldots,n\}$, then $f_{n}(x_{1:n},y_{1:n},x_{n+1}) = y_{i}$ for the smallest $i \in \{1,\ldots,n\}$ with $x_{n+1}=x_{i}$,
and otherwise $f_{n}(x_{1:n},y_{1:n},x_{n+1}) = y_{0}$.
The key property of $f_{n}$ here is that it is \emph{memorization-based}, 
in that any previously-observed point's response $y$ will be faithfully reproduced
if that point is encountered again later in the sequence.  The specific fact that 
it evaluates to $y_{0}$ in the case of a previously-unseen point is unimportant in this
context, and this case can in fact be defined arbitrarily (subject to the function $f_{n}$
being measurable) without affecting the result (and similarly for the choice to 
break ties to favor smaller indices).

Now fix any $\ProcX \in \OKC$ and any measurable function $\target : \X \to \Y$.  
Note that any $i,t \in \nats$ with $i \leq t$ and $X_{t+1} = X_{i}$ has $\target(X_{t+1}) = \target(X_{i})$,
so that 
$\loss( f_{t}(X_{1:t},\target(X_{1:t}),X_{t+1}), \target(X_{t+1}) ) = \loss( \target(X_{i}), \target(X_{t+1}) ) = 0$.
Therefore, we have
\begin{align}
& \limsup_{n \to \infty} \hat{\L}_{\ProcX}(f_{\cdot},\target;n) 
= \limsup_{n \to \infty} \frac{1}{n} \sum_{t=0}^{n-1} \loss( f_{t}(X_{1:t},\target(X_{1:t}),X_{t+1}), \target(X_{t+1}) )
\notag \\ & \leq \limsup_{n \to \infty} \frac{1}{n} \sum_{t=0}^{n-1} \maxloss \ind[ \nexists i \in \{1,\ldots,t\} : X_{t+1} = X_{i} ] 
= \maxloss \limsup_{n \to \infty} \frac{1}{n} \left| \left\{ x \in \X : X_{1:n} \cap \{x\} \neq \emptyset \right\} \right|. \notag
\end{align}
Lemma~\ref{lem:countable-okc-equiv} implies that the rightmost expression above is equal $0$ almost surely.
Since this argument holds for any choice of $\target$, we conclude that $f_{n}$
is strongly universally consistent under $\ProcX$.  Furthermore, since this 
holds for any choice of $\ProcX \in \OKC$, the theorem follows.
\end{proof}

For uncountable $\X$, we can at least state a corollary holding for all \emph{deterministic} processes, 
via a reduction to the case of countable $\X$.

\begin{corollary}
\label{cor:okc-deterministic}
For any deterministic process $\ProcX$,
Condition~\ref{con:okc} is necessary and sufficient for $\ProcX$ to admit strong universal online learning: that is, $\ProcX \in \SUOL$ if and only if $\ProcX \in \OKC$.
\end{corollary}
\begin{proof}[Sketch]
This result follows from essentially the same proof used for Theorem~\ref{thm:okc-countable}; 
the only significant change is in the proof of Lemma~\ref{lem:countable-okc-equiv}, which 
can be established for deterministic processes on general $\X$ by replacing the $z_{k}$ sequence defined in the 
proof by the distinct entries of the sequence $\ProcX$ (noting that the intersection of $\ProcX$ with the complement of this $z_{k}$ sequence is empty).
Alternatively, it can also be established via a reduction to the case of countable $\X$.
Specifically, fix any deterministic process $\ProcX$, and let $\X_{\ProcX}$ denote
the set of \emph{distinct} points $x \in \X$ appearing in the sequence $\ProcX$.
Note that $\X_{\ProcX}$ is countable, and that (with a slight abuse of notation) $\ProcX$
may be thought of as a sequence of $\X_{\ProcX}$-valued variables.
Furthermore, it is straightforward to show that any deterministic $\ProcX$ satisfies Condition~\ref{con:okc} for the space $\X_{\ProcX}$
if and only if it satisfies Condition~\ref{con:okc} for the original space $\X$ 
(since only the intersections of the sets $A_{i}$ with $\X_{\ProcX}$ are relevant for checking this condition).
Thus, since Theorem~\ref{thm:okc-countable} holds for \emph{any} countable space $\X$, 
applying it to the space $\X_{\ProcX}$, we have that $\ProcX$ admits strong universal online learning
if and only if $\ProcX$ satisfies Condition~\ref{con:okc}.
\end{proof}

\subsection{Relation of Online Learning to Inductive and Self-Adaptive Learning}
\label{subsec:online-vs-inductive}

Next, we turn to addressing 
the relation between admission of strong universal online learning 
and admission of strong universal inductive or self-adaptive learning.
Specifically, we find that the latter implies the former, but \emph{not} vice versa (if $\X$ is infinite), 
so that admission of strong universal online learning is a strictly more general condition.
To show this, since we have established in Theorem~\ref{thm:main} that $\SUIL = \SUAL$, 
it suffices to argue that $\SUAL \subseteq \SUOL$, with \emph{strict} inclusion if $|\X| = \infty$:
that is, $\SUOL \setminus \SUAL \neq \emptyset$.
For this we have the following theorem.

\begin{theorem}
\label{thm:suil-subset-suol}
$\SUAL \subseteq \SUOL$, and the inclusion is \emph{strict} if and only if $|\X| = \infty$.
\end{theorem}
\begin{proof}
We begin by showing $\SUAL \subseteq \SUOL$.  In fact, we will establish 
a stronger claim: that there exists a \emph{single} online learning rule $\hat{f}_{n}$
that is strongly universally consistent for \emph{every} $\ProcX \in \SUAL$.
Specifically, let $\hat{g}_{n,m}$ be an optimistically universal self-adaptive learning rule. 
The existence of such a rule was established in Theorem~\ref{thm:optimistic-self-adaptive},
and an explicit construction is given in \eqref{eqn:sual-rule}, as established by Theorem~\ref{thm:doubly-universal-adaptive}.
Now fix any $y_{0} \in \Y$, and 
for each $i \in \nats$ define an online learning rule $\hat{h}^{(i)}_{n}$ as follows.
For each $n \in \nats \cup \{0\}$, for any sequences $x_{1:(n+1)} \in \X^{n+1}$ and $y_{1:n} \in \Y^{n}$,
if $n < i$, then define $\hat{h}^{(i)}_{n}(x_{1:n},y_{1:n},x_{n+1}) = y_{0}$,
and if $n \geq i$, then define $\hat{h}^{(i)}_{n}(x_{1:n},y_{1:n},x_{n+1}) = \hat{g}_{i,n}(x_{1:n},y_{1:i},x_{n+1})$.
Measurability of $\hat{h}^{(i)}_{n}$ follows from measurability of $\hat{g}_{i,n}$,
so that this is a valid definition of an online learning rule.

Given this definition of the sequence $\left\{ \hat{h}^{(i)}_{n} \right\}_{i=1}^{\infty}$, 
denote by $\hat{f}_{n}$ the online learning rule guaranteed to exist 
by Lemma~\ref{lem:general-weighted-majority} (defined explicitly in the proof above), 
satisfying the property described there relative to this sequence $\left\{ \hat{h}^{(i)}_{n} \right\}_{i=1}^{\infty}$.
Now fix any $\ProcX \in \SUAL$ and any measurable $\target : \X \to \Y$,
and for each $i,n \in \nats$, define $\hat{\beta}_{i,n} = \hat{\L}_{\ProcX}(\hat{h}^{(i)}_{\cdot},\target;n)$.
In particular, note that since $\loss$ is always finite, it holds that $\forall i \in \nats$, 
\begin{align*}
\limsup_{n \to \infty} \hat{\beta}_{i,n}
& = \limsup_{n \to \infty} \frac{1}{n} \sum_{t=0}^{n-1} \loss\!\left( \hat{h}^{(i)}_{t}(X_{1:t},\target(X_{1:t}),X_{t+1}),\target(X_{t+1}) \right)
\\ & = \limsup_{n \to \infty} \frac{1}{n} \sum_{t=i}^{n-1} \loss\!\left( \hat{g}_{i,t}(X_{1:t},\target(X_{1:i}),X_{t+1}),\target(X_{t+1}) \right)
\\ & = \limsup_{n \to \infty} \frac{1}{n+1} \sum_{t=i}^{i+n} \loss\!\left( \hat{g}_{i,t}(X_{1:t},\target(X_{1:i}),X_{t+1}),\target(X_{t+1}) \right)
= \hat{\L}_{\ProcX}(\hat{g}_{i,\cdot},\target;i).
\end{align*}
Since $\hat{g}_{n,m}$ is strongly universally consistent under $\ProcX$, 
it follows that $\lim\limits_{i \to \infty} \limsup\limits_{n \to \infty} \hat{\beta}_{i,n} = 0$ on an event $E$ of probability one.
In particular, on $E$, Lemma~\ref{lem:converging-array-path} implies that there exists a sequence $\{i_{n}\}_{i=1}^{\infty}$ in $\nats$
with $i_{n} \leq n$ for every $n$, such that $\lim\limits_{n \to \infty} \hat{\beta}_{i_{n},n} = 0$.
Therefore, since $\ln(i_{n}) \leq \ln(n) = o(n)$, the property of $\hat{f}_{n}$ guaranteed by Lemma~\ref{lem:general-weighted-majority} 
implies that $\lim\limits_{n \to \infty} \hat{\L}_{\ProcX}(\hat{f}_{\cdot},\target;n) = 0$ almost surely.
Since this argument holds for any choice of $\target$, we conclude that $\hat{f}_{n}$ is strongly universally consistent under $\ProcX$,
and since this holds for any choice of $\ProcX \in \SUAL$, it follows that $\SUAL \subseteq \SUOL$.

$\SUOL$ and $\SUAL$ are trivially equal if $|\X| < \infty$, since then \emph{every} process $\ProcX$ is contained in $\KC$,
and Theorem~\ref{thm:main} implies $\SUAL = \KC$, while we have just established that $\SUOL \supseteq \SUAL$, 
so every process is contained in both $\SUAL$ and $\SUOL$.
Now consider the case $|\X| = \infty$.
To see that $\SUOL \setminus \SUAL \neq \emptyset$
in this case, 
in light of Corollary~\ref{cor:okc-deterministic}, 
together with Theorem~\ref{thm:main}, 
it suffices to construct a \emph{deterministic} process in $\OKC \setminus \KC$.
Toward this end, we let $\{z_{i}\}_{i=1}^{\infty}$ be an arbitrary sequence of distinct elements of $\X$,
and define a deterministic process $\ProcX$ as follows.
For each $t \in \nats$, define $i_{t} = \lfloor \log_{2}(2t) \rfloor$,
and let $X_{t} = z_{i_{t}}$.
For any sequence $\{A_{k}\}_{k=1}^{\infty}$ of disjoint elements of $\Borel$,
and any $T \in \nats$, 
\begin{equation*}
|\{ k \in \nats : X_{1:T} \cap A_{k} \neq \emptyset \} | 
\leq |\{ i \in \nats : X_{1:T} \cap \{z_{i}\} \neq \emptyset \}|
= \lfloor \log_{2}(2T) \rfloor = o(T).
\end{equation*}
Therefore, $\ProcX \in \OKC$.
However, let $A_{k} = \{ z_{i} : i \geq k \}$ for each $k \in \nats$, 
and note that each $A_{k}$ is countable, hence in $\Borel$, 
and that $A_{k} \downarrow \emptyset$.  
Then note that every $i \in \nats$ has $\frac{1}{2^{i}-1} \sum\limits_{t=1}^{2^{i}-1} \ind_{\{z_{i}\}}(X_{t}) = \frac{2^{i-1}}{2^{i}-1} > \frac{1}{2}$.
Thus, since each $|A_{k}| = \infty$, we have $\hat{\mu}_{\ProcX}(A_{k}) \geq \frac{1}{2}$ for every $k \in \nats$, 
so that $\lim\limits_{k \to \infty} \hat{\mu}_{\ProcX}(A_{k}) \geq \frac{1}{2} > 0$.
Since $\ProcX$ is deterministic, this violates the requirement of Condition~\ref{con:kc}, and therefore $\ProcX \notin \KC$.
\end{proof}

The proof of Theorem~\ref{thm:suil-subset-suol} actually establishes two additional results.
First, since the online learning rule $\hat{f}_{n}$ constructed in the proof has no dependence 
on the distribution of the process $\ProcX$ from $\SUAL$, this proof also establishes the following corollary.

\begin{corollary}
\label{cor:online-all-sual}
There exists an online learning rule that is strongly universally consistent under every $\ProcX \in \SUAL$.
\end{corollary}

Note that this is a weaker claim than would be required for positive resolution of Open Problem~\ref{prob:optimistic-online}, 
since (as established by Theorem~\ref{thm:suil-subset-suol}) the set of processes admitting strong 
universal online learning is a \emph{strict} superset of the set of processes admitting 
strong universal self-adaptive learning (if $\X$ is infinite).

Second, since Theorem~\ref{thm:okc-nec} establishes that $\SUOL \subseteq \OKC$, 
and Theorem~\ref{thm:main} establishes that $\SUAL = \KC$, 
Theorem~\ref{thm:suil-subset-suol} also establishes that $\KC \subseteq \OKC$ 
(a fact that one can easily verify from their definitions as well).  Furthermore, the above proof that 
the inclusion $\SUAL \subseteq \SUOL$ is strict if $|\X|=\infty$ establishes this fact 
by constructing a deterministic process $\ProcX \in \OKC \setminus \KC$ 
(which thereby verifies the claim due to Corollary~\ref{cor:okc-deterministic} and Theorem~\ref{thm:main}).
Thus, it also establishes that the inclusion $\KC \subseteq \OKC$ is strict in the case $|\X|=\infty$.
Also, as noted in the above proof, if $|\X| < \infty$, then $\KC$ contains \emph{every} process. 
Since $\KC \subseteq \OKC$, this clearly implies that if $|\X| < \infty$, then $\KC = \OKC$.
Altogether, we conclude that the above proof also establishes the following result.

\begin{corollary}
\label{cor:kc-subset-okc}
$\KC \subseteq \OKC$, and the inclusion is \emph{strict} if and only if $|\X| = \infty$.
\end{corollary}

\subsection{Invariance of $\SUOL$ to the Choice of Loss Function}
\label{subsec:suol-invariance-to-loss}

In this subsection, we are interested in the question of whether the family $\SUOL$
is invariant to the choice of loss function (subject to the basic constraints from Section~\ref{subsec:notation}).  
Recall that we established above that this property holds for the families $\SUIL$ and $\SUAL$ (as implied by their 
equivalence to $\KC$ from Theorem~\ref{thm:main}, regardless of the choice of $(\Y,\loss)$).  
Furthermore, a positive resolution of Open Problem~\ref{prob:suol-equals-okc} would immediately imply this 
property for $\SUOL$, since Condition~\ref{con:okc} has no dependence on $(\Y,\loss)$.
However, since Open Problem~\ref{prob:suol-equals-okc} remains open at this time, 
it is interesting to directly explore the question of invariance of $\SUOL$ to the choice of $(\Y,\loss)$.
Specifically, we prove two relevant results.
First, we show that $\SUOL$ is invariant to the choice of $(\Y,\loss)$, under the additional constraint that $(\Y,\loss)$ 
is \emph{totally bounded}: that is, $\forall \eps > 0$, $\exists \Y_{\eps} \subseteq \Y$ s.t. $|\Y_{\eps}|<\infty$ and $\sup\limits_{y \in \Y} \inf\limits_{y_{\eps} \in \Y_{\eps}} \loss(y_{\eps},y) \leq \eps$.
For instance, $\loss$ as 
any $L_{p}$ loss $(y,y^{\prime}) \mapsto |y-y^{\prime}|^{p}$ ($p \in (0,\infty)$) 
with $\Y$ any bounded subset of $\reals$ would satisfy this.
In particular, this means that, in characterizing
the family of processes $\SUOL$ for totally bounded losses, it suffices to characterize this set for the simplest case
of \emph{binary classification}:
$(\Y,\loss) = (\{0,1\},\loss_{\zo})$, 
where for any $\Y$ we generally denote by $\loss_{\zo} : \Y^{2} \to [0,\infty)$ the $0$-$1$ loss on $\Y$,
defined by $\loss_{\zo}(y,y^{\prime}) = \ind[ y \neq y^{\prime} ]$ for all $y,y^{\prime} \in \Y$.
Second, we also find that the set $\SUOL$ is invariant among (bounded, separable) losses that are \emph{not} totally bounded
(e.g., the $0$-$1$ loss with $\Y = \nats$).
We leave open the question of whether or not these two $\SUOL$ sets are equal (Open Problem~\ref{prob:suol-multiclass} below).
We begin with the totally bounded case.

\begin{theorem}
\label{thm:suol-invariant-to-loss}
The set $\SUOL$ is invariant to the specification of $(\Y,\loss)$,
subject to $(\Y,\loss)$ being totally bounded with $\maxloss > 0$.
\end{theorem}
\begin{proof}
To disambiguate notation in this proof, for any near-metric space $(\Y^{\prime},\loss^{\prime})$,
we denote by $\SUOL_{(\Y^{\prime},\loss^{\prime})}$ the family $\SUOL$ as it would be defined
if $(\Y,\loss)$ were specified as $(\Y^{\prime},\loss^{\prime})$.  
As above,
define the measurable subsets of $\Y^{\prime}$ as the elements of the Borel $\sigma$-algebra 
generated by the topology induced by $\loss^{\prime}$.
Let $\loss_{\zo}$ be the $0$-$1$ loss on $\{0,1\}$, as defined above.
To establish the theorem, it suffices to verify the claim that 
$\SUOL_{(\Y^{\prime},\loss^{\prime})} = \SUOL_{(\{0,1\},\loss_{\zo})}$
for all totally bounded near-metric spaces $(\Y^{\prime},\loss^{\prime})$ 
with $\sup\limits_{y,y^{\prime} \in \Y^{\prime}} \loss^{\prime}(y,y^{\prime}) > 0$.
Fix any such $(\Y^{\prime},\loss^{\prime})$.

The inclusion $\SUOL_{(\Y^{\prime},\loss^{\prime})} \subseteq \SUOL_{(\{0,1\},\loss_{\zo})}$ is quite straightforward, as follows.
For any $\ProcX \in \SUOL_{(\Y^{\prime},\loss^{\prime})}$, 
letting $\hat{f}_{n}$ be an online learning rule that is strongly universally consistent under $\ProcX$ 
(for the specification $(\Y,\loss) = (\Y^{\prime},\loss^{\prime})$), 
we can define an online learning rule $\hat{f}^{\zo}_{n}$ for the specification $(\Y,\loss) = (\{0,1\},\loss_{\zo})$ as follows.
Let $z_{0},z_{1} \in \Y^{\prime}$ be such that $\loss^{\prime}(z_{0},z_{1}) > 0$.
For any $n \in \nats \cup \{0\}$, and any sequences $x_{1:(n+1)}$ in $\X$ and $y_{1:n}$ in $\{0,1\}$, 
define a sequence $y_{1:n}^{\prime}$ with $y_{i}^{\prime} = z_{y_{i}}$ for each $i \in \{1,\ldots,n\}$,
and then define $\hat{f}^{\zo}_{n}(x_{1:n},y_{1:n},x_{n+1}) = \argmin\limits_{y \in \{0,1\}} \loss^{\prime}\!\left(\hat{f}_{n}(x_{1:n},y_{1:n}^{\prime},x_{n+1}),z_{y}\right)$
(breaking ties in favor of $y=0$).
In particular, that $\hat{f}^{\zo}_{n}$ is a measurable function $\X^{n}\times\{0,1\}^{n}\times\X \to \{0,1\}$ follows immediately from measurability of $\hat{f}_{n}$.
Then note that, for any measurable function $f : \X \to \{0,1\}$, defining $f^{\prime} : \X \to \Y^{\prime}$ as $f^{\prime}(x) = z_{f(x)}$ (which is clearly also measurable),
we have $\forall t \in \nats \cup \{0\}$, 
\begin{align*}
& \ind\!\left[ \hat{f}^{\zo}_{t}(X_{1:t},f(X_{1:t}),X_{t+1}) \neq f(X_{t+1}) \right] 
\\ & \leq \ind\!\left[ \loss^{\prime}\!\!\left(\hat{f}_{t}(X_{1:t},\!f^{\prime}(X_{1:t}),\!X_{t+1}),\!f^{\prime}(X_{t+1})\right) \!=\!\!\! \max_{y \in \{0,1\}} \!\loss^{\prime}\!\!\left(\hat{f}_{t}(X_{1:t},\!f^{\prime}(X_{1:t}),\!X_{t+1}),\!z_{y}\right) \right]
\\ & \leq \ind\!\!\left[ \loss^{\prime}\!\!\left(\hat{f}_{t}(X_{1:t},\!f^{\prime}\!(X_{1:t}),\!X_{t+1}),\!f^{\prime}\!(X_{t+1})\right) \!\geq\!\!\!\! \sum_{y \in \{0,1\}}\! \frac{1}{2} \loss^{\prime}\!\!\left(\hat{f}_{t}(X_{1:t},\!f^{\prime}\!(X_{1:t}),\!X_{t+1}),\!z_{y}\right) \right]
\\ & \leq \ind\!\left[ \loss^{\prime}\!\left(\hat{f}_{t}(X_{1:t},f^{\prime}(X_{1:t}),X_{t+1}),f^{\prime}(X_{t+1})\right) \geq \frac{1}{2 \triconst} \loss^{\prime}(z_{0},z_{1}) \right]
\\ & \leq \frac{2\triconst}{\loss^{\prime}(z_{0},z_{1})} \loss^{\prime}\!\left(\hat{f}_{t}(X_{1:t},f^{\prime}(X_{1:t}),X_{t+1}),f^{\prime}(X_{t+1})\right),
\end{align*}
where the second-to-last inequality is due to the relaxed triangle inequality.
Therefore, under the specification $(\Y,\loss) = (\{0,1\},\loss_{\zo})$, we have 
\begin{align*}
\limsup_{n \to \infty} \hat{\L}_{\ProcX}\!\left(\hat{f}^{\zo}_{\cdot},f;n\right)
= \limsup_{n \to \infty} \frac{1}{n} \sum_{t=0}^{n-1} \ind\!\left[ \hat{f}^{\zo}_{t}(X_{1:t},f(X_{1:t}),X_{t+1}) \neq f(X_{t+1}) \right]
\\ \leq \frac{2\triconst}{\loss^{\prime}(z_{0},z_{1})} \limsup_{n \to \infty} \frac{1}{n} \sum_{t=0}^{n-1} \loss^{\prime}\!\left(\hat{f}_{t}(X_{1:t},f^{\prime}(X_{1:t}),X_{t+1}),f^{\prime}(X_{t+1})\right)
= 0 \text{ (a.s.)},
\end{align*}
where the last equality (to which the ``almost surely'' qualifier applies) is due to strong universal consistency of $\hat{f}_{n}$
(and the fact that $z_{0},z_{1}$ were chosen to satisfy $\loss^{\prime}(z_{0},z_{1}) > 0$).
Since this argument holds for any choice of measurable $f : \X \to \{0,1\}$, 
we conclude that $\hat{f}^{\zo}_{n}$ is strongly universally consistent under $\ProcX$ (for the specification $(\Y,\loss) = (\{0,1\},\loss_{\zo})$),
so that $\ProcX \in \SUOL_{(\{0,1\},\loss_{\zo})}$.
Since this argument holds for any $\ProcX \in \SUOL_{(\Y^{\prime},\loss^{\prime})}$, we conclude that $\SUOL_{(\Y^{\prime},\loss^{\prime})} \subseteq \SUOL_{(\{0,1\},\loss_{\zo})}$.

The proof of the converse inclusion is somewhat more involved.
Specifically, fix any $\ProcX \in \SUOL_{(\{0,1\},\loss_{\zo})}$,
and let $\hat{f}^{\zo}_{n}$ be an online learning rule that is 
strongly universally consistent under $\ProcX$ (for the specification $(\Y,\loss) = (\{0,1\},\loss_{\zo})$).
We then define an online learning rule $\hat{f}^{\prime}_{n}$ for the specification $(\Y,\loss) = (\Y^{\prime},\loss^{\prime})$ totally bounded, as follows.
For each $\eps > 0$, let $\Y_{\eps}^{\prime} \subseteq \Y^{\prime}$ be such that 
$|\Y_{\eps}^{\prime}| < \infty$ and $\sup\limits_{y \in \Y^{\prime}} \inf\limits_{y_{\eps} \in \Y_{\eps}^{\prime}} \loss^{\prime}(y_{\eps},y) \leq \eps$,
as guaranteed to exist by total boundedness.
For each $y \in \Y^{\prime}$, let $g_{\eps}(y) = \argmin\limits_{y_{\eps} \in \Y_{\eps}^{\prime}} \loss^{\prime}(y_{\eps},y)$, breaking ties to favor smaller indices in some fixed enumeration of $\Y_{\eps}^{\prime}$.
Then, for each $y \in \Y^{\prime}$ and each $y_{\eps} \in \Y_{\eps}^{\prime}$, define $h_{\eps}^{(y_{\eps})}(y) = \ind[ g_{\eps}(y) = y_{\eps} ]$.
One can easily verify that $g_{\eps}$ and $h_{\eps}^{(y_{\eps})}$ are measurable functions,
and furthermore that for every $y \in \Y^{\prime}$, exactly one $y_{\eps} \in \Y_{\eps}^{\prime}$ has $h_{\eps}^{(y_{\eps})}(y) = 1$ while every $y_{\eps}^{\prime} \in \Y_{\eps}^{\prime} \setminus \{y_{\eps}\}$ has $h_{\eps}^{(y_{\eps}^{\prime})}(y) = 0$.

For any $n \in \nats \cup \{0\}$, and any sequences $x_{1:(n+1)}$ in $\X$ and $y_{1:n}$ in $\Y^{\prime}$, 
define 
\begin{equation*}
\hat{f}_{n}^{(\eps)}(x_{1:n},y_{1:n},x_{n+1}) = \argmax\limits_{y_{\eps} \in \Y_{\eps}^{\prime}} \hat{f}^{\zo}_{n}(x_{1:n},h_{\eps}^{(y_{\eps})}(y_{1:n}),x_{n+1}),
\end{equation*} 
breaking ties to favor $y_{\eps}$ with a smaller index in a fixed enumeration of $\Y_{\eps}^{\prime}$.
Again, one can easily verify that $\hat{f}_{n}$ is a measurable function $\X^{n}\times(\Y^{\prime})^{n}\times\X \to \Y^{\prime}$,
which follows immediately from measurability of $\hat{f}^{\zo}_{n}$, the $h_{\eps}^{(y_{\eps})}$ functions, and the $\argmax$.
Thus, $\hat{f}_{n}^{(\eps)}$ defines an online learning rule.

Now note that, for any measurable function $f : \X \to \Y^{\prime}$, and each $y_{\eps} \in \Y_{\eps}^{\prime}$, 
the composed function $x \mapsto h_{\eps}^{(y_{\eps})}(f(x))$ is a measurable function $\X \to \{0,1\}$,
and therefore (by strong universal consistency of $\hat{f}^{\zo}_{n}$) with probability one, 
\begin{equation*}
\limsup_{n \to \infty} \frac{1}{n} \sum_{t = 0}^{n-1} \loss_{\zo}\!\left( \hat{f}^{\zo}_{t}(X_{1:t},h_{\eps}^{(y_{\eps})}(f(X_{1:t})),X_{t+1}), h_{\eps}^{(y_{\eps})}(f( X_{t+1} )) \right) = 0.
\end{equation*}
By the union bound, this holds simultaneously for all $y_{\eps} \in \Y_{\eps}^{\prime}$ with probability one.
Furthermore, note that if 
$\hat{f}^{\zo}_{t}(X_{1:t},h_{\eps}^{(y_{\eps})}(f(X_{1:t})),X_{t+1}) = h_{\eps}^{(y_{\eps})}(f( X_{t+1} ))$
for every $y_{\eps} \in \Y_{\eps}^{\prime}$, then 
$\hat{f}^{(\eps)}_{t}(X_{1:t},f(X_{1:t}),X_{t+1}) = g_{\eps}(f(X_{t+1}))$.
We therefore have that, under the specification $(\Y,\loss) = (\Y^{\prime},\loss^{\prime})$, 
\begin{align*}
& \limsup_{n \to \infty} \hat{\L}_{\ProcX}\!\left(\hat{f}_{\cdot}^{(\eps)},f;n\right)
= \limsup_{n \to \infty} \frac{1}{n} \sum_{t=0}^{n-1} \loss^{\prime}\!\left( \hat{f}^{(\eps)}_{t}(X_{1:t},f(X_{1:t}),X_{t+1}), f(X_{t+1}) \right)
\\ & \leq \limsup_{n \to \infty} \frac{1}{n} \sum_{t=0}^{n-1} \left( \loss^{\prime}(g_{\eps}(f(X_{t+1})),f(X_{t+1})) + \maxloss \ind\!\left[ \hat{f}^{(\eps)}_{t}(X_{1:t},f(X_{1:t}),X_{t+1}) \neq g_{\eps}(f(X_{t+1})) \right] \right)
\\ & \leq \limsup_{n \to \infty} \frac{1}{n} \sum_{t=0}^{n-1} \left( \eps + \maxloss \sum_{y_{\eps} \in \Y_{\eps}^{\prime}} \loss_{\zo}\!\left( \hat{f}^{\zo}_{t}(X_{1:t},h_{\eps}^{(y_{\eps})}(f(X_{1:t})),X_{t+1}), h_{\eps}^{(y_{\eps})}(f(X_{t+1})) \right) \right)
\\ & \leq \eps + \maxloss \sum_{y_{\eps} \in \Y_{\eps}^{\prime}} \limsup_{n \to \infty} \frac{1}{n} \sum_{t=0}^{n-1} \loss_{\zo}\!\left( \hat{f}^{\zo}_{t}(X_{1:t},h_{\eps}^{(y_{\eps})}(f(X_{1:t})),X_{t+1}), h_{\eps}^{(y_{\eps})}(f(X_{t+1})) \right)
= \eps \text{ (a.s.)},
\end{align*}
where the inequality on this last line is due to finiteness of $|\Y_{\eps}^{\prime}|$.

We now apply this argument to values $\eps \in \{1/i : i \in \nats\}$.
For any measurable $\target : \X \to \Y^{\prime}$, for each $i,n \in \nats$, 
define $\beta_{i,n}^{\target} = \hat{\L}_{\ProcX}\!\left(\hat{f}_{\cdot}^{(1/i)},\target;n\right)$ (under the specification $(\Y,\loss) = (\Y^{\prime},\loss^{\prime})$).
By the above argument, together with a union bound, on an event of probability one, we have 
\begin{equation*}
\lim_{i \to \infty} \limsup_{n \to \infty} \beta_{i,n}^{\target} 
\leq \lim_{i \to \infty} 1/i
= 0.
\end{equation*}
Thus, since these $\beta_{i,n}^{\target}$ are also nonnegative, Lemma~\ref{lem:converging-array-path} implies that, 
on this event, there exists a sequence $\{i_{n}\}_{n=1}^{\infty}$ in $\nats$, with $i_{n} \leq n$ for every $n \in \nats$, 
such that $\lim\limits_{n \to \infty} \beta_{i_{n},n}^{\target} = 0$.
Therefore, applying Lemma~\ref{lem:general-weighted-majority} to the sequence $\left\{ \hat{f}_{n}^{(1/i)} \right\}_{i=1}^{\infty}$ 
of online learning rules, we conclude that there exists an online learning rule $\hat{f}_{n}$ such that, 
for this process $\ProcX$, for any measurable $\target : \X \to \Y^{\prime}$, under the specification $(\Y,\loss) = (\Y^{\prime},\loss^{\prime})$, 
$\lim\limits_{n \to \infty} \hat{\L}_{\ProcX}\!\left( \hat{f}_{\cdot},\target;n \right) = 0$ almost surely: 
that is, $\hat{f}_{n}$ is strongly universally consistent under $\ProcX$.  In particular, this implies $\ProcX \in \SUOL_{(\Y^{\prime},\loss^{\prime})}$.
Since this argument holds for any $\ProcX \in \SUOL_{(\{0,1\},\loss_{\zo})}$, we conclude that $\SUOL_{(\{0,1\},\loss_{\zo})} \subseteq \SUOL_{(\Y^{\prime},\loss^{\prime})}$.
Combining this with the first part, we have that $\SUOL_{(\Y^{\prime},\loss^{\prime})} = \SUOL_{(\{0,1\},\loss_{\zo})}$, 
and since these arguments apply to any totally bounded $(\Y^{\prime},\loss^{\prime})$ with $\sup\limits_{y,y^{\prime} \in \Y^{\prime}} \loss^{\prime}(y,y^{\prime}) > 0$,
this completes the proof.
\end{proof}

Next, we have the analogous result for losses that are \emph{not} totally bounded.

\begin{theorem}
\label{thm:suol-invariant-to-loss-not-tb}
The set $\SUOL$ is invariant to the specification of $(\Y,\loss)$,
subject to being separable with $\maxloss < \infty$ but \emph{not} totally bounded.
\end{theorem}
\begin{proof}
This proof follows the same line as that of Theorem~\ref{thm:suol-invariant-to-loss}, 
but with a few important differences.
We continue the notational conventions introduced there, 
but in this context we let $\loss_{\zo}$ denote the $0$-$1$ loss on $\nats$:
that is, $\forall y,y^{\prime} \in \nats$, $\loss_{\zo}(y,y^{\prime}) = \ind[ y \neq y^{\prime} ]$.
To establish the theorem, it suffices to verify the claim that 
$\SUOL_{(\Y^{\prime},\loss^{\prime})} = \SUOL_{(\nats,\loss_{\zo})}$
for all separable near-metric spaces $(\Y^{\prime},\loss^{\prime})$
with $\sup\limits_{y,y^{\prime} \in \Y^{\prime}} \loss^{\prime}(y,y^{\prime}) < \infty$
that are \emph{not} totally bounded.
Fix any such space $(\Y^{\prime},\loss^{\prime})$. 

We again begin with the inclusion $\SUOL_{(\Y^{\prime},\loss^{\prime})} \subseteq \SUOL_{(\nats,\loss_{\zo})}$.
For any $\ProcX \in \SUOL_{(\Y^{\prime},\loss^{\prime})}$, 
letting $\hat{g}_{n}$ be an online learning rule that is strongly universally consistent under $\ProcX$ 
(for the specification $(\Y,\loss) = (\Y^{\prime},\loss^{\prime})$), 
we can define an online learning rule $\hat{g}^{\nats}_{n}$ for the specification $(\Y,\loss) = (\nats,\loss_{\zo})$ as follows.
Since $(\Y^{\prime},\loss^{\prime})$ is not totally bounded, $\exists \eps > 0$ such that any $\Y_{\eps}^{\prime} \subseteq \Y^{\prime}$ 
with $\sup\limits_{y \in \Y} \inf\limits_{y_{\eps} \in \Y_{\eps}^{\prime}} \loss^{\prime}(y_{\eps},y) \leq \eps$ necessarily has $|\Y_{\eps}^{\prime}| = \infty$.
In particular, this implies that for any finite sequence $z_{1},\ldots,z_{k} \in \Y^{\prime}$, $k \in \nats$, there exists $z_{k+1} \in \Y^{\prime}$ 
with $\min\limits_{i \leq k} \loss^{\prime}(z_{i},z_{k+1}) > \eps$.  Thus, starting from any initial $z_{1} \in \Y^{\prime}$, we can inductively 
construct an infinite sequence $z_{1},z_{2},\ldots \in \Y^{\prime}$ with $\inf\limits_{i,j \in \nats : i \neq j} \loss^{\prime}(z_{i},z_{j}) \geq \eps > 0$.
For any $n \in \nats \cup \{0\}$, and any sequences $x_{1:(n+1)}$ in $\X$ and $y_{1:n}$ in $\nats$, 
define a sequence $y_{1:n}^{\prime}$ with $y_{i}^{\prime} = z_{y_{i}}$ for each $i \in \{1,\ldots,n\}$,
and then define $\hat{g}^{\nats}_{n}(x_{1:n},y_{1:n},x_{n+1})$ as the (unique) value $y \in \nats$ with
$\loss^{\prime}\!\left(\hat{g}_{n}(x_{1:n},y_{1:n}^{\prime},x_{n+1}),z_{y}\right) < \eps/(2\triconst)$, if such a $y \in \nats$ exists, 
and otherwise define it to be $z_{1}$.  One can easily check that $\hat{g}^{\nats}_{n}$ is a measurable function, 
due to measurability of $\hat{g}_{n}$.
Then for any measurable $f : \X \to \nats$, defining $f^{\prime} : \X \to \Y^{\prime}$ as $f^{\prime}(x) = z_{f(x)}$ (which is clearly also measurable),
we have (under the specification $(\Y,\loss) = (\nats,\loss_{\zo})$)
\begin{align*}
& \limsup_{n \to \infty} \hat{\L}_{\ProcX}\!\left(\hat{g}^{\nats}_{\cdot},f;n\right)
= \limsup_{n \to \infty} \frac{1}{n} \sum_{t=0}^{n-1} \ind\!\left[ \hat{g}^{\nats}_{t}(X_{1:t},f(X_{1:t}),X_{t+1}) \neq f(X_{t+1}) \right]
\\ & \leq \limsup_{n \to \infty} \frac{1}{n} \sum_{t=0}^{n-1} \ind\!\left[ \loss^{\prime}\!\left(\hat{g}_{t}(X_{1:t},f^{\prime}(X_{1:t}),X_{t+1}),f^{\prime}(X_{t+1})\right) \geq \eps/(2\triconst) \right]
\\ & \leq \frac{2\triconst}{\eps} \limsup_{n \to \infty} \frac{1}{n} \sum_{t=0}^{n-1} \loss^{\prime}\!\left(\hat{g}_{t}(X_{1:t},f^{\prime}(X_{1:t}),X_{t+1}),f^{\prime}(X_{t+1})\right)
= 0 \text{ (a.s.)},
\end{align*}
where the last equality (to which the ``almost surely'' qualifier applies) is due to strong universal consistency of $\hat{g}_{n}$
(and the fact that $\eps > 0$).
Since this argument holds for any choice of measurable $f : \X \to \nats$, 
we conclude that $\hat{g}^{\nats}_{n}$ is strongly universally consistent under $\ProcX$ (for the specification $(\Y,\loss) = (\nats,\loss_{\zo})$),
so that $\ProcX \in \SUOL_{(\nats,\loss_{\zo})}$.
Since this argument holds for any $\ProcX \in \SUOL_{(\Y^{\prime},\loss^{\prime})}$, we conclude that $\SUOL_{(\Y^{\prime},\loss^{\prime})} \subseteq \SUOL_{(\nats,\loss_{\zo})}$.

For the converse inclusion, 
fix any $\ProcX \in \SUOL_{(\nats,\loss_{\zo})}$,
and let $\hat{f}^{\nats}_{n}$ be any online learning rule that is 
strongly universally consistent under $\ProcX$ (for the specification $(\Y,\loss) = (\nats,\loss_{\zo})$).
We then define an online learning rule $\hat{f}^{\prime}_{n}$ for the specification $(\Y,\loss) = (\Y^{\prime},\loss^{\prime})$ as follows.
Let $\tilde{\Y}^{\prime}$ be a countable subset of $\Y^{\prime}$ such that $\sup\limits_{y \in \Y^{\prime}} \inf\limits_{\tilde{y} \in \tilde{\Y}^{\prime}} \loss^{\prime}(\tilde{y},y) = 0$; 
such a set $\tilde{\Y}^{\prime}$ is guaranteed to exist by separability of $(\Y^{\prime},\loss^{\prime})$ 
(and furthermore, is necessarily infinite, due to $(\Y^{\prime},\loss^{\prime})$ not being totally bounded).
Enumerate the elements of $\tilde{\Y}^{\prime}$ as $\tilde{y}_{1},\tilde{y}_{2},\ldots$, 
and for each $\eps > 0$ and each $y \in \Y^{\prime}$, 
define $h_{\eps}(y) = \min\!\left\{ i \in \nats : \loss^{\prime}(\tilde{y}_{i},y) \leq \eps \right\}$.
One can easily check that this is a measurable function $\Y^{\prime} \to \nats$.

For any $n \in \nats \cup \{0\}$, and any $x_{1:n} \in \X^{n}$, $y_{1:n} \in (\Y^{\prime})^{n}$, and $x \in \X$,  
define 
$\hat{f}_{n}^{(\eps)}(x_{1:n},y_{1:n},x)$ $= \tilde{y}_{i}$ for $i = \hat{f}^{\nats}_{n}(x_{1:n},h_{\eps}(y_{1:n}),x)$.
That $\hat{f}_{n}^{(\eps)}$ is a measurable function $\X^{n}\times(\Y^{\prime})^{n}\times\X \to \Y^{\prime}$ follows immediately 
from measurability of $\hat{f}^{\nats}_{n}$ and $h_{\eps}$.
Thus, $\hat{f}_{n}^{(\eps)}$ defines an online learning rule.
Now, for any measurable function $f : \X \to \Y^{\prime}$, 
the composed function $x \mapsto h_{\eps}(f(x))$ is a measurable function $\X \to \nats$,
and therefore (by strong universal consistency of $\hat{f}^{\nats}_{n}$)
\begin{equation*}
\limsup_{n \to \infty} \frac{1}{n} \sum_{t = 0}^{n-1} \loss_{\zo}\!\left( \hat{f}^{\nats}_{t}(X_{1:t},h_{\eps}(f(X_{1:t})),X_{t+1}), h_{\eps}(f( X_{t+1} )) \right) = 0 \text{ (a.s.)}.
\end{equation*}
We therefore have that, under the specification $(\Y,\loss) = (\Y^{\prime},\loss^{\prime})$, 
\begin{align*}
& \limsup_{n \to \infty} \hat{\L}_{\ProcX}\!\left(\hat{f}_{\cdot}^{(\eps)},f;n\right)
= \limsup_{n \to \infty} \frac{1}{n} \sum_{t=0}^{n-1} \loss^{\prime}\!\left( \hat{f}^{(\eps)}_{t}(X_{1:t},f(X_{1:t}),X_{t+1}), f(X_{t+1}) \right)
\\ & \leq \limsup_{n \to \infty} \frac{1}{n} \sum_{t=0}^{n-1} \left( \loss^{\prime}(\tilde{y}_{h_{\eps}(f(X_{t+1}))},f(X_{t+1})) \!+\! \maxloss \ind\!\left[ \hat{f}^{\nats}_{t}(X_{1:t},h_{\eps}(f(X_{1:t})),X_{t+1}) \!\neq\! h_{\eps}(f(X_{t+1})) \right] \right)
\\ & \leq \eps + \maxloss \limsup_{n \to \infty} \frac{1}{n} \sum_{t=0}^{n-1} \loss_{\zo}\!\left( \hat{f}^{\nats}_{t}(X_{1:t},h_{\eps}(f(X_{1:t})),X_{t+1}), h_{\eps}(f(X_{t+1})) \right)
= \eps \text{ (a.s.)}.
\end{align*}

The rest of this proof follows identically to the analogous part of the proof of Theorem~\ref{thm:suol-invariant-to-loss}.
Briefly, for any measurable $\target : \X \to \Y^{\prime}$, for each $i,n \in \nats$, 
defining $\beta_{i,n}^{\target} = \hat{\L}_{\ProcX}\!\left(\hat{f}_{\cdot}^{(1/i)},\target;n\right)$ (under the specification $(\Y,\loss) = (\Y^{\prime},\loss^{\prime})$),
by the union bound, on an event of probability one, we have 
\begin{equation*}
\lim_{i \to \infty} \limsup_{n \to \infty} \beta_{i,n}^{\target} 
\leq \lim_{i \to \infty} 1/i
= 0.
\end{equation*}
Therefore Lemma~\ref{lem:converging-array-path} (with $j_{n} = n$) 
and Lemma~\ref{lem:general-weighted-majority} imply that there exists an online learning rule $\hat{f}_{n}$ such that, 
for this process $\ProcX$, for any measurable $\target : \X \to \Y^{\prime}$, under the specification $(\Y,\loss) = (\Y^{\prime},\loss^{\prime})$, 
$\lim\limits_{n \to \infty} \hat{\L}_{\ProcX}\!\left( \hat{f}_{\cdot},\target;n \right) = 0$ almost surely.
This implies $\ProcX \in \SUOL_{(\Y^{\prime},\loss^{\prime})}$.
Since this argument holds for any $\ProcX \in \SUOL_{(\nats,\loss_{\zo})}$, we conclude $\SUOL_{(\nats,\loss_{\zo})} \subseteq \SUOL_{(\Y^{\prime},\loss^{\prime})}$.
Combining this with the first part, we have $\SUOL_{(\Y^{\prime},\loss^{\prime})} = \SUOL_{(\nats,\loss_{\zo})}$,
and since these arguments apply to any separable near-metric space $(\Y^{\prime},\loss^{\prime})$
with $\sup\limits_{y,y^{\prime} \in \Y^{\prime}} \loss^{\prime}(y,y^{\prime}) < \infty$ that is not totally bounded,
this completes the proof.
\end{proof}

Since the reductions used to construct the learning rules in the above two proofs do not explicitly 
depend on the distribution of the process $\ProcX$, these proofs also establish another interesting property: 
namely, invariance to the specification of $(\Y,\loss)$ in the existence of optimistically universal online learning rules.
Specifically, the proofs of Theorems~\ref{thm:suol-invariant-to-loss} and \ref{thm:suol-invariant-to-loss-not-tb} 
can also be used to establish the following corollary. 

\begin{corollary}
\label{cor:optimistic-online-equiv}
For any separable near-metric space $(\Y^{\prime},\loss^{\prime})$ with $0 < \sup\limits_{y,y^{\prime} \in \Y^{\prime}} \loss^{\prime}(y,y^{\prime}) < \infty$, the following hold.
\begin{itemize}
\item[$\bullet$] If $(\Y^{\prime},\loss^{\prime})$ is totally bound, 
there exists an optimistically universal online learning rule when $(\Y,\loss) = (\Y^{\prime},\loss^{\prime})$
if and only if there exists an optimistically universal online learning rule when $(\Y,\loss) = (\{0,1\},\loss_{\zo})$.  
\item[$\bullet$] If $(\Y^{\prime},\loss^{\prime})$ is not totally bound, 
there exists an optimistically universal online learning rule when $(\Y,\loss) = (\Y^{\prime},\loss^{\prime})$
if and only if there exists an optimistically universal online learning rule when $(\Y,\loss) = (\nats,\loss_{\zo})$. 
\end{itemize}
\end{corollary}

The question of whether the two $\SUOL$ sets from the above 
Theorems~\ref{thm:suol-invariant-to-loss} and \ref{thm:suol-invariant-to-loss-not-tb}  
are equivalent remains an interesting open problem.

\begin{problem}
\label{prob:suol-multiclass}
Is the set $\SUOL$ invariant to the specification of $(\Y,\loss)$,
subject to $(\Y,\loss)$ being separable with $0 < \maxloss < \infty$?
\end{problem}

In particular, in the notation of the above proofs, 
Theorems~\ref{thm:suol-invariant-to-loss} and \ref{thm:suol-invariant-to-loss-not-tb}  imply 
this problem is equivalent to the question of whether 
$\SUOL_{(\{0,1\},\loss_{\zo})} = \SUOL_{(\nats,\loss_{\zo})}$: that is, whether the set of processes that admit strong universal online learning 
is the same for \emph{binary} classification as for \emph{multiclass} classification with a \emph{countably infinite} number of possible classes.

\section{No Consistent Test for Existence of a Universally Consistent Learner}
\label{sec:no-consistent-test}

It is also interesting to ask to what extent admission of universal consistency is actually an \emph{assumption}, 
rather than a testable hypothesis: that is, is there any way to \emph{detect} whether or not a given data sequence $\ProcX$
admits strong universal learning (in any of the above senses)? 
It turns out the answer is \emph{no}.

In our present context, a \emph{hypothesis test} is a sequence of (possibly random)\footnote{In 
the case of random $\hat{t}_{n}$, we will suppose $\hat{t}_{n}$ is independent from $\ProcX$.}
measurable functions $\hat{t}_{n} : \X^{n} \to \{0,1\}$, $n \in \nats \cup \{0\}$.
We say $\hat{t}_{n}$ is \emph{consistent} for a set of processes $\ProcSet$ if, 
for every 
$\ProcX \in \ProcSet$, $\hat{t}_{n}(X_{1:n}) \xrightarrow{P} 1$,
and for every $\ProcX \notin \ProcSet$, $\hat{t}_{n}(X_{1:n}) \xrightarrow{P} 0$.
We have the following theorem.\footnote{There is actually a fairly simple proof of this 
theorem if $\X$ is uncountable and $(\X,\T)$ is a Polish space.  In that case, we can 
simply use the fact that no test can distinguish between an i.i.d.\ process with a 
given nonatomic marginal distribution versus a deterministic process chosen randomly
among the sample paths of the i.i.d.\ process.  However, the proof we present here 
has the advantage of applying also to countable $\X$, and indeed it remains valid 
even if we restrict to \emph{deterministic} processes.}

\begin{theorem}
\label{thm:no-consistent-test-for-suil}
If $\X$ is infinite, there is no consistent hypothesis test for $\SUIL$, $\SUAL$, or $\SUOL$.
\end{theorem}
\begin{proof}
Suppose $\X$ is infinite and fix any hypothesis test $\hat{t}_{n}$.
Let $\{w_{i}\}_{i=0}^{\infty}$ be any sequence of distinct elements of $\X$.
We construct a process $\ProcX$ inductively, as follows.
Let $n_{0} = 0$.
For the purpose of this inductive definition, suppose, for some $k \in \nats$,
that $n_{k-1}$ is defined, and that $X_{t}$ is defined for every $t \in \nats$ with $t \leq n_{k-1}$.
Let $X_{t}^{(k)} = X_{t}$ for every $t \in \nats$ with $t \leq n_{k-1}$.
If $(k+1)/2 \in \nats$ (i.e., $k$ is odd), 
then let $X_{t}^{(k)} = w_{0}$ for every $t \in \nats$ with $t > n_{k-1}$.
Otherwise, if $k/2 \in \nats$ (i.e., $k$ is even),
then let $X_{t}^{(k)} = w_{t}$ for every $t \in \nats$ with $t > n_{k-1}$.
If $\exists n \in \nats$ with $n > n_{k-1}$ such that 
\begin{equation}
\label{eqn:no-consistent-test-switch-condition}
\P\!\left( \hat{t}_{n}(X_{1:n}^{(k)}) = \ind[ (k+1)/2 \in \nats ] \right) > 1/2,
\end{equation}
then define $n_{k} = n$ for some such value of $n$,
and define $X_{t} = X_{t}^{(k)}$ for every $t \in \{n_{k-1}+1,\ldots,n_{k}\}$.
Otherwise, if no such $n$ exists, define $X_{t} = X_{t}^{(k)}$ for every $t \in \nats$ with $t > n_{k-1}$,
in which case the inductive definition is complete (upon reaching the smallest value of $k$
for which no such $n$ exists).

The above inductive definition specifies a deterministic process $\ProcX$.
Now consider two cases.
First, suppose there is a maximum value $k^{*}$ of $k \in \nats$
for which $n_{k-1}$ is defined.  In this case, there is no $n > n_{k^{*}-1}$ 
satisfying \eqref{eqn:no-consistent-test-switch-condition} with $k = k^{*}$.
Furthermore, by the definition of $X_{t}^{(k^{*})}$ for every $t \leq n_{k^{*}-1}$,
and by our choice of $X_{t}$ for every $t > n_{k^{*}-1}$, we have $\ProcX = \{X_{t}^{(k^{*})}\}_{t=1}^{\infty}$.
Together, these imply that $\forall n \in \nats$ with $n > n_{k^{*}-1}$,
\begin{equation}
\label{eqn:no-consistent-test-bounded-prob}
\P\!\left( \hat{t}_{n}(X_{1:n}) = \ind[ (k^{*}+1)/2 \in \nats ] \right) \leq 1/2.
\end{equation}

If $(k^{*}+1)/2 \in \nats$, then $X_{t} = w_{0}$ for every $t \in \nats$ with $t > n_{k^{*}-1}$.
In this case, for any $A \in \Borel$, $\hat{\mu}_{\ProcX}(A) = \ind_{A}(w_{0})$.
Thus, for any monotone sequence $\{A_{i}\}_{i=1}^{\infty}$ of sets in $\Borel$ with $A_{i} \downarrow \emptyset$, 
$\lim\limits_{i \to \infty} \E\!\left[ \hat{\mu}_{\ProcX}(A_{i}) \right] = \lim\limits_{i \to \infty} \ind_{A_{i}}(w_{0}) = \ind_{\lim\limits_{i \to \infty} A_{i}}(w_{0}) = \ind_{\emptyset}(w_{0}) = 0$.
Therefore, $\ProcX$ satisfies Condition~\ref{con:kc} (i.e., $\ProcX \in \KC$).
Since Theorem~\ref{thm:main} implies $\SUIL = \SUAL = \KC$, 
we also have that $\ProcX \in \SUIL$ and $\ProcX \in \SUAL$.
Also, since Theorem~\ref{thm:suil-subset-suol} implies $\SUAL \subset \SUOL$, 
we have $\ProcX \in \SUOL$ as well.
However, \eqref{eqn:no-consistent-test-bounded-prob} implies
$\limsup\limits_{n \to \infty} \P( \hat{t}_{n}(X_{1:n}) \neq 1 ) \geq 1/2$, 
so that $\hat{t}_{n}(X_{1:n})$ fails to converge in probability to $1$,
and hence $\hat{t}_{n}$ is not consistent for any of $\SUIL$, $\SUAL$, or $\SUOL$.

On the other hand, if $(k^{*}+1)/2 \notin \nats$, then $X_{t} = w_{t}$ for every $t \in \nats$ with $t > n_{k^{*}-1}$.
In this case, letting $A_{i} = \{ w_{i} \} \in \Borel$ for each $i \in \nats$, 
these $A_{i}$ sets are disjoint, and for any $T \in \nats$, 
$| \{ i \in \nats : X_{1:T} \cap A_{i} \neq \emptyset \} | \geq T - n_{k^{*}-1} \neq o(T)$, 
so that $\ProcX$ fails to satisfy Condition~\ref{con:okc}: that is, $\ProcX \notin \OKC$.
Since Theorem~\ref{thm:okc-nec} implies $\SUOL \subseteq \OKC$,
and Theorems~\ref{thm:main} and \ref{thm:suil-subset-suol} imply $\SUIL = \SUAL \subset \SUOL$,
we also have that $\ProcX \notin \SUOL$, $\ProcX \notin \SUAL$, and $\ProcX \notin \SUIL$.
However, \eqref{eqn:no-consistent-test-bounded-prob} implies
$\limsup\limits_{n \to \infty} \P( \hat{t}_{n}(X_{1:n}) \neq 0 ) \geq 1/2$,
so that $\hat{t}_{n}(X_{1:n})$ fails to converge in probability to $0$,
and hence $\hat{t}_{n}$ is not consistent for any of $\SUIL$, $\SUAL$, or $\SUOL$.

For the remaining case, suppose $n_{k}$ is defined for all $k \in \nats \cup \{0\}$,
so that $\{n_{k}\}_{k=0}^{\infty}$ is an infinite strictly-increasing sequence of nonnegative integers.
For each $k \in \nats$, our choice of $n_{k}$ guarantees that \eqref{eqn:no-consistent-test-switch-condition}
is satisfied with $n = n_{k}$.
Furthermore, for every $k \in \nats$, our definition of $X_{t}^{(k)}$ for values $t \leq n_{k-1}$,
and our choice of $X_{t}$ for values $t \in \{n_{k-1}+1,\ldots,n_{k}\}$ imply that $X_{1:n_{k}} = X_{1:n_{k}}^{(k)}$.
Thus, every $k \in \nats$ satisfies $\P( \hat{t}_{n_{k}}(X_{1:n_{k}}) = \ind[ (k+1)/2 \in \nats ] ) > 1/2$.
In particular, this implies that 
\begin{equation*}
\limsup_{n \to \infty} \P( \hat{t}_{n}(X_{1:n}) \neq 1 ) \geq \limsup_{j \to \infty} \P( \hat{t}_{n_{2j}}(X_{1:n_{2j}}) = 0 ) \geq 1/2,
\end{equation*}
while
\begin{equation*}
\limsup_{n \to \infty} \P( \hat{t}_{n}(X_{1:n}) \neq 0 ) \geq \limsup_{j \to \infty} \P( \hat{t}_{n_{2j+1}}(X_{1:n_{2j+1}}) = 1 ) \geq 1/2.
\end{equation*}
Thus, $\hat{t}_{n}(X_{1:n})$ fails to converge in probability to any value: that is, 
it neither converges in probability to $0$ nor converges in probability to $1$.
Therefore, in this case as well, we find that $\hat{t}_{n}$ is not consistent for any of $\SUIL$, $\SUAL$, or $\SUOL$.

Thus, regardless of which of these is the case, we have established that $\hat{t}_{n}$ is not a consistent test for $\SUIL$, $\SUAL$, or $\SUOL$.
\end{proof}

Recall that, if $\X$ is \emph{finite}, \emph{every} $\ProcX$ admits strong universal inductive learning:
any sequence $A_{k} \downarrow \emptyset$ has $A_{k} = \emptyset$ for all sufficiently large $k$, 
so that every $\ProcX$ has $\lim\limits_{k \to \infty} \E\!\left[ \hat{\mu}_{\ProcX}(A_{k}) \right] = \hat{\mu}_{\ProcX}(\emptyset) = 0$,
and hence satisfies Condition~\ref{con:kc}, which implies $\ProcX \in \SUIL \cap \SUAL \cap \SUOL$ by Theorems~\ref{thm:main} and \ref{thm:suil-subset-suol}.
Therefore, the \emph{constant} function $\hat{t}_{n}(\cdot) = 1$ is a consistent test for $\SUIL$, $\SUAL$, and $\SUOL$ in this case.
Thus, we may conclude the following corollary.

\begin{corollary}
\label{thm:consistent-test-nec-suf}
There exist consistent hypothesis tests for each of $\SUIL$, $\SUAL$, and $\SUOL$ if and only if $\X$ is finite.
\end{corollary}

Note that, since Theorem~\ref{thm:main} implies $\SUIL = \KC$, 
this corollary also holds for consistent tests of $\KC$.  It is also easy to 
see that the proof above can further extend this corollary to consistent tests of $\OKC$ as well.

\section{Unbounded Losses}
\label{sec:unbounded-losses}

In this section, we depart from the above discussion by considering the case of unbounded losses.
Specifically, we retain the assumption that $(\Y,\loss)$ is a separable near-metric space, but now we
replace the assumption that $\loss$ is bounded (i.e., $\maxloss < \infty$) with the complementary assumption that $\maxloss = \infty$.
To be clear, we suppose $\loss(y_{1},y_{2})$ is finite for every $y_{1},y_{2} \in \Y$, but is \emph{unbounded},
in that $\sup\limits_{y_{1},y_{2} \in \Y} \loss(y_{1},y_{2}) = \infty$.  All of the other restrictions from Section~\ref{subsec:notation} (e.g., that $(\Y,\loss)$ is a separable near-metric space) 
remain unchanged.
In this setting, we find that the condition necessary and sufficient for a process to admit universal learning becomes significantly stronger.
Indeed,
not even all i.i.d.\ processes admit universal learning when $\maxloss = \infty$.
However, we are nevertheless able to establish results on the existence of optimistically universal learning rules and consistent tests.
We again find that the set of processes admitting strong universal learning is invariant to $\loss$ (subject to $\maxloss = \infty$),
and specified by a simple condition.  Specifically, consider the following condition.

\begin{condition}
\label{con:ukc}
Every monotone sequence $\{A_{k}\}_{k=1}^{\infty}$ of sets in $\Borel$ with $A_{k} \downarrow \emptyset$ satisfies
\begin{equation*}
\left| \left\{ k \in \nats : \ProcX \cap A_{k} \neq \emptyset \right\} \right| < \infty \text{ (a.s.)}.
\end{equation*}
\end{condition}

We denote by $\UKC$ the set of processes $\ProcX$ satisfying Condition~\ref{con:ukc}.
We can also state an equivalent form of Condition~\ref{con:ukc} 
in terms of countable measurable partitions of $\X$, as follows.

\begin{lemma}
\label{lem:ukc-partition-equiv}
A process $\ProcX$ satisfies Condition~\ref{con:ukc} if and only if every 
disjoint sequence $\{A_{i}\}_{i=1}^{\infty}$ in $\Borel$ with $\bigcup\limits_{i=1}^{\infty} A_{i} = \X$ 
(i.e., every countable measurable partition) satisfies
\begin{equation}
\label{eqn:ukc-partition-equiv}
\left| \left\{ k \in \nats : \ProcX \cap A_{k} \neq \emptyset \right\} \right| < \infty \text{ (a.s.)}.
\end{equation}
\end{lemma}
\begin{proof}
First suppose $\ProcX$ satisfies Condition~\ref{con:ukc}.
Given any disjoint sequence $\{A_{k}\}_{k=1}^{\infty}$ in $\Borel$ with $\bigcup\limits_{k=1}^{\infty} A_{k} = \X$, 
we can define a sequence $B_{k} = \bigcup\limits_{i=k}^{\infty} A_{i}$ in $\Borel$ with $B_{k} \downarrow \emptyset$.
Then note that 
$\left| \left\{ k \in \nats : \ProcX \cap A_{k} \neq \emptyset \right\} \right| 
\leq \sup\left\{ k \in \nats : \ProcX \cap A_{k} \neq \emptyset \right\} 
= \left| \left\{ k \in \nats : \ProcX \cap B_{k} \neq \emptyset \right\} \right|$,
and Condition~\ref{con:ukc} implies the rightmost expression is finite almost surely.
Thus, \eqref{eqn:ukc-partition-equiv} holds for all such sequences $\{A_{k}\}_{k=1}^{\infty}$.

For the converse direction, suppose $\ProcX$ satisfies \eqref{eqn:ukc-partition-equiv} 
for every disjoint sequence $\{A_{i}\}_{i=1}^{\infty}$ in $\Borel$ with $\bigcup\limits_{i=1}^{\infty} A_{i} = \X$.
Let $\{B_{k}\}_{k=1}^{\infty}$ be any monotone sequence in $\Borel$ with $B_{k} \downarrow \emptyset$, 
and for simplicity also define $B_{0} = \X$. 
We can define a disjoint sequence $\{A_{k}\}_{k=1}^{\infty}$ in $\Borel$ with $\bigcup\limits_{k=1}^{\infty} A_{k} = \X$ 
by letting $A_{k} = B_{k-1} \setminus B_{k}$ for each $k \in \nats$.
Then note that 
$\left| \left\{ k \in \nats : \ProcX \cap B_{k} \neq \emptyset \right\} \right|
= \sup\!\left\{ k \in \nats \cup \{0\} : \ProcX \cap B_{k} \neq \emptyset \right\}
= \sup\{ k \in \nats \cup \{0\} : \ProcX \cap A_{k+1} \neq \emptyset \}$, 
and this rightmost quantity is finite if and only if 
$| \{ k \in \nats : \ProcX \cap A_{k} \neq \emptyset \} | < \infty$.
Together with \eqref{eqn:ukc-partition-equiv}, this implies  
$\left| \left\{ k \in \nats : \ProcX \cap B_{k} \neq \emptyset \right\} \right| < \infty \text{ (a.s.)}$.
Since this holds for every such sequence $\{B_{k}\}_{k=1}^{\infty}$, it follows that $\ProcX$ satisfies Condition~\ref{con:ukc}.
\end{proof}

It is straightforward to see that any process satisfying Condition~\ref{con:ukc} necessarily also satisfies Condition~\ref{con:kc}: i.e., $\UKC \subseteq \KC$.
Specifically, for any $\ProcX \in \UKC$, for any sequence $\{A_{k}\}_{k=1}^{\infty}$ in $\Borel$ with $A_{k} \downarrow \emptyset$, 
with probability one every sufficiently large $k$ has $\ProcX \cap A_{k} = \emptyset$, which implies $\lim\limits_{k \to \infty} \hat{\mu}_{\ProcX}(A_{k}) = 0$;
thus, $\ProcX \in \KC$ by Lemma~\ref{lem:limsup-equiv}. 

Condition~\ref{con:ukc} will turn out to be the key condition for determining whether a 
given process admits strong universal learning (in \emph{any} of the three protocols: inductive, self-adaptive, or online) when the loss is unbounded,
analogous to the role of Condition~\ref{con:kc} for the case of bounded losses in inductive and self-adaptive learning. 
This is stated formally in the following theorem.

\begin{theorem}
\label{thm:unbounded-main}
When $\maxloss = \infty$, the following statements are equivalent for any process $\ProcX$.
\begin{itemize}
\item[$\bullet$] $\ProcX$ satisfies Condition~\ref{con:ukc}.
\item[$\bullet$] $\ProcX$ admits strong universal inductive learning.
\item[$\bullet$] $\ProcX$ admits strong universal self-adaptive learning.
\item[$\bullet$] $\ProcX$ admits strong universal online learning.
\end{itemize}
Equivalently, when $\maxloss = \infty$, $\SUOL = \SUAL = \SUIL = \UKC$.
\end{theorem}

We present the proof of this result in Section~\ref{subsec:unbounded-main-result} below.
One remarkable consequence of this result is that, unlike Theorem~\ref{thm:main} for bounded losses, 
this theorem includes \emph{online} learning among the equivalences.  This is noteworthy 
for two reasons.  First, in the case of bounded losses, we found (in Theorem~\ref{thm:suil-subset-suol}) 
that $\SUOL$ is typically \emph{not} equivalent to $\SUAL$ and $\SUIL$, instead forming a 
strict superset of these.  This therefore creates an interesting distinction between bounded 
and unbounded losses regarding the relative strengths of these settings.  A second interesting 
contrast to the above analysis of bounded losses is that, in the case of unbounded losses, 
Theorem~\ref{thm:unbounded-main} establishes a concise condition that 
is necessary and sufficient for a process to admit strong universal online learning; this 
contrasts with the analysis of online learning for bounded losses in Section~\ref{sec:online},
where we fell short of provably establishing a concise characterization of the processes 
admitting strong universal online learning (see Open Problem~\ref{prob:suol-equals-okc}).

In addition to the above equivalence, we also find that in \emph{all three} learning settings studied here, 
for unbounded losses, there exist optimistically universal learning rules.  
This again contrasts with the results for bounded losses, in the inductive setting (cf.\ Theorem~\ref{thm:no-optimistic-inductive}). 
We have the following theorem, 
the proof of which is given in Section~\ref{subsec:unbounded-main-result} below.

\begin{theorem}
\label{thm:unbounded-optimistically-universal}
When $\maxloss = \infty$, there exists an optimistically universal (inductive / self-adaptive / online) learning rule.
\end{theorem}

Indeed, we find that effectively the \emph{same} learning strategy, 
described in \eqref{eqn:unbounded-universal2-inductive-rule} below, 
suffices for optimistically universal learning in all three of these settings.

\subsection{A Question Concerning the Number of Distinct Values}
\label{subsec:unbounded-examples}

It is worth noting that Condition~\ref{con:ukc} is quite restrictive.
In fact, it is even violated by many i.i.d. processes: namely, all those with the marginal distribution of $X_{t}$ having infinite support.
Clearly any process $\ProcX$ such that the number of distinct points $X_{t}$ is (almost surely) finite 
satisfies Condition~\ref{con:ukc}.  Indeed, for deterministic processes or for countable $\X$, one can easily
show that this is \emph{equivalent} to Condition~\ref{con:ukc}.  
But in general, it is not presently known whether there exist processes $\ProcX$ satisfying Condition~\ref{con:ukc} 
for which the number of distinct $X_{t}$ values is \emph{infinite} with nonzero probability.  Thus we have 
the following open question.

\begin{problem}
\label{prob:ukc-infinite-support}
For some uncountable $\X$, does there exist $\ProcX \in \UKC$ such that,
with nonzero probability, 
$|\{ x \in \X : \ProcX \cap \{x\} \neq \emptyset \}| = \infty$?
\end{problem}

Either answer to this question would be interesting.
If no such processes $\ProcX$ exist, then the proof of Theorem~\ref{thm:unbounded-main} below could be dramatically simplified, 
since it would then be completely trivial to construct a strongly universally consistent learning rule (in any of the three settings) under $\ProcX \in \UKC$, 
simply using memorization (once $n$ is sufficiently large, all the distinct points will have been observed in the training sample).
On the other hand, if there do exist such processes, then it would indicate that $\UKC$ is in fact a somewhat rich family of processes,
and that the learning problem is indeed nontrivial.
It is straightforward to show that, if such processes do exist for $\X = [0,1]$ (with the standard topology), 
then there would also exist processes of this type that are \emph{convergent} (to a nondeterministic limit point) almost surely;\footnote{For 
instance, for $\{U_{t}\}_{t=0}^{\infty}$ i.i.d.\ ${\rm Uniform}(0,2/3)$, the process $X_{t} = U_{0} + 2^{-t} U_{t}$ is convergent to 
the nondeterministic limit $U_{0}$.}
thus, in attempting to answer Open Problem~\ref{prob:ukc-infinite-support} (in the case of $\X = [0,1]$), 
it suffices to focus on convergent processes.

\subsection{An Equivalent Condition}
\label{subsec:unbounded-equivalent}

Before getting into the discussion of consistency under processes in $\UKC$, 
we first note an elegant equivalent formulation of the condition, which may 
help to illuminate its relevance to the problem of learning with unbounded losses.
Specifically, we have the following result.

\begin{lemma}
\label{lem:unbounded-effectively-bounded}
A process $\ProcX$ satisfies Condition~\ref{con:ukc} if and only if
every measurable function $f : \X \to \reals$ satisfies
\begin{equation*}
\sup_{t \in \nats} f(X_{t}) < \infty \text{ (a.s.)}.
\end{equation*}
\end{lemma}
\begin{proof}
First, suppose $\ProcX \in \UKC$, and fix any measurable $f : \X \to \reals$.
For each $k \in \nats$, define $A_{k} = f^{-1}([k-1,\infty))$.
Since $f(x) < \infty$ for every $x \in \X$, we have $A_{k} \downarrow \emptyset$.
Thus, by the definition of $\UKC$,
with probability one $\exists k_{0} \in \nats$ such that $\ProcX \cap A_{k_{0}+1} = \emptyset$;
in other words, with probability one, $\exists k_{0} \in \nats$ such that every $t \in \nats$ has $f(X_{t}) < k_{0}$, so that 
$\sup\limits_{t \in \nats} f(X_{t}) \leq k_{0} < \infty$.

For the other direction, suppose $\ProcX$ is such that every measurable $f : \X \to \reals$ satisfies
$\sup\limits_{t \in \nats} f(X_{t}) < \infty \text{ (a.s.)}$.
Fix any monotone sequence $\{A_{k}\}_{k=1}^{\infty}$ of sets in $\Borel$ with $A_{k} \downarrow \emptyset$,
and define a function $f : \X \to \reals$ such that, $\forall x \in \X$, $f(x) = \sum\limits_{k=1}^{\infty} \ind_{A_{k}}(x) = \left| \left\{ k \in \nats : x \in A_{k} \right\} \right|$.
Note that, since $A_{k} \downarrow \emptyset$, we indeed have $f(x) \in \reals$ for every $x \in \X$.
Furthermore,
$f$ is clearly measurable (being a limit of simple functions).
Therefore $\sup\limits_{t \in \nats} f(X_{t}) < \infty \text{ (a.s.)}$.
Also note that monotonicity of the sequence $\{A_{k}\}_{k=1}^{\infty}$ implies 
$\forall x \in \X$, $f(x)  = \max\!\left(\left\{ k \!\in\! \nats : x \in A_{k} \right\} \cup \{0\} \right)$.
Thus, defining $\hat{k} = \sup\limits_{t \in \nats} f(X_{t})$, on the event (of probability one) that $\hat{k} < \infty$,
every $k \in \nats$ with $k > \hat{k}$ has $\ProcX \cap A_{k} = \emptyset$,
so that $| \{ k \in \nats : \ProcX \cap A_{k} \neq \emptyset \} | \leq \hat{k} < \infty$
(in fact, the first inequality holds with equality).
Since this holds for any choice of monotone sequence $\{A_{k}\}_{k=1}^{\infty}$ in $\Borel$ with $A_{k} \downarrow \emptyset$,
we have that $\ProcX \in \UKC$.
\end{proof}

\subsection{Proofs of the Main Results for Unbounded Losses}
\label{subsec:unbounded-main-result}

This subsection presents the proofs of Theorems~\ref{thm:unbounded-main} and \ref{thm:unbounded-optimistically-universal}.
As with Theorem~\ref{thm:main}, we prove Theorem~\ref{thm:unbounded-main} via a sequence
of lemmas, corresponding to the implications among the various statements claimed to be equivalent.
The first of these is analogous to Lemma~\ref{lem:suil-subset-sual}, showing that processes admitting
strong universal inductive learning also admit strong universal self-adaptive learning.  The proof is identical
to that of Lemma~\ref{lem:suil-subset-sual}, and as such is omitted.

\begin{lemma}
\label{lem:unbounded-suil-subset-sual}
When $\maxloss = \infty$, $\SUIL \subseteq \SUAL$.
\end{lemma}

Next, we have a result analogous to Lemma~\ref{lem:sual-subset-kc}, showing that 
any process admitting strong universal self-adaptive or online learning necessarily satisfies Condition~\ref{con:ukc}.

\begin{lemma}
\label{lem:unbounded-sual-subset-ukc}
When $\maxloss = \infty$, $\SUAL \cup \SUOL \subseteq \UKC$.
\end{lemma}
\begin{proof}
Fix any $\ProcX$ that fails to satisfy Condition~\ref{con:ukc}.
Then there exists a monotone sequence $\{B_{k}\}_{k=1}^{\infty}$ in $\Borel$ with $B_{k} \downarrow \emptyset$
such that, on a $\sigma(\ProcX)$-measurable event $E$ of probability strictly greater than zero,
\begin{equation}
\label{eqn:unbounded-nec-ubkc-violated}
\left|\left\{ k \in \nats : \ProcX \cap B_{k} \neq \emptyset \right\}\right| = \infty.
\end{equation}
Furthermore, monotonicity of $B \mapsto \ProcX \cap B$ implies that, without loss of generality, we may suppose $B_{1} = \X$.
Also, by monotonicity of $\{B_{k}\}_{k=1}^{\infty}$,
on the event $E$, 
\eqref{eqn:unbounded-nec-ubkc-violated} implies that
\begin{equation}
\label{eqn:unbounded-nec-B-every}
\forall k \in \nats, \ProcX \cap B_{k} \neq \emptyset.
\end{equation}
Now for each $i \in \nats$, define $A_{i} = B_{i} \setminus B_{i+1}$.  
Note that, due to monotonicity of the $\{B_{k}\}_{k=1}^{\infty}$ sequence and the facts that $B_{k} \downarrow \emptyset$ and $B_{1} = \X$, 
$\{A_{i}\}_{i=1}^{\infty}$ is a disjoint sequence in $\Borel$ with $\bigcup\limits_{i=1}^{\infty} A_{i} = \X$.
Thus, for every $t \in \nats$, there exists a unique ($X_{t}$-dependent) variable $i_{t} \in \nats$ with $X_{t} \in A_{i_{t}}$.
Also note that every $j \in \nats$ has $B_{j} = \bigcup\limits_{i \geq j} A_{i}$, 
again due to monotonicity of $\{B_{k}\}_{k=1}^{\infty}$ and the fact that $B_{k} \downarrow \emptyset$.

For each $j \in \nats$, define a random variable 
\begin{equation*}
\tau_{j} = \begin{cases}
\min\!\left\{ t \in \nats : X_{t} \in B_{j} \right\}, & \text{ if } \ProcX \cap B_{j} \neq \emptyset\\
0, & \text{ otherwise}
\end{cases}.
\end{equation*}
Note that, on the event $E$, \eqref{eqn:unbounded-nec-B-every} implies that we
have 
$\tau_{j} = \min\{ t \in \nats : X_{t} \in B_{j} \}$
for every $j \in \nats$ (and that this minimum exists and is well-defined).
Let $\{T_{j}\}_{j=1}^{\infty}$ be a nondecreasing sequence of (nonrandom) values in $\nats \cup \{0\}$ such that,
for each $j \in \nats$, 
\begin{equation*}
\P\!\left( \tau_{j} > T_{j} \right) < 2^{-j}.
\end{equation*}
Such a sequence must exist, since $\tau_{j}$ is almost surely finite, so that $\lim\limits_{t \to \infty} \P( \tau_{j} > t ) = 0$ \citep*[e.g.,][Theorem A.19]{schervish:95}.
Since $\sum\limits_{j=1}^{\infty} \P\left( \tau_{j} > T_{j} \right) < \sum\limits_{j=1}^{\infty} 2^{-j} = 1 < \infty$,
the Borel-Cantelli Lemma implies that, on a $\sigma(\ProcX)$-measurable event $E^{\prime}$ of probability one, 
$\exists \iota_{0} \in \nats$ such that $\forall j \in \nats$ with $j \geq \iota_{0}$, 
$\tau_{j} \leq T_{j}$.
For each $i \in \nats$, let $y_{i,0},y_{i,1} \in \Y$ be such that $\loss(y_{i,0},y_{i,1}) > T_{i}$.
For every $\kappa \in [0,1)$ and $i \in \nats$, define $\kappa_{i} = \lfloor 2^{i} \kappa \rfloor - 2 \lfloor 2^{i-1} \kappa \rfloor$:
the $i^{{\rm th}}$ bit of the binary representation of $\kappa$.
Then for each $\kappa \in [0,1)$, $i \in \nats$, and $x \in A_{i}$, define $\target_{\kappa}(x) = y_{i,\kappa_{i}}$.
Note that $(x,\kappa) \mapsto \target_{\kappa}(x)$ is measurable in the product $\sigma$-algebra 
(under $\Borel$ for the $x$ argument, and the usual Borel $\sigma$-algebra on $[0,1)$ for the $\kappa$ argument), 
since the inverse image of any measurable set $C \subseteq \Y$ is 
a countable union of measurable rectangle sets: namely, 
$\bigcup\limits_{\substack{i \in \nats, b \in \{0,1\} : \\  y_{i,b} \in C}} ( A_{i} \times \{ \kappa : \kappa_{i} = b \} )$.

For the purpose of treating both self-adaptive and online learning simultaneously,
for any $n,m \in \nats \cup \{0\}$, let $f_{n,m}$ denote any (possibly random) measurable function $\X^{m} \times \Y^{m} \times \X \to \Y$.
We will see below that any online learning rule can be expressed as such a function by simply disregarding the $n$ index, 
while any self-adaptive learning rule can be expressed as such a function by disregarding the $\Y$-valued arguments beyond the first $n$ (when $m \geq n$).
Additionally, for every $x \in \X$, $n,m \in \nats \cup \{0\}$, and every $\kappa \in [0,1)$,
for brevity we define $f_{n,m}^{\kappa}(x) = f_{n,m}(X_{1:m},\target_{\kappa}(X_{1:m}),x)$
(a composition of measurable functions, and therefore measurable); 
equivalently,  $f_{n,m}^{\kappa}(x) = f_{n,m}(X_{1:m},\{y_{i_{t},\kappa_{i_{t}}}\}_{t=1}^{m},x)$.
We generally have 
\begin{align}
& \sup_{\kappa \in [0,1)} \E\!\left[ \limsup_{n \to \infty} \limsup_{t \to \infty} \frac{1}{t} \sum_{m=0}^{t-1} \loss\!\left( f_{n,m}^{\kappa}(X_{m+1}), y_{i_{m+1},\kappa_{i_{m+1}}} \right) \right]
\notag \\ & \geq \int_{0}^{1} \E\!\left[ \limsup_{n \to \infty} \limsup_{t \to \infty} \frac{1}{t} \sum_{m=0}^{t-1} \loss\!\left( f_{n,m}^{\kappa}(X_{m+1}), y_{i_{m+1},\kappa_{i_{m+1}}} \right) \right] {\rm d}\kappa. \label{eqn:unbounded-abstract-integral}
\end{align}
We therefore aim to establish that this last expression is strictly greater than $0$.

Since $\loss$ is nonnegative, Tonelli's theorem implies that the last expression in \eqref{eqn:unbounded-abstract-integral} equals
\begin{align}
& \E\!\left[ \int_{0}^{1} \limsup_{n \to \infty} \limsup_{t \to \infty} \frac{1}{t} \sum_{m=0}^{t-1} \loss\!\left( f_{n,m}^{\kappa}(X_{m+1}), y_{i_{m+1},\kappa_{i_{m+1}}} \right) {\rm d}\kappa \right] \notag
\\ & \geq \E\!\left[ \ind_{E \cap E^{\prime}} \int_{0}^{1} \limsup_{n \to \infty} \limsup_{t \to \infty} \frac{1}{t} \sum_{m=0}^{t-1} \loss\!\left( f_{n,m}^{\kappa}(X_{m+1}), y_{i_{m+1},\kappa_{i_{m+1}}} \right) {\rm d}\kappa \right]. \label{eqn:unbounded-nec-pre2-fatou}
\end{align}
Since $B_{k} \downarrow \emptyset$, for any $t \in \nats$ there exists $k_{t} \in \nats$ with $X_{1:t} \cap B_{k_{t}} = \emptyset$, 
which (by monotonicity of $\{B_{j}\}_{j=1}^{\infty}$) implies that on the event $E$ (so that \eqref{eqn:unbounded-nec-B-every} holds), 
every integer $j \geq k_{t}$ has $\tau_{j} > t$.
Thus, on $E$, $\tau_{j} \to \infty$ as $j \to \infty$.
Therefore,
the expression on the right hand side of \eqref{eqn:unbounded-nec-pre2-fatou} is at least as large as
\begin{align}
& \E\!\left[ \ind_{E \cap E^{\prime}} \int_{0}^{1} \limsup_{n \to \infty} \limsup_{j \to \infty} \frac{1}{\tau_{j}} \sum_{m=0}^{\tau_{j}-1} \loss\!\left( f_{n,m}^{\kappa}(X_{m+1}), y_{i_{m+1},\kappa_{i_{m+1}}} \right) {\rm d}\kappa \right] \notag
\\ & \geq \E\!\left[ \ind_{E \cap E^{\prime}} \int_{0}^{1} \limsup_{n \to \infty} \limsup_{j \to \infty} \frac{1}{\tau_{j}} \left( \loss\!\left( f_{n,\tau_{j}-1}^{\kappa}(X_{\tau_{j}}), y_{i_{\tau_{j}},\kappa_{i_{\tau_{j}}}} \right) \land \tau_{j} \right) {\rm d}\kappa \right]. \label{eqn:unbounded-nec-pre-fatou}
\end{align}
In particular, since $\forall n, j \in \nats$ with $\tau_{j} > 0$, we have  
$\frac{1}{\tau_{j}} \left( \loss\!\left( f_{n,\tau_{j}-1}^{\kappa}(X_{\tau_{j}}), y_{i_{\tau_{j}},\kappa_{i_{\tau_{j}}}} \right) \land \tau_{j} \right) \leq 1$,
Fatou's lemma (applied twice) implies that \eqref{eqn:unbounded-nec-pre-fatou} is at least as large as
\begin{equation}
\label{eqn:unbounded-nec-post-fatou}
\E\!\left[ \ind_{E \cap E^{\prime}} \limsup_{n \to \infty} \limsup_{j \to \infty} \frac{1}{\tau_{j}} \int_{0}^{1} \left( \loss\!\left( f_{n,\tau_{j}-1}^{\kappa}(X_{\tau_{j}}), y_{i_{\tau_{j}},\kappa_{i_{\tau_{j}}}} \right) \land \tau_{j} \right) {\rm d}\kappa \right].
\end{equation}

Now note that on the event $E$, for every $j \in \nats$, 
minimality of $\tau_{j}$ implies that every $t \in \nats$ with $t < \tau_{j}$ has $X_{t} \notin B_{j}$,
and since $B_{j} = \bigcup\limits_{i \geq j} A_{i}$, this implies $i_{t} < j$.
Furthermore, on $E$, by definition of $\tau_{j}$ we have $X_{\tau_{j}} \in B_{j} = \bigcup\limits_{i \geq j} A_{i}$, 
so that $i_{\tau_{j}} \geq j$ for every $j \in \nats$.  Together these facts imply that on $E$, every $j \in \nats$ has 
$i_{\tau_{j}} \notin \{ i_{1},\ldots,i_{\tau_{j}-1} \}$, so that 
$f_{n,\tau_{j}-1}^{\kappa}(X_{\tau_{j}})$ is functionally independent of $\kappa_{i_{\tau_{j}}}$.
Therefore, for $K \sim {\rm Uniform}([0,1))$ independent of $\ProcX$ and $f_{n,\tau_{j}-1}$, 
it holds that $f_{n,\tau_{j}-1}^{K}(X_{\tau_{j}})$ is conditionally independent of $K_{i_{\tau_{j}}}$ given 
$K_{i_{1}},\ldots,K_{i_{\tau_{j}-1}}$, 
$\ProcX$, and $f_{n,\tau_{j}-1}$, 
on the event $E$.
Furthermore, on this event, $K_{i_{\tau_{j}}}$ is conditionally independent of 
$K_{i_{1}},\ldots,K_{i_{\tau_{j}-1}}$ 
given $\ProcX$ and $f_{n,\tau_{j}-1}$, 
and the conditional distribution of $K_{i_{\tau_{j}}}$ is ${\rm Bernoulli}(\frac{1}{2})$, given $\ProcX$ and $f_{n,\tau_{j}-1}$, on this event.
Therefore, on the event $E$, 
\begin{align}
& \int_{0}^{1} \left( \loss\!\left( f_{n,\tau_{j}-1}^{\kappa}(X_{\tau_{j}}), y_{i_{\tau_{j}},\kappa_{i_{\tau_{j}}}} \right) \land \tau_{j} \right)  {\rm d}\kappa
= \E\!\left[ \left( \loss\!\left( f_{n,\tau_{j}-1}^{K}(X_{\tau_{j}}), y_{i_{\tau_{j}},K_{i_{\tau_{j}}}} \right) \land \tau_{j} \right) \Big| \ProcX, f_{n,\tau_{j}-1} \right] \notag
\\ & = \E\!\left[ \E\!\left[ \left( \loss\!\left( \! f_{n,\tau_{j}-1}\!\!\left(\! X_{\!1:(\tau_{j}-1)},\!\{y_{i_{t},K_{i_{t}}}\}_{t=1}^{\tau_{j}-1}\!\!,\!X_{\tau_{j}} \!\right)\!\!, y_{i_{\tau_{j}}\!,K_{i_{\tau_{j}}}} \!\right) \!\land\! \tau_{j} \right) \!\Big| \ProcX, \!\{K_{i_{t}}\}_{t=1}^{\tau_{j}-1}\!\!,\! f_{n,\tau_{j}-1} \right] \Big| \ProcX,\! f_{n,\tau_{j}-1} \right] \notag
\\ & = \E\!\left[ \sum_{b \in \{0,1\}} \frac{1}{2} \left( \loss\!\left( f_{n,\tau_{j}-1}\!\left( X_{1:(\tau_{j}-1)},\{y_{i_{t},K_{i_{t}}}\}_{t=1}^{\tau_{j}-1},X_{\tau_{j}} \right), y_{i_{\tau_{j}},b} \right) \land \tau_{j} \right) \Big| \ProcX, f_{n,\tau_{j}-1} \right]. \label{eqn:unbounded-nec-integral}
\end{align}
Since $\tau_{j} \geq 0$, one can easily verify that $\loss(\cdot,\cdot) \land \tau_{j}$ is a pseudo-near-metric 
(i.e., a near-metric except that $\loss(y,y^{\prime})$ might sometimes be $0$ even for $y \neq y^{\prime}$)
with $\triconst$ as the constant in the relaxed triangle inequality.
Thus, by the relaxed triangle inequality, 
\begin{multline}
\label{eqn:unbounded-nec-bernoulli}
\sum_{b \in \{0,1\}} \left( \loss\!\left( f_{n,\tau_{j}-1}\left(X_{1:(\tau_{j}-1)},\{y_{i_{s},K_{i_{s}}}\}_{s=1}^{\tau_{j}-1},X_{\tau_{j}}\right), y_{i_{\tau_{j}},b} \right) \land \tau_{j} \right)
\\ \geq \frac{1}{\triconst} \left( \loss\!\left( y_{i_{\tau_{j}},0}, y_{i_{\tau_{j}},1} \right) \land \tau_{j} \right)
\geq \frac{1}{\triconst} \left( T_{i_{\tau_{j}}} \land \tau_{j} \right).
\end{multline}
As established above, on the event $E$, every $j \in \nats$ has $i_{\tau_{j}} \geq j$.
Since $\{T_{i}\}_{i=1}^{\infty}$ is nondecreasing, this implies that, on $E$,
$T_{i_{\tau_{j}}} \geq T_{j}$.
Furthermore, on the event $E^{\prime}$, every $j \geq \iota_{0}$ has $T_{j} \geq \tau_{j}$.
Combining this with \eqref{eqn:unbounded-nec-integral} and \eqref{eqn:unbounded-nec-bernoulli}
yields that, on the event $E \cap E^{\prime}$, $\forall n,j \in \nats$
with 
$j \geq \iota_{0}$, 
\begin{equation*}
\int_{0}^{1} \left( \loss\!\left( f_{n,\tau_{j}-1}^{\kappa}(X_{\tau_{j}}), y_{i_{\tau_{j}},\kappa_{i_{\tau_{j}}}} \right) \land \tau_{j} \right) {\rm d}\kappa
\geq \E\!\left[ \frac{1}{2\triconst} \tau_{j} \Big| \ProcX, f_{n,\tau_{j}-1} \right] 
= \frac{1}{2\triconst} \tau_{j},
\end{equation*}
where the rightmost equality follows from $\sigma(\ProcX)$-measurability of $\tau_{j}$.
Therefore, the expression in \eqref{eqn:unbounded-nec-post-fatou} is at least as large as
\begin{equation*}
\E\!\left[ \ind_{E \cap E^{\prime}} \limsup_{n \to \infty} \limsup_{j \to \infty} \frac{1}{\tau_{j}} \left( \frac{1}{2\triconst} \tau_{j} \right) \right]
= \frac{1}{2\triconst} \P\left( E \cap E^{\prime} \right)
\geq \frac{1}{2\triconst} \left( \P(E) - \P((E^{\prime})^{c}) \right)
= \frac{1}{2\triconst} \P(E),
\end{equation*}
where the rightmost equality is due to the fact that $\P(E^{\prime}) = 1$.
In particular, recall that $\P(E) > 0$, so that the above is strictly greater than zero.

Altogether, we have established that the last expression in \eqref{eqn:unbounded-abstract-integral} is strictly greater than $0$.
By the inequality in \eqref{eqn:unbounded-abstract-integral} this implies $\exists \kappa \in [0,1)$ such that
\begin{equation*}
\E\!\left[ \limsup\limits_{n \to \infty} \limsup\limits_{t \to \infty} \frac{1}{t} \sum_{m=0}^{t-1} \loss\!\left( f_{n,m}^{\kappa}(X_{m+1}), y_{i_{m+1},\kappa_{i_{m+1}}} \right) \right] > 0,
\end{equation*}
which further implies \citep*[see e.g., Theorem 1.6.5 of][]{ash:00} that, with probability strictly greater than zero,
\begin{equation*}
\limsup\limits_{n \to \infty} \limsup\limits_{t \to \infty} \frac{1}{t} \sum_{m=0}^{t-1} \loss\!\left( f_{n,m}^{\kappa}(X_{m+1}), y_{i_{m+1},\kappa_{i_{m+1}}} \right) > 0.
\end{equation*}

This argument applies to any measurable functions $f_{n,m} : \X^{m} \times \Y^{m} \times \X \to \Y$ (possibly random).
In particular, for any online learning rule $h_{n}$, we can define a function $f_{n,m}(x_{1:m},y_{1:m},x) = h_{m}(x_{1:m},y_{1:m},x)$ 
(for every $n,m \in \nats \cup \{0\}$ and $x_{1:m} \in \X^{m}$, $y_{1:m} \in \Y^{m}$, $x \in \X$), 
in which case any $\kappa \in [0,1)$ has 
\begin{equation*}
\limsup_{n \to \infty} \hat{\L}_{\ProcX}(h_{\cdot},\target_{\kappa};n)
= \limsup_{n \to \infty} \limsup_{t \to \infty} \frac{1}{t} \sum_{m=0}^{t-1} \loss\!\left( f_{n,m}^{\kappa}(X_{m+1}), y_{i_{m+1},\kappa_{i_{m+1}}} \right).
\end{equation*}
Therefore, the above argument implies that $\exists \kappa \in [0,1)$ for which, 
with probability strictly greater than $0$, 
$\limsup\limits_{n \to \infty} \hat{\L}_{\ProcX}(h_{\cdot},\target_{\kappa};n) > 0$,
so that $h_{n}$ is not strongly universally consistent under $\ProcX$. 
Since this argument applies to any online learning rule $h_{n}$, 
this implies $\ProcX \notin \SUOL$, 
and since the argument applies to any process $\ProcX$ failing to satisfy Condition~\ref{con:ukc}, 
we conclude that $\SUOL \subseteq \UKC$.

Similarly, for any self-adaptive learning rule $g_{n,m}$, for every $n,m \in \nats \cup \{0\}$ with $m \geq n$, 
we can define a function $f_{n,m}(x_{1:m},y_{1:m},x) = g_{n,m}(x_{1:m},y_{1:n},x)$ 
(for every $x_{1:m} \in \X^{m}$, $y_{1:m} \in \Y^{m}$, $x \in \X$).
For $n,m \in \nats \cup \{0\}$ with $m < n$, we can simply define $f_{n,m}(x_{1:m},y_{1:m},x)$ as an arbitrary fixed $y \in \Y$
(invariant to the arguments $x_{1:m} \in \X^{m}$, $y_{1:m} \in \Y^{m}$, $x \in \X$).
Then for any $\kappa \in [0,1)$, we have 
\begin{align*}
\limsup_{n \to \infty} \hat{\L}_{\ProcX}(g_{n,\cdot},\target_{\kappa};n)
&= \limsup_{n \to \infty} \limsup_{t \to \infty} \frac{1}{t+1} \sum_{m=n}^{n+t} \loss\!\left( f_{n,m}^{\kappa}(X_{m+1}), y_{i_{m+1},\kappa_{i_{m+1}}} \right)
\\ & = \limsup_{n \to \infty} \limsup_{t \to \infty} \frac{1}{t} \sum_{m=0}^{t-1} \loss\!\left( f_{n,m}^{\kappa}(X_{m+1}), y_{i_{m+1},\kappa_{i_{m+1}}} \right).
\end{align*}
Therefore, the above argument implies that $\exists \kappa \in [0,1)$ for which,
with probability strictly greater than $0$, 
$\limsup\limits_{n \to \infty} \hat{\L}_{\ProcX}(g_{n,\cdot},\target_{\kappa};n) > 0$,
so that $g_{n,m}$ is not strongly universally consistent under $\ProcX$.
Since this argument applies to any self-adaptive learning rule $g_{n,m}$, 
this implies $\ProcX \notin \SUAL$,
and since the argument applies to any process $\ProcX$ failing to satisfy Condition~\ref{con:ukc},
we conclude that $\SUAL \subseteq \UKC$, which completes the proof.
\end{proof}

To argue sufficiency of $\UKC$ for strong universal inductive learning,
we propose a new type of learning rule, suitable for learning with unbounded losses 
under processes in $\UKC$.  Specifically, let $\eps_{0} = \infty$, and for each $k \in \nats$, let $\eps_{k} = (2\triconst)^{-k}$.
Given a sequence $\{\tilde{f}_{i}\}_{i=1}^{\infty}$ of measurable functions $\X \to \Y$ (described below),
and any $n \in \nats$, $x_{1:n} \in \X^{n}$, and $y_{1:n} \in \Y^{n}$,
define $\hat{i}_{n,0}(x_{1:n},y_{1:n}) = 1$, and for each $k \in \nats$, inductively define
\begin{multline}
\label{eqn:unbounded-universal2-hati}
\hat{i}_{n,k}(x_{1:n},y_{1:n}) = \min\left\{ i \in \nats : \max_{1 \leq t \leq n} \loss\!\left( \tilde{f}_{i}(x_{t}), y_{t} \right) \leq \eps_{k}, \text{ and} \right. 
\\ \left. \sup_{x \in \X} \loss\!\left( \tilde{f}_{i}(x), \tilde{f}_{\hat{i}_{n,k-1}(x_{1:n},y_{1:n})}(x) \right) \leq \triconst\eps_{k-1}+\eps_{k} \right\},
\end{multline}
if it exists.  For completeness, 
if the set on the right hand side of \eqref{eqn:unbounded-universal2-hati} is empty for a given $k \in \nats$,
let us define $\hat{i}_{n,k}(x_{1:n},y_{1:n}) = \hat{i}_{n,k-1}(x_{1:n},y_{1:n})$.
Fix any sequence $\{k_{n}\}_{n=1}^{\infty}$ in $\nats$ with $k_{n} \to \infty$. 
Then, for any $n \in \nats$, and any $x_{1:n} \in \X^{n}$, $y_{1:n} \in \Y^{n}$, and $x \in \X$, 
define 
\begin{equation}
\label{eqn:unbounded-universal2-inductive-rule}
\hat{f}_{n}(x_{1:n},y_{1:n},x) = \tilde{f}_{\hat{i}_{n,k_{n}}(x_{1:n},y_{1:n})}(x).
\end{equation}
We will argue in the proof of Lemma~\ref{lem:unbounded-ukc-subset-suil} below that $\hat{f}_{n}$ is measurable, 
and hence \eqref{eqn:unbounded-universal2-inductive-rule} defines a valid inductive learning rule.
We will see below that, for an appropriate choice of the sequence $\{\tilde{f}_{i}\}_{i=1}^{\infty}$, 
this inductive learning rule is strongly universally consistent under every $\ProcX \in \UKC$, even for unbounded losses.
To specify an appropriate sequence $\{\tilde{f}_{i}\}_{i=1}^{\infty}$, and to study the performance of the resulting learning rule,
we first prove modified versions of Lemmas~\ref{lem:approximating-sets} and \ref{lem:approximating-sequence},
under the restriction of $\ProcX$ to $\UKC$.  

\begin{lemma}
\label{lem:unbounded-approximating-sets}
There exists a countable set $\T_{1} \subseteq \Borel$ such that, $\forall \ProcX \in \UKC$,
$\forall A \in \Borel$, with probability one, $\exists \hat{A} \in \T_{1}$ s.t. $\ProcX \cap \hat{A} = \ProcX \cap A$.
\end{lemma}
\begin{proof}
This proof follows along similar lines to the proof of Lemma~\ref{lem:approximating-sets},
and indeed the set $\T_{1}$ will be the same as defined in that proof.
Let $\T_{0}$ be as in the proof of Lemma~\ref{lem:approximating-sets}.
As in the proof of Lemma~\ref{lem:approximating-sets}, there is an immediate proof based on the 
monotone class theorem \citep*[][Theorem 1.3.9]{ash:00}, by taking $\T_{1}$ as the algebra generated by $\T_{0}$ (which, one can show, is a countable set),
and then showing that the collection of sets $A$ for which the claim holds forms a monotone class (straightforwardly using Condition~\ref{con:ukc} for this part).
However, as was the case for Lemma~\ref{lem:approximating-sets}, we will instead establish the claim
with a \emph{smaller} set $\T_{1}$, 
which thereby simplifies the problem of implementing the resulting learning rule.
Specifically, as in the proof of Lemma~\ref{lem:approximating-sets}, take $\T_{1} = \left\{ \bigcup \A : \A \subseteq \T_{0}, |\A| < \infty \right\}$,
which (as discussed in that proof) is a countable set.
Fix any $\ProcX \in \UKC$, and let 
\begin{equation*}
\Lambda = \left\{ A \in \Borel : \P\!\left( \exists \hat{A} \in \T_{1} \text{ s.t. } \ProcX \cap \hat{A} = \ProcX \cap A \right) = 1 \right\}.
\end{equation*}

For any $A \in \T$, as mentioned in the proof of Lemma~\ref{lem:approximating-sets}, 
$\exists \{B_{i}\}_{i=1}^{\infty}$ in $\T_{0}$ such that $A = \bigcup\limits_{i=1}^{\infty} B_{i}$.
Then letting $A_{k} = \bigcup\limits_{i=1}^{k} B_{i}$ for each $k \in \nats$, 
we have $A_{k} \bigtriangleup A = A \setminus A_{k} \downarrow \emptyset$ (monotonically),
and $A_{k} \in \T_{1}$ for each $k \in \nats$.  Therefore, by Condition~\ref{con:ukc},
with probability one, $\exists k \in \nats$ such that $\ProcX \cap (A_{k} \bigtriangleup A) = \emptyset$,
which implies $\ProcX \cap A_{k} = \ProcX \cap A$.
Thus, $A \in \Lambda$.  Since this holds for any $A \in \T$, we have $\T \subseteq \Lambda$.

Next, we argue that $\Lambda$ is a $\sigma$-algebra,
beginning with the property of being closed under complements.
First, consider any $A \in \T_{1}$.
Since $\T_{1} \subseteq \T$, it follows that $\X \setminus A$ is a closed set.
Since $(\X,\T)$ is metrizable, this implies $\exists \{B_{i}\}_{i=1}^{\infty}$ in $\T$ such that $\X \setminus A = \bigcap\limits_{i=1}^{\infty} B_{i}$ 
\citep*[][Proposition 3.7]{kechris:95}.
Defining $C_{k} = \bigcap\limits_{i=1}^{k} B_{i}$ for each $k \in \nats$, we have that 
$C_{k} \bigtriangleup (\X \setminus A) = C_{k} \setminus (\X \setminus A) \downarrow \emptyset$ (monotonically),
and $C_{k} \in \T$ for each $k \in \nats$.
In particular, by Condition~\ref{con:ukc}, this implies that on an event $E_{0}^{(A)}$ of probability one,
there exists $k_{0} \in \nats$ such that
$\ProcX \cap (C_{k_{0}} \bigtriangleup (\X \setminus A))= \emptyset$, which implies 
$\ProcX \cap C_{k_{0}} = \ProcX \cap (\X \setminus A)$.
Furthermore, for each $k \in \nats$, since $C_{k} \in \T \subseteq \Lambda$, 
there is an event $E_{k}^{(A)}$ of probability one, on which 
$\exists \hat{A}_{k} \in \T_{1}$ with $\ProcX \cap \hat{A}_{k} = \ProcX \cap C_{k}$.
Altogether, on the event $\bigcap\limits_{k=0}^{\infty} E_{k}^{(A)}$ (which has probability one, by the union bound), 
$\ProcX \cap \hat{A}_{k_{0}} = \ProcX \cap (\X \setminus A)$.
Thus, every $A \in \T_{1}$ has $(\X \setminus A) \in \Lambda$.
Now define $E^{(\T_{1})} = \bigcap\limits_{A \in \T_{1}} \bigcap\limits_{k=0}^{\infty} E_{k}^{(A)}$,
which has probability one by the union bound (since $\T_{1}$ is countable).

Next, consider any $A \in \Lambda$, and suppose the event (of probability one), denoted $E^{\prime}$, 
that $\exists \hat{A} \in \T_{1} \text{ s.t. } \ProcX \cap \hat{A} = \ProcX \cap A$ holds,
which also implies $\ProcX \cap (\X \setminus \hat{A}) = \ProcX \cap (\X \setminus A)$.
Since $\hat{A} \in \T_{1}$, on the event $E^{(\T_{1})}$ we have that $\exists \hat{A}^{\prime} \in \T_{1}$
with $\ProcX \cap \hat{A}^{\prime} = \ProcX \cap (\X \setminus \hat{A})$.
Thus, on the event $E^{\prime} \cap E^{(\T_{1})}$, we have 
$\ProcX \cap \hat{A}^{\prime} = \ProcX \cap (\X \setminus A)$.
Since $E^{\prime} \cap E^{(\T_{1})}$ has probability one (by the union bound), 
we have that $\X \setminus A \in \Lambda$.  Since this argument holds for any $A \in \Lambda$,
we have that $\Lambda$ is closed under complements.

Next, we show that $\Lambda$ is closed under countable unions.
Let $\{A_{i}\}_{i=1}^{\infty}$ be a sequence in $\Lambda$, and let $A = \bigcup\limits_{i=1}^{\infty} A_{i}$.
Since each $A_{i} \in \Lambda$, by the union bound there is an event $E$ of probability one, 
on which there exists a sequence $\{\hat{A}_{i}\}_{i=1}^{\infty}$ in $\T_{1}$
such that $\forall i \in \nats$, $\ProcX \cap A_{i} = \ProcX \cap \hat{A}_{i}$.
Furthermore, since $A \bigtriangleup \bigcup\limits_{i=1}^{k} A_{i} = A \setminus \bigcup\limits_{i=1}^{k} A_{i} \downarrow \emptyset$ (monotonically),
Condition~\ref{con:ukc} implies that, on an event $E^{\prime\prime}$ of probability one, 
$\exists k \in \nats$ such that $\ProcX \cap \left(A \bigtriangleup \bigcup\limits_{i=1}^{k} A_{i}\right) = \emptyset$,
which implies $\ProcX \cap \bigcup\limits_{i=1}^{k} A_{i} = \ProcX \cap A$.
Since, for any $k \in \nats$, $\ProcX \cap \bigcup\limits_{i=1}^{k} A_{i}$ is simply the subsequence of $\ProcX$ 
consisting of all entries appearing in any of the $\ProcX \cap A_{i}$ subsequences with $i \leq k$, 
and (on $E$) each $\ProcX \cap A_{i} = \ProcX \cap \hat{A}_{i}$, 
together we have that on the event $E \cap E^{\prime\prime}$ (which has probability one, by the union bound),
$\exists k \in \nats$ such that 
$\ProcX \cap \bigcup\limits_{i=1}^{k} \hat{A}_{i} = \ProcX \cap \bigcup\limits_{i=1}^{k} A_{i} = \ProcX \cap A$.
Since it follows immediately from its definition that the set $\T_{1}$ is closed under finite unions, 
we have that $\bigcup\limits_{i=1}^{k} \hat{A}_{i} \in \T_{1}$.
Therefore, $A \in \Lambda$.  Since this holds for any choice of the sequence $\{A_{i}\}_{i=1}^{\infty}$ in $\Lambda$,
we have that $\Lambda$ is closed under countable unions.

Finally, recalling that $\T$ is a topology, we have $\X \in \T$,
and since $\T \subseteq \Lambda$, this implies $\X \in \Lambda$.
Altogether, we have established that $\Lambda$ is a $\sigma$-algebra.
Therefore, since $\Borel$ is the $\sigma$-algebra generated by $\T$, and $\T \subseteq \Lambda$,
it immediately follows that 
$\Borel \subseteq \Lambda$
(which also implies $\Lambda = \Borel$).
Since this argument holds for any choice of $\ProcX \in \UKC$, the lemma 
follows.
\end{proof}

\begin{lemma}
\label{lem:unbounded-constrained-approximating-sequence}
There exists a countable set $\tilde{\F}$ of measurable functions $\X \to \Y$
such that, for every $\ProcX \in \UKC$, for every measurable $f : \X \to \Y$,
with probability one, $\forall \eps > 0$, $\forall \tilde{f}_{1} \in \tilde{\F}$, $\exists \tilde{f}_{2} \in \tilde{\F}$
with 
\begin{align*}
\sup_{x \in \X} \loss\!\left( \tilde{f}_{2}(x), \tilde{f}_{1}(x) \right) & \leq \eps + \triconst \sup_{t \in \nats} \loss\!\left( \tilde{f}_{1}(X_{t}), f(X_{t}) \right) 
\\ \text{and }~~ \sup_{t \in \nats} \loss\!\left( \tilde{f}_{2}(X_{t}), f(X_{t}) \right) & \leq \eps.
\end{align*}
\end{lemma}
\begin{proof}
The construction, and first half of this proof, will proceed analogously to the proof of Lemma~\ref{lem:approximating-functions}, but with a few important changes.
Specifically, let $\T_{1}$ be as in Lemma~\ref{lem:unbounded-approximating-sets},
let $\tilde{\Y}$ be a countable subset of $\Y$ with $\sup\limits_{y \in \Y} \inf\limits_{\tilde{y} \in \tilde{\Y}} \loss(\tilde{y},y) = 0$ (which exists by separability of $(\Y,\loss)$), 
let $A_{0} = \X$, let $y_{0}$ be an arbitrary value in $\Y$,
and for any $k \in \nats$, any sequence $\{A_{i}\}_{i=1}^{k}$ in $\Borel$, and any sequence $\{y_{i}\}_{i=1}^{k}$ in $\Y$, 
define $\tilde{f}(x;\{y_{i}\}_{i=1}^{k},\{A_{i}\}_{i=1}^{k}) = y_{\max\{i \in \{0,\ldots,k\}:x \in A_{i}\}}$.
Define $\T_{2} = \left\{ \bigcap_{i=1}^{k} A_{i} : k \in \nats, A_{1},\ldots,A_{k} \in \T_{1} \right\}$: the finite intersections of sets in $\T_{1}$.
This is a countable set since $\T_{1}$ is countable.
Then define 
\begin{equation*}
\tilde{\F} = \left\{ \tilde{f}(\cdot;\{y_{i}\}_{i=1}^{k}, \{A_{i}\}_{i=1}^{k}) : k \in \nats, \forall i \leq k, y_{i} \in \tilde{\Y}, A_{i} \in \T_{2} \right\},
\end{equation*}
which is a countable set (since $\tilde{\Y}$ and $\T_{2}$ are countable).
Enumerate the elements of $\tilde{\Y}$ as $\tilde{y}_{1},\tilde{y}_{2},\ldots$; 
for simplicity, we will suppose this sequence is infinite, which is always the case if $\maxloss = \infty$,
and otherwise can be achieved by repeating elements if necessary in the general case.
As in the proof of Lemma~\ref{lem:approximating-functions}, 
for each $\eps > 0$, let $B_{\eps,1} = \{ y \in \Y : \loss(\tilde{y}_{1},y) \leq \eps \}$ 
and for each integer $i \geq 2$ inductively define $B_{\eps,i} = \{ y \in \Y : \loss(\tilde{y}_{i},y) \leq \eps \} \setminus \bigcup\limits_{j=1}^{i-1} B_{\eps,j}$.
For each $\eps > 0$, this defines a disjoint sequence $\{B_{\eps,i}\}_{i=1}^{\infty}$ in $\Borel_{y}$ with $\bigcup\limits_{i=1}^{\infty} B_{\eps,i} = \Y$.

Fix any $\ProcX \in \UKC$, any measurable $f : \X \to \Y$, and any $\eps > 0$.
For each $i \in \nats$, define $C_{\eps,i} = f^{-1}(B_{\eps,i})$, 
an element of $\Borel$ (by measurability of $f$ and $B_{\eps,i}$).
Note that $\bigcup\limits_{i=1}^{\infty} C_{\eps,i} = f^{-1}\!\left( \bigcup\limits_{i=1}^{\infty} B_{\eps,i} \right) = f^{-1}(\Y) = \X$, 
and since the $B_{\eps,i}$ sets are disjoint over the values of $i$, the sets $C_{\eps,i}$ are also disjoint over $i$.
It follows that $\lim\limits_{k \to \infty} \bigcup\limits_{i=k}^{\infty} C_{\eps,i} = \emptyset$, with 
$\bigcup\limits_{i=k}^{\infty} C_{\eps,i}$ nonincreasing in $k$, so that Condition~\ref{con:ukc} implies 
that, on an event $E_{\eps,1}$ of probability one, $\exists k_{0} \in \nats$ s.t. $\ProcX \cap \bigcup\limits_{i=k_{0}+1}^{\infty} C_{\eps,i} = \emptyset$.
Since $\bigcup\limits_{i=1}^{\infty} C_{\eps,i} = \X$, this also means $\ProcX \cap \bigcup\limits_{i=1}^{k_{0}} C_{\eps,i} = \ProcX$.
Furthermore, by the union bound and the defining property of $\T_{1}$ from Lemma~\ref{lem:unbounded-approximating-sets},
on an event $E_{\eps,2}$ of probability one, $\forall i \in \nats$, $\exists \tilde{A}_{\eps,i} \in \T_{1}$ with $\ProcX \cap \tilde{A}_{\eps,i} = \ProcX \cap C_{\eps,i}$.
This also means that, when the events $E_{\eps,1}$ and $E_{\eps,2}$ occur simultaneously,  
we have $\ProcX \cap \bigcup\limits_{i=1}^{k_{0}} \tilde{A}_{\eps,i} = \ProcX$.

At this point, we may note that the function $\tilde{f}(\cdot; \{\tilde{y}_{i}\}_{i=1}^{k_{0}}, \{ \tilde{A}_{\eps,i} \}_{i=1}^{k_{0}} )$ would suffice as a specification of $\tilde{f}_{2}$ 
for the purpose of satisfying the \emph{second} requirement in the lemma: that is, $\sup\limits_{t \in \nats} \loss\!\left( \tilde{f}_{2}(X_{t}), f(X_{t}) \right) \leq \eps$.
Indeed if this were the only requirement, we could have used $\T_{1}$ in place of $\T_{2}$ in the specification of $\tilde{\F}$ above.
However, we must modify this function in order to also satisfy the first requirement, constraining the supremum distance between $\tilde{f}_{2}$ and a given $\tilde{f}_{1}$ in $\tilde{\F}$.

Toward this end, supposing the event $E_{\eps,1} \cap E_{\eps,2}$ occurs and that $k_{0}$ and $\{\tilde{A}_{\eps,i}\}_{i=1}^{k_{0}}$ are as above, 
fix any function $\tilde{f}_{1} \in \tilde{\F}$, and let $k_{1} \in \nats$, $\{y_{i}\}_{i=1}^{k_{1}} \in \tilde{\Y}^{k_{1}}$, and $\{A_{i}\}_{i=1}^{k_{1}} \in \T_{2}^{k_{1}}$ 
be such that $\tilde{f}_{1}(\cdot) = \tilde{f}(\cdot; \{y_{i}\}_{i=1}^{k_{1}},\{A_{i}\}_{i=1}^{k_{1}})$.  
Let $A_{0}$ and $y_{0}$ be as specified above, and define $\tilde{A}_{\eps,0} = \X$ and $\tilde{y}_{0} = y_{0}$.
Let $k_{2} = (k_{0}+1)(k_{1}+1)-1$.
For each $i \in \{0,\ldots,k_{0}\}$ and $j \in \{0,1,\ldots,k_{1}\}$, 
define $\hat{A}_{j(k_{0}+1)+i} = \tilde{A}_{\eps,i} \cap A_{j}$.
Also, for each $i \in \{0,\ldots,k_{0}\}$ and $j \in \{0,1,\ldots,k_{1}\}$, 
define $\hat{D}_{i,j} = \hat{A}_{j(k_{0}+1)+i} \setminus \bigcup\limits_{j^{\prime}=j(k_{0}+1)+i+1}^{k_{2}} \hat{A}_{j^{\prime}}$; 
if 
$\ProcX \cap \hat{D}_{i,j} \neq \emptyset$, 
define $\hat{y}_{j (k_{0}+1)+i} = \tilde{y}_{i}$,
and otherwise define $\hat{y}_{j (k_{0}+1)+i} = y_{j}$.
Note that $\hat{A}_{1},\ldots,\hat{A}_{k_{2}}$ are elements of $\T_{2}$
and $\hat{y}_{1},\ldots,\hat{y}_{k_{2}}$ are elements of $\tilde{\Y}$, 
and therefore, defining $\tilde{f}_{2}(\cdot) = \tilde{f}(\cdot;\{\hat{y}_{i}\}_{i=1}^{k_{2}},\{\hat{A}_{i}\}_{i=1}^{k_{2}})$, 
we have that $\tilde{f}_{2} \in \tilde{\F}$.

Now, for every $x \in \X$, denote by $(\i(x),\j(x))$ the unique value in $\{0,\ldots,k_{0}\}\times\{0,\ldots,k_{1}\}$ 
such that $\j(x)(k_{0}+1)+\i(x) = \max\!\left\{ j^{\prime} \in \{0,\ldots,k_{2}\} : x  \in \hat{A}_{j^{\prime}} \right\}$ (noting that $\hat{A}_{0} = \X$, so that this is always well-defined).
By definition, we have $\tilde{f}_{2}(x) = \hat{y}_{\j(x)(k_{0}+1)+\i(x)}$ (noting that $\hat{y}_{0} = y_{0}$, so that this equality holds even when $(\i(x),\j(x))=(0,0)$).
Since $\tilde{A}_{\eps,0} = \X$, it holds that every $j \in \{0,\ldots,k_{1}\}$ has $\hat{A}_{j(k_{0}+1)} = A_{j}$.
In particular, this implies that, for every $x \in \X$, it holds that $\j(x) = \max\{ j^{\prime} \in \{0,\ldots,k_{1}\} : x \in A_{j^{\prime}} \}$: 
that is, by definition of $(\i(x),\j(x))$, we have $x \in \hat{A}_{\j(x)(k_{0}+1)+\i(x)} \subseteq A_{\j(x)}$, 
and maximality of $\j(x)(k_{0}+1)+\i(x)$ implies that every $j^{\prime} \in \{\j(x)+1,\ldots,k_{1}\}$ 
has $x \notin \hat{A}_{j^{\prime} (k_{0}+1)} = A_{j^{\prime}}$.
Moreover, for any $i \in \{0,\ldots,k_{0}\}$, if $x \in \tilde{A}_{\eps,i}$, then since $x \in A_{\j(x)}$ as well, 
we have $x \in \hat{A}_{\j(x)(k_{0}+1)+i} = \tilde{A}_{\eps,i} \cap A_{\j(x)}$; 
in particular, this is true of the \emph{largest} $i \in \{0,\ldots,k_{0}\}$ with $x \in \tilde{A}_{\eps,i}$.
It immediately follows that $\i(x) = \max\{ i^{\prime} \in \{0,\ldots,k_{0}\} : x \in \tilde{A}_{\eps,i^{\prime}} \}$.

Now note that, for any $t \in \nats$, 
since $X_{t} \in \hat{D}_{\i(X_{t}),\j(X_{t})}$, 
we have $\ProcX \cap \hat{D}_{\i(X_{t}),\j(X_{t})} \neq \emptyset$.
Therefore, by definition, $\hat{y}_{\j(X_{t})(k_{0}+1)+\i(X_{t})} = \tilde{y}_{\i(X_{t})}$.
Furthermore, since $\hat{A}_{\j(X_{t})(k_{0}+1)+\i(X_{t})} \subseteq \tilde{A}_{\eps,\i(X_{t})}$, we have $X_{t} \in \tilde{A}_{\eps,\i(X_{t})}$,
and since $\ProcX \cap \bigcup\limits_{i=1}^{k_{0}} \tilde{A}_{\eps,i} = \ProcX$, there exists $i \in \{1,\ldots,k_{0}\}$ with $X_{t} \in \tilde{A}_{\eps,i}$.
Therefore, the fact (established above) that $\i(X_{t}) = \max\{ i^{\prime} \in \{0,\ldots,k_{0}\} : X_{t} \in \tilde{A}_{\eps,i^{\prime}} \}$
implies $\i(X_{t}) \neq 0$.
Since every $i \in \nats$ has $\ProcX \cap \tilde{A}_{\eps,i} = \ProcX \cap C_{\eps,i}$, this further implies that 
$X_{t} \in C_{\eps,\i(X_{t})}$, and therefore $\loss\!\left(\tilde{y}_{\i(X_{t})},f(X_{t})\right) \leq \eps$, 
so that altogether we have $\loss\!\left(\tilde{f}_{2}(X_{t}),f(X_{t})\right) \leq \eps$.
Since this is true of every $t \in \nats$, we conclude that 
$\sup\limits_{t \in \nats} \loss\!\left(\tilde{f}_{2}(X_{t}),f(X_{t})\right) \leq \eps$.

Next, note that for any $x \in \X$, since (as established above) $\j(x) = \max\{ j^{\prime} \in \{0,\ldots,k_{1}\} : x \in A_{j^{\prime}} \}$, 
we have $\tilde{f}_{1}(x) = y_{\j(x)}$.
In particular, if 
$\ProcX \cap \hat{D}_{\i(x),\j(x)} = \emptyset$,
then, by definition, we have $\tilde{f}_{2}(x) = y_{\j(x)}$, 
so that in this case $\loss\!\left( \tilde{f}_{2}(x), \tilde{f}_{1}(x) \right) = \loss( y_{\j(x)}, y_{\j(x)} ) = 0$.
On the other hand, if 
$\ProcX \cap \hat{D}_{\i(x),\j(x)} \neq \emptyset$, 
then, by definition, we have $\tilde{f}_{2}(x) = \tilde{y}_{\i(x)}$.
In this case, letting $t \in \nats$ be such that 
$X_{t} \in \hat{D}_{\i(x),\j(x)}$, 
it immediately follows that $(\i(X_{t}),\j(X_{t})) = (\i(x),\j(x))$, 
so that we also have $\tilde{f}_{1}(X_{t}) = y_{\j(x)}$ and $\tilde{f}_{2}(X_{t}) = \tilde{y}_{\i(x)}$.
Thus, in this case we have 
$\loss\!\left( \tilde{f}_{2}(x), \tilde{f}_{1}(x) \right)
= \loss\!\left( \tilde{y}_{\i(x)},y_{\j(x)} \right) 
= \loss\!\left( \tilde{f}_{2}(X_{t}), \tilde{f}_{1}(X_{t}) \right)$.
Since every $x \in \X$ satisfies one of these two cases, and since each $X_{t}$ itself takes a value in $\X$, we conclude that
\begin{equation*}
\sup\limits_{x \in \X} \loss\!\left( \tilde{f}_{2}(x), \tilde{f}_{1}(x) \right) = \sup\limits_{t \in \nats} \loss\!\left( \tilde{f}_{2}(X_{t}), \tilde{f}_{1}(X_{t}) \right).
\end{equation*}
Combining this with the relaxed triangle inequality and 
the fact (established above) that $\sup\limits_{t \in \nats} \loss\!\left(\tilde{f}_{2}(X_{t}),f(X_{t})\right) \leq \eps$, 
we conclude that 
\begin{align*}
\sup\limits_{x \in \X} \loss\!\left( \tilde{f}_{2}(x), \tilde{f}_{1}(x) \right) 
& \leq
\triconst \sup\limits_{t \in \nats} \loss\!\left( \tilde{f}_{2}(X_{t}), f(X_{t}) \right) 
+ \triconst \sup\limits_{t \in \nats} \loss\!\left( \tilde{f}_{1}(X_{t}), f(X_{t}) \right) 
\\ & \leq
\triconst \eps 
+ \triconst \sup\limits_{t \in \nats} \loss\!\left( \tilde{f}_{1}(X_{t}), f(X_{t}) \right).
\end{align*}

The above results hold for any fixed $\eps > 0$.
Now letting $\eps_{k}^{\prime} = 2^{-k}$ for each $k \in \nats$,
we have that on the event $\bigcap\limits_{k = 1}^{\infty} \left( E_{\eps_{k}^{\prime},1} \cap E_{\eps_{k}^{\prime},2} \right)$,
for any $\tilde{f}_{1} \in \tilde{\F}$ and any $\eps > 0$, letting $k = \left\lceil \log_{2}((\triconst/\eps) \lor 2) \right\rceil$,
we have that $\exists \tilde{f}_{2} \in \tilde{\F}$ with 
\begin{equation*}
\sup_{t \in \nats} \loss(\tilde{f}_{2}(X_{t}),f(X_{t})) \leq \eps_{k}^{\prime} \leq \eps,
\end{equation*}
and 
\begin{equation*}
\sup_{x \in \X} \loss(\tilde{f}_{2}(x),\tilde{f}_{1}(x)) 
\leq \triconst \eps_{k}^{\prime} + \triconst \sup_{t \in \nats} \loss(\tilde{f}_{1}(X_{t}),f(X_{t})) 
\leq \eps + \triconst \sup_{t \in \nats} \loss(\tilde{f}_{1}(X_{t}),f(X_{t})).
\end{equation*}
Noting that the event $\bigcap\limits_{k=1}^{\infty} \left( E_{\eps_{k}^{\prime},1} \cap E_{\eps_{k}^{\prime},2} \right)$ has probability one (by the union bound)
completes the proof.
\end{proof}

We are now ready to present a result establishing that any process satisfying Condition~\ref{con:ukc} necessarily 
admits strong universal inductive (and online) learning.  
This is analogous to Lemma~\ref{lem:kc-subset-suil} from the bounded case.
For clarity, we make explicit the fact that
this result holds for $\maxloss = \infty$, though it clearly also holds for $\maxloss < \infty$
(since $\UKC \subseteq \KC$).

\begin{lemma}
\label{lem:unbounded-ukc-subset-suil}
When $\maxloss = \infty$, $\UKC \subseteq \SUIL \cap \SUOL$. 
\end{lemma}
\begin{proof}
We begin by showing that $\UKC \subseteq \SUIL$.
Let $\hat{f}_{n}$ be the inductive learning rule specified by \eqref{eqn:unbounded-universal2-inductive-rule},
where the sequence $\{\tilde{f}_{i}\}_{i=1}^{\infty}$ is chosen as an enumeration of the elements of the countable set $\tilde{\F}$ from Lemma~\ref{lem:unbounded-constrained-approximating-sequence}.
We establish the stated result by arguing that $\hat{f}_{n}$ is strongly universally consistent for every $\ProcX \in \UKC$,
which thereby establishes that every $\ProcX \in \UKC$ admits strong universal inductive learning.

To verify that $\hat{f}_{n}$ is a measurable function, 
we note that any measurable $B \subseteq \Y$ has 
$\hat{f}_{n}^{-1}(B) 
= \bigcup\limits_{i \in \nats} \left(\hat{i}_{n,k_{n}}^{-1}(\{i\}) \times \X\right) \cap \left( \X^{n} \times \Y^{n} \times \tilde{f}_{i}^{-1}(B) \right)$.
Since each $\tilde{f}_{i}$ is a measurable function, it suffices to verify measurability of $\hat{i}_{n,k}$ for all $n, k$.
Note that $\hat{i}_{n,0}$ is constant, hence trivially measurable.  For the purpose of induction, let us suppose some $k \in \nats$ has $\hat{i}_{n,k-1}$ measurable.
For any $i \in \nats$, let $J_{i,k} = \!\left\{ j \in \nats : \sup\limits_{x \in \X} \loss\!\left( \tilde{f}_{i}(x), \tilde{f}_{j}(x) \right) \leq \triconst \eps_{k-1} + \eps_{k} \right\}$,
and $A_{i,k} = \hat{i}_{n,k-1}^{-1}(J_{i,k}) \cap \left\{ (x_{1:n},y_{1:n}) : \max\limits_{1 \leq t \leq n} \loss\!\left(\tilde{f}_{i}(x_{t}),y_{t}\right) \leq \eps_{k} \right\}$.
Since $\hat{i}_{n,k-1}$ is measurable (by assumption) and $\loss$ and $\tilde{f}_{i}$ are measurable functions, we observe that $A_{i,k}$ is a measurable set.
Then note that any $i \in \nats$ has 
$\hat{i}_{n,k}^{-1}(\{i\}) = \left( A_{i,k} \setminus \bigcup\limits_{i^{\prime} < i} A_{i^{\prime},k} \right) \cup \left( \hat{i}_{n,k-1}^{-1}(\{i\}) \setminus \bigcup\limits_{i^{\prime} \in \nats} A_{i^{\prime},k} \right)$, 
where the second term is due to the case when the set on the right hand side of \eqref{eqn:unbounded-universal2-hati} is empty.
Thus, $\hat{i}_{n,k}^{-1}(\{i\})$ is a measurable set.
Since any $C \subseteq \nats$ has 
$\hat{i}_{n,k}^{-1}(C) = \bigcup\limits_{i \in C} \hat{i}_{n,k}^{-1}(\{i\})$, 
we conclude that $\hat{i}_{n,k}$ is measurable, 
and this holds for all $k$ by the principle of induction.
Therefore, $\hat{f}_{n}$ is a valid inductive learning rule.

Now fix any $\ProcX \in \UKC$ and any measurable function $\target : \X \to \Y$.
To simplify the notation, let us abbreviate $\hat{i}_{n,k} = \hat{i}_{n,k}(X_{1:n},\target(X_{1:n}))$ for every $n \in \nats$ and $k \in \nats \cup \{0\}$.
Let $E$ denote the event of probability one guaranteed by Lemma~\ref{lem:unbounded-constrained-approximating-sequence},
for the process $\ProcX$ and the function $f = \target$:
that is, on $E$, $\forall \eps > 0$, $\forall i \in \nats$, $\exists j \in \nats$ with 
\begin{align}
\sup_{x \in \X} \loss\!\left( \tilde{f}_{j}(x), \tilde{f}_{i}(x) \right) & \leq \eps + \triconst \sup_{t \in \nats} \loss\!\left( \tilde{f}_{i}(X_{t}), \target(X_{t}) \right) \label{eqn:unbounded-suploss-eps-bound-1}
\\ \text{and }~~ \sup_{t \in \nats} \loss\!\left( \tilde{f}_{j}(X_{t}), \target(X_{t}) \right) & \leq \eps. \label{eqn:unbounded-suploss-eps-bound-2}
\end{align}
Let us suppose this event $E$ occurs.

We now argue by induction that, $\forall k \in \nats \cup \{0\}$,
$\exists i_{k}^{*}, n_{k}^{*} \in \nats$ such that, $\forall n \geq n_{k}^{*}$, $\hat{i}_{n,k} = i_{k}^{*}$
and $\sup\limits_{t \in \nats} \loss\!\left( \tilde{f}_{i_{k}^{*}}(X_{t}), \target(X_{t}) \right) \leq \eps_{k}$,
for $\eps_{k}$ as defined above \eqref{eqn:unbounded-universal2-hati}.
In particular, as a base case, let us define $i_{0}^{*} = 1$ and $n_{0}^{*} = 1$, 
for which the claims trivially hold 
since we have defined $\hat{i}_{n,0} = 1$ for every $n \in \nats$, and
moreover, $\eps_{0} = \infty$, so that we trivially have $\sup\limits_{t \in \nats} \loss\!\left( \tilde{f}_{i_{0}^{*}}(X_{t}), \target(X_{t}) \right) \leq \eps_{0}$.

Now take as an inductive hypothesis that, for some $k \in \nats$,
$\exists i_{k-1}^{*}, n_{k-1}^{*} \in \nats$ such that, $\forall n \geq n_{k-1}^{*}$, it holds that $\hat{i}_{n,k-1} = i_{k-1}^{*}$
and $\sup\limits_{t \in \nats} \loss\!\left( \tilde{f}_{i_{k-1}^{*}}(X_{t}), \target(X_{t}) \right) \leq \eps_{k-1}$.
Then define
\begin{equation*}
i_{k}^{*} = \min\left\{ j \in \nats : \sup_{t \in \nats} \loss\!\left( \tilde{f}_{j}(X_{t}), \target(X_{t}) \right) \leq \eps_{k}
\text{ and } \sup_{x \in \X} \loss\!\left( \tilde{f}_{j}(x), \tilde{f}_{i_{k-1}^{*}}(x) \right) \leq \triconst\eps_{k-1} + \eps_{k} \right\}.
\end{equation*}
Note that, taking $\eps = \eps_{k}$ and $i = i_{k-1}^{*}$ in \eqref{eqn:unbounded-suploss-eps-bound-1} and \eqref{eqn:unbounded-suploss-eps-bound-2}, 
and combining with the fact (from the inductive hypothesis) that $\sup\limits_{t \in \nats} \loss\!\left( \tilde{f}_{i_{k-1}^{*}}(X_{t}), \target(X_{t}) \right) \leq \eps_{k-1}$,
we can conclude that the set of values $j$ on the right hand side of the definition of $i_{k}^{*}$ is nonempty,
so that $i_{k}^{*}$ is a well-defined element of $\nats$.
In particular, by definition, we have $\sup\limits_{t \in \nats} \loss\!\left( \tilde{f}_{i_{k}^{*}}(X_{t}), \target(X_{t}) \right) \leq \eps_{k}$.
Next note that, by minimality of $i_{k}^{*}$, for every $j \in \nats$ 
with $j < i_{k}^{*}$ and $\sup\limits_{x \in \X} \loss\!\left( \tilde{f}_{j}(x), \tilde{f}_{i_{k-1}^{*}}(x) \right) \leq \triconst\eps_{k-1} + \eps_{k}$
(if any such $j$ exists), 
we have $\sup\limits_{t \in \nats} \loss\!\left( \tilde{f}_{j}(X_{t}), \target(X_{t}) \right) > \eps_{k}$, 
so that $\exists t_{j,k} \in \nats$ such that $\loss\!\left( \tilde{f}_{j}(X_{t_{j,k}}), \target(X_{t_{j,k}}) \right) > \eps_{k}$.
Now define
\begin{equation*}
n_{k}^{*} = \max\!\left( \left\{ t_{j,k} : j \!\in\! \{1,\ldots,i_{k}^{*}-1\}, \sup\limits_{x \in \X} \loss\!\left( \tilde{f}_{j}(x), \tilde{f}_{i_{k-1}^{*}}(x) \right) \leq \triconst\eps_{k-1} + \eps_{k} \right\} \cup \left\{ n_{k-1}^{*} \right\} \right),
\end{equation*}
which (being a maximum of a finite subset of $\nats$) is a finite positive integer.
In particular, note that (since $n_{k}^{*} \geq n_{k-1}^{*}$) for any $n \geq n_{k}^{*}$, 
the inductive hypothesis implies $\hat{i}_{n,k-1} = i_{k-1}^{*}$.
Additionally, for any $n \geq n_{k}^{*}$, every $j \in \nats$ with $j < i_{k}^{*}$ and 
$\sup\limits_{x \in \X} \loss\!\left( \tilde{f}_{j}(x), \tilde{f}_{i_{k-1}^{*}}(x) \right) \leq \triconst\eps_{k-1} + \eps_{k}$
has $\max\limits_{1 \leq t \leq n} \loss\!\left( \tilde{f}_{j}(X_{t}), \target(X_{t}) \right) \geq \loss\!\left( \tilde{f}_{j}(X_{t_{j,k}}), \target(X_{t_{j,k}}) \right) > \eps_{k}$.
In particular, this means that any such $j$ is not included in the set on the right hand side of \eqref{eqn:unbounded-universal2-hati} (when $x_{1:n} = X_{1:n}$ and $y_{1:n} = \target(X_{1:n})$).
Furthermore, for $n \geq n_{k}^{*}$, every $j \in \nats$ with $j < i_{k}^{*}$ and $\sup\limits_{x \in \X} \loss\!\left( \tilde{f}_{j}(x), \tilde{f}_{i_{k-1}^{*}}(x) \right) > \triconst\eps_{k-1} + \eps_{k}$
is clearly also not included in the set on the right hand side of \eqref{eqn:unbounded-universal2-hati} in this case (again, since $\hat{i}_{n,k-1} = i_{k-1}^{*}$).
On the other hand, by definition we have 
$\sup\limits_{x \in \X} \loss\!\left( \tilde{f}_{i_{k}^{*}}(x), \tilde{f}_{i_{k-1}^{*}}(x) \right) \leq \triconst\eps_{k-1} + \eps_{k}$,
and
$\max\limits_{1 \leq t \leq n} \loss\!\left( \tilde{f}_{i_{k}^{*}}(X_{t}), \target(X_{t}) \right)
\leq \sup\limits_{t \in \nats} \loss\!\left( \tilde{f}_{i_{k}^{*}}(X_{t}), \target(X_{t}) \right)
\leq \eps_{k}$,
so that, since
$\hat{i}_{n,k-1} = i_{k-1}^{*}$,
we have that $i_{k}^{*}$ \emph{is} included in the set on the right hand
side of \eqref{eqn:unbounded-universal2-hati} (with $x_{1:n} = X_{1:n}$ and $y_{1:n} = \target(X_{1:n})$).
Together with the definition of $\hat{i}_{n,k}$, 
these observations imply that, for any $n \geq n_{k}^{*}$, 
it holds that $\hat{i}_{n,k} = i_{k}^{*}$.

By the principle of induction, we have established the existence of a sequence $\{n_{k}^{*}\}_{k=0}^{\infty}$ in $\nats$ such that,
$\forall k \in \nats \cup \{0\}$, $\forall n \in \nats$ with $n \geq n_{k}^{*}$, we have 
$\sup\limits_{t \in \nats} \loss\!\left( \tilde{f}_{\hat{i}_{n,k}}(X_{t}), \target(X_{t}) \right) \leq \eps_{k}$.
Now for any $n \in \nats$, let $k_{n}^{*} = \max\left\{ k \in \{0,\ldots,k_{n}\} : n \geq n_{k}^{*} \right\}$ (recalling that we defined $n_{0}^{*} = 1$ above, so that $k_{n}^{*}$ always exists).
Note that, by the above guarantee, 
\begin{equation}
\label{eqn:unbounded-eps-k-n-star-bound-1}
\sup\limits_{t \in \nats} \loss\!\left( \tilde{f}_{\hat{i}_{n,k_{n}^{*}}}\!\!(X_{t}), \target(X_{t}) \right) \leq \eps_{k_{n}^{*}}.
\end{equation}
Furthermore, since $k_{n} \to \infty$, and each $n_{k}^{*}$ is finite, we have that $k_{n}^{*} \to \infty$.

Note that, by definition, for each $k \in \{1,\ldots,k_{n}\}$, 
we have $\sup\limits_{x \in \X} \loss\!\left( \tilde{f}_{\hat{i}_{n,k}}(x), \tilde{f}_{\hat{i}_{n,k-1}}(x) \right) \leq \triconst\eps_{k-1} + \eps_{k}$
(noting that this is true even when the set on the right hand side of \eqref{eqn:unbounded-universal2-hati} is empty, 
by our choice to define $\hat{i}_{n,k} = \hat{i}_{n,k-1}$ in that case).
Combining this with an inductive application of the relaxed triangle inequality and subadditivity of the supremum,
and noting that $k_{n}^{*} \leq k_{n}$ (by definition), this implies
\begin{align*}
& \sup_{x \in \X} \loss\!\left( \tilde{f}_{\hat{i}_{n,k_{n}}}(x), \tilde{f}_{\hat{i}_{n,k_{n}^{*}}}(x) \right)
\leq \sup_{x \in \X} \sum_{k=k_{n}^{*}+1}^{k_{n}} \triconst^{k-k_{n}^{*}}\loss\!\left( \tilde{f}_{\hat{i}_{n,k}}(x), \tilde{f}_{\hat{i}_{n,k-1}}(x) \right)
\\ & \leq \sum_{k=k_{n}^{*}+1}^{k_{n}} \sup_{x \in \X} \triconst^{k-k_{n}^{*}} \loss\!\left( \tilde{f}_{\hat{i}_{n,k}}(x), \tilde{f}_{\hat{i}_{n,k-1}}(x) \right)
\\ & \leq \sum_{k=k_{n}^{*}+1}^{k_{n}} \triconst^{k-k_{n}^{*}} \left( \triconst\eps_{k-1} + \eps_{k} \right)
\leq \sum_{k=k_{n}^{*}+1}^{\infty} \triconst^{k-k_{n}^{*}} \left( \triconst\eps_{k-1} + \eps_{k} \right).
\end{align*}
If $k_{n}^{*} \geq 1$, 
then by our choice of $\eps_{k} = (2\triconst)^{-k}$ for every $k \in \nats$, 
the rightmost expression above equals $\triconst^{-k_{n}^{*}} (2 \triconst^{2} + 1) \cdot 2^{-k_{n}^{*}} = (2 \triconst^{2} + 1) \eps_{k_{n}^{*}}$;
on the other hand, if $k_{n}^{*} = 0$, then our choice of $\eps_{0} = \infty$ implies 
the expression is $\infty = (2\triconst^{2}+1)\eps_{0}$. 
Thus, either way, we have
\begin{equation}
\label{eqn:unbounded-eps-k-n-star-bound-2}
\sup_{x \in \X} \loss\!\left( \tilde{f}_{\hat{i}_{n,k_{n}}}(x), \tilde{f}_{\hat{i}_{n,k_{n}^{*}}}(x) \right) \leq (2 \triconst^{2} + 1) \eps_{k_{n}^{*}}.
\end{equation}
Therefore, by the relaxed triangle inequality, $\forall n \in \nats$,
\begin{align*}
& \sup_{t \in \nats} \loss\!\left( \tilde{f}_{\hat{i}_{n,k_{n}}}(X_{t}), \target(X_{t}) \right)
\leq \sup_{t \in \nats} \triconst \left( \loss\!\left( \tilde{f}_{\hat{i}_{n,k_{n}^{*}}}(X_{t}), \target(X_{t}) \right) + \loss\!\left( \tilde{f}_{\hat{i}_{n,k_{n}}}(X_{t}), \tilde{f}_{\hat{i}_{n,k_{n}^{*}}}(X_{t}) \right) \right)
\\ & \leq \triconst \sup_{t \in \nats} \loss\!\left( \tilde{f}_{\hat{i}_{n,k_{n}^{*}}}(X_{t}), \target(X_{t}) \right) + \triconst \sup_{x \in \X} \loss\!\left( \tilde{f}_{\hat{i}_{n,k_{n}}}(x), \tilde{f}_{\hat{i}_{n,k_{n}^{*}}}(x) \right)
\leq 2 (\triconst^{3} + \triconst) \eps_{k_{n}^{*}},
\end{align*}
where the last inequality is due to \eqref{eqn:unbounded-eps-k-n-star-bound-1} and \eqref{eqn:unbounded-eps-k-n-star-bound-2}.
Since $k_{n}^{*} \to \infty$ and $\eps_{k} \to 0$, 
and since $\hat{f}_{n}(X_{1:n},\target(X_{1:n}),\cdot) = \tilde{f}_{\hat{i}_{n,k_{n}}}(\cdot)$ by its definition in \eqref{eqn:unbounded-universal2-inductive-rule},
and $\loss$ is non-negative, 
we may conclude that
\begin{equation*}
\sup_{t \in \nats} \loss\!\left( \hat{f}_{n}(X_{1:n}, \target(X_{1:n}), X_{t}), \target(X_{t}) \right) \to 0.
\end{equation*}

Since all of the above claims hold on the event $E$, which has probability one,
and since the above argument holds for \emph{any} choice of measurable function $\target : \X \to \Y$,
we may conclude that, for any measurable $\target : \X \to \Y$, 
\begin{equation}
\label{eqn:unbounded-sup-convergence}
\sup_{t \in \nats} \loss\!\left( \hat{f}_{n}(X_{1:n}, \target(X_{1:n}), X_{t}), \target(X_{t}) \right) \to 0 \text{ (a.s.)}.
\end{equation}
This further implies that, for any measurable $\target : \X \to \Y$, 
\begin{align*}
\lim_{n \to \infty} \hat{\L}_{\ProcX}\!\left( \hat{f}_{n}, \target ; n \right)
& = \lim_{n \to \infty} \limsup_{m \to \infty} \frac{1}{m} \sum_{t=n+1}^{n+m} \loss\!\left( \hat{f}_{n}(X_{1:n}, \target(X_{1:n}), X_{t}), \target(X_{t}) \right)
\\ & \leq \lim_{n \to \infty} \sup_{t \in \nats} \loss\!\left( \hat{f}_{n}(X_{1:n}, \target(X_{1:n}), X_{t}), \target(X_{t}) \right)
= 0 \text{ (a.s.)}.
\end{align*}
Thus, since $\hat{\L}_{\ProcX}$ is non-negative, we conclude that 
the inductive learning rule $\hat{f}_{n}$ is strongly universally consistent under $\ProcX$.
In particular, this implies that $\ProcX$ \emph{admits} strong universal inductive learning: that is, $\ProcX \in \SUIL$.

The above argument can also be used to show that $\ProcX \in \SUOL$.
Specifically, consider this same $\hat{f}_{n}$ function defined above, 
but now interpreted as an \emph{online} learning rule.
We then have, for any measurable $\target : \X \to \Y$, 
\begin{align}
\lim_{n \to \infty} \hat{\L}_{\ProcX}\!\left( \hat{f}_{\cdot}, \target; n \right)
& = \lim_{n \to \infty} \frac{1}{n} \sum_{t=0}^{n-1} \loss\!\left( \hat{f}_{t}(X_{1:t},\target(X_{1:t}),X_{t+1}),\target(X_{t+1}) \right)
\notag \\ & \leq \lim_{n \to \infty} \frac{1}{n} \sum_{t=0}^{n-1} \sup_{m \in \nats} \loss\!\left( \hat{f}_{t}(X_{1:t},\target(X_{1:t}),X_{m}),\target(X_{m}) \right). \label{eqn:unbounded-online-sup-relxation}
\end{align}
The convergence in \eqref{eqn:unbounded-sup-convergence} implies
$\sup\limits_{m \in \nats} \loss\!\left( \hat{f}_{t}(X_{1:t},\target(X_{1:t}),X_{m}),\target(X_{m}) \right) \to 0 \text{ (a.s.)}$ as $t \to \infty$.
Thus, since the arithmetic mean of the first $n$ elements in any convergent sequence in $\reals$ is also convergent (as $n \to \infty$) with the same limit value, 
this immediately implies that 
the final expression in \eqref{eqn:unbounded-online-sup-relxation} equals $0$ almost surely.  
Since this holds for any measurable $\target : \X \to \Y$, and $\loss$ is non-negative, 
we have that $\hat{f}_{n}$ is also a strongly universally consistent \emph{online} learning rule under $\ProcX$.
In particular, this implies that $\ProcX$ admits strong universal online learning:
that is, $\ProcX \in \SUOL$.

Finally, since the above arguments hold for \emph{any} choice of $\ProcX \in \UKC$,
we may conclude that $\UKC \subseteq \SUIL \cap \SUOL$, which completes the proof.
\end{proof}

Combining the above lemmas immediately provides the following proof of Theorem~\ref{thm:unbounded-main}.\\

\begin{proof}[of Theorem~\ref{thm:unbounded-main}]
Taking Lemmas~\ref{lem:unbounded-suil-subset-sual}, \ref{lem:unbounded-sual-subset-ukc}, and \ref{lem:unbounded-ukc-subset-suil} together, 
we have that $\SUIL \cup \SUOL \subseteq \SUAL \cup \SUOL \subseteq \UKC \subseteq \SUIL \cap \SUOL \subseteq \SUAL \cap \SUOL$.
This further implies that $\SUAL \bigtriangleup \SUOL = (\SUAL \cup \SUOL) \setminus (\SUAL \cap \SUOL) = \emptyset$,
and similarly $\SUIL \bigtriangleup \SUOL = (\SUIL \cup \SUOL) \setminus (\SUIL \cap \SUOL) = \emptyset$, 
so that $\SUIL = \SUOL = \SUAL$.  Combining this with Lemmas~\ref{lem:unbounded-sual-subset-ukc} and \ref{lem:unbounded-ukc-subset-suil},
we obtain $\SUOL = \SUAL \cup \SUOL \subseteq \UKC \subseteq \SUIL \cap \SUOL = \SUOL$, so that $\SUOL = \UKC$.
Hence $\SUIL = \SUAL = \SUOL = \UKC$, which completes the proof.
\end{proof}

We may also note that the proof of Lemma~\ref{lem:unbounded-ukc-subset-suil} specifically 
establishes that the inductive learning rule $\hat{f}_{n}$ specified in \eqref{eqn:unbounded-universal2-inductive-rule}
(with $\{\tilde{f}_{i}\}_{i=1}^{\infty}$ an enumeration of the countable set $\tilde{\F}$ from Lemma~\ref{lem:unbounded-constrained-approximating-sequence})
is strongly universally consistent for every $\ProcX \in \UKC$, and therefore by Theorem~\ref{thm:unbounded-main} (just established),
for every $\ProcX \in \SUIL$ when $\maxloss = \infty$.
Since the definition of $\hat{f}_{n}$ has no direct dependence on the distribution of $\ProcX$, this 
implies $\hat{f}_{n}$ is an optimistically universal inductive learning rule when $\maxloss = \infty$.
This is particularly interesting, as it contrasts with the fact, established in Theorem~\ref{thm:no-optimistic-inductive} above,
that for \emph{bounded} losses, no optimistically universal inductive learning rule exists (if $\X$ is an uncountable Polish space).
Furthermore, this also means we can easily define an optimistically universal \emph{self-adaptive} learning rule when $\maxloss = \infty$, 
simply defining 
\begin{equation}
\label{eqn:unbounded-universal2-self-adaptive-rule-defn}
\hat{g}_{n,m}(x_{1:m},y_{1:n},x) = \hat{f}_{n}(x_{1:n},y_{1:n},x)
\end{equation}
for every $n,m \in \nats \cup \{0\}$ with $m \geq n$, and every $x_{1:m} \in \X^{m}$, $y_{1:n} \in \Y^{n}$, and $x \in \X$.
In particular, it is clear that $\hat{\L}_{\ProcX}(\hat{g}_{n,\cdot},\target;n) = \hat{\L}_{\ProcX}(\hat{f}_{n},\target;n)$ for this definition of $\hat{g}_{n,m}$.
Thus, since $\hat{f}_{n}$ is strongly universally consistent under every $\ProcX \in \UKC$ by Lemma~\ref{lem:unbounded-ukc-subset-suil}, 
it immediately follows that $\hat{g}_{n,m}$ also has this property, and the fact that it is an optimistically universal self-adaptive learning rule (when $\maxloss = \infty$)
then follows from $\SUAL = \UKC$ (from Theorem~\ref{thm:unbounded-main}, just established).
The proof of Lemma~\ref{lem:unbounded-ukc-subset-suil} also establishes strong universal consistency of $\hat{f}_{n}$ under any $\ProcX \in \UKC$ 
when $\hat{f}_{n}$ is interpreted as an \emph{online} learning rule, so that (since $\UKC = \SUOL$ when $\maxloss = \infty$, again by Theorem~\ref{thm:unbounded-main}) 
$\hat{f}_{n}$ is also an optimistically universal \emph{online} learning rule when $\maxloss = \infty$.
We summarize these findings in the following theorem.

\begin{theorem}
\label{thm:unbounded-universal2-inductive}
When $\maxloss = \infty$, with $\{\tilde{f}_{i}\}_{i=1}^{\infty}$ an enumeration of the countable set $\tilde{\F}$ from Lemma~\ref{lem:unbounded-constrained-approximating-sequence},
the learning rule $\hat{f}_{n}$ from \eqref{eqn:unbounded-universal2-inductive-rule} is an optimistically universal inductive learning rule,
and an optimistically universal online learning rule.  
Moreover, defining $\hat{g}_{n,m}$ as in \eqref{eqn:unbounded-universal2-self-adaptive-rule-defn},
when $\maxloss = \infty$, $\hat{g}_{n,m}$ is an optimistically universal self-adaptive learning rule.
\end{theorem}

In particular, this implies that for unbounded losses, there \emph{exist} optimistically universal (inductive/self-adaptive/online) learning rules, 
so that Theorem~\ref{thm:unbounded-optimistically-universal} immediately follows.

{\vskip 3mm}\noindent {\bf Remark:}~
Interestingly, the proof of Lemma~\ref{lem:unbounded-ukc-subset-suil} in fact establishes a much stronger kind of convergence for $\hat{f}_{n}$ under any $\ProcX \in \UKC$:
for any measurable $\target : \X \to \Y$, 
\begin{equation}
\label{eqn:sup-consistency}
\sup_{t \in \nats} \loss\!\left( \hat{f}_{n}(X_{1:n},\target(X_{1:n}),X_{t}), \target(X_{t}) \right) \to 0 \text{ (a.s.)}.
\end{equation}
Denoting by $\SUIL^{{\rm sup}}$ the set of processes $\ProcX$ that admit the existence of an inductive learning rule $\hat{f}_{n}$ 
satisfying \eqref{eqn:sup-consistency} for every measurable $\target : \X \to \Y$,
we have thus established that $\UKC \subseteq \SUIL^{{\rm sup}}$ when $\maxloss = \infty$.
Furthermore, as shown in the proof of Lemma~\ref{lem:unbounded-ukc-subset-suil}, 
this type of convergence itself implies strong universal consistency of $\hat{f}_{n}$ in the original sense of Definition~\ref{def:suil},
so that $\SUIL^{{\rm sup}} \subseteq \SUIL$.
Thus, since $\SUIL = \UKC$ when $\maxloss = \infty$ (from Theorem~\ref{thm:unbounded-main}, just established), 
we have established that, when $\maxloss = \infty$, $\SUIL^{{\rm sup}} = \SUIL$:
that is, the set of processes $\ProcX$ admitting 
this stronger type of universal consistency is in fact the \emph{same} as those admitting strong universal inductive
learning in the usual sense of Definition~\ref{def:suil}.
It is clear that this is \emph{not} the case
when $\maxloss < \infty$ if $\X$ is infinite.  
Indeed, combining the proof of Lemma~\ref{lem:unbounded-ukc-subset-suil} with 
a straightforward variation on the proof of Lemma~\ref{lem:unbounded-sual-subset-ukc},
one can show that \emph{even when $\maxloss < \infty$}, 
Condition~\ref{con:ukc} remains a necessary and sufficient condition for a process $\ProcX$
to admit the existence of an inductive learning rule satisfying \eqref{eqn:sup-consistency} for all 
measurable functions $\target : \X \to \Y$: that is, $\SUIL^{{\rm sup}} = \UKC$.
For these same reasons, the same is true of the analogous guarantee for self-adaptive or online learning:
that is, regardless of whether $\maxloss = \infty$ or $\maxloss < \infty$, 
Condition~\ref{con:ukc} is necessary and sufficient for there to exist 
a self-adaptive learning rule $\hat{g}_{n,m}$ such that, for all measurable $\target : \X \to \Y$, 
\begin{equation*}
\sup\limits_{t \in \nats : t \geq n} \loss\!\left( \hat{g}_{n,t}(X_{1:t},\target(X_{1:n}),X_{t+1}), \target(X_{t+1}) \right) \to 0 \text{ (a.s.)},
\end{equation*}
and Condition~\ref{con:ukc} is also necessary and sufficient for there to exist 
an online learning rule $\hat{h}_{n}$ such that, for all measurable $\target : \X \to \Y$, 
\begin{equation*}
\loss\!\left( \hat{h}_{n}(X_{1:n},\target(X_{1:n}),X_{n+1}), \target(X_{n+1}) \right) \to 0 \text{ (a.s.)}.
\end{equation*}

\subsection{No Consistent Test for Existence of a Universally Consistent Learner}
\label{subsec:unbounded-no-consistent-test}
As we did in Section~\ref{sec:no-consistent-test} in the case of bounded losses, 
it is also natural to ask whether there exist consistent hypothesis tests
for whether or not a given data process $\ProcX$ admits strong universal
learning, in this case when $\maxloss = \infty$.  As was true for bounded losses,
we again find that the answer is generally \emph{no}.  Formally, we have the following theorem.

\begin{theorem}
\label{thm:unbounded-no-consistent-test-for-suil}
When $\maxloss = \infty$ and $\X$ is infinite, 
there is no consistent hypothesis test for $\SUIL$, $\SUAL$, or $\SUOL$.
\end{theorem}
\begin{proof}
Suppose $\X$ is infinite.
Since Theorem~\ref{thm:unbounded-main} implies $\SUIL = \SUAL = \SUOL = \UKC$ when $\maxloss = \infty$, 
it suffices to prove that there is no consistent hypothesis test for $\UKC$.
Fix any hypothesis test $\hat{t}_{n}$.
Fix $\ProcX$ to be that specific process constructed in the proof of Theorem~\ref{thm:no-consistent-test-for-suil},
relative to this hypothesis test $\hat{t}_{n}$.
The proof of Theorem~\ref{thm:no-consistent-test-for-suil} (combined with Theorem~\ref{thm:main}) establishes that,
for this specific process $\ProcX$, 
if $\ProcX \in \KC$, then $\hat{t}_{n}(X_{1:n})$ fails to converge in probability to $1$, 
and if $\ProcX \notin \KC$, then $\hat{t}_{n}(X_{1:n})$ fails to converge in probability to $0$.

Recall that $\UKC \subseteq \KC$,
so that if $\ProcX \notin \KC$, then $\ProcX \notin \UKC$ as well.
But, as mentioned above, $\hat{t}_{n}(X_{1:n})$ fails to 
converge in probability to $0$ in this case. Thus, in the case that this process $\ProcX \notin \KC$, we have
established that $\hat{t}_{n}$ is not a consistent test for $\UKC$.

On the other hand, in the case that the constructed process $\ProcX$ \emph{is} in $\KC$,
there are two subcases to consider.  First, recalling the construction of $\ProcX$,
if there exists a largest $k \in \nats$ for which $n_{k-1}$ is defined, then for $\ProcX$
to be in $\KC$ we necessarily have $(k+1)/2 \in \nats$ (i.e., $k$ is odd).  In this case,
every $t > n_{k-1}$ has $X_{t} = w_{0} = X_{n_{k-1}+1}$, 
so that for any disjoint sequence $\{A_{i}\}_{i=1}^{\infty}$ in $\Borel$,
\begin{equation*}
\left| \left\{ i \in \nats : \ProcX \cap A_{i} \neq \emptyset \right\} \right|
= \left| \left\{ i \in \nats : X_{1:(n_{k-1}+1)} \cap A_{i} \neq \emptyset \right\} \right| 
\leq n_{k-1}+1 < \infty.
\end{equation*}
Therefore, Lemma~\ref{lem:ukc-partition-equiv} implies 
that $\ProcX \in \UKC$ as well.
But, as mentioned above, in the case that this constructed process $\ProcX \in \KC$, $\hat{t}_{n}(X_{1:n})$ fails to converge in probability to $1$,
so that if $\ProcX \in \KC$ and there is a largest $k \in \nats$ with $n_{k-1}$ defined, this establishes that $\hat{t}_{n}$ is 
not a consistent test for $\UKC$.
Finally, the only remaining case is where $\ProcX \in \KC$ and $n_{k-1}$ is defined for every $k \in \nats$.
In this case, as established in the proof of Theorem~\ref{thm:no-consistent-test-for-suil}, $\hat{t}_{n}(X_{1:n})$
fails to converge in probability \emph{at all} (i.e., \emph{neither} converges in probability to $0$ \emph{nor} converges in probability to $1$),
which trivially establishes that $\hat{t}_{n}$ is not a consistent test for $\UKC$ in this case as well.
\end{proof}

Since it is trivially true that \emph{every} $\ProcX$ is in $\UKC$ when $\X$ is \emph{finite} 
(and hence also in $\SUIL$, $\SUAL$, and $\SUOL$ when $\maxloss = \infty$, by Theorem~\ref{thm:unbounded-main}),
we have the following immediate corollary.

\begin{corollary}
\label{cor:unbounded-consistent-test}
When $\maxloss = \infty$, there exist consistent hypothesis tests for each of $\SUIL$, $\SUAL$, and $\SUOL$
if and only if $\X$ is finite.
\end{corollary}

\section{Noisy Responses}
\label{sec:noise}

In much of the statistical learning theory literature, it is common to 
suppose that the response $Y_{t}$ is \emph{noisy}, so that rather 
than $Y_{t} = \target(X_{t})$ always, we merely have that among all 
measurable functions $f : \X \to \Y$, the conditional expectation $\E[ \loss( f(X_{t}), Y_{t} ) | X_{t} ]$ 
is minimized for $f = \target$.  For instance, in classification with $\loss$ the $0$-$1$ loss and $\Y = \{0,1\}$, 
this corresponds to having $Y_{t} = \target(X_{t})$ with a probability at least $1/2$ given $X_{t}$.
In regression with $\Y$ an interval in $\reals$ and $\loss$ the squared loss ($\loss(a,b)=(a-b)^{2}$), 
it is well known that the point-wise minimizer of $\E[\loss(f(X),Y)|X]$ is the \emph{conditional mean}: 
$\target(X) = \E[Y|X]$ (a.s.).  

It is interesting to consider how the theory developed in the sections above can be modified to accommodate noisy responses.
Here we are still interested in obtaining low long-run average loss.
However, in the presence of noise we generally cannot hope to achieve \emph{zero} average loss in the limit.
We must therefore adjust our goal.  Instead, we will aim to achieve zero \emph{excess} loss, 
relative to a fixed \emph{optimal} function: that is, we will still suppose there is a function $\target$ 
representing an optimal predictor, and we will evaluate our performance relative to this function.

In this context, we achieve two different strengths of results. 
First, we show that for certain restricted types of losses $\loss$, 
there exists a self-adaptive learning rule that is strongly universally consistent for any 
$(\ProcX,\ProcY) = \{(X_{t},Y_{t})\}_{t=1}^{\infty}$ with $\ProcX \in \KC$ and $\ProcY$ satisfying a conditional independence property: 
that is, the $Y_{t}$ variables are conditionally independent given their respective $X_{t}$ variables.
In particular, this result applies to the \emph{squared loss} for $\Y$ any bounded interval in $\reals$.
However, it turns out classification with the $0$-$1$ loss does not satisfy the requirements on $\loss$ for this result.
To address classification, we propose a second, stronger condition, 
where we suppose $Y_{t}$ is a \emph{noisy function} of $X_{t}$: that is, 
the $Y_{t}$ values are conditionally independent 
and the conditional distribution of $Y_{t}$ given $X_{t}$ is a $t$-invariant function of $X_{t}$.
We show that there exists a self-adaptive learning rule that is strongly universally consistent 
for classification with finite $\Y$ and the $0$-$1$ loss, for all processes of this type with $\ProcX \in \KC$.
The question of learning, either with general conditionally independent $\ProcY$ or with noisy functions, 
for general bounded separable losses $\loss$ is left for future work.

\subsection{Definitions}
\label{sec:noise-definitions}

In this general setting, for any measurable $\goodfun : \X \to \Y$, 
for any process $(\ProcX,\ProcY) = \{(X_{t},Y_{t})\}_{t=1}^{\infty}$ on $\X \times \Y$, 
for any inductive learning rule $f_{n}$ and self-adaptive learning rule $g_{n,m}$, and any $n \in \nats$, 
we define the long-run average \emph{excess} loss 
\begin{align*}
\hat{\L}_{\ProcX}(f_{n},\ProcY;n,\goodfun) & = \limsup\limits_{t \to \infty} \frac{1}{t} \sum_{m=n+1}^{n+t} \left( \loss\!\left( f_{n}(X_{1:n}, Y_{1:n}, X_{m}), Y_{m} \right) - \loss\!\left(\goodfun(X_{m}),Y_{m}\right) \right),\\
\hat{\L}_{\ProcX}(g_{n,\cdot},\!\ProcY;n,\goodfun) & = \limsup\limits_{t \to \infty} \frac{1}{t+1} \sum_{m=n}^{n+t} \left( \loss\!\left( g_{n,m}(X_{1:m},Y_{1:n},X_{m+1}), Y_{m+1} \right) \!-\! \loss\!\left(\goodfun\!(X_{m+1}),Y_{m+1}\right) \right).
\end{align*}
We are then interested in learning rules $f_{n}$, $g_{n,m}$ that guarantee that, for all measurable $\goodfun$, 
$\limsup\limits_{n \to \infty} \hat{\L}_{\ProcX}(f_{n},\ProcY;n,\goodfun) \leq 0$ almost surely, 
or $\limsup\limits_{n \to \infty} \hat{\L}_{\ProcX}(g_{n,\cdot},\ProcY;n,\goodfun) \leq 0$ almost surely, respectively, 
for some specific family of processes $(\ProcX,\ProcY)$. 
For brevity, we will not discuss the online setting in detail in this section, 
though an analogous generalization is possible there: that is, 
$\hat{\L}_{\ProcX}(f_{\cdot},\ProcY;\goodfun) = \limsup\limits_{n \to \infty} \frac{1}{n} \sum\limits_{t=0}^{n-1} \left( \loss\!\left( f_{t}(X_{1:t}, Y_{1:t}, X_{t+1}), Y_{t+1} \right) - \loss\!\left(\goodfun(X_{t+1}),Y_{t+1}\right) \right)$.

It is clear that some kind of restriction to the dependences among the $(X_{t},Y_{t})$ variables would be 
required for any positive result to be possible.  
While the argument we follow here can also be applied in more general scenarios (within certain limits), 
as a simple scenario to consider we restrict to the following two key requirements.

\begin{itemize}
\item[\Yone.] We restrict to processes $\ProcY = \{Y_{t}\}_{t=1}^{\infty}$ (dependent on $\ProcX$) 
with the property that there exists a measurable function $\target : \X \to \Y$ with 
\begin{equation*}
\E[ \loss(\target(X_{t}),Y_{t}) | X_{t} ] = \inf_{y \in \Y} \E[ \loss(y,Y_{t}) | X_{t} ] \text{ (a.s.)}
\end{equation*}
for every $t \in \nats$.  In other words, we assume there is a time-invariant optimal function.\footnote{In principle, the theory 
below can be extended to cases where the infimum does not exist, but there exist $\target_{\eps}$ functions 
with $\E[ \loss(\target(X_{t}),Y_{t}) | X_{t} ] \leq \inf_{y \in \Y} \E[ \loss(y,Y_{t}) | X_{t} ] + \eps \text{ (a.s.)}$.  
We restrict to cases where an optimal function $\target$ exists to simplify the exposition.}
\item[\Ytwo.]
We suppose the $Y_{t}$ variables are \emph{conditionally independent} given their respective $X_{t}$ variables.
Formally, we suppose $\ProcY = \{Y_{t}\}_{t=1}^{\infty}$ has the property that $\forall t \in \nats$, 
$Y_{t}$ is conditionally independent of $\{(X_{t^{\prime}},Y_{t^{\prime}})\}_{t^{\prime} \neq t}$ given $X_{t}$.
\end{itemize}

In particular, note that both of these conditions would be satisfied for any i.i.d.\ process $(\ProcX,\ProcY)$ 
if $\Y$ is sequentially compact
(so that the infimum exists in \Yone).  
Henceforth in this section, mentions of $\target$ will always refer to the function $\target$ 
guaranteed to exist by \Yone.
As we argue below in Lemma~\ref{lem:non-negative-noisy-excess}, 
for any process $(\ProcX,\ProcY)$ satisfying \Yone\ and \Ytwo, 
for any
learning rule $f_{n,m}$, 
we have $\hat{\L}_{\ProcX}(f_{n,\cdot},\ProcY;n,\target) \geq 0$ (a.s.).
Moreover, by similar arguments to the proof of Lemma~\ref{lem:non-negative-noisy-excess}, 
it is not hard to show that for any measurable function $\goodfun : \X \to \Y$, 
we have $\hat{\L}_{\ProcX}(f_{n,\cdot},\ProcY;n,\target) \geq \hat{\L}_{\ProcX}(f_{n,\cdot},\ProcY;n,\goodfun)$ (a.s.).
For these reasons, for $(\ProcX,\ProcY)$ satisfying \Yone\ and \Ytwo, 
we say a learning rule $f_{n,m}$ is \emph{consistent} if 
$\hat{\L}_{\ProcX}(f_{n,\cdot},\ProcY;n,\target) \to 0 \text{ (a.s.)}$.

Below we obtain two different results, corresponding to two different families of processes $(\ProcX,\ProcY)$.
These families
are stated formally in the following definitions.

\begin{definition}
\label{def:noises}
We say a process $(\ProcX,\ProcY)$ has \emph{independent noise} if it satisfies properties \Yone\ and \Ytwo\ above.
We say $\ProcY$ is a \emph{noisy function} of $\ProcX$ if $(\ProcX,\ProcY)$ has independent noise, 
and the conditional distribution of $Y_{t}$ given $X_{t}$ is a $t$-invariant function of $X_{t}$.
\end{definition}

The main difference between $(\ProcX,\ProcY)$ merely having independent noise and $\ProcY$ being a noisy function of $\ProcX$ 
is that the former case admits processes where the conditional distribution of $Y_{t}$ given $X_{t}$ varies over time.
For instance, as a simple example of a process $(\ProcX,\ProcY)$ having independent noise, 
consider $\Y = [-1,1]$ and $\loss = \loss_{\sq}$ the squared loss ($\loss_{\sq}(a,b) = (a-b)^2$), 
and take $\ProcX$ as any process, while $Y_{t} = \target(X_{t}) + \left(1-\frac{1}{t}\right)V_{t}$, 
where $\target$ is any measurable function with $\target(x) \in [-1/2,1/2]$ for all $x \in \X$,
and where $\{V_{t}\}_{t=1}^{\infty}$ are independent (and independent of $\ProcX$) with 
$V_{t} \sim {\rm Uniform}([-1/2,1/2])$.  This process $\ProcY$ gets \emph{noisier} over time, 
but the optimal function is always $\target$ (the conditional mean of $Y_{t}$ given $X_{t}$ being $\target(X_{t})$), 
and each $Y_{t}$ is conditionally independent of $\{(X_{t^{\prime}},Y_{t^{\prime}})\}_{t^{\prime} \neq t}$ given $X_{t}$.
Note that $\ProcY$ is \emph{not} a noisy function of $\ProcX$.
However, if instead we had defined $Y_{t} = \target(X_{t})+V_{t}$, then it would be.

We now formally state the criteria for universal consistency with noise.

\begin{definition}
\label{def:independent-noise}
For a self-adaptive learning rule $g_{n,m}$ and a process $\ProcX$ on $\X$,  
we say $g_{n,m}$ is strongly universally consistent 
\emph{with independent noise} under $\ProcX$ if, 
for every process $\ProcY$ such that $(\ProcX,\ProcY)$ has independent noise, 
it holds that 
$\lim\limits_{n \to \infty} \hat{\L}_{\ProcX}(g_{n,\cdot},\ProcY;n,\target) = 0$ (a.s.).
Similarly, for an inductive learning rule $f_{n}$, we say $f_{n}$ is strongly universally consistent with independent noise under $\ProcX$ if, 
for every $\ProcY$
such that $(\ProcX,\ProcY)$ has independent noise, 
$\lim\limits_{n \to \infty} \hat{\L}_{\ProcX}(f_{n},\ProcY;n,\target) = 0$ (a.s.).
We say a process $\ProcX$ admits strong universal (inductive/self-adaptive) learning with independent noise if there exists an (inductive/self-adaptive) 
learning rule that is strongly universally consistent with independent noise under $\ProcX$.
We say a self-adaptive learning rule $g_{n,m}$ is optimistically universal with independent noise 
if it is strongly universally consistent with independent noise under every $\ProcX$ that admits 
strong universal self-adaptive learning with independent noise.
\end{definition}

\begin{definition}
\label{def:noisy-function}
For a self-adaptive learning rule $g_{n,m}$ and a process $\ProcX$ on $\X$, 
we say $g_{n,m}$ is strongly universally consistent 
\emph{for noisy functions} under $\ProcX$ if, 
for every process $\ProcY$ that is a noisy function of $\ProcX$,
it holds that 
$\lim\limits_{n \to \infty} \hat{\L}_{\ProcX}(g_{n,\cdot},\ProcY;n,\target) = 0$ (a.s.).
Similarly, for an inductive learning rule $f_{n}$, we say $f_{n}$ is strongly universally consistent for noisy functions under $\ProcX$ if, 
for every $\ProcY$ that is a noisy function of $\ProcX$, 
$\lim\limits_{n \to \infty} \hat{\L}_{\ProcX}(f_{n},\ProcY;n,\target) = 0$ (a.s.).
We say a process $\ProcX$ admits strong universal (inductive/self-adaptive) learning for noisy functions if there exists an (inductive/self-adaptive) 
learning rule that is strongly universally consistent for noisy functions under $\ProcX$.
We say a self-adaptive learning rule $g_{n,m}$ is optimistically universal for noisy functions 
if it is strongly universally consistent for noisy functions under every $\ProcX$ that admits 
strong universal self-adaptive learning for noisy functions.
\end{definition}

In particular, since any i.i.d.\ process $(\ProcX,\ProcY)$ would have $\ProcY$ a noisy function of $\ProcX$
(if $\Y$ is sequentially compact, so that the infimum exists in \Yone), 
we note that the theory we develop here represents a proper generalization of the standard theory of 
universal consistency under i.i.d.\ processes with bounded losses \citep*[e.g.,][]{devroye:96,gyorfi:02}.

In Section~\ref{sec:independent-noise} below, 
we show that for certain restricted types of losses (essentially, those guaranteeing uniqueness of $\target$), 
any process $\ProcX$ satisfying Condition~\ref{con:kc} admits strong universal inductive and self-adaptive learning with independent noise.
In other words, a process $\ProcX$ admits strong universal learning with independent noise 
if (and only if) it admits strong universal learning \emph{without} noise.
In particular, this includes the squared loss ($\loss(a,b)=(a-b)^2$) for $\Y$ any bounded interval in $\reals$.
We also argue that there is a self-adaptive learning rule that is optimistically universal with independent noise.
Thus, the theory developed above for the noiseless setting completely generalizes to allow independent noise, 
for these loss functions.
However, it happens that the $0$-$1$ loss does not satisfy the requirements for this result.
As this is an important loss for the classification setting, in Section~\ref{sec:noisy-classification} 
we extend the theory to hold for 
learning with the $0$-$1$ loss ($\loss(a,b)=\ind[a \neq b]$) for any finite $\Y$, 
but only for the stronger \emph{noisy function} setting.
Specifically, we show that for this loss, any process $\ProcX$ satisfying Condition~\ref{con:kc} admits 
strong universal inductive and self-adaptive learning for noisy functions.
We also show that there is a self-adaptive learning rule that is optimistically universal for noisy functions.
Again, this result generalizes the results for the noiseless setting to the setting of noisy functions.
The approach to obtaining this result is via the traditional \emph{plug-in} technique \citep*[see e.g.,][]{devroye:96}, 
making use of consistent regression estimators (guaranteed to exist by the aforementioned result for learning with independent noise) 
to identify which element in $\Y$ has the highest likelihood given $X_{t}$.

Before proceeding with our results for learning with independent noise, 
we first discuss the motivation for the special role of $\target$ in the 
definition of universal consistency above.
This is motivated by the fact that, in any process $(\ProcX,\ProcY)$ that has independent noise, 
$\target$ is guaranteed to be an \emph{optimal} function (almost surely), 
so that no prediction rule can be \emph{better} than $\target$ in the limit.
Formally, we have the following lemma.

\begin{lemma}
\label{lem:non-negative-noisy-excess}
If $\maxloss < \infty$, for any deterministic self-adaptive learning rule $\hat{f}_{n,m}$, 
for any process $(\ProcX,\ProcY)$ that has independent noise, 
with probability one, $\forall n \in \nats$, 
$\hat{\L}_{\ProcX}(\hat{f}_{n,\cdot},\ProcY;n,\target) \geq 0$.
\end{lemma}
\begin{proof} 
For any $n,m \in \nats$ with $m \geq n$, 
let $g_{n,m}(X_{m+1}) = \hat{f}_{n,m}(X_{1:m},Y_{1:n},X_{m+1})$, 
and define $\Delta_{m+1}^{n} = \loss(g_{n,m}(X_{m+1}),Y_{m+1}) - \loss(\target(X_{m+1}),Y_{m+1})$.
For any $n \in \nats$ note that, by the conditional independence property \Ytwo, the sequence
\begin{equation*}
\left\{ \left( \Delta_{m+1}^{n} - \E\!\left[ \Delta_{m+1}^{n} \middle| X_{m+1}, g_{n,m}(X_{m+1}) \right] \right) \right\}_{m=n}^{\infty}
\end{equation*}
is a martingale difference sequence with respect to $\left\{\left( X_{1:(m+2)}, Y_{1:(m+1)} \right)\right\}_{m=n}^{\infty}$.
Therefore, Azuma's inequality \citep*[][Theorem 9.1]{devroye:96} implies that, 
for any $t \in \nats \cup \{0\}$, with probability at least $1-\frac{1}{(t+1)^2}$, 
\begin{equation*}
\frac{1}{t+1} \left| \sum_{m=n}^{n+t}  \Delta_{m+1}^{n}  - \E\!\left[ \Delta_{m+1}^{n} \middle| X_{m+1}, g_{n,m}(X_{m+1}) \right] \right| 
\leq \maxloss \sqrt{\frac{2\ln(2 (t+1)^2)}{t+1}}.
\end{equation*}
Since $\sum\limits_{t=0}^{\infty} \frac{1}{(t+1)^2} < \infty$, the Borel-Cantelli lemma implies that, with probability one, 
\begin{equation*}
\limsup_{t \to \infty} \frac{1}{t+1} \left| \sum_{m=n}^{n+t}  \Delta_{m+1}^{n} - \E\!\left[ \Delta_{m+1}^{n} \middle| X_{m+1}, g_{n,m}(X_{m+1}) \right] \right| = 0.
\end{equation*}
By the definition of $\target$ from \Yone, and the conditional independence property \Ytwo, for any 
$m \geq n$, 
$\E\!\left[ \Delta_{m+1}^{n} \middle| X_{m+1}, g_{n,m}(X_{m+1}) \right] \geq 0$ almost surely.
Altogether, by the union bound, on an event of probability one, we have 
\begin{equation*}
\hat{\L}_{\ProcX}(\hat{f}_{n,\cdot},\ProcY;n,\target) 
= \limsup_{t \to \infty} \frac{1}{t+1} \sum_{m=n}^{n+t} \E\!\left[ \Delta_{m+1}^{n} \middle| X_{m+1}, g_{n,m}(X_{m+1}) \right]
\geq 0.
\end{equation*}
Since this holds for any fixed $n \in \nats$, the lemma follows by the union bound over all $n \in \nats$.
\end{proof}

\subsection{Learning with Independent Noise}
\label{sec:independent-noise}

We begin with the general setting of learning with independent noise.
In this subsection we will restrict to the case of bounded losses ($\maxloss < \infty$), 
with $\loss$ satisfying some special properties.
Specifically, we suppose that 
there exist functions $\bclower$ and $\bcupper$ mapping $[0,\maxloss] \to [0,\infty)$, 
strictly increasing and continuous, with $\bclower$ convex and $\bcupper$ concave, 
such that $\bclower(0) = \bcupper(0) = 0$ and 
for any $\Y$-valued random variable $Y$, 
$\exists y^{*} \in \Y$ with $\E[ \loss(y^{*},Y) ] = \inf\limits_{y \in \Y} \E[ \loss(y,Y) ]$ 
(i.e., the infimum is realized in $\Y$), 
and 
$\forall y \in \Y$, 
\begin{equation}
\label{eqn:bcfunc}
\bclower(\loss(y,y^{*}))
\leq \E[ \loss(y,Y) - \loss(y^{*},Y) ]
\leq \bcupper(\loss(y,y^{*})).
\end{equation}

As a simple example of this, 
consider the case of bounded regression with the squared loss: 
$\Y = [0,1]$ and $\loss = \loss_{\sq}$ (where $\loss_{\sq}(y,y^{\prime}) = (y-y^{\prime})^{2}$ is the \emph{squared loss}).
In this case, as mentioned above, it is well known that 
$y^{*} = \E[ Y ]$ uniquely, 
and that for any $y \in [0,1]$,
$\E[ \loss(y,Y) - \loss(y^{*},Y) ]
= \loss(y,y^{*})$,
since for any $[0,1]$-valued random variable $Y$ and for $y^* = \E[Y]$,
it holds that $\E[ (Y-y)^2 ] - \E[ (Y-y^*)^2 ] = y^2 - 2 y \E[Y] + 2 y^* \E[Y] - (y^*)^2 = (y-y^*)^2$.
Thus, for $(\Y,\loss)=([0,1],\loss_{\sq})$, the above condition holds with $\bclower(x)=\bcupper(x)=x$.
As we discuss below, the results also have implications for the classification 
setting via the well-known \emph{plug-in} technique \citep*{devroye:96}.

For losses satisfying \eqref{eqn:bcfunc}, 
we will argue that the following specifications yield learning rules 
that are strongly universally consistent with independent noise.
First we specify an inductive learning rule.
Let $\{\G_{i}\}_{i=1}^{\infty}$, $\{m_{i}\}_{i=1}^{\infty}$, and $\{i_{n}\}_{i=1}^{\infty}$ 
be as in the proof of Lemma~\ref{lem:kc-subset-suil}, 
and note that without loss of generality we can suppose $m_{i}$ is strictly increasing 
(since, for the $\gamma_{i}$ values from the proof of Lemma~\ref{lem:kc-subset-suil}, 
the sequence $i_{n}^{\prime} = \min\{ i : m_{i} = m_{i_{n}} \}$ 
can be used in place of $i_{n}$ in Lemma~\ref{lem:srm} while still retaining 
the guarantee in the lemma, and while still satisfying $i_{n}^{\prime}\to\infty$), 
and that $|\G_{i}| = i$ for all $i \in \nats$ (and indeed, this is the case in the construction given in Lemma~\ref{lem:approximating-sequence}).
Then for any $n \in \nats$, $x_{1:n} \in \X^{n}$, and $y_{1:n} \in \Y^{n}$, 
for each $s \in \{m_{i_{n}},\ldots,n\}$ define the function 
$\tilde{f}_{n,s}(x_{1:n},y_{1:n},\cdot)$ as 
\begin{equation*}
\argmin_{f \in \G_{i_{n}}} \frac{1}{s} \sum_{t=1}^{s} \loss(f(x_{t}),y_{t}).
\end{equation*}
Then define the function $\hat{f}_{n}(x_{1:n},y_{1:n},\cdot)$ as
\begin{equation}
\label{eqn:suil-rule-noise}
\argmin_{f \in \G_{i_{n}}} \max_{m_{i_{n}} \leq s \leq n} \frac{1}{s} \sum_{t=1}^{s} \loss(f(x_{t}),\tilde{f}_{n,s}(x_{1:n},y_{1:n},x_{t})).
\end{equation}

We can specify a self-adaptive learning rule analogously, as follows.
Let $\{\F_{i}\}_{i=1}^{\infty}$, $\{\gamma_{i}\}_{i=1}^{\infty}$, and $\{u_{i}\}_{i=1}^{\infty}$ 
be as in \eqref{eqn:sual-index}, with $u_{i}$ strictly increasing here, 
and note that without loss of generality we can suppose $|\F_{i}| = i$.
Define $\hat{i}_{n,m}$ as in \eqref{eqn:sual-index}.  
Then for any $n,m \in \nats$ ($m \geq n$), $x_{1:m} \in \X^{m}$, and $y_{1:n} \in \Y^{n}$,
for each $s \in \{u_{\hat{i}_{n,m}(x_{1:m})},\ldots,n\}$ define the function 
$\tilde{f}_{n,m,s}(x_{1:m},y_{1:n},\cdot)$ as 
\begin{equation*}
\argmin_{f \in \F_{\hat{i}_{n,m}(x_{1:m})}} \frac{1}{s} \sum_{t=1}^{s} \loss(f(x_{t}),y_{t}).
\end{equation*}
Then define the function $\hat{f}_{n,m}(x_{1:m},y_{1:n},\cdot)$ as
\begin{equation}
\label{eqn:sual-rule-noise}
\argmin_{f \in \F_{\hat{i}_{n,m}(x_{1:m})}} \max_{u_{\hat{i}_{n,m}(x_{1:m})} \leq s \leq n} \frac{1}{s} \sum_{t=1}^{s} \loss(f(x_{t}),\tilde{f}_{n,m,s}(x_{1:m},y_{1:n},x_{t})).
\end{equation}
One can easily verify that the choices in these optimizations can be 
made in such a way that the functions $\hat{f}_{n}$ and $\hat{f}_{n,m}$ are measurable 
(e.g., by breaking ties based on a fixed predefined ordering of each $\F_{i}$ set); 
for simplicity, we will suppose ties in the $\argmin$ are broken deterministically, 
so that the learning rules are deterministic functions.

We have the following theorem, which reveals that a process $\ProcX$ 
admits strong universal (inductive/self-adaptive) learning with independent noise 
if and only if it admits strong universal (inductive/self-adaptive) learning \emph{without} noise.
Furthermore, the above learning rules witness the sufficiency of these conditions, 
which further implies that the self-adaptive learning rule 
\eqref{eqn:sual-rule-noise} is optimistically universal with independent noise.

\begin{theorem}
\label{thm:noisy-equivalence}
If $\loss$ satisfies \eqref{eqn:bcfunc}, 
then Condition~\ref{con:kc} is necessary and sufficient for 
a process $\ProcX$ to admit strong universal (inductive/self-adaptive) learning 
with independent noise. 
Moreover, the self-adaptive learning rule $\hat{f}_{n,m}$ defined by 
\eqref{eqn:sual-rule-noise} is optimistically universal with independent noise.
\end{theorem}

The restrictions to $\loss$ guaranteeing existence of $\bclower$ and $\bcupper$ 
provide an important convenience for us: 
namely, under these conditions, we can still characterize consistency 
as convergence to $\target$.
More specifically, we have the following guarantee, which will be a key component in the proof of Theorem~\ref{thm:noisy-equivalence}.

\begin{lemma}
\label{lem:good-on-data}
If $\ProcX \in \KC$, $(\ProcX,\ProcY)$ has independent noise,
$\loss$ satisfies \eqref{eqn:bcfunc}, 
and $\hat{f}_{n}$ and $\hat{f}_{n,m}$ are as in \eqref{eqn:suil-rule-noise} and \eqref{eqn:sual-rule-noise} respectively,
then 
\begin{equation*}
\limsup_{n \to \infty} \hat{\mu}_{\ProcX}( \loss(\hat{f}_{n}(X_{1:n},Y_{1:n},\cdot),\target(\cdot)) ) = 0 \text{ (a.s.)}
\end{equation*}
and 
\begin{equation*}
\limsup_{n \to \infty} \limsup_{t \to \infty} \frac{1}{t+1} \sum_{m=n}^{n+t} \loss(\hat{f}_{n,m}(X_{1:m},Y_{1:n},X_{m+1}),\target(X_{m+1})) = 0 \text{ (a.s.)}.
\end{equation*}
\end{lemma}
\begin{proof}
For brevity, we only give the detailed proof of the claim for the self-adaptive rule $\hat{f}_{n,m}$.
The proof of the claim for the inductive learning rule $\hat{f}_{n}$ follows analogously, 
merely substituting properties of the sequence $i_{n}$ established in the proof of Lemma~\ref{lem:kc-subset-suil}, 
in place of the analogous properties of $\hat{i}_{n}$ used here.  We provide 
an outline of that analogous argument following the detailed proof for 
the self-adaptive learning rule $\hat{f}_{n,m}$, which we now turn to.

To simplify notation, let $\hat{g}_{n,m}(\cdot) = \hat{f}_{n,m}(X_{1:m},Y_{1:n},\cdot)$ and $\tilde{g}_{n,m,s}(\cdot) = \tilde{f}_{n,m,s}(X_{1:m},Y_{1:n},\cdot)$.
Let $m_{n}^{*}$, $\hat{i}_{n}$, $\target_{i}$, $\alpha_{i}$, $\iota_{0}$, and the event $K$ 
all be as in the proof of Theorem~\ref{thm:doubly-universal-adaptive} (defined relative to the fixed function $\target$ from property \Yone), 
and define $\hat{g}_{n} = \hat{g}_{n,m_{n}^{*}}$.

By the conditional independence property \Ytwo, Hoeffding's inequality 
(applied under the conditional distribution given $\ProcX$) and the law 
of total probability imply that, for any $i, s \in \nats$ and $f \in \F_{i}$, with probability at least $1 - \frac{1}{i^{3} s^{2}}$, 
\begin{equation}
\label{eqn:good-on-data-hoeffding}
\left| \frac{1}{s} \sum_{t=1}^{s} \left( \loss( f(X_{t}), Y_{t} ) - \loss(\target(X_{t}),Y_{t}) \right) - \E\!\left[ \loss( f(X_{t}), Y_{t} ) - \loss(\target(X_{t}),Y_{t}) \middle| X_{t} \right] \right|
\leq 2\maxloss \sqrt{ \frac{\ln( 2 i^{3} s^{2} )}{2s} }.
\end{equation}
By the union bound, for any fixed $i \in \nats$, the inequality \eqref{eqn:good-on-data-hoeffding} 
holds simultaneously for all $f \in \F_{i}$ and $s \in \nats$ with probability at least 
$1 - \sum_{s=1}^{\infty} |\F_{i}| \frac{1}{i^{3} s^{2}} = 1 - \frac{\pi^{2}}{6 i^{2}}$ (since $|\F_{i}|=i$).
Furthermore, since $\sum_{i=1}^{\infty} \frac{\pi^{2}}{6 i^{2}} < \infty$, 
the Borel-Cantelli lemma implies that there is an event $\tilde{K}$ of probability one, 
on which $\exists \iota_{1} \in \nats$ such that
\eqref{eqn:good-on-data-hoeffding} holds simultaneously 
for every $i \geq \iota_{1}$ and every $s \in \nats$.

Now suppose the event $K \cap \tilde{K}$ occurs. Recalling that $\lim\limits_{n\to\infty} \hat{i}_{n} = \infty$ by \eqref{eqn:adaptive-hati-limit}, 
let $\tilde{\nu} \in \nats$ be such that $\forall n \geq \tilde{\nu}$ it holds that $\hat{i}_{n} \geq \max\{ \iota_{0}, \iota_{1} \}$, so that 
both \eqref{eqn:good-on-data-hoeffding} (for $i = \hat{i}_{n}$, and for all $s$) and \eqref{eqn:adaptive-targeti-bound} hold.
Then for any $n \geq \tilde{\nu}$ and any $s \in \{ u_{\hat{i}_{n}}, \ldots, n \}$, for any $m \geq m_{n}^{*}$ we have 
\begin{align*}
& \frac{1}{s} \sum_{t=1}^{s} \loss( \tilde{g}_{n,m,s}(X_{t}), Y_{t} ) - \loss(\target(X_{t}),Y_{t})
\\ & \leq \frac{1}{s} \sum_{t=1}^{s} \loss( \target_{\hat{i}_{n}}(X_{t}), Y_{t} ) - \loss(\target(X_{t}),Y_{t})
\\ & \leq 2 \maxloss \sqrt{ \frac{\ln( 2\hat{i}_{n}^{3} s^{2} )}{2s} } + \frac{1}{s} \sum_{t=1}^{s} \E\!\left[ \loss( \target_{\hat{i}_{n}}(X_{t}), Y_{t} ) - \loss(\target(X_{t}),Y_{t}) \middle| X_{t}, \hat{i}_{n} \right]
\\ & \leq 2 \maxloss \sqrt{ \frac{\ln( 2\hat{i}_{n}^{3} s^{2} )}{2s} } + \frac{1}{s} \sum_{t=1}^{s} \bcupper\!\left( \loss( \target_{\hat{i}_{n}}(X_{t}), \target(X_{t}) ) \right)
\\ & \leq 2 \maxloss \sqrt{ \frac{\ln( 2\hat{i}_{n}^{3} s^{2} )}{2s} } + \bcupper\!\left( \frac{1}{s} \sum_{t=1}^{s} \loss( \target_{\hat{i}_{n}}(X_{t}), \target(X_{t}) ) \right)
\\ & \leq 2 \maxloss \sqrt{ \frac{\ln( 2\hat{i}_{n}^{3} s^{2} )}{2s} } + \bcupper( \alpha_{\hat{i}_{n}} )
\leq 2 \maxloss \sqrt{ \frac{\ln( 2 \hat{i}_{n}^{5} )}{2\hat{i}_{n}} } + \bcupper( \alpha_{\hat{i}_{n}} ),
\end{align*}
where the last four inequalities are due to 
\eqref{eqn:bcfunc}, Jensen's inequality (due to concavity of $\bcupper$), 
\eqref{eqn:adaptive-targeti-bound} and monotonicity of $\bcupper$, 
and the fact that $u_{i}$ is strictly increasing (so that $s \geq u_{\hat{i}_{n}} \geq \hat{i}_{n}$).

Now denote by $\{Y_{t}^{\prime}\}_{t=1}^{\infty}$ a sequence with the same conditional distribution given $\ProcX$ as $\ProcY$, but conditionally independent of $\ProcY$ given $\ProcX$.
Again by \eqref{eqn:good-on-data-hoeffding}, if $n \geq \tilde{\nu}$ and $m \geq m_{n}^{*}$, 
every $s \in \{ u_{\hat{i}_{n}}, \ldots, n \}$ has 
\begin{align*}
& \frac{1}{s} \sum_{t=1}^{s} \loss(\tilde{g}_{n,m,s}(X_{t}),Y_{t}) - \loss(\target(X_{t}),Y_{t})
\\ & \geq - 2\maxloss \sqrt{ \frac{\ln( 2 \hat{i}_{n}^{3} s^{2} )}{2s} } + \frac{1}{s} \sum_{t=1}^{s} \E\!\left[ \loss(\tilde{g}_{n,m,s}(X_{t}),Y_{t}^{\prime}) - \loss(\target(X_{t}),Y_{t}^{\prime}) \middle| X_{t}, \tilde{g}_{n,m,s} \right]
\\ & \geq - 2\maxloss \sqrt{ \frac{\ln( 2 \hat{i}_{n}^{5})}{2\hat{i}_{n}} } + \frac{1}{s} \sum_{t=1}^{s} \bclower\!\left( \loss(\tilde{g}_{n,m,s}(X_{t}),\target(X_{t})) \right)
\\ & \geq - 2\maxloss \sqrt{ \frac{\ln( 2 \hat{i}_{n}^{5})}{2\hat{i}_{n}} } + \bclower\!\left( \frac{1}{s} \sum_{t=1}^{s} \loss(\tilde{g}_{n,m,s}(X_{t}),\target(X_{t})) \right),
\end{align*}
where the last two inequalities use that $u_{i}$ is strictly increasing (so that $s \geq u_{\hat{i}_{n}} \geq \hat{i}_{n}$) 
and \eqref{eqn:bcfunc} along with Jensen's inequality (due to convexity of $\bclower$).

Define  
\begin{equation*}
\beta_{i} = \bclower^{-1}\!\left( 4 \maxloss \sqrt{ \frac{\ln( 2 i^{5} )}{2i} } + \bcupper( \alpha_{i} ) \right),
\end{equation*}
noting that the inverse function $\bclower^{-1}$ is well-defined due to $\bclower$ being continuous and strictly increasing.
Note that $\lim\limits_{n \to \infty} \beta_{\hat{i}_{n}} = 0$ by the facts that 
$\bclower$ and $\bcupper$ are continuous with $\bclower(0) = \bcupper(0) = 0$, 
together with the facts that $\alpha_{i} \to 0$ and $\hat{i}_{n} \to \infty$ (as established in the proof of Theorem~\ref{thm:doubly-universal-adaptive}).
Altogether, we have established that, on $K \cap \tilde{K}$, every $n \geq \tilde{\nu}$, $m \geq m_{n}^{*}$, and $s \in \{ u_{\hat{i}_{n}}, \ldots, n \}$ satisfy 
\begin{equation}
\label{eqn:noise-beta-bound}
\frac{1}{s} \sum_{t=1}^{s} \loss(\tilde{g}_{n,m,s}(X_{t}),\target(X_{t})) \leq \beta_{\hat{i}_{n}}.
\end{equation}

Now to relate this performance guarantee for these $\tilde{g}_{n,m,s}$ functions on the first $n$ data points 
to performance of $\hat{g}_{n,m}$ on the full sequence, note that (just as in the proof of Theorem~\ref{thm:doubly-universal-adaptive}) 
since $\hat{g}_{n,m} = \hat{g}_{n}$ for all $m \geq m_{n}^{*}$, we have 
\begin{align*}
& \limsup_{s \to \infty} \frac{1}{s+1} \sum_{m=n}^{n+s} \loss(\hat{g}_{n,m}(X_{m+1}),\target(X_{m+1}))
\\ & = \limsup_{s \to \infty} \frac{1}{s} \sum_{t=1}^{s} \loss(\hat{g}_{n}(X_{t}),\target(X_{t}))
\leq \sup_{u_{\hat{i}_{n}} \leq s < \infty} \frac{1}{s} \sum_{t=1}^{s} \loss(\hat{g}_{n}(X_{t}),\target(X_{t})).
\end{align*}
Again supposing $K \cap \tilde{K}$ holds and that $n \geq \tilde{\nu}$,
by the relaxed triangle inequality we have
\begin{align*}
& \sup_{u_{\hat{i}_{n}} \leq s < \infty} \frac{1}{s} \sum_{t=1}^{s} \loss(\hat{g}_{n}(X_{t}),\target(X_{t})) 
\\ & \leq \triconst \left( \sup_{u_{\hat{i}_{n}} \leq s < \infty} \frac{1}{s} \sum_{t=1}^{s} \loss(\hat{g}_{n}(X_{t}),\target_{\hat{i}_{n}}(X_{t})) \right) + \triconst \left( \sup_{u_{\hat{i}_{n}} \leq s < \infty} \frac{1}{s} \sum_{t=1}^{s} \loss(\target_{\hat{i}_{n}}(X_{t}),\target(X_{t})) \right)
\\ & \leq \triconst \left( \max_{u_{\hat{i}_{n}} \leq s \leq n} \frac{1}{s} \sum_{t=1}^{s} \loss(\hat{g}_{n}(X_{t}),\target_{\hat{i}_{n}}(X_{t})) \right) + \triconst( \gamma_{\hat{i}_{n}} + \alpha_{\hat{i}_{n}} ),
\end{align*}
where the inequality in the last line is due to \eqref{eqn:adaptive-hati-gamma-bound} and \eqref{eqn:adaptive-targeti-bound}.
By the relaxed triangle inequality again, we have 
\begin{align*}
& \max_{u_{\hat{i}_{n}} \leq s \leq n} \frac{1}{s} \sum_{t=1}^{s} \loss(\hat{g}_{n}(X_{t}),\target_{\hat{i}_{n}}(X_{t}))
\\ & \leq \triconst \left( \max_{u_{\hat{i}_{n}} \leq s \leq n} \frac{1}{s} \sum_{t=1}^{s} \loss(\target_{\hat{i}_{n}}(X_{t}),\tilde{g}_{n,m,s}(X_{t})) \right)
+ \triconst \left( \max_{u_{\hat{i}_{n}} \leq s \leq n} \frac{1}{s} \sum_{t=1}^{s} \loss(\hat{g}_{n}(X_{t}),\tilde{g}_{n,m,s}(X_{t})) \right),
\end{align*}
and since $\hat{i}_{n,m}(X_{1:m}) = \hat{i}_{n}$ and $\hat{g}_{n,m} = \hat{g}_{n}$ for $m \geq m_{n}^{*}$, 
the definition of $\hat{f}_{n,m}$ from \eqref{eqn:sual-rule-noise} implies that in this case this last line is at most
\begin{equation*}
2 \triconst \left( \max_{u_{\hat{i}_{n}} \leq s \leq n} \frac{1}{s} \sum_{t=1}^{s} \loss(\target_{\hat{i}_{n}}(X_{t}),\tilde{g}_{n,m,s}(X_{t})) \right).
\end{equation*}
Furthermore, the relaxed triangle inequality implies 
\begin{align*}
& \max_{u_{\hat{i}_{n}} \leq s \leq n} \frac{1}{s} \sum_{t=1}^{s} \loss(\target_{\hat{i}_{n}}(X_{t}),\tilde{g}_{n,m,s}(X_{t}))
\\ & \leq \triconst \left( \max_{u_{\hat{i}_{n}} \leq s \leq n} \frac{1}{s} \sum_{t=1}^{s} \loss(\target_{\hat{i}_{n}}(X_{t}),\target(X_{t})) \right)
+ \triconst \left( \max_{u_{\hat{i}_{n}} \leq s \leq n} \frac{1}{s} \sum_{t=1}^{s} \loss(\tilde{g}_{n,m,s}(X_{t}),\target(X_{t})) \right)
\\ & \leq \triconst ( \alpha_{\hat{i}_{n}} + \beta_{\hat{i}_{n}} ),
\end{align*}
where the last inequality is due to \eqref{eqn:adaptive-targeti-bound} and \eqref{eqn:noise-beta-bound}.
Altogether, on $K \cap \tilde{K}$, any $n \geq \tilde{\nu}$ has 
\begin{equation*}
\limsup_{s \to \infty} \frac{1}{s+1} \sum_{m=n}^{n+s} \loss(\hat{g}_{n,m}(X_{m+1}),\target(X_{m+1}))
\leq 2 \triconst^{3} ( \alpha_{\hat{i}_{n}} + \beta_{\hat{i}_{n}} ) + \triconst ( \gamma_{\hat{i}_{n}} + \alpha_{\hat{i}_{n}} ).
\end{equation*}
Thus, since $\alpha_{\hat{i}_{n}} \to 0$, $\gamma_{\hat{i}_{n}} \to 0$, and $\beta_{\hat{i}_{n}} \to 0$ (as established above), 
and the event $K \cap \tilde{K}$ holds with probability one (by the union bound), 
the claim for the self-adaptive learning rule $\hat{f}_{n,m}$ in the statement of the lemma follows.

The claim for the inductive learning rule $\hat{f}_{n}$ follows by a very similar argument, 
except replacing $\hat{i}_{n}$ above with the quantity $i_{n}$ from the proof of Lemma~\ref{lem:kc-subset-suil}.
For brevity, we only give an outline here to illustrate the key steps, leaving the details as an exercise for the interested reader.
For this argument, we take the definitions of $\target_{i}$, $\alpha_{i}$, and $\gamma_{i}$ as in the proof of Lemma~\ref{lem:kc-subset-suil}.
Note that an event identical to $\tilde{K}$ now holds for the sets $\G_{i}$.
Defining $\hat{h}_{n}(\cdot) = \hat{f}_{n}(X_{1:n},Y_{1:n},\cdot)$ and $\tilde{h}_{n,s}(\cdot) = \tilde{f}_{n,s}(X_{1:n},Y_{1:n},\cdot)$,
and following the same reasoning as above (using the analogous results from the proof of Lemma~\ref{lem:kc-subset-suil}) 
we have that with probability one, for every sufficiently large $n$, every $s \in \{m_{i_{n}},\ldots,n\}$ satisfies 
$\frac{1}{s} \sum\limits_{t=1}^{s} \loss(\tilde{h}_{n,s}(X_{t}),\target(X_{t})) \leq \beta_{i_{n}}$.
Continuing to follow the same arguments as above, but now using \eqref{eqn:inductive-n-sqrt-gamma-bound}, 
we then have with probability one, for all sufficiently large $n$, 
\begin{align*}
\hat{\mu}_{\ProcX}\!\left( \loss\!\left(\hat{h}_{n}(\cdot),\target(\cdot)\right) \right)
& \leq \triconst \max_{m_{i_{n}} \leq s \leq n} \frac{1}{s} \sum_{t=1}^{s} \loss(\hat{h}_{n}(X_{t}),\target_{i_{n}}(X_{t})) + \triconst \left( \sqrt{\gamma_{i_{n}}} + \alpha_{i_{n}} \right)
\\ & \leq 2 \triconst^{2} \max_{m_{i_{n}} \leq s \leq n} \frac{1}{s} \sum_{t=1}^{s} \loss(\target_{i_{n}}(X_{t}),\tilde{h}_{n,s}(X_{t})) + \triconst \left( \sqrt{\gamma_{i_{n}}} + \alpha_{i_{n}} \right)
\\ & \leq 2 \triconst^{3} ( \alpha_{i_{n}} + \beta_{i_{n}} ) + \triconst \left( \sqrt{\gamma_{i_{n}}} + \alpha_{i_{n}} \right),
\end{align*}
which converges to $0$ as $n \to \infty$.
\end{proof}

To complete the proof of the claims for $\hat{f}_{n}$ and $\hat{f}_{n,m}$ in Theorem~\ref{thm:noisy-equivalence}, 
we will compose the above result with the following general lemma.

\begin{lemma}
\label{lem:reduction-to-realizable}
Fix any process $\ProcX$.
If $\loss$ satisfies \eqref{eqn:bcfunc}, then 
for any deterministic self-adaptive learning rule $f_{n,m}$, 
if, for every $\ProcY$ such that $(\ProcX,\ProcY)$ has independent noise,
\begin{equation*}
\limsup_{n \to \infty} \limsup_{t \to \infty} \frac{1}{t+1} \sum_{m=n}^{n+t} \loss(f_{n,m}(X_{1:m},Y_{1:n},X_{m+1}),\target(X_{m+1})) = 0 \text{ (a.s.)},
\end{equation*}
then $f_{n,m}$ is strongly universally consistent with independent noise under $\ProcX$.
\end{lemma}
\begin{proof}
For any $n,m \in \nats$ with $m \geq n$, 
define $g_{n,m}(x) = f_{n,m}(X_{1:m},Y_{1:n},x)$ 
and $\Delta_{m+1}^{n} = \loss(g_{n,m}(X_{m+1}),Y_{m+1}) - \loss(\target(X_{m+1}),Y_{m+1})$.
By the assumed properties of the loss $\loss$ and the conditional independence property \Ytwo, 
every $m \geq n$ has 
\begin{equation*}
\E\!\left[ \Delta_{m+1}^{n} \middle| X_{m+1}, g_{n,m}(X_{m+1}) \right]
\leq \bcupper\!\left(\loss(g_{n,m}(X_{m+1}),\target(X_{m+1}))\right).
\end{equation*}
Then note that, due to the conditional independence property \Ytwo, for any $n \in \nats$, 
\begin{equation*}
\left\{ \Big( \Delta_{m+1}^{n} - \E\!\left[ \Delta_{m+1}^{n} \middle| X_{m+1}, g_{n,m}(X_{m+1}) \right]\Big) \right\}_{m=n}^{\infty}
\end{equation*}
is a martingale difference sequence with respect to $\{(X_{1:(m+2)},Y_{1:(m+1)})\}_{m=n}^{\infty}$.
Therefore, Azuma's inequality \citep*[e.g.,][Theorem 9.1]{devroye:96} implies that, for any $t \in \nats \cup \{0\}$, 
with probability at least $1-\frac{1}{n^{2} (t+1)^{2}}$, 
\begin{equation*}
\frac{1}{t+1} \left| \sum_{m=n}^{n+t} \Delta_{m+1}^{n}  - \E\!\left[ \Delta_{m+1}^{n} \middle| X_{m+1}, g_{n,m}(X_{m+1}) \right] \right|
\leq \maxloss \sqrt{ \frac{2\ln(2 n^{2} (t+1)^{2} )}{(t+1)} }.
\end{equation*}
Since $\sum\limits_{n = 1}^{\infty} \sum\limits_{t = 0}^{\infty} \frac{1}{n^{2} (t+1)^{2}} < \infty$, 
the Borel-Cantelli lemma implies that, on an event $E_{1}$ of probability one, 
\begin{equation*}
\limsup_{n\to\infty} \limsup_{t \to \infty} \frac{1}{t+1} \left| \sum_{m=n}^{n+t} \Delta_{m+1}^{n} - \E\!\left[ \Delta_{m+1}^{n} \middle| X_{m+1}, g_{n,m}(X_{m+1}) \right] \right|
= 0,
\end{equation*}
so that 
\begin{equation*}
\limsup_{n \to \infty} \limsup_{t\to\infty} \frac{1}{t+1} \sum_{m=n}^{n+t} \Delta_{m+1}^{n} 
\leq \limsup_{n \to \infty} \limsup_{t \to \infty} \frac{1}{t+1} \sum_{m=n}^{n+t} \bcupper\!\left(\loss(g_{n,m}(X_{m+1}),\target(X_{m+1}))\right).
\end{equation*}
By Jensen's inequality, the right hand side above is at most
\begin{equation*}
\limsup_{n \to \infty} \limsup_{t \to \infty} \bcupper\!\left( \frac{1}{t+1} \sum_{m=n}^{n+t} \loss(g_{n,m}(X_{m+1}),\target(X_{m+1}))\right),
\end{equation*}
and since $\bcupper$ is continuous and strictly increasing, and its argument is bounded, this expression equals 
\begin{equation*}
\bcupper\!\left( \limsup_{n \to \infty} \limsup_{t \to \infty} \frac{1}{t+1} \sum_{m=n}^{n+t} \loss(g_{n,m}(X_{m+1}),\target(X_{m+1})) \right).
\end{equation*}
By the assumed property of $f_{n,m}$ in the statement of the lemma, 
and the fact that $\bcupper(0) = 0$, 
this last expression equals $0$ on an event $E_{2}$ of probability one.
Thus, on the event $E_{1} \cap E_{2}$, 
we have 
$\limsup\limits_{n \to \infty} \hat{\L}_{\ProcX}(f_{n,\cdot},\ProcY;n,\target) \leq 0$.
Also recall that Lemma~\ref{lem:non-negative-noisy-excess} implies that, 
on an event $E_{3}$ of probability one, 
$\liminf\limits_{n \to \infty} \hat{\L}_{\ProcX}(f_{n,\cdot},\ProcY;n,\target) \geq 0$.
Altogether, $\hat{\L}_{\ProcX}(f_{n,\cdot},\ProcY;n,\target) \to 0$ on the event $E_{1} \cap E_{2} \cap E_{3}$, 
which has probability one by the union bound.
\end{proof}

We may also note that, since any inductive learning rule $f_{n}$
can be interpreted as a self-adaptive learning rule that simply ignores the additional data $X_{(n+1):m}$, 
Lemma~\ref{lem:reduction-to-realizable} has the further implication that, 
if $\loss$ satisfies \eqref{eqn:bcfunc}, then 
for any deterministic inductive learning rule $f_{n}$, 
if, for every $\ProcY$ such that $(\ProcX,\ProcY)$ has independent noise,
\begin{equation*}
\limsup_{n \to \infty} \hat{\mu}_{\ProcX}( \loss( f_{n}(X_{1:n},Y_{1:n},\cdot), \target(\cdot) ) ) = 0 \text{ (a.s.)},
\end{equation*}
then $f_{n}$ is strongly universally consistent with independent noise under $\ProcX$.

With the above two lemmas in hand, we are ready for the proof of Theorem~\ref{thm:noisy-equivalence}.

\begin{proof}[of Theorem~\ref{thm:noisy-equivalence}]
It follows immediately from Theorem~\ref{thm:main}
that Condition~\ref{con:kc} is a necessary condition for $\ProcX$ to admit 
strong universal learning (either inductive or self-adaptive) with independent noise,
since the noise-free case is a special case of the stated conditions.
Furthermore, Lemmas~\ref{lem:good-on-data} and \ref{lem:reduction-to-realizable}
together imply that Condition~\ref{con:kc} is sufficient for the rules $\hat{f}_{n}$ and $\hat{f}_{n,m}$ 
to be strongly universally consistent with independent noise, 
and therefore also sufficient for $\ProcX$ to admit strong universal (inductive/self-adaptive) 
learning with independent noise.
Finally, note that
the self-adaptive learning rule 
$\hat{f}_{n,m}$ has no direct dependence on the distribution of $\ProcX$, 
aside from the data supplied as its arguments, and yet (as just established) is 
strongly universally consistent with independent noise under every $\ProcX$ satisfying Condition~\ref{con:kc}.
Together with the fact (also just established) that Condition~\ref{con:kc} is necessary and sufficient for $\ProcX$ to 
admit strong universal self-adaptive learning with independent noise, this also establishes the claim 
that $\hat{f}_{n,m}$ is optimistically universal with independent noise.
\end{proof}

In particular, this implies \eqref{eqn:sual-rule-noise} is optimistically universal with independent noise 
for the special case of \emph{regression}, where $\Y$ is any bounded interval of $\reals$ 
and $\loss$ is the squared loss: $\loss_{\sq}(y,y') = (y-y')^{2}$.  
However, since not every $(\Y,\loss)$ satisfies \eqref{eqn:bcfunc}, the questions of 
whether Condition~\ref{con:kc} is sufficient for universal learning 
with independent noise, and whether there exist self-adaptive learning rules that are 
optimistically universal with independent noise, for general (bounded, separable) 
losses $\loss$, remain open.

\begin{problem}
\label{prob:kc-independent-noise}
Is Condition~\ref{con:kc} sufficient for a process $\ProcX$ to 
admit strong universal inductive and self-adaptive learning 
with independent noise, 
for \emph{every} separable near-metric space $(\Y,\loss)$ with $\maxloss < \infty$?
\end{problem}

\begin{problem}
\label{prob:optimistic-independent-noise}
Is it true that, for \emph{every} separable near-metric space $(\Y,\loss)$ with $\maxloss < \infty$, 
there exists a self-adaptive learning rule that is 
optimistically universal with independent noise?
\end{problem}

\subsection{Learning Noisy Functions}
\label{sec:noisy-classification}

While the above results for learning with independent noise are quite general, 
it turns out the important problem of \emph{classification} with the $0$-$1$ loss is not directly covered by 
these results, since it does not guarantee the existence of functions $\bclower$, $\bcupper$ satisfying \eqref{eqn:bcfunc}.
Fortunately, we can extend the above theory to a result on classification for \emph{noisy functions} 
via the well-known \emph{plug-in} classifier technique \citep*[see e.g.,][Theorem 2.2]{devroye:96}.

Specifically, let $\hat{f}_{n}^{\sqsmall}$ and $\hat{f}_{n,m}^{\sqsmall}$ be the inductive and self-adaptive learning rules 
$\hat{f}_{n}$ and $\hat{f}_{n,m}$ 
from \eqref{eqn:suil-rule-noise} and \eqref{eqn:sual-rule-noise}, 
respectively, 
but defined for the setting $(\Y,\loss) = ([0,1],\loss_{\sq})$ (where $\loss_{\sq}(a,b) = (a-b)^2$).\footnote{In fact, 
it is not hard to show that, in the arguments below, it suffices to take $\hat{f}_{n}^{\sqsmall}$ and $\hat{f}_{n,m}^{\sqsmall}$ 
as \emph{any} learning rules that are strongly universally consistent for noisy functions with respect to $([0,1],\loss_{\sq})$ 
under the given $\ProcX \in \KC$, 
and the resulting ``plug-in'' learning rules $\hat{h}_{n}$ and $\hat{h}_{n,m}$ will then be strongly universally consistent 
for noisy functions under $\ProcX$ with respect to $(\Y,\loss_{\zo})$ for finite $\Y$.  As this more general 
reduction requires a few extra steps in the proof, we present the result specialized to \eqref{eqn:suil-rule-noise} and \eqref{eqn:sual-rule-noise} 
for simplicity.}
In the finite classification problem, 
we consider a setting with $\Y$ a finite set with $|\Y| \geq 2$, and $\loss = \loss_{\zo}$ (where $\loss_{\zo}(a,b) = \ind[a \neq b]$).
For any $n \in \nats$, $x_{1:n} \in \X^{n}$, $y_{1:n} \in \Y^{n}$, and $x \in \X$, 
define an inductive learning rule 
\begin{equation}
\label{eqn:suil-rule-noisy-classification}
\hat{h}_{n}(x_{1:n},y_{1:n},x) = \argmax_{y \in \Y} \hat{f}_{n}^{\sqsmall}(x_{1:n},\ind_{\{y\}}(y_{1:n}),x).
\end{equation}
Similarly, for any $n,m \in \nats$ ($m \geq n$), any $x_{1:m} \in \X^{m}$, $y_{1:n} \in \Y^{n}$, and $x \in \X$, 
define a self-adaptive learning rule 
\begin{equation}
\label{eqn:sual-rule-noisy-classification}
\hat{h}_{n,m}(x_{1:m},y_{1:n},x) = \argmax_{y \in \Y} \hat{f}_{n,m}^{\sqsmall}(x_{1:m},\ind_{\{y\}}(y_{1:n}),x).
\end{equation}
Since $\hat{f}_{n}^{\sqsmall}$ and $\hat{f}_{n,m}^{\sqsmall}$ are measurable functions, 
the functions $\hat{h}_{n}$ and $\hat{h}_{n,m}$ can also be defined as measurable functions 
(e.g., by breaking ties in the $\argmax$ based on a pre-specified preference order on the finite set $\Y$); 
for simplicity, let us suppose ties in the $\argmax$ are broken deterministically, 
so that $\hat{h}_{n}$ and $\hat{h}_{n,m}$ are deterministic functions.

Note that when $\ProcY$ is a noisy function of $\ProcX$, for every $y \in \Y$ the conditional probability 
$\P( Y_{t} = y | X_{t} )$ is a $t$-invariant function of $X_{t}$: that is, there is a function $\eta(\cdot;y)$ such that 
$\eta(X_{t};y) = \P( Y_{t} = y | X_{t} )$ for every $t$. 
Moreover, the value $\eta(X_{t};y)$ 
minimizes $\E\!\left[ (\eta(X_{t};y) - \ind_{\{y\}}(Y_{t}))^{2} \middle| X_{t} \right]$ almost surely.
Thus, for $\ProcX \in \KC$,
Lemma~\ref{lem:good-on-data} implies that for each $y \in \Y$,  
the estimators $\hat{f}_{n}^{\sqsmall}(X_{1:n},\ind_{\{y\}}(Y_{1:n}),X_{m+1})$ and $\hat{f}_{n,m}^{\sqsmall}(X_{1:m},\ind_{\{y\}}(Y_{1:n}),X_{m+1})$
will (on average, in the limit) be very close to $\P( Y_{m+1} = y | X_{m+1} )$ for all large $n$.
By this fact, we see that the learning rules $\hat{h}_{n}$ and $\hat{h}_{n,m}$ will 
(on average, in the limit) predict with a value $y$ that nearly maximizes $\P( Y_{m+1} = y | X_{m+1} )$, 
and therefore achieves a nearly-minimal $0$-$1$ loss for all large $n$.
The following theorem formalizes these claims, and summarizes our results on learning for noisy functions.

\begin{theorem}
\label{thm:noisy-classification}
For finite $\Y$ with $|\Y| \geq 2$, and $\loss = \loss_{\zo}$, 
Condition~\ref{con:kc} is necessary and sufficient for a process $\ProcX$ to admit strong universal (inductive/self-adaptive) 
learning for noisy functions.  Moreover, the self-adaptive learning rule $\hat{h}_{n,m}$ defined by \eqref{eqn:sual-rule-noisy-classification} 
is optimistically universal for noisy functions.
\end{theorem}
\begin{proof}
The fact that Condition~\ref{con:kc} is necessary for $\ProcX$ to admit strong universal 
inductive or self-adaptive learning for noisy functions 
is immediate from Theorem~\ref{thm:main} since the noise-free case is a special case of a noisy function: that is, 
defining $Y_{t} = \target(X_{t})$ for a $t$-invariant measurable function $\target$ 
always satisfies the property of $\ProcY$ being a noisy function of $\ProcX$.

For the sufficiency claim, for brevity, we only present the details for the self-adaptive learning rule $\hat{h}_{n,m}$.
The proof for the inductive learning rule $\hat{h}_{n}$ is essentially identical, except replacing 
every occurence of $\hat{h}_{n,m}(X_{1:m},Y_{1:n},X_{m+1})$ with $\hat{h}_{n}(X_{1:n},Y_{1:n},X_{m+1})$ 
and every occurence of $\hat{f}_{n,m}^{\sqsmall}(X_{1:m},\ind_{\{y\}}(Y_{1:n}),X_{m+1})$ with $\hat{f}_{n}^{\sqsmall}(X_{1:n},\ind_{\{y\}}(Y_{1:n}),X_{m+1})$.

We now proceed with the proof for the self-adaptive rule $\hat{h}_{n,m}$.
Suppose $\ProcX$ satisfies Condition~\ref{con:kc}, and that $\ProcY$ is a noisy function of $\ProcX$.
As mentioned above, since $\ProcY$ is a noisy function of $\ProcX$, 
for every $y \in \Y$ the conditional probability $\P( Y_{t} = y | X_{t} )$ is a $t$-invariant function of $X_{t}$. 
Also, as is well known, for any $p \in [0,1]$, it holds that  
$\E\!\left[ (p-\ind_{\{y\}}(Y_{t}))^{2} \middle| X_{t} \right] = \E\!\left[ (p - \P( Y_{t} = y | X_{t} ))^{2} \middle| X_{t} \right] + \E\!\left[ (\P( Y_{t} = y | X_{t} ) - \ind_{\{y\}}(Y_{t}))^{2} \middle| X_{t} \right]$,
which is minimized at $p = \P(Y_{t} = y | X_{t})$ almost surely.
Therefore, the process $\{(X_{t},\ind_{\{y\}}(Y_{t}))\}_{t=1}^{\infty}$ satisfies property \Yone\ for the squared loss $\loss_{\sq}$ on $[0,1]$, 
with the function $x \mapsto \eta(x;y) := \P( Y_{t} = y | X_{t} = x )$ being the function realizing the minimum value of $\E[ \loss_{\sq}(\eta(X_{t};y),\ind_{\{y\}}(Y_{t})) | X_{t} ]$ (a.s.).
Furthermore, since $(\ProcX,\ProcY)$ has independent noise, it follows immediately that, $\forall t \in \nats$, 
the variable $\ind_{\{y\}}(Y_{t})$ is conditionally independent of $\{(X_{t^{\prime}},\ind_{\{y\}}(Y_{t^{\prime}}))\}_{t^{\prime} \neq t}$ given $X_{t}$: 
that is, the process $\{(X_{t},\ind_{\{y\}}(Y_{t}))\}_{t=1}^{\infty}$ also satisfies property \Ytwo.
Therefore, the process $\{(X_{t},\ind_{\{y\}}(Y_{t}))\}_{t=1}^{\infty}$ has independent noise (under the loss $\loss_{\sq}$ on $[0,1]$).
Furthermore, as discussed above, the loss $\loss_{\sq}$ on $[0,1]$ satisfies \eqref{eqn:bcfunc} with $\bclower(x) = \bcupper(x) = x$.
Therefore, Lemma~\ref{lem:good-on-data} and the union bound imply that, on an event $E$ of probability one, we have 
\begin{equation*}
\max_{y \in \Y} \limsup_{n \to \infty} \limsup_{t \to \infty} \frac{1}{t+1} \sum_{m=n}^{n+t} \left( \hat{f}_{n,m}^{\sqsmall}(X_{1:m},\ind_{\{y\}}(Y_{1:n}),X_{m+1}) - \eta(X_{m+1};y) \right)^{2} = 0.
\end{equation*}
Furthermore, since the maximum of a finite number of values is continuous and nondecreasing in those values, 
this implies that on the event $E$, 
\begin{equation*}
\limsup_{n \to \infty} \limsup_{t \to \infty} \max_{y \in \Y} \frac{1}{t+1} \sum_{m=n}^{n+t} \left( \hat{f}_{n,m}^{\sqsmall}(X_{1:m},\ind_{\{y\}}(Y_{1:n}),X_{m+1}) - \eta(X_{m+1};y) \right)^{2} = 0,
\end{equation*}
and again because $\Y$ is a finite set, this implies that 
\begin{align*}
& \limsup_{n \to \infty} \limsup_{t \to \infty} \frac{1}{t+1} \sum_{m=n}^{n+t} \max_{y \in \Y}  \left( \hat{f}_{n,m}^{\sqsmall}(X_{1:m},\ind_{\{y\}}(Y_{1:n}),X_{m+1}) - \eta(X_{m+1};y) \right)^{2}
\\ & \leq \limsup_{n \to \infty} \limsup_{t \to \infty} \frac{1}{t+1} \sum_{m=n}^{n+t} \sum_{y \in \Y}  \left( \hat{f}_{n,m}^{\sqsmall}(X_{1:m},\ind_{\{y\}}(Y_{1:n}),X_{m+1}) - \eta(X_{m+1};y) \right)^{2}
\\ & \leq \limsup_{n \to \infty} \limsup_{t \to \infty} |\Y| \max_{y \in \Y} \frac{1}{t+1} \sum_{m=n}^{n+t} \left( \hat{f}_{n,m}^{\sqsmall}(X_{1:m},\ind_{\{y\}}(Y_{1:n}),X_{m+1}) - \eta(X_{m+1};y) \right)^{2} = 0.
\end{align*}
By Jensen's inequality, this further implies that on the event $E$, 
\begin{equation*}
\limsup_{n \to \infty} \limsup_{t \to \infty} \left( \frac{1}{t+1} \sum_{m=n}^{n+t} \max_{y \in \Y}  \left| \hat{f}_{n,m}^{\sqsmall}(X_{1:m},\ind_{\{y\}}(Y_{1:n}),X_{m+1}) - \eta(X_{m+1};y) \right| \right)^{2} = 0,
\end{equation*}
which (since $x \mapsto x^2$ is continuous and nondecreasing on $[0,1]$) implies 
\begin{equation}
\label{eqn:noisy-classification-good-y}
\limsup_{n \to \infty} \limsup_{t \to \infty} \frac{1}{t+1} \sum_{m=n}^{n+t} \max_{y \in \Y}  \left| \hat{f}_{n,m}^{\sqsmall}(X_{1:m},\ind_{\{y\}}(Y_{1:n}),X_{m+1}) - \eta(X_{m+1};y) \right| = 0.
\end{equation}

For each $n,m \in \nats$ with $m \geq n$, define 
$\hat{Y}_{n,m+1} = \hat{h}_{n,m}(X_{1:m},Y_{1:n},X_{m+1})$ and 
$\Delta_{m+1}^{n} = \ind\!\left[ \hat{Y}_{n,m+1} \neq Y_{m+1} \right] - \ind\!\left[ \target(X_{m+1}) \neq Y_{m+1} \right]$.
Then note that, due to the conditional independence property \Ytwo, for each $n \in \nats$, the sequence 
\begin{equation*}
\left\{ \Delta_{m+1}^{n} - \E\!\left[ \Delta_{m+1}^{n} \middle| X_{m+1}, \hat{Y}_{n,m+1} \right] \right\}_{m=n}^{\infty}
\end{equation*} 
is a martingale difference sequence with respect to $\{ (X_{1:(m+2)},Y_{1:(m+1)}) \}_{m=n}^{\infty}$.
Therefore, Azuma's inequality \citep*[e.g.,][Theorem 9.1]{devroye:96} implies that, for any $t \in \nats \cup \{0\}$, 
with probability at least $1 - \frac{1}{n^2 (t+1)^{2}}$,  
\begin{equation*}
\frac{1}{t+1} \left| \sum_{m=n}^{n+t} \Delta_{m+1}^{n} - \E\!\left[ \Delta_{m+1}^{n} \middle| X_{m+1}, \hat{Y}_{n,m+1} \right] \right| \leq \sqrt{\frac{2 \ln( 2 n^{2} (t+1)^{2} )}{t+1} }.
\end{equation*}
Since $\sum\limits_{n = 1}^{\infty} \sum\limits_{t = 0}^{\infty} \frac{1}{n^{2} (t+1)^{2}} < \infty$, 
the Borel-Cantelli lemma implies that, on an event $E^{\prime}$ of probability one, 
\begin{equation*}
\limsup_{n \to \infty} \limsup_{t \to \infty} \frac{1}{t+1} \left| \sum_{m=n}^{n+t} \Delta_{m+1}^{n} - \E\!\left[ \Delta_{m+1}^{n} \middle| X_{m+1}, \hat{Y}_{n,m+1} \right] \right| = 0,
\end{equation*}
which implies 
\begin{equation*}
\limsup_{n \to \infty} \hat{\L}_{\ProcX}(\hat{h}_{n,\cdot},\ProcY;n,\target) 
= \limsup_{n \to \infty} \limsup_{t \to \infty} \frac{1}{t+1} \sum_{m=n}^{n+t} \E\!\left[ \Delta_{m+1}^{n} \middle| X_{m+1}, \hat{Y}_{n,m+1} \right].
\end{equation*}

Next, note that for any $n,m \in \nats$ with $m \geq n$, 
by the conditional independence property \Ytwo, 
it holds that
\begin{align*}
& \E\!\left[ \Delta_{m+1}^{n} \middle| X_{m+1}, \hat{Y}_{n,m+1} \right] 
\\ & = \P\!\left( Y_{m+1} \neq \hat{Y}_{n,m+1} \middle| X_{m+1}, \hat{Y}_{n,m+1} \right) - \P\!\left( Y_{m+1} \neq \target(X_{m+1}) \middle| X_{m+1} \right)
\\ & = \P\!\left( Y_{m+1} = \target(X_{m+1}) \middle| X_{m+1} \right) - \P\!\left( Y_{m+1} = \hat{Y}_{n,m+1} \middle| X_{m+1}, \hat{Y}_{n,m+1} \right).
\end{align*}
Recalling that $\eta(X_{m+1};y) = \P( Y_{m+1} = y | X_{m+1} )$, 
the conditional independence property \Ytwo\ implies the last expression above equals
\begin{align*}
& \eta(X_{m+1};\target(X_{m+1})) - \eta(X_{m+1};\hat{Y}_{n,m+1})
\\ & \leq \eta(X_{m+1};\target(X_{m+1})) - \hat{f}_{n,m}^{\sqsmall}(X_{1:m},\ind_{\{\target(X_{m+1})\}}(Y_{1:n}),X_{m+1}) 
\\ & ~~~+ \max_{y \in \Y} \hat{f}_{n,m}^{\sqsmall}(X_{1:m},\ind_{\{y\}}(Y_{1:n}),X_{m+1}) - \eta(X_{m+1};\hat{Y}_{n,m+1})
\\ & = \eta(X_{m+1};\target(X_{m+1})) - \hat{f}_{n,m}^{\sqsmall}(X_{1:m},\ind_{\{\target(X_{m+1})\}}(Y_{1:n}),X_{m+1}) 
\\ & ~~~+ \hat{f}_{n,m}^{\sqsmall}(X_{1:m},\ind_{\{\hat{Y}_{n,m+1}\}}(Y_{1:n}),X_{m+1}) - \eta(X_{m+1};\hat{Y}_{n,m+1})
\\ & \leq 2 \max_{y \in \Y} \left| \hat{f}_{n,m}^{\sqsmall}(X_{1:m},\ind_{\{y\}}(Y_{1:n}),X_{m+1}) - \eta(X_{m+1};y) \right|.
\end{align*}
Therefore, on the event $E$, \eqref{eqn:noisy-classification-good-y} implies we have 
\begin{align*}
& \limsup_{n \to \infty} \limsup_{t \to \infty} \frac{1}{t+1} \sum_{m=n}^{n+t} \E\!\left[ \Delta_{m+1}^{n} \middle| X_{m+1}, \hat{Y}_{n,m+1} \right]
\\ & \leq 2 \limsup_{n \to \infty} \limsup_{t \to \infty} \frac{1}{t+1} \sum_{m=n}^{n+t} \max_{y \in \Y} \left| \hat{f}_{n,m}^{\sqsmall}(X_{1:m},\ind_{\{y\}}(Y_{1:n}),X_{m+1}) - \eta(X_{m+1};y) \right| = 0.
\end{align*}
Altogether, on the event $E \cap E^{\prime}$, it holds that 
$\limsup\limits_{n \to \infty} \hat{\L}_{\ProcX}(\hat{h}_{n,\cdot},\ProcY;n,\target) \leq 0$.
Also, Lemma~\ref{lem:non-negative-noisy-excess} implies that, on an event $E^{\prime\prime}$ of probability one, 
$\liminf\limits_{n \to \infty} \hat{\L}_{\ProcX}(\hat{h}_{n,\cdot},\ProcY;n,\target) \geq 0$.
Thus, on the event $E \cap E^{\prime} \cap E^{\prime\prime}$ of probability one (by the union bound), we have 
$\hat{\L}_{\ProcX}(\hat{h}_{n,\cdot},\ProcY;n,\target) \to 0$.

Since this holds for any $\ProcX \in \KC$, it immediately follows that Condition~\ref{con:kc} is sufficient 
for $\ProcX$ to admit strong universal self-adaptive learning for noisy functions.
Moreover, note that the definition of $\hat{h}_{n,m}$ has no dependence on the distributions of $\ProcX$ or $\ProcY$ 
beyond the data supplied as its arguments, 
and we have shown that $\hat{h}_{n,m}$ is strongly universally consistent for noisy functions under every $\ProcX$ satisfying Condition~\ref{con:kc}.
Since we have just established Condition~\ref{con:kc} is necessary and sufficient for $\ProcX$ to admit strong universal self-adaptive learning for noisy functions,
this also completes the proof that $\hat{h}_{n,m}$ is optimistically universal for noisy functions.
\end{proof}

We leave open the question of whether the above result for classification 
can be extended to general (bounded, separable) losses $\loss$, as stated 
in the following open problems.  

\begin{problem}
\label{prob:kc-noisy-function}
Is Condition~\ref{con:kc} sufficient for a process $\ProcX$ to 
admit strong universal inductive and self-adaptive learning 
for noisy functions, 
for \emph{every} separable near-metric space $(\Y,\loss)$ with $\maxloss < \infty$?
\end{problem}

\begin{problem}
\label{prob:optimistic-noisy-function}
Is it true that, for \emph{every} separable near-metric space $(\Y,\loss)$ with $\maxloss < \infty$, 
there exists a self-adaptive learning rule that is 
optimistically universal for noisy functions?
\end{problem}

Of course, Open Problem~\ref{prob:kc-noisy-function} would also be resolved 
by a positive resolution of Open Problem~\ref{prob:kc-independent-noise}, 
which represents a strictly stronger result.  Moreover, a positive resolution 
of both Open Problems~\ref{prob:kc-independent-noise} and \ref{prob:optimistic-independent-noise} together 
would also positively resolve Open Problem~\ref{prob:optimistic-noisy-function}.

\section{Extensions}
\label{sec:extensions}

Here we briefly mention two simple extensions of the above theory.
First, we present a straightforward extension to losses $\loss$ beyond near-metrics, 
admitting any loss dominated by a nondecreasing function of a near-metric loss.
Second, we present an extension of the results to \emph{weak} universal consistency.
In this latter case, we find that all of the results for inductive and self-adaptive learning 
above hold without modification for weak consistency as well.  However, interestingly, 
this is not true for online learning, as we find that the 
set of processes admitting weak universal online learning is a \emph{strict superset} 
of the set $\SUOL$ of processes admitting strong universal online learning 
(if $\X$ is infinite and $\maxloss < \infty$).

\subsection{More-General Loss Functions}
\label{subsec:nonmetric-losses}

For simplicity, we have chosen to restrict the loss function $\loss$ to be a \emph{near-metric} in the above results. 
However, as mentioned in Section~\ref{subsec:notation}, most of the theory developed above extends to a much 
broader family of loss functions, including all functions $\loss : \Y^{2} \to [0,\infty)$ that are merely \emph{dominated} by a separable near-metric 
$\lossdom$, in the sense that
$\forall y,y^{\prime} \in \Y, \loss(y,y^{\prime}) \leq \domfunc(\lossdom(y,y^{\prime}))$ for some continuous nondecreasing unbounded function 
$\domfunc : [0,\infty) \to [0,\infty)$ with $\domfunc(0)=0$,
and that also satisfy a non-triviality condition: 
$\!\sup\limits_{y_{0},y_{1} \in \Y} \inf\limits_{y \in \Y} \max\{ \loss(y,y_{0}),\loss(y,y_{1}) \} > 0$.
The measurable sets $\Borel_{y}$ are then defined as the Borel $\sigma$-algebra generated by the topology induced by $\lossdom$,
and we also require that $\loss$ be a measurable function with respect to this.
For instance, this extension admits asymmetric losses, 
such as in discrete classification with asymmetric misclassification costs. 

Here we briefly elaborate on the (minor) changes to the above theory yielding this generalization.
For any $z \in [0,\infty)$, define $\domfunc^{-1}(z) = \inf\{ x \in [0,\infty) : \domfunc(x) \geq z \}$; 
this always exists since the conditions on $\domfunc$ guarantee that its range is $[0,\infty)$, 
and moreover by continuity of $\domfunc$ we have $\domfunc(\domfunc^{-1}(z)) = z$.
Still defining $\maxloss = \sup\limits_{y,y^{\prime} \in \Y} \loss(y,y^{\prime})$, 
in the case of bounded losses ($\maxloss < \infty$), 
note that we can suppose $\lossdom$ is also bounded without loss of generality,
and in fact that it is bounded by $\domfunc^{-1}(\maxloss)$, since 
the near-metric $(y,y^{\prime}) \mapsto \lossdom(y,y^{\prime}) \land \domfunc^{-1}(\maxloss)$
still satisfies the requirement $\loss(y,y^{\prime}) \leq \domfunc(\lossdom(y,y^{\prime}) \land \domfunc^{-1}(\maxloss))$.
Then we can simply replace $\loss$ with $\lossdom$ in the learning rules proposed in \eqref{eqn:suil-rule} and \eqref{eqn:sual-rule}, 
and the resulting performance guarantees in terms of the loss $\lossdom$
then imply universal consistency under $\loss$
under the same conditions.
To see this, note that 
for any $\hat{y},y^{\star} \in \Y$, for any $\eps > 0$, we have
\begin{equation*}
\loss\!\left( \hat{y}, y^{\star} \right)
\leq \domfunc\!\left( \lossdom\!\left( \hat{y}, y^{\star} \right) \right)
\leq \eps + \maxloss \ind\!\left[ \lossdom\!\left( \hat{y}, y^{\star} \right) > \domfunc^{-1}(\eps) \right]
\leq \eps + \frac{\maxloss}{\domfunc^{-1}(\eps)} \lossdom\!\left( \hat{y}, y^{\star} \right),
\end{equation*}
noting that $\domfunc^{-1}(\eps) > 0$.  Plugging this inequality into the three $\hat{\L}_{\ProcX}$ definitions,
and noting that it holds for all $\eps > 0$, 
it easily follows that, in any of the three learning settings discussed above, 
strong universal consistency under the loss $\lossdom$ implies strong universal consistency under the loss $\loss$.

Furthermore, in the results where it is needed to argue inconsistency of a given learning rule 
(Lemma~\ref{lem:sual-subset-kc}, Theorems~\ref{thm:no-optimistic-inductive} and \ref{thm:okc-nec}),
the only property of $\loss$ used in those arguments is the stated non-triviality condition; 
more specifically, this condition is represented there by the
fact that, for $\loss$ a near-metric, any distinct $y_{0},y_{1} \in \Y$ have 
$\inf\limits_{y \in \Y}  \left( \loss(y,y_{0}) + \loss(y,y_{1}) \right) \geq \frac{1}{\triconst}\loss(y_{0},y_{1}) > 0$, 
but the arguments would hold just as well for these more-general losses $\loss$ by 
replacing $\frac{1}{\triconst}\loss(y_{0},y_{1})$ with $\inf\limits_{y \in \Y} \max\{ \loss(y,y_{0}),\loss(y,y_{1}) \}$
and choosing $y_{0},y_{1} \in \Y$ specifically to make this latter quantity nonzero.

These generalizations can be applied to all of the results involving a loss function in 
Sections~\ref{sec:intro} through \ref{subsec:online-vs-inductive}.
Section~\ref{subsec:suol-invariance-to-loss} is the only place (involving bounded losses) where somewhat-nontrivial modifications are necessary 
to extend the results to these more-general losses, 
simply due to needing an appropriate generalization of the notion of ``total boundedness'' for the arguments to remain valid.

The results on unbounded losses in Section~\ref{sec:unbounded-losses} can also be generalized. 
In this case, the same trick of using $\lossdom$ in place of $\loss$ in the definition of the learning rule \eqref{eqn:unbounded-universal2-inductive-rule} 
again works for establishing universal consistency with $\loss$ under $\ProcX \in \UKC$ in Lemma~\ref{lem:unbounded-ukc-subset-suil}, 
but in this case it follows from the stronger guarantee \eqref{eqn:sup-consistency} for $\lossdom$ 
(together with continuity and monotonicity of $\domfunc$, and $\domfunc(0)=0$) rather than 
from directly relating $\hat{\L}_{\ProcX}$ for the losses $\lossdom$ and $\loss$:
that is, the learning rule defined in terms of $\lossdom$ satisfies the convergence in \eqref{eqn:sup-consistency} for the loss $\lossdom$ under $\ProcX \in \UKC$,
and the properties of $\domfunc$ imply that it remains true for $\domfunc(\lossdom(\cdot,\cdot))$, and hence also for the loss $\loss$.
However, the complementary result in Lemma~\ref{lem:unbounded-sual-subset-ukc} requires an 
additional restriction to $\loss$ for the argument there to generalize: 
namely, that $\sup\limits_{y_{0},y_{1} \in \Y} \inf\limits_{y \in \Y} \max\{ \loss(y,y_{0}),\loss(y,y_{1}) \} = \infty$,
a property satisfied by most unbounded losses studied in the literature anyway.
Using this to replace the values $\loss(y_{i,0},y_{i,1})$ appearing in the proof of Lemma~\ref{lem:unbounded-sual-subset-ukc}
with values $\inf\limits_{y \in \Y} \max\{ \loss(y,y_{i,0}),\loss(y,y_{i,1}) \}$ (both in the definition of $y_{i,0},y_{i,1}$, and in \eqref{eqn:unbounded-nec-bernoulli}), 
the result is then extended to these more-general loss functions.
Together, these modifications allow us to extend all of the results in Section~\ref{sec:unbounded-losses} to 
these more-general loss functions $\loss$.

\subsection{Weak Universal Consistency}
\label{subsec:weak-consistency}

It is straightforward to extend the above results on inductive and self-adaptive learning (Sections~\ref{sec:main} and \ref{sec:universal2})
to \emph{weak} universal consistency as well,
where the definition of weakly universally consistent learning is as above except replacing 
the \emph{almost sure} convergence of $\hat{\L}_{\ProcX}$ to $0$ with convergence \emph{in probability}.
The proof of \emph{necessity} of Condition~\ref{con:kc} for inductive learning and self-adaptive learning 
(from Lemmas~\ref{lem:suil-subset-sual} and \ref{lem:sual-subset-kc}) 
can easily be modified to show necessity of Condition~\ref{con:kc} for \emph{weak} universal consistency by inductive or self-adaptive learning rules as well.
Specifically, the proof of Lemma~\ref{lem:sual-subset-kc} in this case would follow the same argument, 
but starting from $\sup\limits_{\kappa \in [0,1)} \limsup\limits_{n\to\infty} \E\!\left[ \hat{\L}_{\ProcX}(g_{n,\cdot},\target_{\kappa};n) \right]$ 
instead of $\sup\limits_{\kappa \in [0,1)} \E\!\left[ \limsup\limits_{n\to\infty} \hat{\L}_{\ProcX}(g_{n,\cdot},\target_{\kappa};n) \right]$.
After relaxing $\sup\limits_{\kappa \in [0,1)}$ to an integral over $\kappa \in [0,1)$ (as in the present proof) and applying Fatou's lemma 
to exchange the integral operator with the $\limsup\limits_{n \to \infty}$, the proof proceeds identically as before, 
and the final conclusion follows by noting that if $\lim\limits_{n \to \infty} \hat{\mu}_{\ProcX}\!\left( \bigcup \{A_{i} : X_{1:n} \cap A_{i} = \emptyset\} \right) > 0$
with nonzero probability, then (by the monotone convergence theorem) 
\begin{equation*}
\lim\limits_{n\to\infty} \E\!\left[ \hat{\mu}_{\ProcX}\!\left( \bigcup \{A_{i} : X_{1:n} \cap A_{i} = \emptyset\} \right) \right]
= \E\!\left[ \lim\limits_{n\to\infty} \hat{\mu}_{\ProcX}\!\left( \bigcup \{A_{i} : X_{1:n} \cap A_{i} = \emptyset\} \right) \right] > 0.
\end{equation*}
For brevity, we leave the details of the proof as an exercise for the interested reader.
Since strong universal consistency implies weak universal consistency, the sufficiency of Condition~\ref{con:kc}  
for universal consistency of inductive or self-adaptive learning (from Lemmas~\ref{lem:kc-subset-suil} and \ref{lem:suil-subset-sual}),  
as well as the result on optimistically universal self-adaptive learning (Theorem~\ref{thm:doubly-universal-adaptive}), 
continue to hold for the \emph{weak} universal consistency criterion in place of \emph{strong} universal consistency.
In particular, this means that the set of processes $\WUIL$ (or $\WUAL$) admitting weak universal inductive (or self-adaptive) learning 
is equal to $\SUIL$ (or $\SUAL$), both of which are equal $\KC$ by Theorem~\ref{thm:main}. 
Additionally, it follows from statements made in the proof of Theorem~\ref{thm:no-optimistic-inductive} 
that Theorem~\ref{thm:no-optimistic-inductive} remains valid for weak universal consistency as well.
Again, the details are left as an exercise for the interested reader.

Interestingly, the extension to weak consistency in the \emph{online} learning setting (with $\maxloss < \infty$) 
is substantially more involved, and indeed the set of processes that admit \emph{weak} universal online learning ($\WUOL$) 
is in fact a \emph{strict superset} of $\SUOL$ (if $\X$ is infinite).  
That it is a superset easily follows from the fact that almost sure convergence implies convergence in probability,
so the interesting detail here is that there exist processes $\ProcX$ that admit weak universal online learning
but \emph{not} strong universal online learning.
To see this, consider the following construction of a process $\ProcX$. 
Let $\{z_{i}\}_{i=0}^{\infty}$ be distinct elements of $\X$ (supposing $\X$ is infinite), and let $\{B_{k}\}_{k=1}^{\infty}$ be independent random variables 
with $B_{k} \sim {\rm Bernoulli}(\frac{1}{k})$.  Then for each $k \in \nats$ and each $t \in \{2^{k-1},\ldots,2^{k}-1\}$, 
if $B_{k} = 1$, then set $X_{t} = z_{t}$, and if $B_{k} = 0$, then set $X_{t} = z_{0}$.
Since $\sum\limits_{k=1}^{\infty} \frac{1}{k} = \infty$, the second Borel-Cantelli lemma implies that, with probability one, 
there exists an infinite strictly-increasing sequence $\{k_{i}\}_{i=1}^{\infty}$ in $\nats$ with $B_{k_{i}} = 1$ for every $i \in \nats$.
On this event, every $k \in \{ k_{i} : i \in \nats\}$ has 
$|\{ j \in \nats : X_{1:(2^{k}-1)} \cap \{z_{j}\} \neq \emptyset\}| \geq 2^{k-1}$,
so that 
$|\{ j \in \nats : X_{1:T} \cap \{z_{j}\} \neq \emptyset\}| \neq o(T) \text{ (a.s.)}$.
Thus, $\ProcX \notin \OKC$, and hence by Theorem~\ref{thm:okc-nec}, 
$\ProcX \notin \SUOL$.
However, if we take $f_{n}$ as the simple memorization-based online learning rule 
(from the proof of Theorem~\ref{thm:okc-countable}), 
then for any $n \in \nats$ and measurable $\target : \X \to \Y$,
we have 
$\E\!\left[ \hat{\L}_{\ProcX}(f_{\cdot},\target;n) \right] 
\leq \frac{\maxloss}{n} \E\left[ | \{ j \in \nats \cup \{0\} : X_{1:n} \cap \{z_{j}\} \neq \emptyset \} | \right]
\leq \frac{\maxloss}{n} \left( 1 + \sum\limits_{k=1}^{\lfloor\log_{2}(2n)\rfloor} \frac{1}{k} 2^{k-1} \right)
\leq \frac{\maxloss}{n} \left( 1 +\!\!\! \int\limits_{1}^{\lfloor\log_{2}(4n)\rfloor} \frac{1}{x} 2^{x-1} {\rm d}x \right)$.
Since $\int\limits_{1}^{t} \frac{1}{x} 2^{x-1} {\rm d}x = o\!\left( \int\limits_{1}^{t} 2^{x} {\rm d}x \right)$ as $t \to \infty$
(by L'H\^{o}pital's rule and the fundamental theorem of calculus), 
and $\int\limits_{1}^{t} 2^{x} {\rm d}x = \frac{1}{\ln(2)} 2^{t}$,
we conclude that 
$\int\limits_{1}^{\lfloor\log_{2}(4n)\rfloor} \frac{1}{x} 2^{x-1} {\rm d}x = o(n)$, 
so that $\E\!\left[ \hat{\L}_{\ProcX}(f_{\cdot},\target;n) \right] \to 0$, 
which implies $\hat{\L}_{\ProcX}(f_{\cdot},\target;n) \xrightarrow{P} 0$ by Markov's inequality.
Thus, $\ProcX$ admits weak universal online learning.

Following arguments analogous to the proof of Theorem~\ref{thm:okc-nec}, 
one can show that a \emph{necessary} condition for a process $\ProcX$ to admit weak universal online learning 
is that every disjoint sequence $\{A_{i}\}_{i=1}^{\infty}$ in $\Borel$ satisfies 
$\E[ | \{ i \in \nats : X_{1:T} \cap A_{i} \neq \emptyset \} | ] = o(T)$.
This represents a sort of \emph{weak} form of Condition~\ref{con:okc}.
Furthermore, following similar arguments to the proof of Theorem~\ref{thm:okc-countable}, 
one can show that in the special case of \emph{countable} $\X$, 
this condition is both necessary \emph{and sufficient} for $\ProcX$
to admit weak universal online learning.
However, as was the case for Condition~\ref{con:okc} and strong universal consistency (Open Problem~\ref{prob:suol-equals-okc}),
in the general case (allowing uncountable $\X$) 
it remains an open problem to determine whether this weaker form of Condition~\ref{con:okc} is equivalent to the condition that $\ProcX$ admits 
weak universal online learning.
Likewise, it also remains an open problem to determine whether there generally exist optimistically universal online learning rules 
under this weak consistency criterion instead of the strong consistency criterion.

In the case of unbounded losses,
one can show that Theorems~\ref{thm:unbounded-main} and \ref{thm:unbounded-optimistically-universal} 
extend to weak universal consistency without modification.
Specifically, since almost sure convergence implies convergence in probability, 
Theorem~\ref{thm:unbounded-main} immediately implies sufficiency of 
Condition~\ref{con:ukc} for a process to admit weak universal learning (in all three settings).
Furthermore, the same construction used in the proof of Lemma~\ref{lem:unbounded-sual-subset-ukc} 
can be used to show that Condition~\ref{con:ukc} is also necessary for 
weak universal learning (again in all three settings) when $\maxloss = \infty$.
Briefly, for any $\ProcX \notin \UKC$, 
in the notation defined in the proof of Lemma~\ref{lem:unbounded-sual-subset-ukc}, 
we would have that for any online learning rule $h_{n}$, 
every $j \in \nats$ has 
$\P\!\left( \hat{\L}_{\ProcX}(h_{\cdot},\target_{K};T_{j}) > \frac{1}{2\triconst} \right) 
\geq \frac{1}{2} \P( 0 < \tau_{j} \leq T_{j} ) > \frac{1}{2} ( \P(E) - 2^{-j} )$,
which is bounded away from $0$ for all sufficiently large $j$.
Since one can also show that $T_{j} \to \infty$, it follows that 
$\exists \kappa \in [0,1)$ such that 
$\limsup\limits_{n \to \infty} \P\!\left( \hat{\L}_{\ProcX}(h_{\cdot},\target_{\kappa};n) > \frac{1}{2\triconst} \right) > 0$, 
so that $h_{n}$ is not weakly universally consistent under $\ProcX$.
Similarly, for any self-adaptive learning rule $g_{n,m}$, 
we would have that for any $n \in \nats$, 
$\P\!\left( \hat{\L}_{\ProcX}(g_{n,\cdot},\target_{K};n) \geq \frac{1}{2\triconst} \right) \geq \P(E \cap E^{\prime}) > 0$, 
which implies $\exists \kappa \in [0,1)$ such that 
$\limsup\limits_{n \to \infty} \P\!\left( \hat{\L}_{\ProcX}(g_{n,\cdot},\target_{\kappa};n) \geq \frac{1}{2\triconst} \right) > 0$,
so that $g_{n,m}$ is not weakly universally consistent under $\ProcX$.
The same argument holds for any inductive learning rule $f_{n}$ as well.
The details of these arguments are left as an exercise for the interested reader.
Together with Theorems~\ref{thm:unbounded-main} and \ref{thm:unbounded-optimistically-universal}
and the fact that almost sure convergence implies convergence in probability, 
this also implies that (when $\maxloss = \infty$) there exists an optimistically universal learning rule 
(in all three settings) under this weak consistency criterion as well.

\section{Open Problems}
\label{sec:open-problems}

For convenience, we conclude the paper by briefly gathering in summary form 
the main open problems posed in the sections above, 
along with additional general directions for future study.
The statements dependent on $\loss$ regard the case $\maxloss < \infty$, 
and always restrict to $(\Y,\loss)$ a separable near-metric space.

\begin{itemize}
\item Open Problem~\ref{prob:optimistic-online}: 
Does there exist an optimistically universal online learning rule?
\item Open Problem~\ref{prob:suol-equals-okc}: Is $\SUOL = \OKC$?
\item Open Problem~\ref{prob:suol-multiclass}: 
Is the set $\SUOL$ invariant to the specification of $(\Y,\loss)$,
subject to $(\Y,\loss)$ being separable with $0 < \maxloss < \infty$?
\item Open Problem~\ref{prob:ukc-infinite-support}: For some uncountable $\X$, 
do there exist processes $\ProcX \in \UKC$ such that, with nonzero probability, 
the number of distinct $x \in \X$ appearing in $\ProcX$ is infinite?
\item Open Problem~(\ref{prob:kc-independent-noise} / \ref{prob:kc-noisy-function}):
Does every $\ProcX \in \KC$ admit strong universal inductive and self-adaptive learning (with independent noise / for noisy functions)
for every $(\Y,\loss)$?
\item Open Problem~(\ref{prob:optimistic-independent-noise} / \ref{prob:optimistic-noisy-function}): 
Is it true that, for every $(\Y,\loss)$, there exists a self-adaptive learning rule that is optimistically universal (with independent noise / for noisy functions)?
\end{itemize}

\bibliography{learning}

\end{document}